\documentclass{article}

\usepackage[final]{neurips_2021}

\usepackage[utf8]{inputenc} 
\usepackage[T1]{fontenc}    
\usepackage{url}            
\usepackage{booktabs}       
\usepackage{amsfonts}       
\usepackage{nicefrac}       
\usepackage{microtype}      
\usepackage{graphicx}
\usepackage{natbib}
\usepackage{doi}

\usepackage{array}
\usepackage[export]{adjustbox}
\usepackage{algorithm}
\usepackage[noend]{algpseudocode}
\usepackage{amssymb}
\usepackage{amsmath}
\usepackage{amsthm}
\usepackage[shortlabels]{enumitem}
\usepackage{graphicx}
\usepackage[bb=dsserif]{mathalpha}
\usepackage{mathtools}
\usepackage{placeins}
\usepackage{subcaption}
\usepackage{tikz}
\usepackage[compact]{titlesec}
\usepackage{thm-restate}
\usepackage{thmtools}
\usetikzlibrary{calc}
\usetikzlibrary{intersections}
\usepackage{verbatim}
\usepackage{xcolor}

\newtheorem{theorem}{Theorem}
\newtheorem{lemma}[theorem]{Lemma}
\newtheorem{proposition}[theorem]{Proposition}
\newtheorem{corollary}[theorem]{Corollary}
\theoremstyle{definition}
\newtheorem{definition}[theorem]{Definition}
\newtheorem{remark}[theorem]{Remark}

\titlespacing{\paragraph}{0pt}{0ex}{1ex}

\DeclareMathOperator*{\argmax}{argmax}

\algnewcommand{\LineComment}[1]{\State \(\triangleright\) #1}

\newcommand{\cbias}{C_{\text{\normalfont bias}}}
\newcommand{\cgrad}{C_{\text{\normalfont grad}}}
\newcommand{\defeq}{\vcentcolon=}
\newcommand{\dist}{\text{\normalfont dist}}
\newcommand{\net}{\mathcal{N}}
\newcommand{\nin}{n_0}

\newcommand{\var}{\mathrm{Var}}
\newcommand{\vol}{\operatorname{vol}}
\newcommand{\conv}{\operatorname{conv}}

\hypersetup{
pdftitle={On the Expected Complexity of Maxout Networks},
pdfsubject={cs.LG, stat.ML},
pdfauthor={Hanna Tseran, Guido Mont\'ufar},
pdfkeywords={Linear regions of neural networks, maxout units, Expected complexity, Decision boundary, Parameter initialisation},
}

\date{\today}

\title{On the Expected Complexity of Maxout Networks} 

\author{Hanna Tseran\\
Max Planck Institute for Mathematics in the Sciences
\\
hanna.tseran@mis.mpg.de}
\author{
  Hanna Tseran\\
  Max Planck Institute for Mathematics in the Sciences\\
  04103 Leipzig, Germany\\
  \texttt{hanna.tseran@mis.mpg.de}\\
   \And
   Guido Mont\'ufar \\
   Department of Mathematics and Department of Statistics, UCLA\\
   Los Angeles, CA 90095, USA; \\
   Max Planck Institute for Mathematics in the Sciences\\
   04103 Leipzig, Germany \\
   \texttt{montufar@math.ucla.edu} \\
}

\begin{document}

\maketitle

\begin{abstract}
Learning with neural networks relies on the complexity of the representable functions, but more importantly, the particular assignment of typical parameters to functions of different complexity. Taking the number of activation regions as a complexity measure, recent works have shown that the practical complexity of deep ReLU networks is often far from the theoretical maximum. In this work, we show that this phenomenon also occurs in networks with maxout (multi-argument) activation functions and when considering the decision boundaries in classification tasks. We also show that the parameter space has a multitude of full-dimensional regions with widely different complexity, and obtain nontrivial lower bounds on the expected complexity. Finally, we investigate different parameter initialization procedures and show that they can increase the speed of convergence in training. 
\end{abstract}


\section{Introduction}

We are interested in the functions parametrized by artificial feedforward neural networks with maxout units. Maxout units compute parametric affine functions followed by a fixed multi-argument activation function of the form $(s_1,\ldots, s_K)\mapsto \max\{s_1,\ldots, s_K\}$ and can be regarded as a natural generalization of ReLUs, which have a single-argument activation function $s\mapsto \max\{0,s\}$.
For any choice of parameters, these networks subdivide their input space into activation regions where different pre-activation features attain the maximum and the computed function is (affine) linear.
We are concerned with the
expected number of activation regions and their volume given probability distributions of parameters, as well as corresponding properties for the decision boundaries in classification tasks. 
We show that different architectures can attain very different numbers of regions with positive probability, but for
parameter distributions for which the conditional densities of bias values and the expected gradients of activation values are bounded, the expected number of regions is at most polynomial in the rank $K$ and the total number of units.

\paragraph{Activation regions of neural networks} 
For neural networks with piecewise linear activation functions, the number of activation regions serves as a complexity measure and summary description which has proven useful in the investigation of approximation errors, Lipschitz constants, speed of convergence, implicit biases of parameter optimization, and robustness against adversarial attacks.
In particular, \citet{pascanu2013number,NIPS2014_5422,telgarsky2015representation,pmlr-v49-telgarsky16} obtained depth separation results showing that deep networks can represent functions with many more linear regions than any of the functions that can be represented by shallow networks with the same number of units or parameters. This implies that certain tasks require enormous shallow networks but can be solved with small deep networks.
The geometry of the boundaries between linear regions has been used to study function-preserving transformations of the network weights \citep{phuong2019functional, serra2020lossless}
and robustness \citep{pmlr-v89-croce19a, lee2019towards}. 
\citet{steinwart2019sober} demonstrated empirically that the distribution of regions at initialization can be related to the speed of convergence of gradient descent, and \citet{NEURIPS2019_1f6419b1,86441} related the density of breakpoints at initialization to the curvature of the solutions after training.
The properties of linear regions in relation to training have been recently studied by \citet{Zhang2020Empirical}.
The number of regions has also been utilized to study the eigenvalues of the neural tangent kernel and Lipschitz constants \citep{nguyen2020tight}.

\paragraph{Maximum number of regions}
Especially the maximum number of linear regions has been studied intensively. 
In particular, \citet{montufar2017notes,serra2018bounding} improved the upper bounds from \citet{NIPS2014_5422} by accounting for output dimension bottlenecks across layers.
\citet{hinz2019framework} introduced a histogram framework for a fine grained analysis of such dimensions in ReLU networks. Based on this, \citet{xie2020general, Hinz2021UsingAH} obtained still tighter upper bounds for ReLU networks.
The maximum number of regions has been studied not only for fully connected networks, but also convolutional neural networks \citep{xiong2020number}, graph neural networks (GNNs) and message passing simplicial networks (MPSN) \citep{bodnar2021weisfeiler}. 

\paragraph{Expected number of regions}
Although the maximum possible number of regions gives useful complexity bounds and insights into different architectures, in practice one may be more interested in the expected behavior for typical choices of the parameters. 
The first results on the expected number of regions were obtained by \citet{pmlr-v97-hanin19a,NIPS2019_8328} for the case of ReLU networks or single-argument piecewise linear activations. 
They show that if one has a distribution of parameters such that the conditional densities of bias values are bounded and the expected gradients of activation values are bounded, then the expected number of linear regions can be much smaller than the maximum theoretically possible number. Moreover, they obtain bounds for the expected number and volume of lower dimensional linear pieces of the represented functions. 
These results do not directly apply to the case of maxout units, but we will adapt the proofs to obtain corresponding results. 

\paragraph{Regions of maxout networks}
Most previous works focus on ReLUs or single-argument activation functions.
In this case, the linear regions of individual layers are described by hyperplane arrangements, which have been investigated since the 19th century \citep{Steiner1826,10.2307/2303424,zaslavsky1975facing}. 
Hence, the main challenge in these works is the description of compositions of several layers.
In contrast, the linear regions of maxout layers are described by complex arrangements that are not so well understood yet. 
The study of maxout networks poses significant challenges already at the level of individual layers and in fact single units.
For maxout networks, the maximum possible number of regions has been studied by \citet{pascanu2013number,NIPS2014_5422,serra2018bounding}. 
Recently, \citet{sharp2021} obtained counting formulas and sharp (asymptotic) upper bounds for the number of regions of shallow (deep) maxout networks. 
However, their focus was on the maximum possible value, and not on the generic behavior, which we investigate here. 

\paragraph{Related notions}
The activation regions of neural networks can be approached from several perspectives. 
In particular, the functions represented by networks with piecewise linear activations correspond to so-called tropical rational functions and can be studied from the perspective of tropical geometry \citep{zhang2018tropical, charisopoulos2018tropical}.
In the case of piecewise affine convex nonlinearities, these can be studied in terms of so-called max-affine splines \citep{NIPS2019_9712}.
A related but complementary notion of network expressivity is trajectory length, proposed by \citet{raghu2017expressive}, which measures transitions between activation patterns along one-dimensional paths on the input space, which also leads to depth separation results. 
Recent work \citep{hanin2021deep} shows that ReLU networks preserve expected length. 

\paragraph{Contributions} 
We obtain the following results for maxout networks.
\begin{itemize}[leftmargin=*]
\itemsep.01em 
    \item There are widely different numbers of linear regions that are attained with positive probability over the parameters (Theorem~\ref{thm:positive_measure}). 
    There is a non-trivial lower bound on the number of linear regions that holds for almost every choice of the parameters (Theorem~\ref{thm:lower_bound}). 
    These results advance the maximum complexity analysis of \citet{sharp2021} from the perspective of generic parameters. 
    \item For common parameter distributions, the expected number of activation regions is polynomial in the number of units (Theorem~\ref{th:main_result}). 
    Moreover, the expected volume of activation regions of different dimensions is polynomial in the number of units (Theorem~\ref{th:semi_main_upper_bound}). 
    These results correspond to maxout versions of results from \citet{NIPS2019_8328} and \citet{pmlr-v97-hanin19a}. 
    \item For multi-class classifiers, we obtain an upper bound on the expected number of linear pieces (Theorem~\ref{th:decision_boundary}) and the expected volume (Theorem~\ref{th:decision_boundary_volume}) of the decision boundary, along with a lower bound on the expected distance between input points and decision boundaries (Corollary~\ref{cor:dist_to_db}). 
    \item We provide an algorithm and implementation for counting the number of linear regions of maxout networks (Algorithm~\ref{algorith:exact_count}).  
    \item We present parameter initialization procedures for maxout networks maximizing the number of regions or normalizing the mean activations across layers (similar to \citealt{glorot2010understanding,he2015delving}), and observe experimentally that these can lead to faster convergence of training. 
\end{itemize}

\section{Activation regions of maxout networks} 
\label{section:definitions}

We consider feedforward neural networks with $n_0$ inputs, $L$ hidden layers of widths $n_1,\ldots, n_L$, and no skip connections, which implement functions of the form $f = \psi \circ \phi_{L} \circ \cdots \circ \phi_1$. 
The $l$-th hidden layer implements a function $\phi_l\colon \mathbb{R}^{n_{l-1}}\to\mathbb{R}^{n_l}$ with output coordinates, i.e.\ units, given by trainable affine functions followed by a fixed real-valued activation function, and $\psi\colon\mathbb{R}^{n_L}\to\mathbb{R}^{n_{L+1}}$ is a linear output layer. 
We denote the total number of hidden units by $N=n_1+\cdots+ n_L$, and index them by $z\in [N]:=\{1,\ldots, N\}$. The collection of all trainable parameters is denoted by $\theta$. 

We consider networks with maxout units, introduced by \citet{goodfellow2013maxout}.
A rank-$K$ maxout unit with $n$ inputs implements a function $\mathbb{R}^n\to\mathbb{R}$; $x\mapsto \max_{k \in [K]} \{ w_{k} \cdot x + b_{k}\}$, where $w_{k} \in \mathbb{R}^{n}$ and $b_{k} \in \mathbb{R}$, $k\in [K]$, are trainable weights and biases. 
The activation function $(s_1,\ldots, s_K)\mapsto \max\{s_1,\ldots, s_K\}$ can be regarded as a multi-argument generalization of the rectified linear unit (ReLU) activation function $s\mapsto \max\{0,s\}$. 
The $K$ arguments of the maximum are called the pre-activation features of the maxout unit. 
For unit $z$ in a maxout network, we denote $\zeta_{z,k}(x;\theta)$ its $k$-th pre-activation feature, considered as a function of the input to the network. 

For any choice of the trainable parameters, the function represented by a maxout network is piecewise linear, meaning it splits the input space into countably many regions over each of which it is linear. 

\begin{definition}[Linear regions]
    Let $f\colon \mathbb{R}^{\nin}\to\mathbb{R}$ be a piecewise linear function. 
    A linear region of $f$ is a maximal connected subset of $\mathbb{R}^{\nin}$ on which $f$ has a constant gradient. 
\end{definition}

We will relate the linear regions of the represented functions to activation regions defined next. 

\begin{definition}[Activation patterns]
An activation pattern of a network with $N$ rank-$K$ maxout units is an assignment of a non-empty set $J_z\subseteq[K]$ to each unit $z\in[N]$. 
An activation pattern $J=(J_z)_{z\in[N]}$ with $\sum_{z\in[N]} (|J_z| - 1) = r$ is called an $r$-partial activation pattern. 
The set of all possible activation patterns is denoted $\mathcal{P}$, and the set of $r$-partial activation patterns is denoted $\mathcal{P}_r$. 
An activation sub-pattern is a pattern where we disregard all $J_z$ with $|J_z|=1$. The set of all possible activation sub-patterns is denoted $\mathcal{S}$, and the set of $r$-partial activation sub-patterns is denoted $\mathcal{S}_r$. 
\end{definition}

\begin{definition}[Activation regions]
Consider a network $\net$ with $N$ maxout units. For any parameter value $\theta$ and any activation pattern $J$, the corresponding activation region is 
    \begin{align*}
        \mathcal{R}(J, \theta) \defeq \big\{ x \in \mathbb{R}^{\nin} \ \big| \ \argmax\limits_{k \in [K]} \zeta_{z, k}(x; \theta)  = J_z  \ \text{ \normalfont for each $z\in[N]$} \big\}. 
    \end{align*}
For any $r \in \{0, \dots, \nin\}$ we denote the union of $r$-partial activation regions by
     \begin{align*}
        \mathcal{X}_{\net, r}(\theta) \defeq 
        \bigcup_{J\in \mathcal{P}_r} \mathcal{R}(J; \theta). 
    \end{align*} 
\end{definition}  

By these definitions, we have a decomposition of the input space as a disjoint union of activation regions, $\mathbb{R}^{\nin} = \sqcup_{J\in \mathcal{P}} \mathcal{R}(J,\theta)$. See Figure~\ref{fig:1}. 
Next we observe that for almost every choice of $\theta$, $r$-partial activation regions are either empty or relatively open convex polyhedra of co-dimension $r$. 
In particular, for almost every choice of the parameters, if $r$ is larger than $n_0$, the $r$-partial activation regions are empty. Therefore, in our discussion we only need to consider $r$ up to $n_0$. 

\begin{restatable}[$r$-partial activation regions are relatively open convex polyhedra]{lemma}{convactreglemma}
    \label{lem:conv_act_reg}
    Consider a maxout network $\net$. Let $r \in \{0, \dots, \nin \}$ and $J\in\mathcal{P}_r$.
    Then for any $\theta$, $\mathcal{R}(J,\theta)$ is a relatively open convex polyhedron in $\mathbb{R}^{\nin}$. 
    For almost every $\theta$, it is either empty or has co-dimension $r$. 
\end{restatable}
The proof of Lemma~\ref{lem:conv_act_reg} is given in Appendix \ref{app:intro_proofs}. 
Next we show that for almost every choice of $\theta$, 
$0$-partial activation regions and linear regions correspond to each other.

\begin{restatable}[Activation regions vs linear regions]{lemma}{actvslinlemma}
    \label{lem:act_vs_lin}
    Consider a maxout network $\net$. 
    The set of parameter values $\theta$ for which the represented function has the same gradient on two distinct activation regions is a null set.
    In particular, for almost every $\theta$,
    linear regions and activation regions correspond to each other.
\end{restatable}
The proof of Lemma~\ref{lem:act_vs_lin} is given in Appendix \ref{app:intro_proofs}. 
We note that for specific parameters, linear regions can be the union of several activation regions and can be non-convex. 
Such situation is more common in ReLU networks, whose units can more readily output zero, thereby hiding the activation pattern of the units in the previous layers. 

To summarize the above observations, for almost every $\theta$, the $0$-partial activation regions are $\nin$-dimensional open convex polyhedra which agree with the linear regions of the represented function, and for $r=1,\ldots,\nin$ the $r$-partial activation regions are co-dimension-$r$ polyhedral pieces of the boundary between linear regions.
Next we investigate the number non-empty $r$-partial activation regions and their volume within given subsets of the input space. 
We are concerned with their generic numbers, where we use ``generic'' in the standard sense, to refer to a positive Lebesgue measure event.

\section{Numbers of regions attained with positive probability}
\label{sec:numbers}

We start with a simple upper bound. 
\begin{restatable}[Simple upper bound on the number of $r$-partial activation patterns]{lemma}{trivialupperboundlemma}
    \label{lem:trivial_upper_bound}
    Let $r\in\mathbb{N}_0$. The number of $r$-partial activation patterns and sub-patterns in a network with a total of $N$ rank-$K$ maxout units are upper bounded by $|\mathcal{P}_r| \leq \binom{r K}{2r} \binom{N}{r} K^{N - r}$ and $|\mathcal{S}_r|\leq \binom{r K}{2r} \binom{N}{r}$ respectively. 
\end{restatable}
The upper bound has asymptotic order $O(N^r K^{N+r})$ in $K$ and $N$. 
The proof of Lemma~\ref{lem:trivial_upper_bound} is given in Appendix \ref{app:intro_proofs}, where we also provide an exact but unhandy counting formula.

By definition, the number of $r$-partial activation patterns is a trivial upper bound on the number of non-empty $r$-partial activation regions for any choice of parameters.
Depending on the network architecture, this bound may not be attainable for any choice of the parameters.
\citet[Theorems~3.7 and 3.12]{sharp2021} obtained bounds for the maximum number of linear regions. For a shallow network with $n_0$ inputs and a single layer of $n_1$ rank-$K$ maxout units it has order $\Theta((n_1 K)^{n_0})$ in $K$ and $n_1$, and for a deep network with $n_0$ inputs and $L$ layers of $n_1,\ldots, n_L$ rank-$K$ maxout units it has order $\Theta(\prod_{l=1}^L(n_l K)^{n_0})$ in $K$ and $n_1,\ldots, n_L$.
Hence the maximum number of non-empty activation regions can be very large, especially for deep networks. 

Intuitively, linear regions have a non-zero volume and cannot `disappear' under small perturbations of parameters.
This raises the question about which numbers of linear regions are attained with positive probability, i.e.\ over positive Lebesgue measure subsets of parameter values. Figure~\ref{fig:1} shows that the number of linear regions of a maxout network is a very intricate function of the parameter values. 

For a network with $n_0$ inputs and a single layer of $n_1$ ReLUs, the maximum number of linear regions is $\sum_{j=0}^{n_0}\binom{n_1}{j}$, and is attained for almost all parameter values. This is a consequence of the generic behavior of hyperplane arrangements \citep[see][]{10.2307/2303424,zaslavsky1975facing,NIPS2014_5422}. 
In contrast, shallow maxout networks can attain different numbers of linear regions with positive probability. The intuitive reason is that the nonlinear locus of maxout units is described not only by linear equations $\langle w_i, x\rangle + b_i = \langle w_j, x\rangle +b_j$ but also linear inequalities $\langle w_i, x\rangle + b_i \geq \langle w_k, x\rangle +b_k$. See Figure~\ref{fig:1} for an example. 
We obtain the following result. 

\begin{figure}
    \centering
\begin{tikzpicture}
\definecolor{reg11}{rgb}{0.5, 0.0, 0.0}
\definecolor{reg21}{rgb}{0.8, 0.0, 0.0}
\definecolor{reg22}{rgb}{0, 0.0, 0.8}
\definecolor{reg32}{rgb}{0, 0.0, 0.5}
\definecolor{reg33}{rgb}{0, .8, .8}
\node
(P) at (0,0) {\includegraphics[width = 4cm]{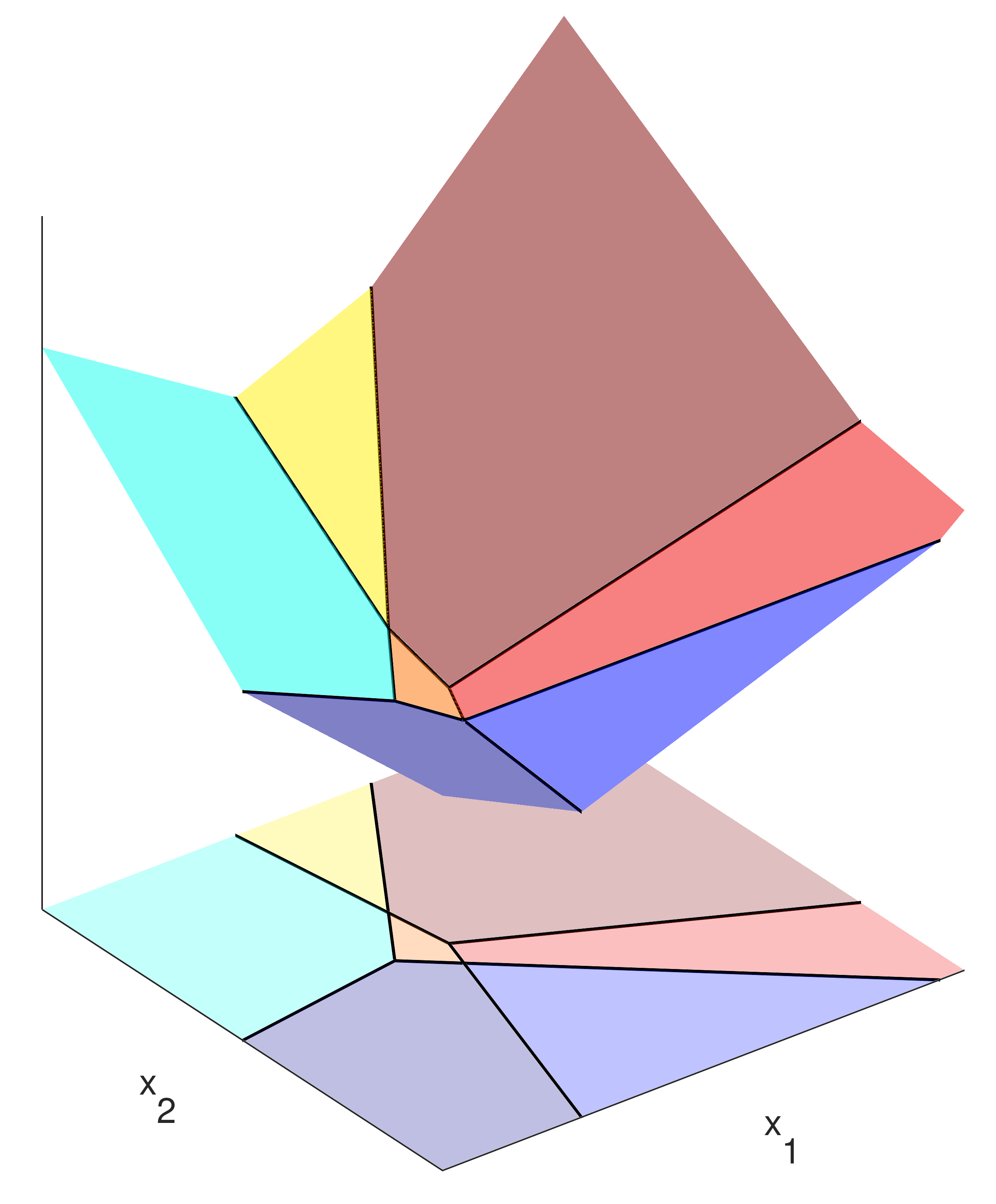}}; 

\node[fill=white, inner sep=1pt,opacity=1] at (0,2) {\small Activation regions}; 
\node at (1.25,-.9) {$\mathcal{X}_1$};  
\node at (.2,-1.1) {\scalebox{.5}{\textcolor{reg11}{$\mathcal{R}(1,1)$}}};  
\node at (.7,-1.7) {\scalebox{.5}{\textcolor{reg22}{$\mathcal{R}(2,2)$}}};  
\node at (-.25,-1.9) {\scalebox{.5}{\textcolor{reg32}{$\mathcal{R}(3,2)$}}};  
\node at (-1.1,-1.3) {\scalebox{.5}{\textcolor{reg33}{$\mathcal{R}(3,3)$}}};  

\node[fill=white] at (1.1,-2.2) {\small \textcolor{black}{$x_1$}};  
\node[fill=white] at (-1.4,-2) {\small \textcolor{black}{$x_2$}};  
\node[fill=white] at (-1.6,1.4) {\small \textcolor{white}{$f$}};  
 
\node (R) at (7,0){\includegraphics[width = 4cm]{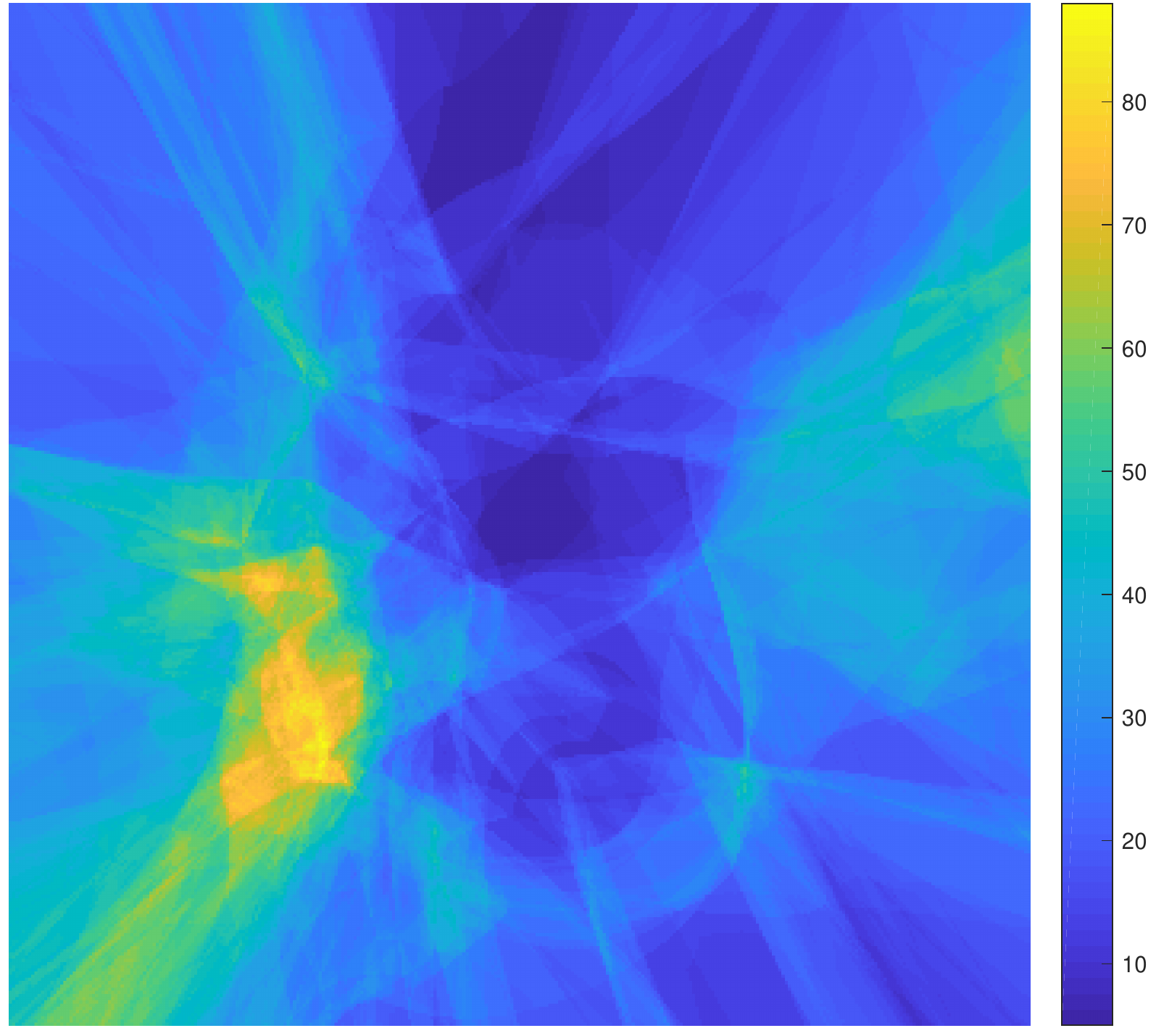}}; 
\node at (7,2) {\small Number of linear regions}; 
\node at (7,-2) {\small $\xi_1$};
\node at (4.75,0) {\rotatebox{90}{\small $\xi_2$}};

\draw[very thin] (4.2,1.2)--(5.5,1);
\draw[very thin] (4.2,-1.2)--(5.5,-1);
\node (F1) at (4,1.2) {\includegraphics[clip=true, trim=0cm 13.7cm 0cm .3cm, height=1.2cm]{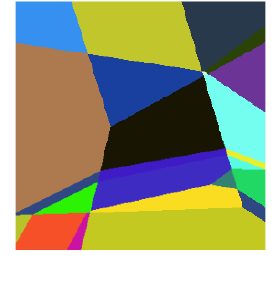}};
\node (F2) at (4,-1.2) {\includegraphics[clip=true, trim=0cm -.3cm 0cm 14.3cm, height=1.2cm]{images/regions_over_parameters/figure2Dslice_parameters_d_1_resolutioninput_320_resolutionpar_320_net_2_3_3_3_k_3_functions-trim.pdf}};

\end{tikzpicture}
    \caption{
        Left: Shown is a piecewise linear function $\mathbb{R}^2\to\mathbb{R}$ represented by a network with a layer of two rank-3 maxout units for a choice of the parameters. 
        The input space is subdivided into activation regions $\mathcal{R}(J;\theta)$ with linear regions separated by $\mathcal{X}_1(\theta)$. 
        Right: Shown is the number of linear regions of a 3 layer maxout network over a portion of the input space as a function of a 2D affine subspace of parameter values $\theta(\xi_1,\xi_2)$. 
        Shown are also two examples of the input-space subdivisions of functions represented by the network for different parameter values. 
        More details about this figure are given in Appendix~\ref{app:experiments}.  
        As the figure illustrates, the function taking parameters to number of regions is rather intricate. In this work we characterize values attained with positive probability and upper bound the expected value given a parameter distribution.
    }
    \label{fig:1}
\end{figure}

\begin{restatable}[Numbers of linear regions]{theorem}{theorempositivemeasure}
\label{thm:positive_measure}
\hspace{1em}
    \begin{itemize}[leftmargin=*]
    \item 
    Consider a rank-$K$ maxout unit with $n_0$ inputs. This corresponds to a network with an input layer of size $n_0$ and single maxout layer with a single maxout unit. 
    For each $1\leq k\leq K$, there is a set of parameter values for which the number of linear regions is $k$. For $\min\{K, n_0+1\}\leq k\leq K$, the corresponding set has positive measure, and else it is a null set. 
    \item 
    Consider a layer of $n_1$ rank-$K$ maxout units with $n_0$ inputs. This corresponds to a network with a single maxout layer, $L=1$, and $n_L=n_1$. 
    For each choice of $1\leq k_1,\ldots, k_{n_1} \leq K$, there are parameters for which the number of linear regions is $\sum_{j=0}^{n_0} \sum_{S \in {\binom{[n_1]}{j}}}\prod_{i\in S}(k_i-1)$.
    For $\min\{K,n_0+1\}\leq k_1,\ldots, k_{n_1}\leq K$, the corresponding set has positive measure. Here $S \in \binom{[n_1]}{j}$ means that $S$ is a subset of $[n_1]:=\{1,\ldots, n_1\}$ of cardinality $|S|=j$. 
    \item 
    Consider a network with $n_0$ inputs and $L$ layers of $n_1,\ldots, n_L$ rank-$K$ maxout units, $K\geq 2$, $\frac{n_l}{n_0}$ even. 
    Then, for each choice of $1\leq k_{li}\leq K$, $i=1,\ldots, n_0$, $l=1,\ldots, L$, there are parameters for which the number of linear regions is $\prod_{l=1}^{L} \prod_{i=1}^{n_0}(\frac{n_l}{n_0} (k_{li}-1)+1)$. 
    There is a positive measure subset of parameters for which the latter is the number of linear regions over $(0,1)^{n_0}$. 
    \end{itemize}
\end{restatable}

The proof is provided in Appendix~\ref{app:posvol}.
The result shows that maxout networks have a multitude of positive measure subsets of parameters over which they attain widely different numbers of linear regions. 
In the last statement of the theorem we consider inputs from a cube, but  qualitatively similar statements can be formulated for the entire input space. 

There are specific parameter values for which the network represents functions with very few linear regions (e.g., setting the weights and biases of the last layer to zero). However, the smallest numbers of regions are only attained over null sets of parameters: 

\begin{restatable}[Generic lower bound on the number of linear regions]{theorem}{theoremlowerbound}
\label{thm:lower_bound}
Consider a rank-$K$ maxout network, $K\geq 2$, with $\nin$ inputs, $n_1$ units in the first layer, and any number of additional nonzero width layers. 
Then, for almost every choice of the parameters, the number of linear regions is at least $\sum_{j=0}^{\nin}{\binom{n_1}{j}}$ and the number of bounded linear regions is at least $\binom{n_1-1}{n_0}$.
\end{restatable}
This lower bound has asymptotic order $\Omega(n_1^{n_0})$ in $K$ and $n_1,\ldots, n_L$. 
The proof is provided in Appendix~\ref{app:posvol}.
To our knowledge, this is the first non-trivial probability-one lower bound for a maxout network. 
Note that this statement does not apply to ReLU networks unless they have a single layer of ReLUs.
In the next section we investigate the expected number of activation regions for given probability distributions over the parameter space.

\section{Expected number and volume of activation regions}
\label{section:number}

For the expected number of activation regions we obtain the following upper bound, which corresponds to a maxout version of \citep[Theorem~10]{NIPS2019_8328}. 
\begin{restatable}[Upper bound on the expected number of partial activation regions]{theorem}{mainresulttheorem}
    \label{th:main_result}
     Let $\net$ be a fully-connected feed-forward maxout network with $\nin$ inputs and a total of $N$ rank $K$ maxout units. 
     Suppose we have a probability distribution over the parameters so that: 
    \begin{enumerate}[leftmargin=*]
        \item The distribution of all weights has a density with respect to the Lebesgue measure on $\mathbb{R}^{\# \text{\normalfont weights}}$.
        \item Every collection of biases has a conditional density with respect to Lebesgue measure given the values of all other weights and biases. 
        \item There exists $C_{\text{\normalfont grad}} > 0$ so that for any $t\in \mathbb{N}$ and any pre-activation feature $\zeta_{z,k}$, 
        \begin{align*}
            \sup\limits_{x \in \mathbb{R}^{\nin}} \mathbb{E}[\| \nabla \zeta_{z, k}(x) \|^t] \leq C^t_{\text{\normalfont grad}}.
        \end{align*} 
        \item There exists $C_{\text{\normalfont bias}} > 0$ so that for any pre-activation features $\zeta_1, \ldots, \zeta_t$ from any neurons, the conditional density of their biases $\rho_{b_{1}, \dots, b_{t}}$ given all the other weights and biases satisfies 
        \begin{align*}
            \sup \limits_{b_{1}, \dots, b_{t}\in \mathbb{R}} \rho_{b_{1}, \dots, b_{t}} (b_{1}, \dots, b_{t}) \leq C_{\text{\normalfont bias}}^t.
        \end{align*}
    \end{enumerate}
    Fix $r \in \{0, \dots, \nin\}$ and let $T = 2^5 \cgrad \cbias$. 
    Then, there exists $\delta_0\leq 1/(2 \cgrad \cbias)$ such that for all cubes $C \subseteq\mathbb{R}^{\nin}$ with side length $\delta>\delta_0$ we have 
    \begin{align*}
        \frac{\mathbb{E}[\# \text{ \normalfont$r$-partial activation regions of } \mathcal{N} \text{ \normalfont in } C]}{\vol(C)} \leq
        \begin{cases}
            \binom{rK}{2r} \binom{N}{r} K^{N - r}, \quad N \leq \nin
            \\[5pt]
            \frac{(T K N)^{\nin} \binom{\nin K}{2 \nin} }{(2 K)^r \nin!}, \quad N \geq \nin
        \end{cases}. 
    \end{align*}
    Here the expectation is taken with respect to the distribution of weights and biases in $\net$. 
    Of particular interest is the case $r=0$, which corresponds to the number of linear regions. 
\end{restatable}

The proof of Theorem~\ref{th:main_result} is given in Appendix~\ref{app:expected_number}. 
The upper bound has asymptotic order $O(N^{\nin} K^{3\nin- r})$ in $K$ and $N$, which is polynomial.
In contrast, \citet{sharp2021} shows that the maximum number of linear regions of a deep network of width $n$ is $\Theta((nK)^{\frac{n_0}{n} N})$, which for constant width is exponential in $N$; see Appendix~\ref{app:posvol}. 
We present an analogue of Theorem~\ref{th:main_result} for networks without biases in Appendix~\ref{app:zero_bias}. 

When the rank is $K=2$, the formula coincides with the result obtained previously by \citet[Theorem~10]{NIPS2019_8328} for ReLU networks, up to a factor $K^r$. 
For some settings, we expect that the result can be further improved. For instance, for iid Gaussian weights and biases, one can show that the expected number of regions of a rank $K$ maxout unit grows only like $\log K$, as we discuss in Appendix~\ref{app:single_unit}. 

We note that the constants $C_{\text{bias}}$ and $C_{\text{grad}}$ only need to be evaluated over the inputs in the region $C$. 
Intuitively, the bound on the conditional density of bias values corresponds to a bound on the density of non-linear locations over the input.
The bound on the expected gradient norm of the pre-activation features is determined by the distribution of weights. 
We provide more details in Appendix~\ref{app:constants}.

For the expected volume of the $r$-dimensional part of the non-linear locus we obtain the following upper bound, which corresponds to a maxout version of  \citep[Corollary~7]{pmlr-v97-hanin19a}. 
    
\begin{restatable}[Upper bound on the expected volume of the non-linear locus]{theorem}{semimainupperboundtheorem}
        \label{th:semi_main_upper_bound} 
    Consider a bounded measurable set $S \subset \mathbb{R}^{\nin}$ and the settings of Theorem~\ref{th:main_result} with constants $C_{\text{grad}}$ and $C_{\text{bias}}$ evaluated over $S$. Then, for any $r \in \{1, \ldots, \nin\}$,
    \begin{align*}
        \frac{\mathbb{E}[\vol_{\nin - r}(\mathcal{X}_{\net, r} \cap S)]}{\vol_{\nin}(S)} \leq   (2 \cgrad \cbias)^r \binom{rK}{2r} \binom{N}{r}. 
    \end{align*}
\end{restatable}
The proof of Theorem~\ref{th:semi_main_upper_bound} is given in Appendix~\ref{app:boundary_volume}.
When the rank is $K=2$, the formula coincides with the result obtained previously by \citet[Corollary~7]{pmlr-v97-hanin19a} for ReLU networks.
A table comparing the results for maxout and ReLU networks is given in Appendix~\ref{app:expected_number}.

\section{Expected number of pieces and volume of the decision boundary}

In the case of classification problems, we are primarily interested in the decision boundary, rather than the overall function. We define an $M$-class classifier by appending an argmax gate to a network with $M$ outputs.
The decision boundary is then a union of certain $r$-partial activation regions for the network with a maxout unit as the output layer.
For simplicity, here we present the results for the $n_0-1$-dimensional regions, which we call `pieces', and present the results for
arbitrary values of $r$ in Appendix \ref{app:decision_boundary}. 
The number of pieces of the decision boundary is at most equal to the number of activation regions in the original network times $\binom{M}{2}$. 
A related statement appeared in \cite{alfarra2020on}.
For specific choices of the network parameters, the decision boundary does intersects most activation regions and can have as many as $\Omega (M^2  \prod_{l=1}^L(n_l K)^{n_0})$ pieces (see Appendix~\ref{app:decision_boundary}).
However, in general this upper bound can be improved. 
For the expected number of pieces and volume of the decision boundary we obtain the following results. 
We write $\mathcal{X}_{\operatorname{DB}}$ for the decision boundary, and $\mathcal{X}_{\operatorname{DB},r}$ for the union of $r$-partial activation regions which include equations from the decision boundary (generically these are the co-dimension-$r$ pieces of the decision boundary). 

\begin{restatable}[Upper bound on the expected number of linear pieces of the decision boundary]{theorem}{theoremdecisionboundary}
    \label{th:decision_boundary}
    Let $\net$ be a fully-connected feedforward maxout network, with $\nin$ inputs, a total of $N$ rank-$K$ maxout units, and $M$ linear output units used for multi-class classification. 
    Under the assumptions of Theorem~\ref{th:main_result}, there exists $\delta_0\leq 1/(2 \cgrad \cbias)$ such that for all cubes $C \subseteq\mathbb{R}^{\nin}$ with side length $\delta>\delta_0$,
    \begin{align*}
        \frac{\mathbb{E}\big[\substack{\# \text{ \normalfont linear pieces in the } \\  \text{\normalfont decision boundary of } \mathcal{N} \text{ \normalfont in } C}\big]}{\vol(C)} \leq
        \begin{cases}
            \binom{M}{2}  K^{N} , \quad N \leq \nin
            \\[5pt]
            \frac{(2^4 \cgrad \cbias)^{\nin} (2 K N)^{\nin - 1} }{(\nin - 1)!} \binom{M}{2} \binom{K(\nin - 1)}{2 (\nin - 1)}, \quad N \geq \nin
        \end{cases}.
    \end{align*}
    Here the expectation is taken with respect to the distribution of weights and biases in $\net$. 
\end{restatable} 

For binary classification, $M = 2$, this bound has asymptotic order $O((K^3N)^{\nin - 1})$ in $K$ and $N$. 
For the expected volume we have the following. 

\begin{restatable}[Upper bound on the volume of the $(\nin-r)$-skeleton of the decision boundary]{theorem}{lemmaboundaryvolume}
    \label{th:decision_boundary_volume}
    Consider a bounded measurable set $S \subset \mathbb{R}^{\nin}$. 
    Consider the notation and assumptions of Theorem~\ref{th:main_result}, 
    whereby the constants $C_{\text{grad}}$ and $C_{\text{bias}}$ are over $S$. 
    Then, for any $r \in \{1, \ldots, \nin\}$ we have
    \begin{align*}
        \frac{\mathbb{E}[\vol_{\nin - r}(\mathcal{X}_{\operatorname{DB}, r} \cap S)]}{\vol_{\nin}(S)} \leq (2 \cgrad \cbias)^r \sum_{i = 1} ^{\min\{M - 1, r\}}  \binom{M}{i+1} \binom{K(r - i)}{2 (r - i)} \binom{N}{r-i}.
    \end{align*}
\end{restatable}

Moreover, the expected distance to the decision boundary can be bounded as follows. 

\begin{restatable}[Distance to the decision boundary]{corollary}{cordbdistance}
    \label{cor:dist_to_db}
    Suppose $\net$ is as in Theorem \ref{th:main_result}. For any compact set $S \subset \mathbb{R}^{\nin}$ let $x$ be a uniform point in $S$.
    There exists $c > 0$ independent of $S$ so that
    \begin{align*}
        \mathbb{E} [\text{\normalfont distance}(x, \mathcal{X}_{\operatorname{DB}} )] \geq
        \frac{c}{2 \cgrad \cbias M^{m + 1} m},
    \end{align*}
    where $m \defeq \min\{M - 1, \nin \}$.
\end{restatable}

The proofs are presented in Appendix~\ref{app:decision_boundary}, where we also extend Theorem~\ref{th:decision_boundary} to address the expected number of co-dimension-$r$ pieces of the decision boundary. 
A corresponding result applies for the case of ReLU networks (see details in Appendix~\ref{app:decision_boundary}). 

\section{Experiments}
\label{section:experiments}

In the experiments we used fully-connected networks. We describe the network architecture in terms of the depth and total number of units, with units evenly distributed across the layers with larger lower layers if needed. 
For instance, a network of depth $3$ with $110$ units has 3 hidden layers of widths $37, 37, 36$. 
Details and additional experiments are presented in Appendix~\ref{app:experiments}. 
The computer implementation of the key functions is available on GitHub at \url{https://github.com/hanna-tseran/maxout_complexity}. 

\paragraph{Initialization procedures} 
We consider several initialization procedures detailed in Appendix~\ref{app:init}: 
1) ReLU-He initializes the parameters as iid samples from the distribution proposed by \cite{he2015delving} for ReLUs. 
2) Maxout-He follows a similar reasoning to normalize the expected norm of activation vectors across layers, but for the case of maxout networks. The weight distribution has standard deviation depending on $K$ and the assumed type of data distribution, as shown in Table~\ref{tb:std}. 
3) ``Sphere'' ensures each unit has the maximum number of regions. 
4) ``Many regions'' ensures each layer has the maximum number of regions. 

\begin{table}[ht]
    \caption{Standard deviation of the weight distribution for maxout-He initialization. 
    }
    \label{tb:std}
    \centering
    {\small 
    \begin{tabular}{c
    c}
        \toprule
        Maxout rank & 
        Standard deviation \\
        \midrule
        $2$ &  
        $\sqrt{1/n_l}$ \\
        $3$ &
        $\sqrt{2 \pi /((\sqrt{3} + 2 \pi) n_l)}$ \\
        $4$ & 
        $\sqrt{ \pi /((\sqrt{3} + \pi) n_l)}$ \\
        $5$ & 
        $\sqrt{ 0.5555 / n_l}$ \\
        \midrule
        ReLU & 
        $\sqrt{2/n_l}$ \\
        \bottomrule
    \end{tabular}
    }
\end{table}

\paragraph{Algorithm for counting activation regions} 
Several approaches for counting linear regions of ReLU networks have been considered \citep[e.g.,][]{serra2018bounding, NIPS2019_8328, serra2020empirical, xiong2020number}.
For maxout networks we count the activation regions and pieces of the decision boundary by iterative addition of linear inequality constraints and feasibility verification using linear programming. 
Pseudocode and complexity analysis are provided in Appendix~\ref{app:algorithm}.

\paragraph{Number of regions and decision boundary for different networks}
Figure~\ref{fig:formula} shows a close agreement, up to constants, of the theoretical upper bounds on the expected number of activation regions and on the expected number of linear pieces of the decision boundaries with the empirically obtained values for different networks. Further comparisons with constants and different values of $K$ are provided in Appendix~\ref{app:experiments}.
Figure~\ref{fig:neurons_to_depth} shows that for common parameter distributions, the growth of the expected number of activation regions is more significantly determined by the total number of neurons than by the network's depth.
In fact, we observe that for high rank units and certain types of distributions, deeper networks may have fewer activation regions. 
We attribute this to the fact that higher rank units tend to have smaller images (since they compute the max of more pre-activation features).
Figure~\ref{fig:shallow_and_rank} shows how $\nin$ and $K$ affect the number of activation regions. 
For small input dimension, the number of regions per unit tends to be smaller than $K$. 
Indeed, for iid Gaussian parameters the number of regions per unit scales as $\log K$ (see Appendix~\ref{app:single_unit}). 

\paragraph{Number of regions during training} 
We consider the 10 class classification task with the MNIST dataset~\citep{lecun2010mnist} and optimization with Adam \citep{DBLP:journals/corr/KingmaB14} using different initialization strategies. 
Notice that for deep skinny fully connected networks the task is non-trivial. 
Figure~\ref{fig:training} shows how the number of activation regions evolves during training. Shown are also the linear regions and decision boundaries over a 2D slice of the input space through 3 training data points. Figure~\ref{fig:loss} shows the training loss and accuracy curves for the different initializations. We observe that maxout networks with maxout-He, sphere, and many regions converge faster than with naive He initialization.

\begin{figure}
    \centering
    \begin{tabular}{cc}
        {\small Number of activation regions} & {\small Number of pieces in the decision boundary}\\
        \includegraphics[trim=10 0 5 0, clip, width=0.35\textwidth]{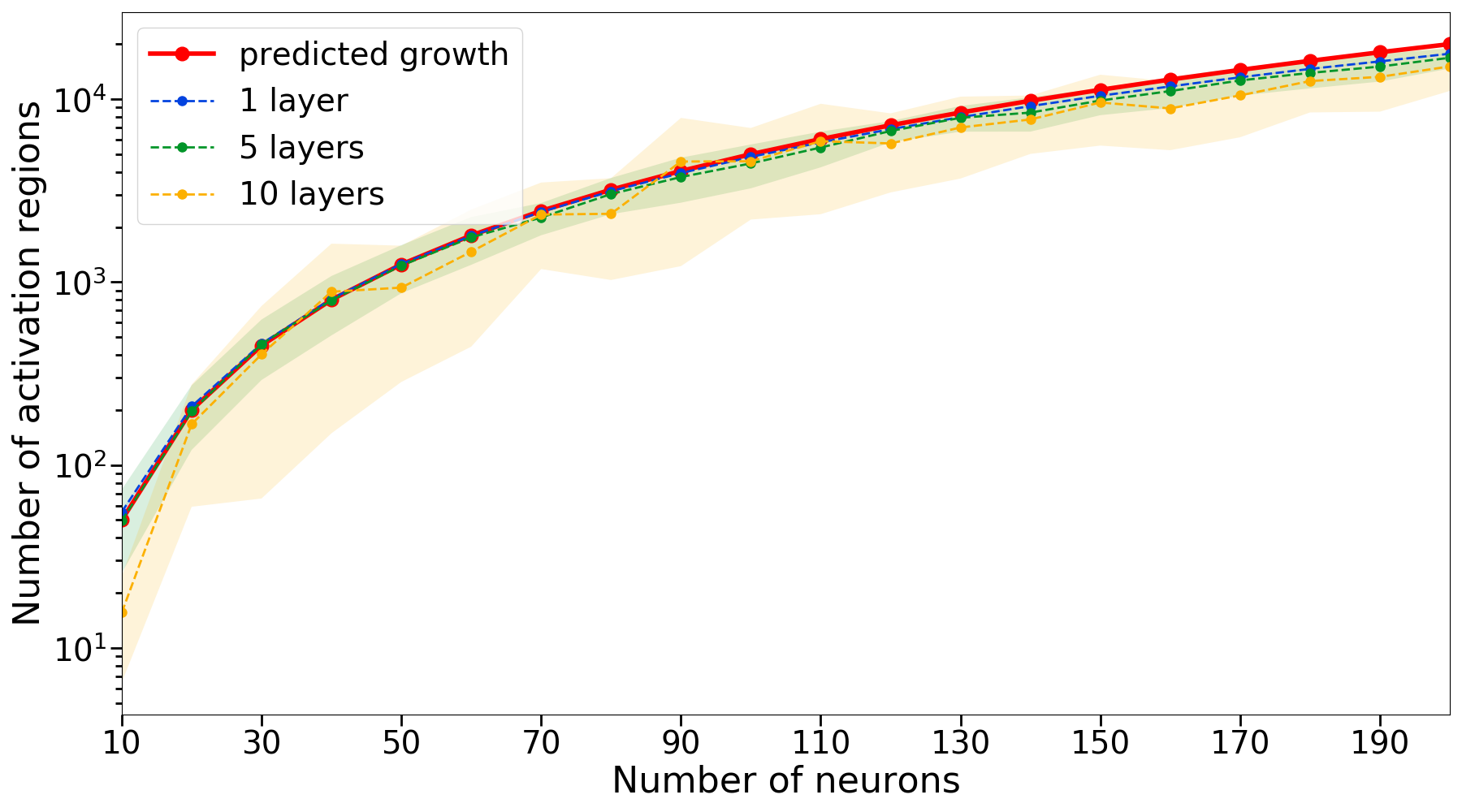}&
        \includegraphics[width=0.35\textwidth]{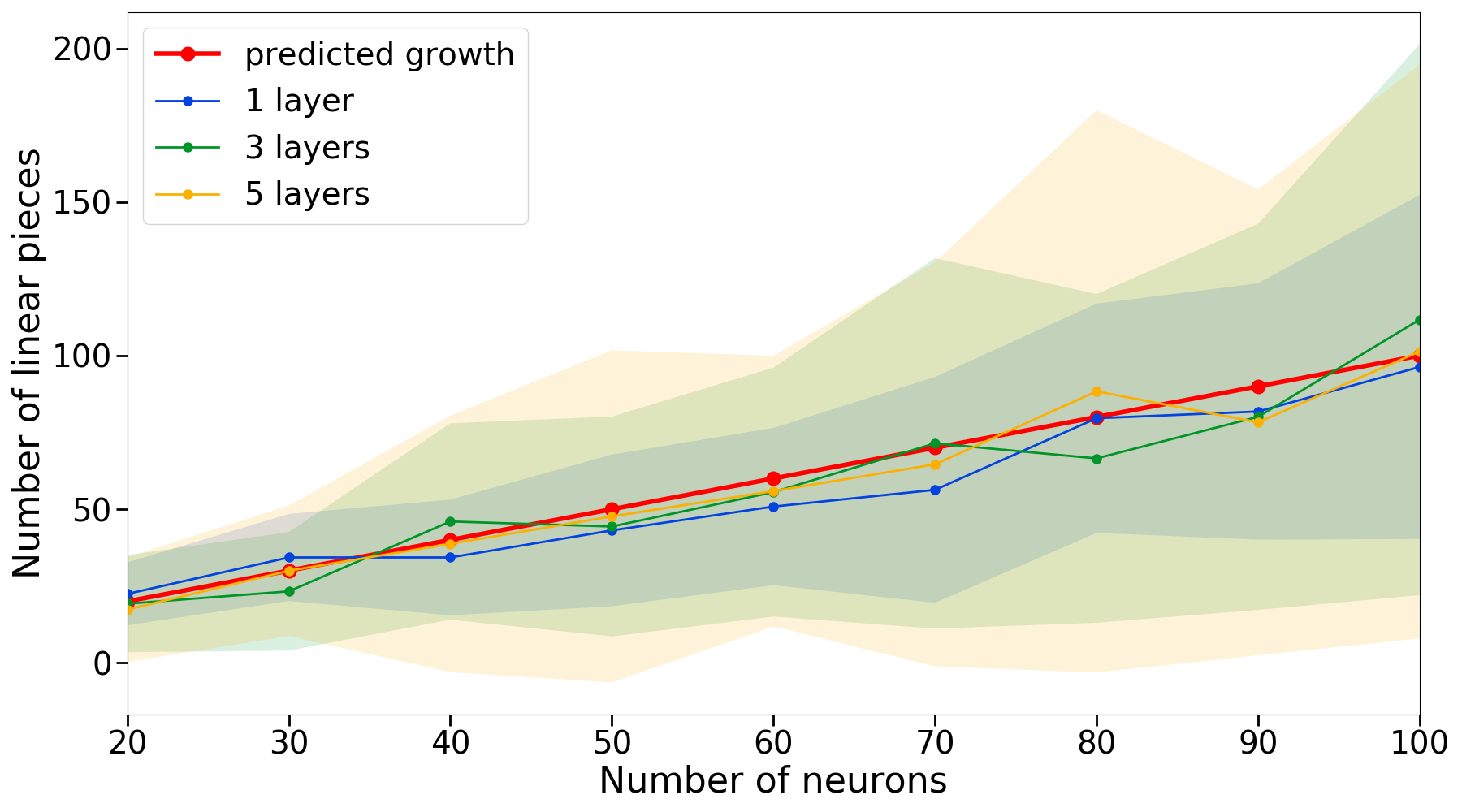}
    \end{tabular}
    \caption{Shown are means and stds for 30 maxout-He normal initializations for networks with $K = 2$ and $\nin = 2$. Left: Comparison of the theoretically predicted growth $O(N^{\nin} / \nin!)$ and the experimentally obtained number of regions for networks with different architectures. Right: Comparison of the theoretically predicted growth $O(N)$ and the experimentally obtained number of linear pieces of the decision boundary for networks with different architectures.
    }
    \label{fig:formula} 
\end{figure}
    
\begin{figure}
    \centering
    \begin{tabular}{cc}
        {\small Maxout rank $K = 2$} & {\small Maxout rank $K = 5$}\\
        \begin{subfigure}{.47\textwidth}
            \centering
            \includegraphics[trim=10 10 10 20, clip, width=.52\textwidth] {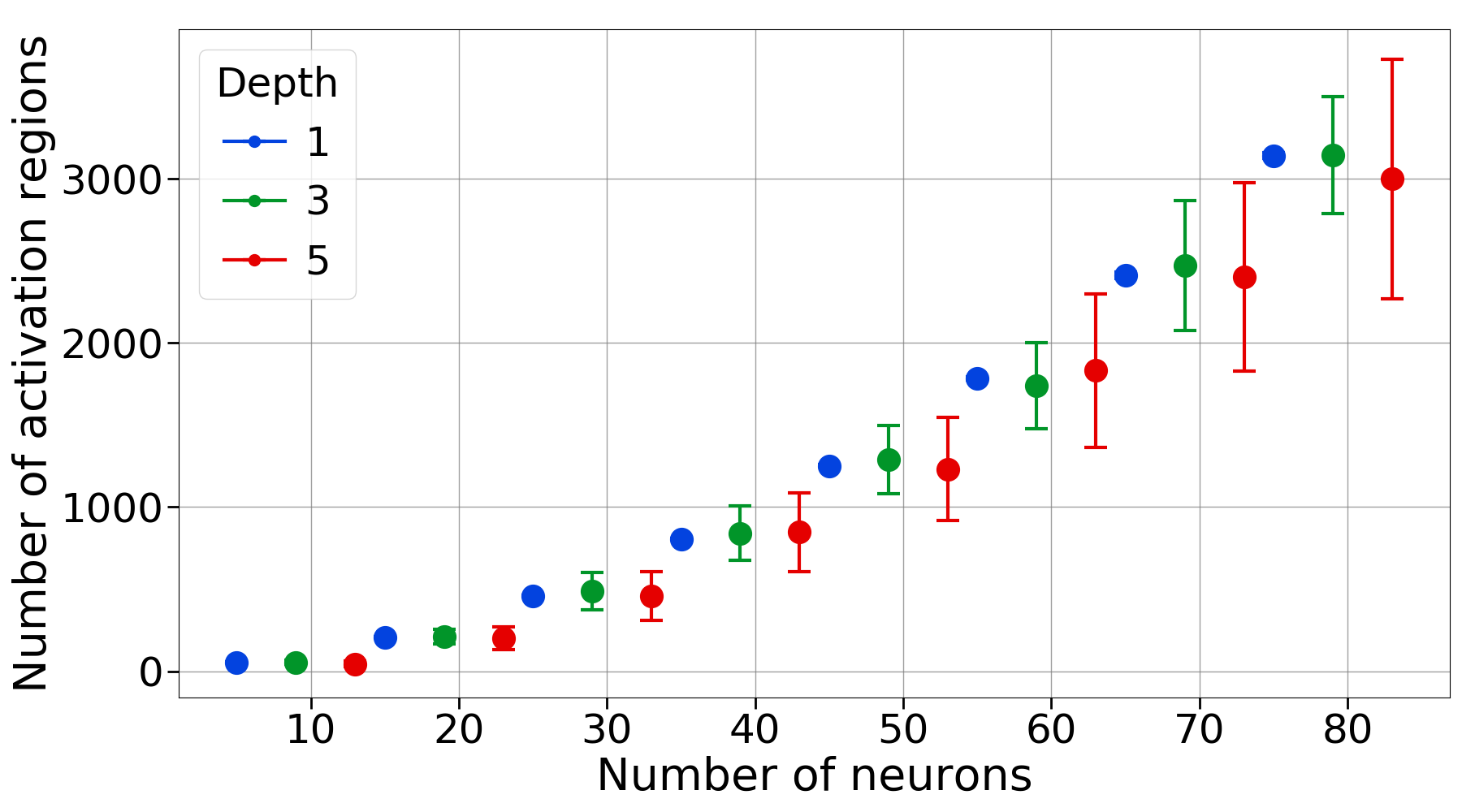}
            \includegraphics[trim=30 40 15 10, clip, width=.42\textwidth] {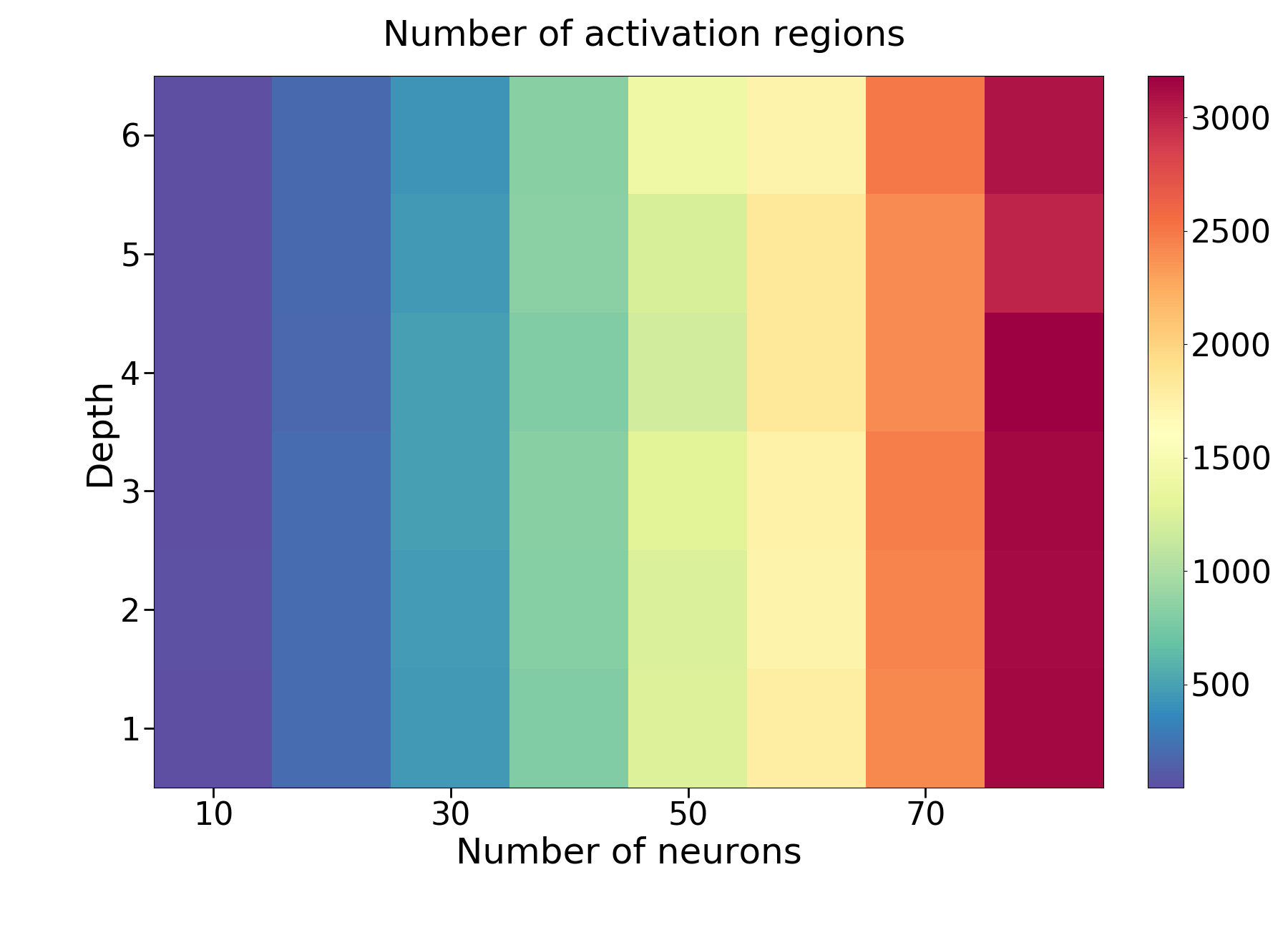}
        \end{subfigure}&
        \begin{subfigure}{.47\textwidth}
            \centering
            \includegraphics[trim=10 10 10 20, clip, width=.52\textwidth] {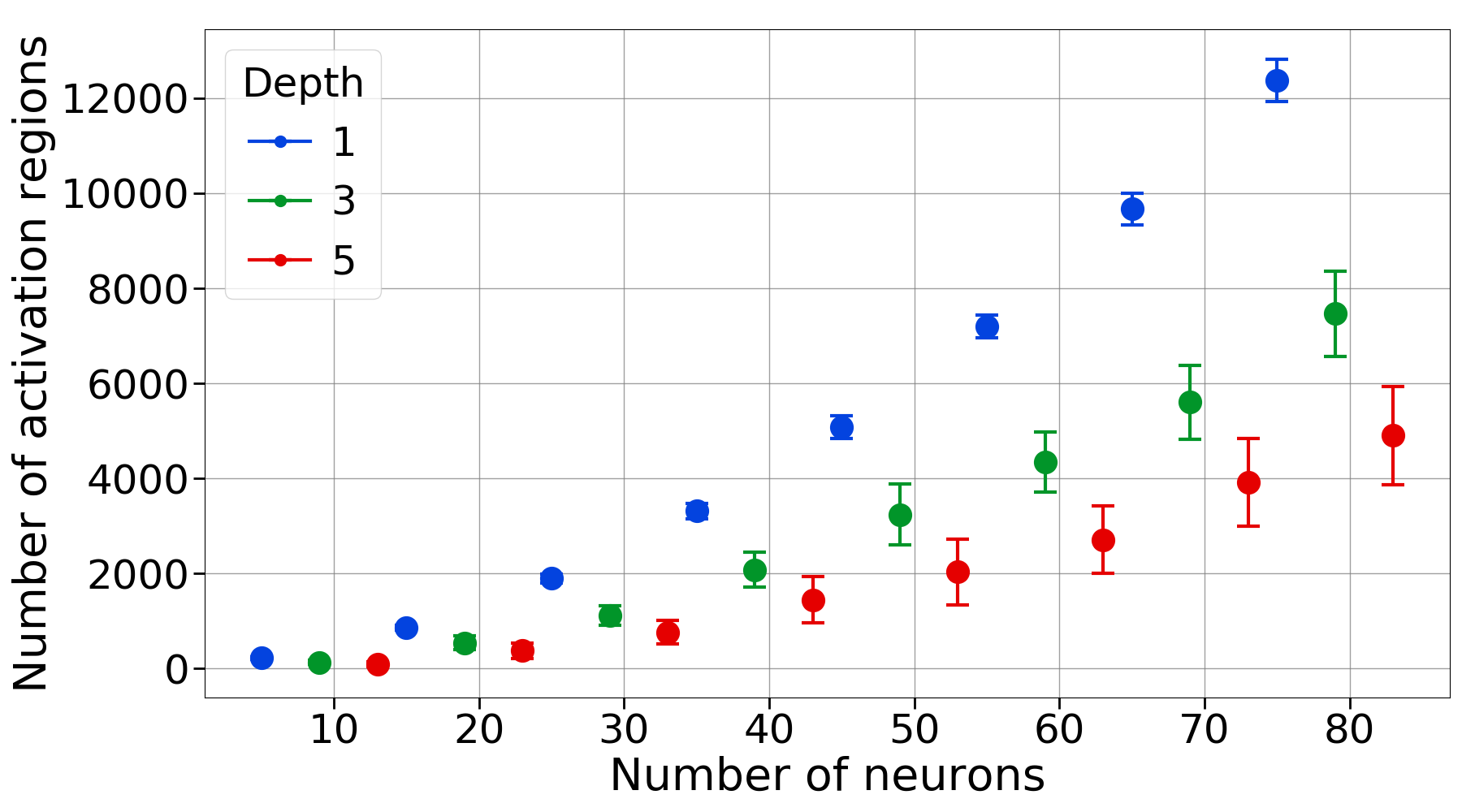}
            \includegraphics[trim=30 40 15 10, clip, width=.42\textwidth] {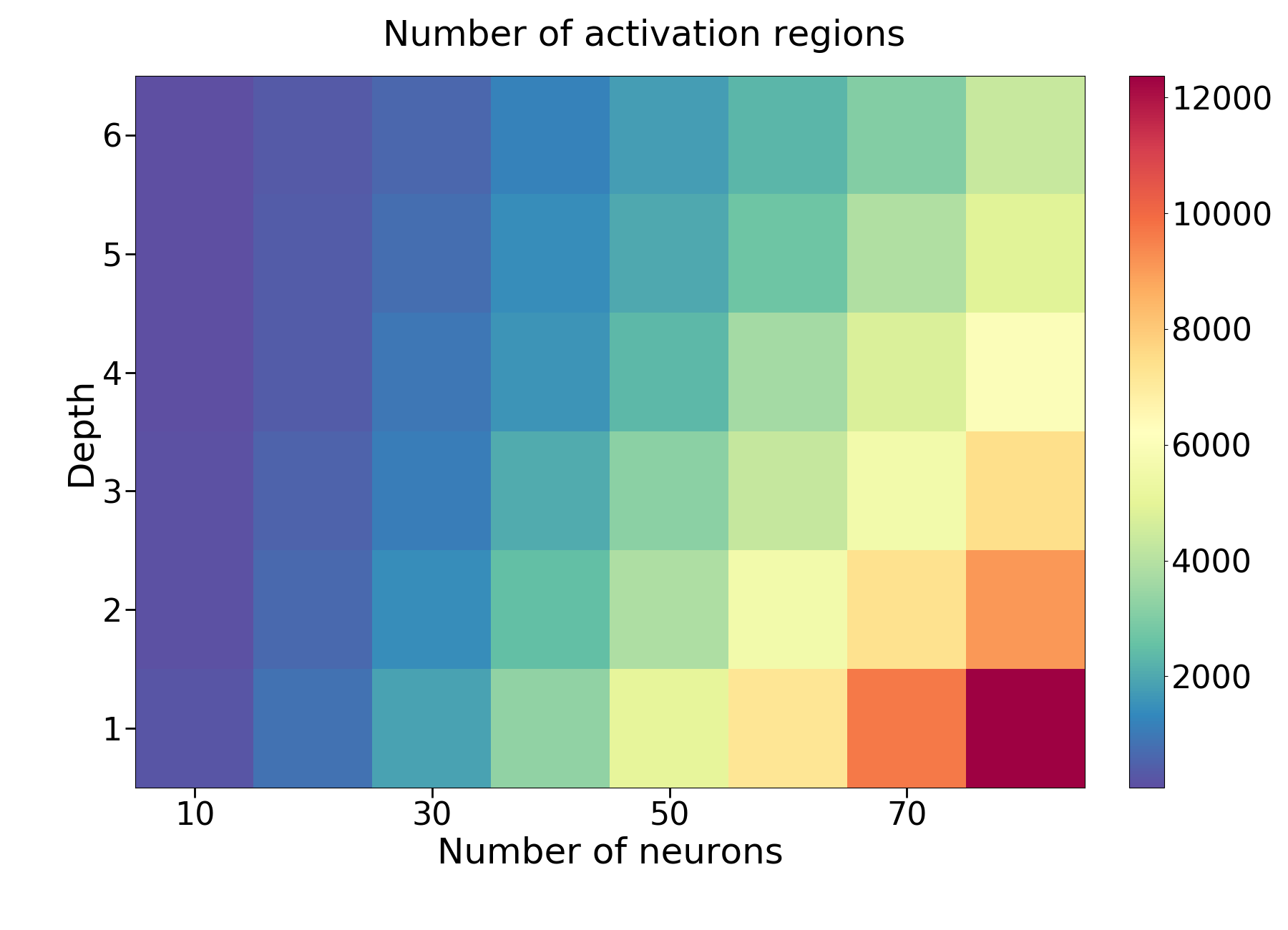}
        \end{subfigure}
    \end{tabular}
    \caption{
    Effect of the depth and number of neurons on the number of activation regions at initialization for networks with $\nin = 2$. 
    Shown are means and stds for 30 maxout-He normal initializations.
}
    \label{fig:neurons_to_depth} 
\end{figure}

\begin{figure}
    \begin{subfigure}{1.\textwidth}
        \centering
        \begin{tabular}{cc}
        {\small Contribution per unit in a shallow network} & {\small Effect of the maxout rank} \\
        \includegraphics[width=0.35\textwidth]{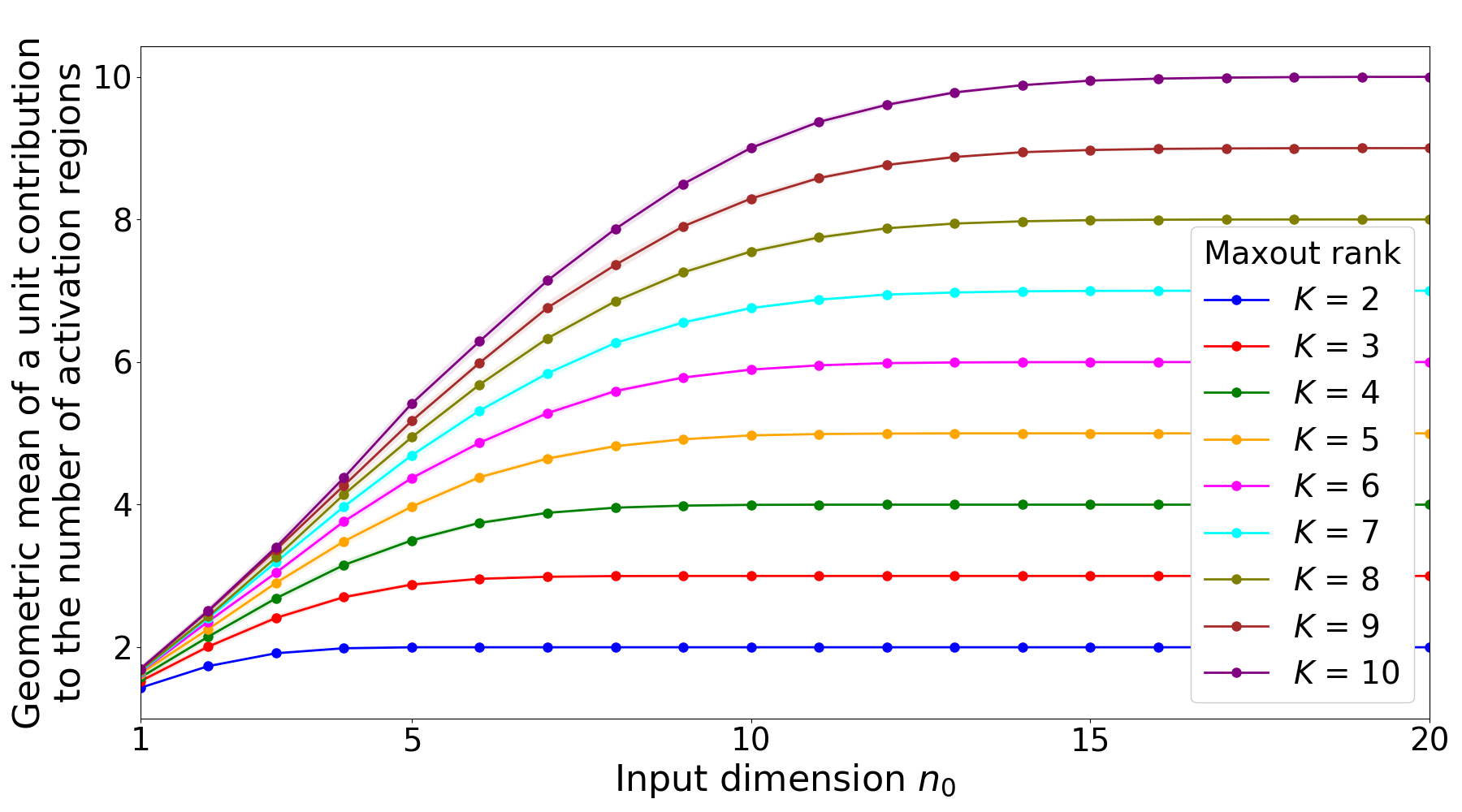}&
        \includegraphics[width=0.35\textwidth]{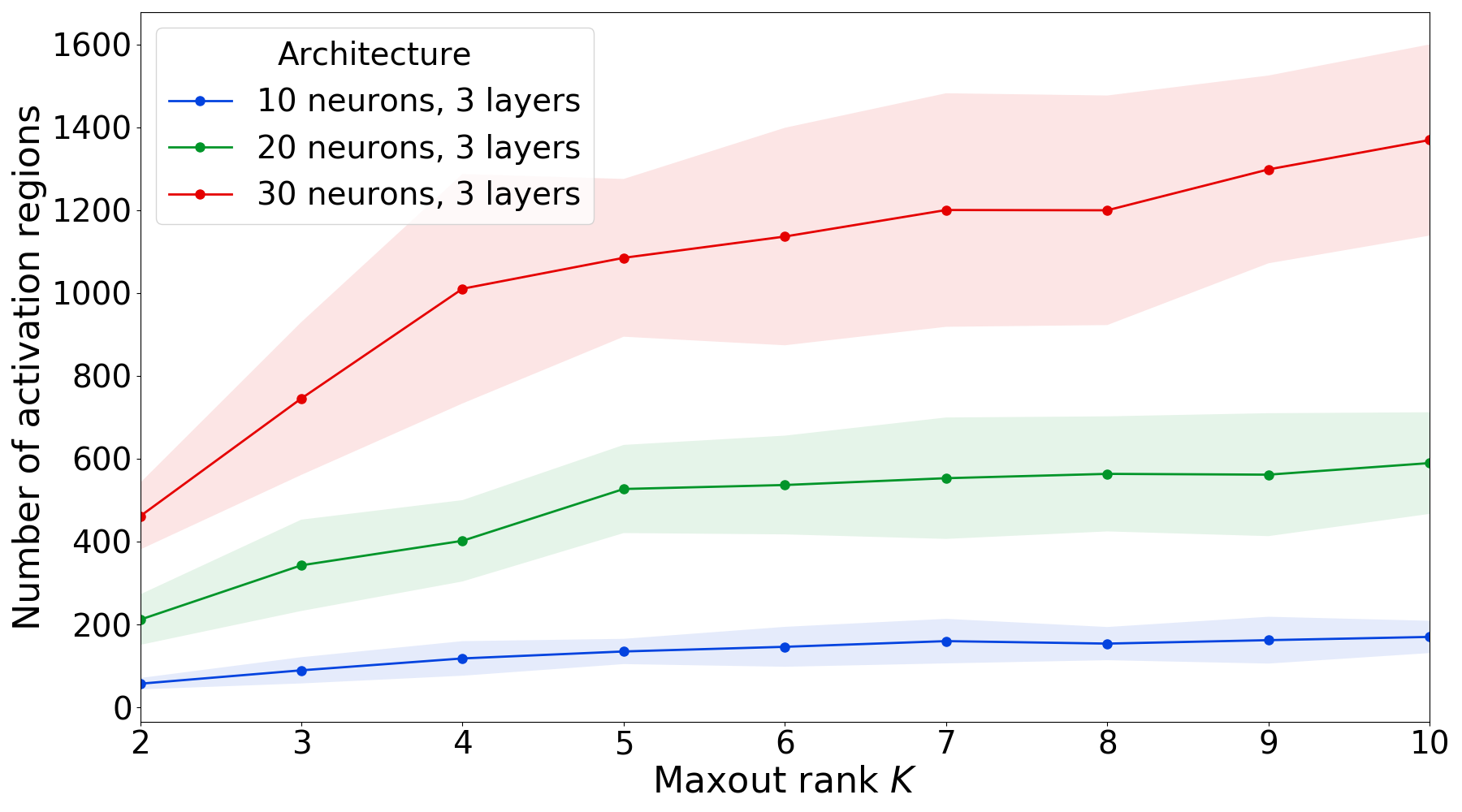}
        \end{tabular}
    \end{subfigure}
    \caption{
        Left: Plotted is $\#\text{regions}^{1/N}$ for a shallow network with $N=5$.
        The multiplicative contribution per unit increases with the input dimension until the trivial upper bound $K$ is reached.
        Right: Number of regions of 3 layer networks with $\nin = 2$ depending on $K$. 
        Shown are means and stds for 30 ReLU-He normal initializations.
        }
    \label{fig:shallow_and_rank} 
\end{figure}

\begin{figure}
    \begin{tabular}{cccc}
        \centering
        & {\small Number of regions during training} & {\small Before training} & {\small After 100 epochs} \\
        
        \begin{minipage}[c]{0.03\textwidth}
            \begin{flushright}
                {\rotatebox{90}{\small Linear regions}}
            \end{flushright}
        \end{minipage}
        &
        \begin{minipage}[c]{0.38\textwidth}
            \begin{center}
                \includegraphics[trim=10 10 10 10, clip, width=0.9\textwidth]{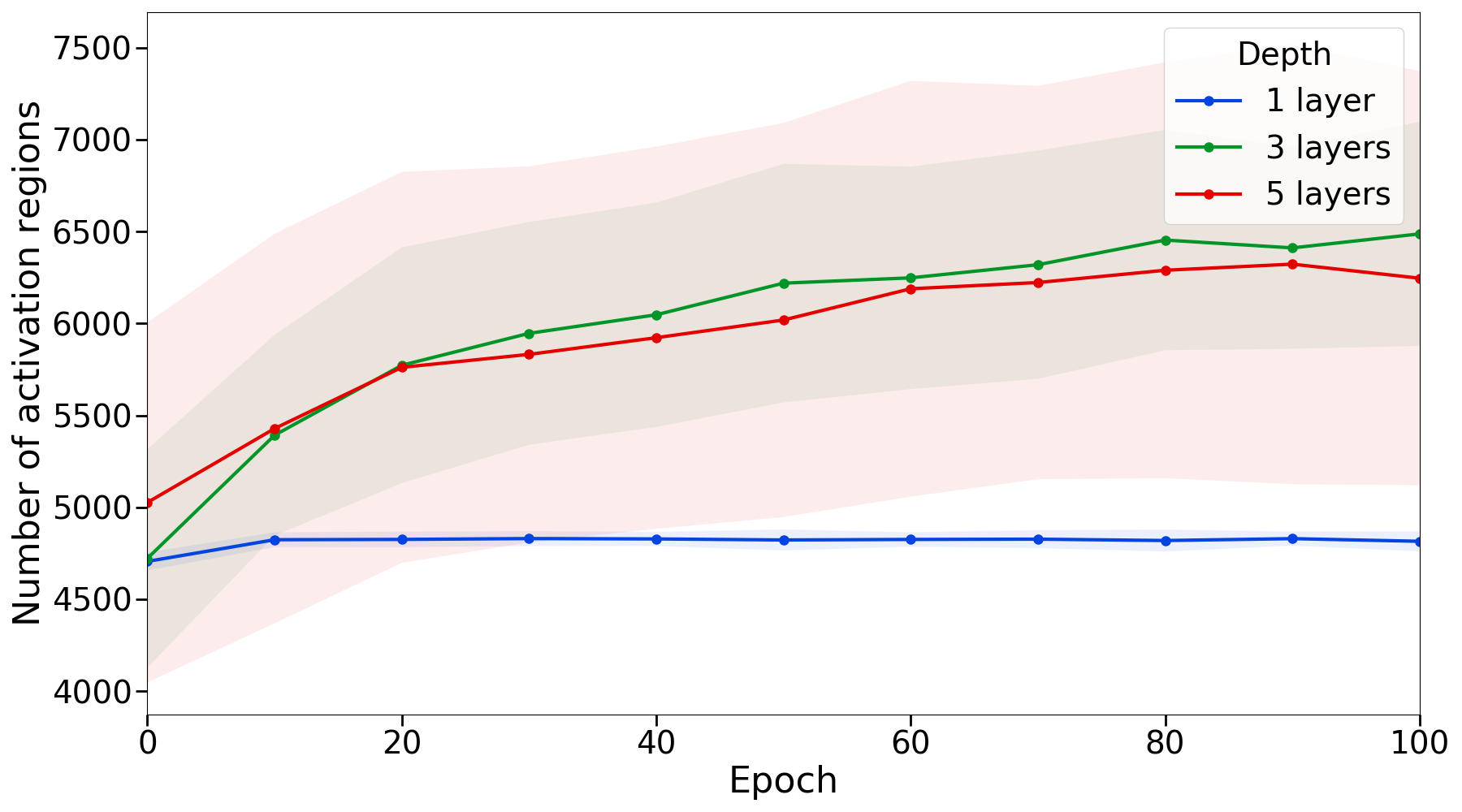}
            \end{center}
        \end{minipage}
        &
        \begin{minipage}[c]{0.2\textwidth} 
            \begin{center}
                \includegraphics[width=0.9\textwidth]{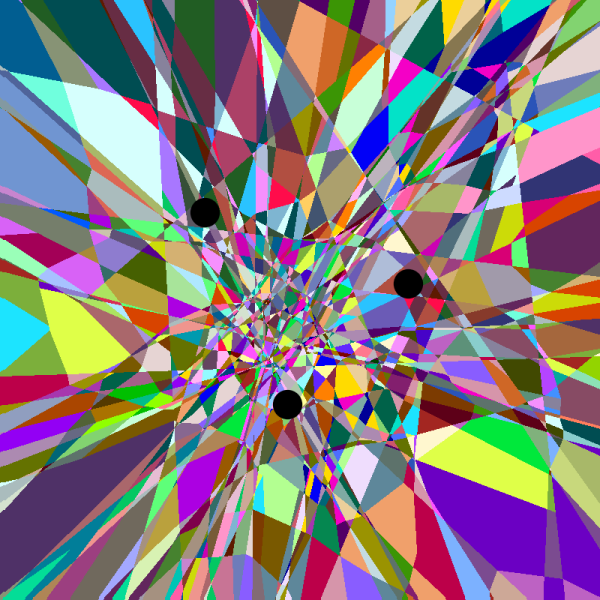}
            \end{center}
        \end{minipage}
        &
        \begin{minipage}[c]{0.2\textwidth} 
            \begin{center}
                \includegraphics[width=0.9\textwidth]{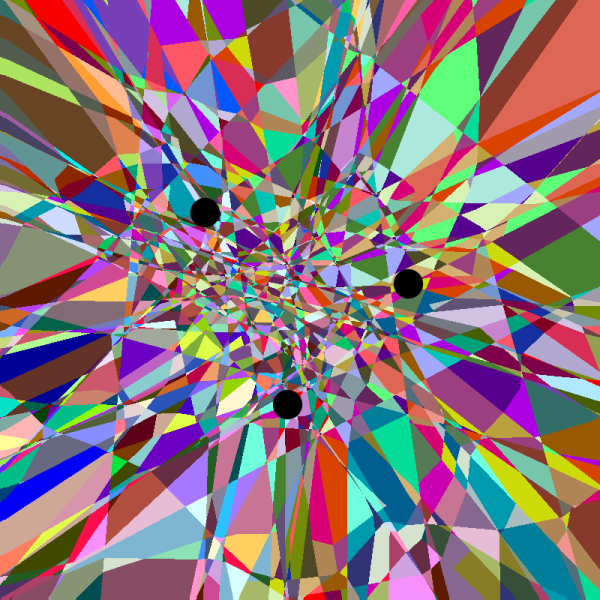}
            \end{center}
        \end{minipage} \\
       
        \vspace{.01cm}\\
        & {\small Number of pieces during training} & {\small Before training} & {\small After 100 epochs} \\
        
        \begin{minipage}[c]{0.03\textwidth}
            \begin{flushright}
                {\rotatebox{90}{\small Decision boundary}}
            \end{flushright}
        \end{minipage}
        &
        \begin{minipage}[c]{0.38\textwidth} 
            \begin{center}
                \includegraphics[trim=10 10 10 10, clip, width=0.9\textwidth]{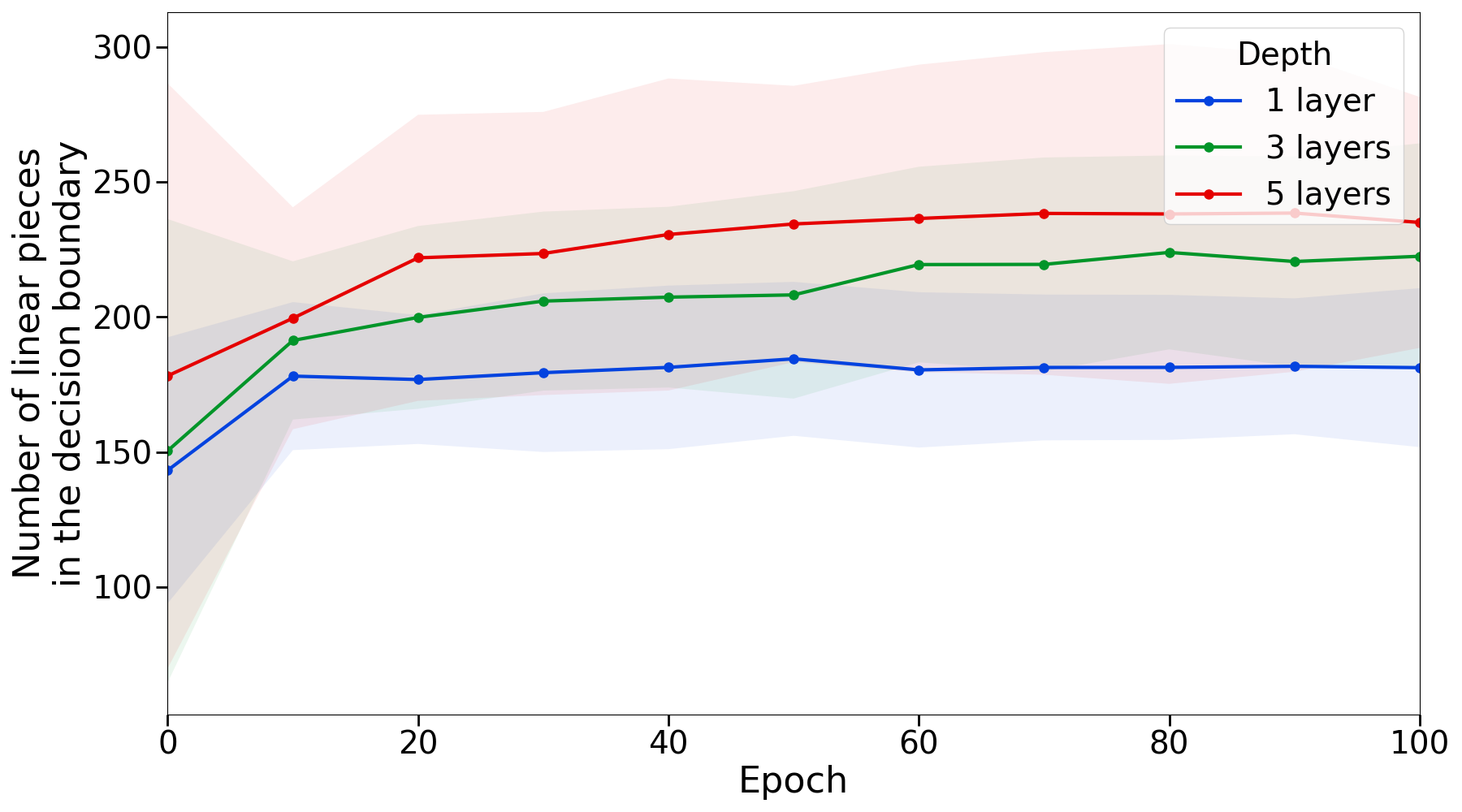}
            \end{center}
        \end{minipage}
         &
        \begin{minipage}[c]{0.2\textwidth} 
            \begin{center}
                \includegraphics[width=0.9\textwidth]{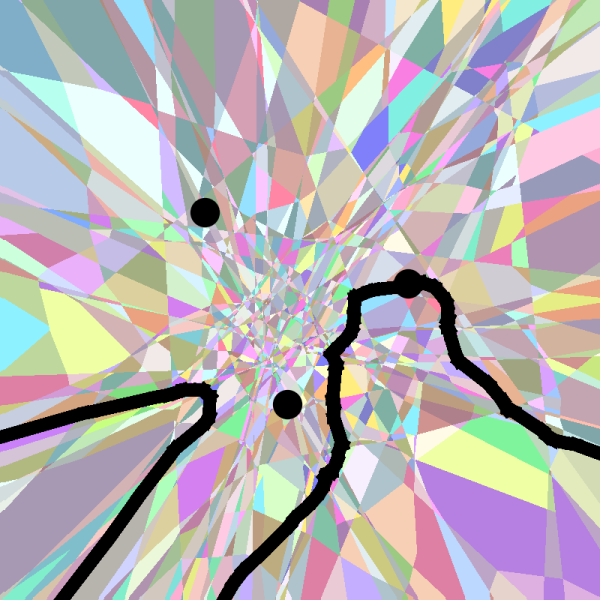}
            \end{center}
        \end{minipage} &
        \begin{minipage}[c]{0.2\textwidth} 
          \begin{center}
              \includegraphics[width=0.9\textwidth]{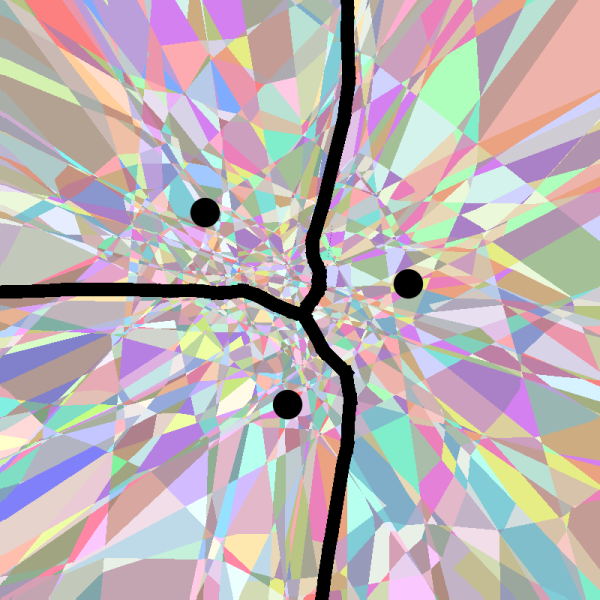}
          \end{center}
        \end{minipage}
    \end{tabular}
    \caption{Evolution of the linear regions and the decision boundary during training on the MNIST dataset in a slice determined by three random points from different classes. 
    The network had $100$ maxout units of rank $K = 2$, and was initialized using maxout-He normal initialization. The right panel is for the $3$ layer network. 
    As expected, for the shallow rank-$2$ network, the number of regions is approximately constant. For deep networks we observe a moderate increase in the number of regions as training progresses, especially around the training data.  However, the number of regions remains far from the theoretical maximum. This is consistent with previous observations for ReLU networks. 
    There is also a slight increase in the number of linear pieces in the decision boundary, and at the end of training the decision boundary clearly separates the three reference points. 
    }
    \label{fig:training} 
\end{figure}

\begin{figure}
    \begin{subfigure}{1.\textwidth}
        \centering
        \begin{tabular}{cc}
        {\small Loss on the training set} & {\small Accuracy on the validation set} \\
        \includegraphics[width=0.35\textwidth]{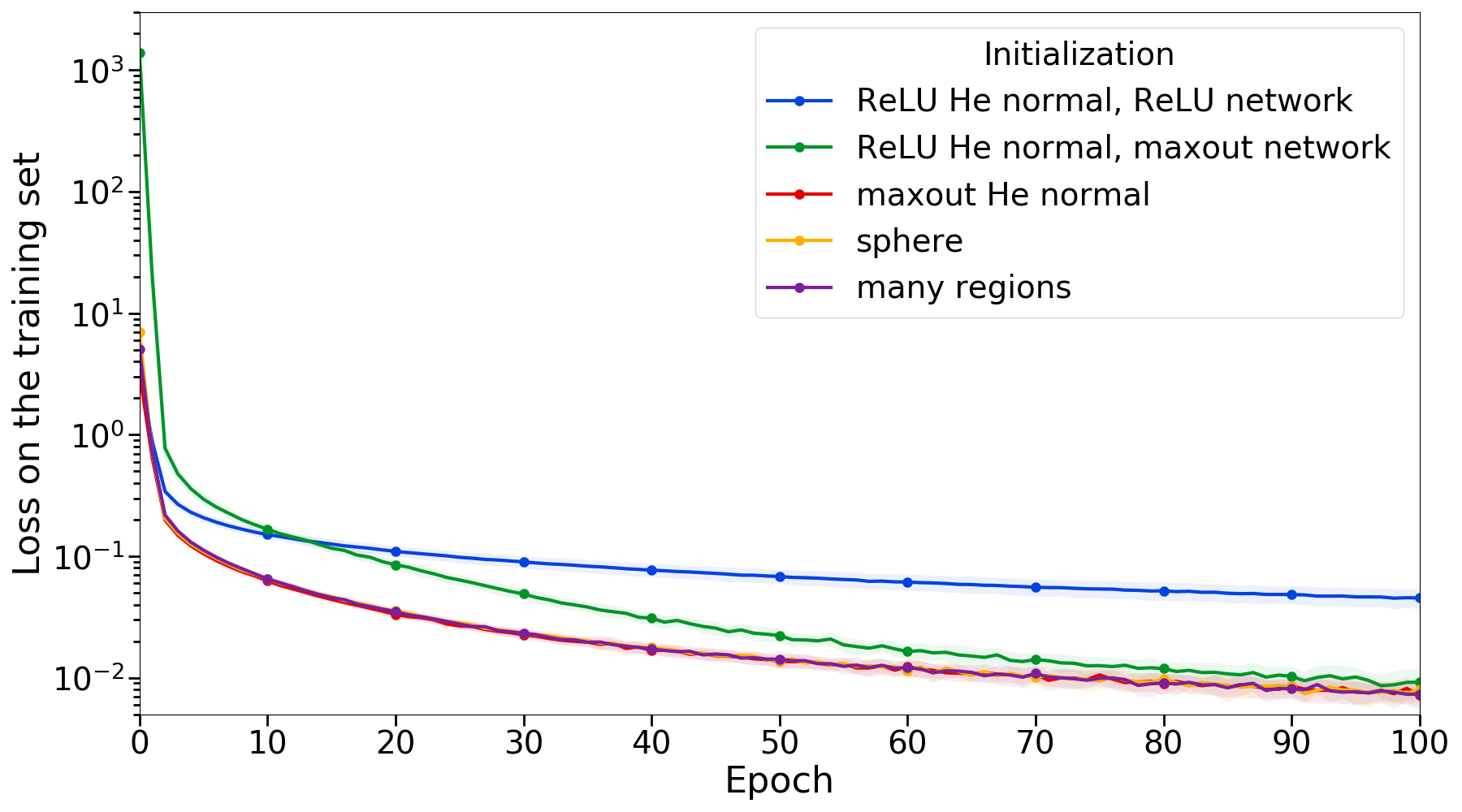}&
        \includegraphics[width=0.35\textwidth]{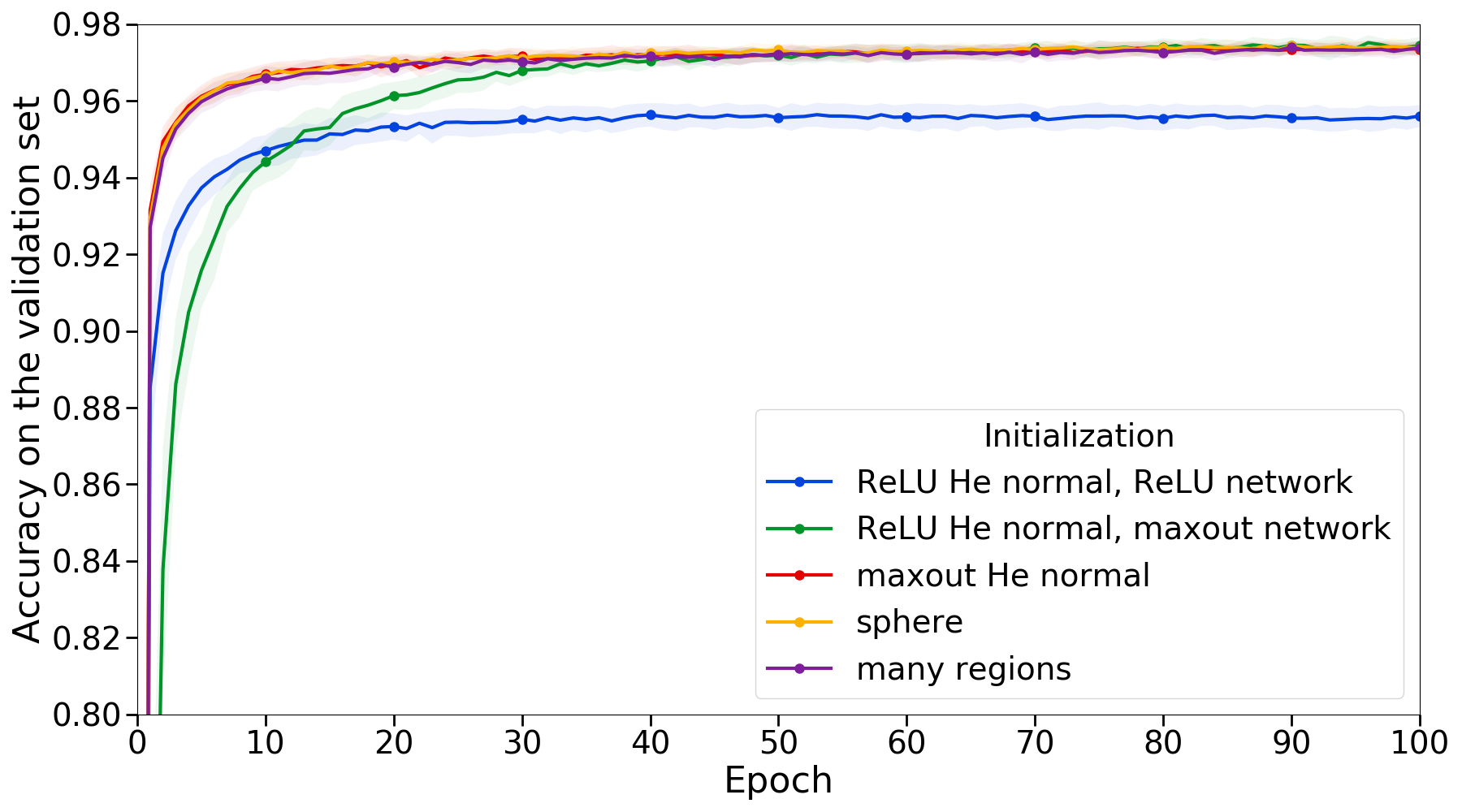}
        \end{tabular}
    \end{subfigure}
    \caption{
        Comparison of training on MNIST with different initializations. All networks had $200$ units, $10$ layers, and maxout networks had rank $K = 5$. Shown are averages and std (barely noticeable) over 30 repetitions. 
        The type of initialization has a significant impact on the training time of maxout networks, with maxout-He, sphere, and many regions giving better results for deep networks and larger maxout rank (more details on this in Appendix~\ref{app:experiments}). 
    }
    \label{fig:loss} 
\end{figure}

\newpage 
\section{Discussion}
\label{section:discussion}

We advance a line of analysis recently proposed by \citet{pmlr-v97-hanin19a,NIPS2019_8328}, where the focus lies on the expected complexity of the functions represented by neural networks rather than worst case bounds. Whereas previous works focus on single-argument activations, our results apply to networks with multi-argument maxout units. 
We observe that maxout networks can assume widely different numbers of linear regions with positive probability and then computed an upper bound on the expected number of regions and volume given properties of the parameter distribution, covering the case of zero biases. 
Further, taking the standpoint of classification, we obtained corresponding results for the decision boundary of maxout (and ReLU) networks, along with bounds on the expected distance to the decision boundary.

Experiments show that the theoretical bounds capture the general behavior. We present algorithms for enumerating the regions of maxout networks and proposed parameter initialization strategies with two types of motivations, one to increase the number of regions, and second, to normalize the variance of the activations similar to \citet{glorot2010understanding} and \citet{he2015delving}, but now for maxout. We observed experimentally that this can improve training in maxout networks.

\paragraph{Limitations} In our theory and experiments we have considered only fully connected networks. 
The analysis and implementation of other architectures for experiments with more diverse datasets are interesting extensions.
By design, the results focus on parameter distributions which have a density.

\paragraph{Future work}
In future work we would like to obtain a fine grained description of the distribution of activation regions over the input space depending on the parameter distribution and explore the relations to speed of convergence and implicit biases in gradient descent.
Of significant interest would be an extension of the presented results to specific types of parameter distributions, including such which do not have a density or those one might obtain after training. 

\paragraph{Discussion of potential negative societal impacts} 
This is foundational research and not tied to particular applications. 
To the best of our knowledge there are no direct paths to negative societal impacts. 
 
\paragraph{Funding transparency statement}
 This project has received funding from the European Research Council (ERC) under the European Union’s Horizon 2020 research and innovation programme (grant agreement n\textsuperscript{o} 757983).  

\newpage 
{\small
    \bibliographystyle{apalike}
    \bibliography{main}
}
\newpage

\appendix

\section*{Appendix} 

The appendix is organized as follows. 
\begin{itemize}[leftmargin=*]
    \item Appendix~\ref{app:intro_proofs} Proofs related to activation patterns and activation regions. 
    \item Appendix~\ref{app:posvol} Proofs related to the numbers of regions attained with positive probability. 
    \item Appendix~\ref{app:single_unit} Expected number of regions for large rank. 
    \item Appendix~\ref{app:boundary_volume} Proofs related to the expected volume of activation regions. 
    \item Appendix~\ref{app:expected_number} Proofs related to the expected number of activation regions. 
    \item Appendix~\ref{app:constants} Upper bounding the constants. 
    \item Appendix~\ref{app:zero_bias} Proofs related to the expected number of regions for networks with zero bias. 
    \item Appendix~\ref{app:decision_boundary} Proofs related to the decision boundary. 
    \item Appendix~\ref{app:algorithm} Algorithm for counting regions and pieces of the decision boundary. 
    \item Appendix~\ref{app:init} Initialization procedures. 
    \item Appendix~\ref{app:experiments} Details on the experiments and additional experiments. 
\end{itemize}

\section{Proofs related to activation patterns and activation regions}
\label{app:intro_proofs}  

\subsection{Number of activation patterns}

\trivialupperboundlemma*
\begin{proof}[Proof of Lemma~\ref{lem:trivial_upper_bound}]
    To get an $r$-partial activation pattern one needs at most $r$ neurons. The number of ways to choose them is $\binom{N}{r}$. The number of ways to choose a pre-activation feature that attains a maximum in the rest of neurons is $K^{N - r}$. The $r$ chosen neurons have in total $r K$ pre-activation features. Out of them, we need to choose $r$ features that attain maximum, and $r$ additional features to construct the pre-activation pattern, so $2r$ features in total. 
    We ignore the restriction that there needs to be at least one feature from each neuron, which gives us an upper-bound ${r \binom{K}{2r}}$. 
    Notice that this way we also count $r$-partial patterns that require less than $r$ neurons. Combining everything, we get the desired result. 
    For the sub-patters, we simply ignore the term $K^{n-r}$. 
\end{proof}

We will use the above upper bound in our calculations due to its simplicity. 
For completeness, we note that the exact number of partial activation patterns can be given as follows. 

\begin{proposition}[Number of $r$-partial activation patterns]
\label{proposition:maxnrpartialactivation}
For a network with a total of $N$ rank-$K$ maxout units the number of distinct $r$-partial activation patterns is 
$$
|\mathcal{P}_r| = \sum_{\substack{(N_0,\ldots, N_{K-1})\in\mathbb{N}_0^K\colon \\  \sum_{j=0}^{K-1} N_j = N, \sum_{j=0}^{K-1} j N_j = r}} \binom{N}{ N_0,\ldots, N_{K-1}} \prod_{j=0}^{K-1} \binom{K}{1+j}^{N_j}. 
$$
If $K=2$ then the summation index takes only one value $(N_0,N_1)=(N-r,r)$ and the expression simplifies to $ \binom{N}{N-r} 2^{N-r}$. 
\end{proposition}
\begin{proof}
We have $N$ neurons. For a given activation pattern, for $j=0,\ldots, K-1$, denote $N_j$ the number of neurons with $(1+j)$ pre-activation features attaining the maximum.
Since every neuron has indecision in the range $0,\ldots, K-1$, we have $\sum_{j=0}^{K-1} N_j=N$. 
The $r$-partial activation patterns are precisely those for which $\sum_j j N_j=r$. 
The number of distinct ways in which we can partition the set of $N$ neurons into $K$ sets of cardinalities $N_0,\ldots, N_{K-1}$ is precisely $\binom{N}{N_0,\ldots, N_{K-1}}$. 
For each $j$, the number of ways in which a given neuron can have $(1+j)$ pre-activation features attaining the maximum is $\binom{K}{1+j}$. 
\end{proof}

\subsection{Generic correspondence between activation regions and linear regions}

For a fixed activation pattern $J$, a \emph{computation path} $\gamma$ is a path in the computation graph of the network $\net$ that goes from input to the output through one of the units in each layer, where $\gamma = (\gamma_0,\gamma_1,\ldots,\gamma_L)$, $\gamma_l\in[n_l]\times[K]$ specifies a unit and a corresponding pre-activation feature in layer $l$.
For any input $x$ in the activation region $\mathcal{R}(J,\theta)$, the gradient with respect to $x$ can be expressed through the computation paths as 
\begin{align*}
\nabla \net(x,\theta) = W_x^{(L+1)}W_x^{(L)} \cdots W_x^{(1)}, \qquad  
\frac{\partial}{\partial x_i} \net(x,\theta) = \sum_{\substack{\text{paths }\gamma\\\text{ starting at $i$}}} \prod_{l=1}^{L+1} w_{\gamma}^{(l)},
\end{align*}
where in $W_x^{(l)} \in \mathbb{R}^{n_{l}\times n_{l-1}}$ is a piecewise constant matrix valued function of the input $x$ with rows corresponding to the pre-activation features that attain the maximum according to the pattern $J$, 
and $w_{\gamma}^{(l)}\in\mathbb{R}$ are corresponding weights on the edge of $\gamma$ between the layer $(l - 1)$ and $l$, again depending on $J$. 
For a simple example of when one linear region is a union of several activation regions in a maxout network, consider a network with one of the weights in the single linear output unit set to zero. Such a situation can happen, for instance, at initialization, though with probability $0$. Then, switching between the maximums in the unit in the previous layer to which this weight connects will not be visible when we compute the gradient, and several activation regions created by the transitions between maximums in this unit will become a part of the same linear region. 

\actvslinlemma*
\begin{proof}[Proof of the Lemma~\ref{lem:act_vs_lin}]
    Consider two different non-empty activation regions corresponding to activation patterns $J_1$ and $J_2$ for which $\nabla \net (x; \theta)$ has the same value. 
    This means that $\nin$ equations of the form
    \begin{align*}
         \sum_{\text{paths }\gamma \in \Gamma_{1, i}} \prod_{l=1}^{L+1} w_{\gamma}^{(l)} = \sum_{\text{paths }\gamma \in \Gamma_{2, i}} \prod_{l=1}^{L+1} w_{\gamma}^{(l)}
    \end{align*}
    are satisfied, where $\Gamma_{1, i}, \Gamma_{2, i}$ are collections of paths starting at $i$ corresponding to the activation patterns $J_1$ and $J_2$ respectively.
    For different values of $i$ the sets of paths differ only at the input layer. 
    
    Based on this equation, there exists $c_{\gamma, i} \in \{\pm 1\}$ and a non-empty collection of paths $\Gamma_i$  (the symmetric difference of $\Gamma_{1,i}$ and $\Gamma_{2,i}$) so that 
    \begin{align*}
        \sum_{\text{paths }\gamma \in \Gamma_i} c_{\gamma, i} \prod_{l=1}^{L+1} w_{\gamma}^{(l)} = 0.
    \end{align*}
    This is a polynomial equation in the weights of the network. 
    Each monomial occurs either with coefficient $1$ or $-1$. In particular, this polynomial is not identically zero. 
    The zero set of a polynomial is of measure zero on $\mathbb{R}^{\# \text{\normalfont weights}}$ unless it is identically zero, see e.g. \citet{caron2005zero}.
    We have a system of $\nin$ such equations (one for each $i$). 
    The intersection of the solution sets is again a set of measure zero.
    The total number of pairs of activation regions is finite, upper bounded by $\binom{K^N}{2}$. 
    A countable union of measure zero sets is of measure zero, thus the set of weights for which two activation regions have the same gradient values has measure zero with respect to the Lebesgue measure on $\mathbb{R}^{\# \text{\normalfont weights}}$.
\end{proof}

\subsection{Partial activation regions} 

Now we introduce several objects that are needed to discuss $r$-partial activation regions. 

\begin{definition}
    \label{def:indecision_locus}
    
    Fix a value $\theta$ of the trainable parameters.
    For a neuron $z$ in $\net$ and a set $J_z \subseteq [K]$, the \textbf{$J_z$-activation region} of a unit $z$ is
    \begin{align*}
        \mathcal{H}(J_z;\theta) \defeq \{x_0 \in \mathbb{R}^{\nin} \ | \ \argmax\limits_{k\in [K]} \zeta_{z, k} (x_{l(z) - 1}; \theta) = J_z  \}.
    \end{align*}
    More generally, for a set of neurons $\mathcal{Z} = \{z\}$ and a corresponding list of sets $J_{\mathcal{Z}} = (J_z)_{z\in\mathcal{Z}}$,  the corresponding $J_\mathcal{Z}$-activation region is 
    \begin{align}
        \mathcal{H}(J_\mathcal{Z}; \theta) \defeq \bigcap\limits_{z\in \mathcal{Z}} \mathcal{H}({J_z};\theta). 
        \label{eq:big_h}
    \end{align}
    
If we specify an activation pattern for every neuron, $J_{[N]}$, so that $\mathcal{Z}=[N]$, then we write 
    \begin{align*}
 \mathcal{R}(J_{[N]}; \theta) = \mathcal{H}(J_{[N]};\theta). 
    \end{align*}

Recall that an activation pattern $J_{[N]}$ with with the property that  $\sum_{z}(|J_z|-1)=r$ is called an $r$-partial activation pattern. 
To distinguish such patterns, we denote them by $J^r \in\mathcal{P}_r$.
The union of all corresponding activation regions is denoted 
    $$
    \mathcal{X}_{\mathcal{N},r}(\theta) = \bigcup_{J^r\in\mathcal{P}_r} \mathcal{R}(J^r;\theta). 
    $$    
\end{definition}

\convactreglemma*
\begin{proof}[Proof of Lemma~\ref{lem:conv_act_reg}]
    Fix an $r$-partial activation pattern $J^r\in\mathcal{P}_r$.
    Over the activation region $\mathcal{R}(J;\theta)$, the $k$-th pre-activation feature of each neuron $z$ is a linear function of the input to the network, namely 
    \begin{align*}
         w^*_{z, k} \cdot x + b^*_{z, k} =  w^{(l(z))}_{z, k} (w^{(l(z) - 1)} \cdots ( w^{(1)} \cdot x + b^{(1)}) \cdots + b^{(l(z) - 1)}) + b^{(l(z))}_{z, k}, 
    \end{align*} 
    where $w^*_{z, k}$ and $b^*_{z, k}, k \in [K]$ denote the  weights and biases of this linear function, which depend on the weights and biases and activation values of the units up to unit $z$.
    For each $z$ specify a fixed element $j_0\in J_z$. 
    The activation region can be written as
    \begin{align*}
        \begin{split}
            \bigcap\limits_{z \in [N]} \big\{ x \in \mathbb{R}^{\nin} \ | \ & w^*_{z, j_0} \cdot x + b^*_{z, j_0} = w^*_{z, j} \cdot x + b^*_{z, j}, \quad \forall j \in J_z \setminus \{j_0\};\\
            &w^*_{z, j_0} \cdot x + b^*_{z, j_0} > w^*_{z, i} \cdot x + b^*_{z, i}, \quad \forall i \in [K] \setminus J_z \big\}. 
        \end{split}
    \end{align*}
    This means that an $r$-partial activation region is determined by a set of strict linear inequalities and $r$ linear equations.
    The equations are represented by vectors $v_{z,j} =(w^*_{z,j_0},b^*_{z_{j_0}})-(w^*_{z,j},b^*_{z,j})$ for all $j\in J_z\setminus\{j_0\}$ for all $z$ for which $|J_z|>1$. For generic parameters these equations are linearly independent. Indeed, the vectors being linearly dependent means that there is a matrix $V^\top V$, where $V$ has rows $v_{z,j}$, with vanishing determinant.
    By similar arguments as in the proof of Lemma~\ref{lem:act_vs_lin}, the set of parameters solving a polynomial system has measure zero.
    Hence, for generic choices of parameters, the $r$ linear equations are independent and the polyhedron will have a co-dimension $r$ (or otherwise be empty). 
\end{proof}

The same result can be obtained for $r$-partial activation regions of ReLU networks since ReLU activation regions can be similarly written as a system of linear equations and inequalities.

We can make a statement about the shape of $r$-partial activation regions of maxout networks.
Recall that a \emph{convex polyhedron} is the closure of the solution set to finite system of linear inequalities.
If it is bounded, it is called a convex polytope.
The dimension of a polyhedron is the dimension of the smallest affine space containing it. 

The next statement follows immediately from Lemma \ref{lem:conv_act_reg}.
    \begin{lemma}[$\mathcal{X}_{\net, r}$ consists of $(\nin -r)$-dimensional pieces]
        \label{lem:hyperplane_ball}
        With probability $1$ with respect to the distribution of the network parameters $\theta$, for any $x \in \mathcal{X}_{\net, r}$ there exists $\varepsilon > 0$ (depending on $x$ and $\theta$) s.t.\ $\mathcal{X}_{\net, r}$ intersected with the $\varepsilon$ ball $B_{\varepsilon}(x)$ is equal to the intersection of this ball with an $(\nin - r)$-dimensional affine subspace of $\mathbb{R}^{\nin}$.
    \end{lemma} 
    
    \begin{restatable}[$r$-partial activation regions are relatively open convex polyhedra]{corollary}{smallersumcor}
        \label{cor:smaller_sum}
        Recall that an an $r$-partial activation sub-pattern $\hat{J} \in \mathcal{S}_r$ is a list $\hat{J} = (J_z)_{z\in Z}$ of sets $J_z\subseteq [K]$, $z\in Z\subseteq[N]$ with $|J_z|>1$ and $\sum_{z \in Z} (|J_z| - 1) = r$. 
        For almost all choices of the parameter (i.e., except for a null set  with respect to the Lebesgue measure), 
        \begin{align*}
            \vol_{\nin - r} \left( \mathcal{X}_{\net, r}(\theta)\right)
            = \sum_{\hat{J} \in\mathcal{S}_r} \vol_{\nin - r} (\mathcal{H}(\hat{J} ;\theta)). 
        \end{align*}    
    \end{restatable}

\begin{proof}[Proof of Corollary~\ref{cor:smaller_sum}]
    Given $\hat{J}\in \mathcal{S}_r$, we denote $Z\subseteq [N]$ the corresponding list of neurons.
    Using the notion of indecision loci from Definition \ref{def:indecision_locus}, we can re-write $\mathcal{X}_{\net, r}(\theta)$ as 
        \begin{align*}
            &\mathcal{X}_{\net, r}(\theta) 
            = \bigcup_{J \in \mathcal{P}_r} \mathcal{R}(J ; \theta)
            = \bigcup_{J \in \mathcal{P}_r} \mathcal{H} (J ; \theta) 
            = \bigcup_{J \in \mathcal{P}_r} \bigcap\limits_{z\in[N]} \mathcal{H}(J_z;\theta)\\ 
            = &\bigcup_{J \in \mathcal{P}_r} \left[\bigcap\limits_{z\in Z} \mathcal{H}(J_z;\theta) \cap \bigcap\limits_{z\in [N]\setminus Z} \mathcal{H}(J_z; \theta)\right]\\
            = &\bigcup_{\hat{J}\in\mathcal{S}_r} \left[\bigcap\limits_{z\in Z} \mathcal{H}(J_z; \theta) \cap \bigcup_{J_{z}\in [K], z\in [N]\setminus Z} \bigcap\limits_{z\in[N]\setminus Z} \mathcal{H}(J_z;\theta)\right]\\
            = &\bigcup_{\hat{J}\in \mathcal{S}_r} \left[ \bigcap_{z\in Z} \mathcal{H}(J_z;\theta) \cap \bigcap_{z \notin Z} \bigcup_{k \in [K]} \mathcal{H}(J_z=\{k\}; \theta)\right].
        \end{align*}
    
    Therefore,
    \begin{align*}
        \vol_{\nin - r} \left( \mathcal{X}_{\net, r}(\theta)\right)
        = \sum_{\hat J \in \mathcal{S}_r} \vol_{\nin - r} \left( \bigcap_{z\in Z} \mathcal{H} ({J_z}; \theta) \cap \bigcap_{z \notin Z} \bigcup_{k \in [K]} \mathcal{H}(J_z=\{k\}; \theta)\right).
    \end{align*}
    
    Notice that $\left(\bigcap_{z \notin Z} \bigcup_{k \in [K]} \mathcal{H}(J_z=\{k\}; \theta)\right)^c$ is a zero measure set in $\mathcal{X}_{\mathcal{N},r}(\theta)$, because over that set, by Lemma~\ref{lem:hyperplane_ball} the co-dimension of the corresponding activation regions is larger than $r$. 
    Therefore, for any given $\hat J =(J_z)_{z\in Z}\in \mathcal{S}_r$,  
    \begin{align*}
        \vol_{\nin - r} \left( \bigcap_{z\in Z} \mathcal{H} ({J_z}; \theta) \cap \bigcap_{z \notin Z} \bigcup_{k \in [K]} \mathcal{H}(J_z=\{k\}; \theta)\right) 
        = \vol_{\nin - r} \left( \bigcap_{z\in Z} \mathcal{H} ({J_z}; \theta) \right).
    \end{align*}
    This completes the proof. 
\end{proof}

\section{Proofs related to the generic numbers of regions}
\label{app:posvol}

\subsection{Number of regions and Newton polytopes}
We start with the observation that the linear regions of a maxout unit correspond to the upper vertices of a polytope constructed from its parameters. 

\begin{definition}
\label{def:Newton-polytope}
Consider a function of the form $f \colon \mathbb{R}^n\to\mathbb{R};\; f(x) = \max\{w_j\cdot x + b_j\}$, 
where $w_j\in\mathbb{R}^n$ and $b_j\in \mathbb{R}$, $j=1,\ldots,M$. 
The \emph{lifted Newton polytope} of $f$ is defined as $P_f := \operatorname{conv}\{(w_j,b_j)\in\mathbb{R}^{n+1}\colon j=1,\ldots, M\}$.
\end{definition}

\begin{definition}
Let $P$ be a polytope in $\mathbb{R}^{n+1}$ and let $F$ be a face of $P$. 
An outer normal vector of $F$ is a vector $v\in\mathbb{R}^{n+1}$ with $\langle v, p-q \rangle > 0$ for all $p\in F$, $q\in P\setminus F$ and $\langle v, p-q \rangle = 0$ for all $p,q\in F$.
The face $F$ is an \emph{upper face} of $P$ if it has an outer normal vector $v$ whose last coordinate is positive, $v_{n+1}>0$.
It is a \emph{strict upper face} if each of its outer normal vectors has a positive last coordinate. 
\end{definition} 
The Newton polytope is a fundamental object in the study of polynomials. 
The naming in the context of piecewise linear functions stems from the fact that piecewise linear functions can be regarded as differences of so-called tropical polynomials. The connections between such polynomials and neural networks with piecewise linear activation functions have been discussed in several recent works \citep{zhang2018tropical,charisopoulos2018tropical,alfarra2020on}. 
For details on tropical geometry, see \citep{MaclaganSturmfels15,JoswigBook}. 
Although in the context of (tropical) polynomials the coefficients are integers, such a restriction is not needed in our discussion. 

\begin{figure}
    \centering
\begin{tikzpicture}
\node at (0,0) { 
\begin{tikzpicture}[dot/.style={circle,inner sep=1pt,fill,label={#1},name=#1},
  extended line/.style={shorten >=0,shorten <=-#1},
  extended line/.default=1.5cm]
  
\coordinate (A) at (-1,-.5);
\coordinate (B) at (-.5,1);
\coordinate (C) at (1,-1);
\coordinate (O) at (-.125,-.125);

\draw[blue!80, thick, fill=blue!10] (A) node[below left] {$(w_1,-b_1)$}
  -- (B) node[above]{$(w_2,-b_2)$}
  -- (C) node[right]{$(w_3,-b_3)$}
  -- cycle; 
\fill[blue!80] (A) circle[radius=1.5pt];
\fill[blue!80] (B) circle[radius=1.5pt];  
\fill[blue!80] (C) circle[radius=1.5pt];  

\node at (.75,-.125) {\textcolor{blue!80}{$P'_f$}};
\node at (.75,1) {\textcolor{red!80}{$N_{P'_f}$}};  

\draw[extended line, red!80, thick] ($(A)!(O)!(B)$) -- (O);
\draw[extended line, red!80, thick] ($(B)!(O)!(C)$) -- (O);

\draw [red!80, thick,name path=pathAC] (O) -- ($($(A)!(O)!(C)$)!-1.75cm!(O)$);

\draw[black!80,name path=pathinput] (-2.5,-1.75) -- (1.5,-1.75) node[below] {\textcolor{black!80}{$\mathbb{R}^{\nin}\times\{-1\}$}};

\path [name intersections={of=pathAC and pathinput,by=E}];
\node [circle, inner sep=1pt, minimum size=1.5pt, fill=red!80, draw=none] at (E) {}; 
\end{tikzpicture} 
};
\end{tikzpicture}
    \caption{The linear regions of a function $f(x) = \max_{j}\{\langle w_j,x\rangle + b_j\}$ correspond to the lower vertices of the polytope $P'_f = \operatorname{conv}_j\{(w_j,-b_j)\}\subseteq \mathbb{R}^{\nin+1}$, or, equivalently, the upper vertices of the lifted Newton polytope $P_f = \operatorname{conv}_j\{(w_j,b_j)\}\subseteq \mathbb{R}^{\nin+1}$. 
    The linear regions of $f$ can also be described as the intersection of the normal fan $N_{P'_f}$, consisting of outer normal cones of faces of $P'_f$, with the affine space $\mathbb{R}^{\nin}\times\{-1\}$. 
    } 
    \label{fig:newton-polytope}
\end{figure}
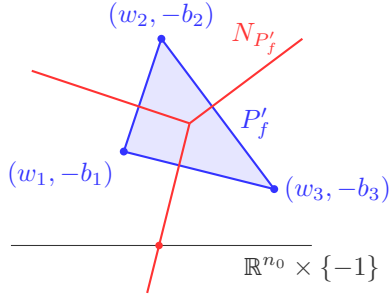

A convex analysis interpretation of the Newton polytope can be given as follows. Consider a piecewise linear convex function $f\colon \mathbb{R}^n\to\mathbb{R};\; x\mapsto \max_j\{w_j\cdot x + b_j\}$. Then the upper faces of its lifted Newton polytope $P_f$ correspond to the graph $\{(x^\ast, -f^\ast(x^\ast))\colon x^\ast\in \mathbb{R}^n\cap \operatorname{dom}(f^\ast)\}$ of the negated convex conjugate $f^\ast\colon \mathbb{R}^n\to\mathbb{R};\; x^\ast\mapsto \sup_{x\in\mathbb{R}^n} \langle x,x^\ast\rangle -f(x)$, which is a convex piecewise linear function.
This implies that the upper vertices of $P_f$ are the points $(w_j,b_j)\in\mathbb{R}^{n+1}$ for which $f(x) = w_j\cdot x +b_j$ over a neighborhood of inputs. Hence the upper vertices of the Newton polytope correspond to the linear regions of $f$. 
This relationship holds more generally for boundaries between linear regions and other lower dimensional linear features of the graph of the function. 
We will use the following result, which is well known in tropical geometry \cite[see][]{JoswigBook}. 

\begin{proposition}[Regions correspond to upper faces] 
\label{prop:Newton}
The $r$-partial activation regions of a function $f(x) = \max_j\{w_j\cdot x+b_j\}$ correspond to the $r$-dimensional upper faces of its lifted Newton polytope $P_f$. Moreover, the bounded activation regions correspond to the strict upper faces of $P_{f}$.
\end{proposition}

The situation is illustrated in Figure~\ref{fig:newton-polytope}. 

\subsection{Bounds on the maximum number of linear regions}
\label{sec:upperbound} 

For reference, we briefly recall results providing upper bounds on the maximum number of linear regions of maxout networks. 
The maximum number of regions of maxout networks was studied by \cite{pascanu2013number, NIPS2014_5422}, showing that deep networks can represent functions with many more linear regions than any of the functions that can be represented by a shallow network with the same number of units or parameters.
\citet{serra2018bounding} obtained an upper bound for deep maxout networks based on multiplying upper bounds for individual layers.
These bounds were recently improved by \citet{sharp2021}, who obtained the following result, here stated in a simplified form. 

\begin{restatable}[Maximum number of linear regions, \citealt{sharp2021}]{theorem}{theoremupperbound}
\label{thm:upper_bound}
\hspace{1em}
\begin{itemize}[leftmargin=*]
    \item
For a network with $n_0$ inputs and a single layer of $n_1$ rank-$K$ maxout units, the maximum number of linear regions is $\sum_{j=0}^{n_0}\binom{n_1}{j}(K-1)^j$.
\item 
For a network with $n_0$ inputs and $L$ layers of $n_1,\ldots, n_L$ rank-$K$ maxout units, if $n\leq n_0$, $\frac{n_l}{n}$ even, and $e_l = \min\{n_0,\ldots, n_{l-1}\}$, the maximum number of linear regions is lower bounded by $\prod_{l=1}^{L} (\frac{n_l}{n} (K-1)+1)^n$ and upper bounded by $\prod_{l=1}^L \sum_{j=0}^{e_l} \binom{n_l}{j} (K-1)^j$.
\end{itemize}
\end{restatable} 

\subsection{Numbers of regions attained over positive measure subsets of parameters}

A layer of maxout units can attain several different numbers of linear regions with positive probability over the parameters. This is illustrated in Figure~\ref{fig:pos}. 
We obtain the following result, describing numbers of linear regions that can be attained by maxout units, layers, and deep maxout networks with positive probability over the parameters. 

\begin{figure}
    \centering
\begin{tikzpicture}[every node/.style={black,above right, inner sep=1pt}]

\path[fill=gray!10] (-1.25,-1.25) rectangle (1.25cm,1.25cm);
\node at (-1.25,-1.25) {\textcolor{black!80}{$\mathbb{R}^{\nin}$}};

\definecolor{color0}{rgb}{1,0,0}
\definecolor{color1}{RGB}{250,128,114}
\definecolor{color2}{RGB}{139,0,0}

\draw[name path=line11, double=color0, white, thick] (0,0) -- (1.25,1) node [right] {}; 
\draw[name path=line12, double=color0, white, thick] (0,0) -- (1.25,-1) node [right] {}; 
\draw[name path=line13, double=color0, white, thick] (0,0) -- (-1.25,0) node [left] {}; 

\fill[name intersections={of=line11 and line12,total=\t}, color0, draw=white, thick] {(intersection-1) circle (1.5pt) node {}};

\coordinate (O) at (-.5,.25);
\coordinate (A) at (-.5,-1.25);
\coordinate (B) at (-1,1.25);
\coordinate (C) at (0,1.25);

\draw[name path=line21, double=color1, white, thick] (O) -- (A) node [above] {}; 
\draw[name path=line22, double=color1, white, thick] (O) -- (B) node [above] {}; 
\draw[name path=line23, double=color1, white, thick] (O) -- (C) node [below] {}; 

\fill[name intersections={of=line21 and line22,total=\t}, color1, draw=white, thick] {(intersection-1) circle (1.5pt) node {}};

\coordinate (O) at (.625,-1);
\coordinate (A) at (.625,-1.25);
\coordinate (B) at (.125,1.25);
\coordinate (C) at (1.125,1.25);

\draw[name path=line21, double=color2, white, thick] (O) -- (A) node [above] {}; 
\draw[name path=line22, double=color2, white, thick] (O) -- (B) node [above] {}; 
\draw[name path=line23, double=color2, white, thick] (O) -- (C) node [below] {}; 

\fill[name intersections={of=line21 and line22,total=\t}, color2, draw=white, thick] {(intersection-1) circle (1.5pt) node {}};

\end{tikzpicture}
    \caption{A layer of maxout units of rank $K\geq 3$ attains several different numbers of linear regions with positive probability over the parameters. For a layer with two rank-$3$ maxout units, some neighborhoods of parameters give 6 linear regions and others 9, with nonlinear loci given by perturbations of the red-pink and red-darkred lines. }
    \label{fig:pos}
\end{figure}
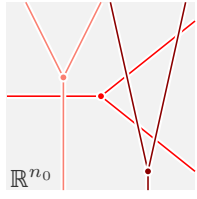

\theorempositivemeasure*

The strategy of the proof is as follows. We first show that there are parameters such that individual rank-$K$ maxout units behave as rank-$k$ maxout units, for any $1\leq k\leq K$, and there are positive measure subsets of the parameters for which they behave as rank-$k$ maxout units, for any $n+1\leq k\leq K$.
Further, there are positive measure subsets of the parameters of individual rank-$K$ maxout units for which, over the positive orthant $\mathbb{R}^n_{\geq0}$, they behave as rank-$k$ maxout units, for any $1\leq k\leq K$.
Then we use a similar strategy as \citet{sharp2021} to construct parameters of a network with units of pre-specified ranks which attain a particular number of linear regions. 

\begin{proposition}
\label{prop:genericposorth}
Consider a rank-$K$ maxout unit with $n$ inputs restricted to $\mathbb{R}^n_{\geq0}$. 
For any $1\leq k\leq K$, there is a positive measure subset of parameters for which the behaves as a rank-$k$ maxout unit. Moreover, this set can be made to contain parameters representing any desired function that can be computed by a rank-$k$ maxout unit.
\end{proposition}
\begin{proof}
We need to show that for any choices of $(w_i,b_i)$, $i\in[k]$, there are generic choices of $(w_j,b_j)$, $j\in[K]\setminus[k]$, so that for each $J\subseteq[K]$ with $J\not\subseteq [k]$, the corresponding activation region $\mathcal{R}(J, \theta)$ does not intersect $\mathbb{R}^n_{\geq0}$.
Notice that, if $j\in J\setminus[k]$, then the corresponding activation region $\mathcal{R}(J,\theta)$ is contained in the arrangement consisting of hyperplanes $H_{ji} = \{x\colon (w_j-w_i)\cdot x +(b_j-b_i) =0 \}$, $i\in J\setminus\{j\}$.
For each $j\in[K]\setminus[k]$, we choose $w_j = j c (-1,\ldots, -1)+\epsilon_j$, $b_j=-j c'+\epsilon'_j$ for some $c > 2\max\{\|w_{i}\|_\infty \colon i\in[k]\}$, $c'>2\max\{b_i\colon i\in [k]\}$ and small $\epsilon_j\in\mathbb{R}^n$, $\epsilon'_j\in\mathbb{R}$.
Then, for each $j\in [K]\setminus[k]$ and $i\in[K]$, $j<j$, the hyperplane $H_{ji}$ has a normal vector $(w_j-w_i)\in\mathbb{R}_{<0}$ and an intercept $b_j-b_i<0$, and hence it does not intersect $\mathbb{R}^{n}_{\geq0}$. 
\end{proof}

We are now ready to prove the theorem. 

\begin{proof}[Proof of Theorem~\ref{thm:positive_measure}] 
\emph{Single unit.} 
Consider a maxout unit $\max_{j\in[K]}\{w_j\cdot x + b_j\}$. 
To have this behave as a rank-$k$ maxout unit, $1\leq k\leq K$, we simply set $(w_j,b_j) = (w_1,b_1-1)$, $j\in[K]\setminus[k]$. This is a non-generic choice of parameters.
Consider now a rank-$k$ maxout unit with $n+1\leq k$ and generic parameters $(w_i,b_i)$, $i\in[k]$. 
We want to show that there are generic choices of $(w_j,b_j)$, $j\in[K]\setminus[k]$ so that $\max_{j\in[K]}\{w_j\cdot x + b_j\} = \max_{i\in[k]}\{w_i\cdot x + b_i\}$ for all $x\in\mathbb{R}^n$. 
In view of Proposition~\ref{prop:Newton}, this is equivalent to $(w_j,b_j)$, $j\in[K]\setminus[k]$ not being upper vertices of the lifted Newton polytope $P = \conv\{(w_j,b_j)\colon j\in [K]\}$. 
Since any generic $n+1$ points in $\mathbb{R}^n$ are affinely independent, we have that the convex hull $\conv\{w_i\in\mathbb{R}^n\colon i\in [k]\}$ has full dimension $n$. 
Hence, any $w_j = \frac{1}{k}\sum_{i\in[k]}w_i + \epsilon_j $ and $b_j = \min_{i\in[k]}\{b_i\}-1 + \epsilon'_j$ with sufficiently small $\epsilon_j\in \mathbb{R}^n$, $\epsilon'_j\in \mathbb{R}$, $j\in[K]\setminus[k]$ are strictly below $\conv\{(w_i,b_i)\colon i\in[k]\}$ and are not upper vertices of $P$. 

\emph{Single layer.} We use the previous item to obtain $n_1$ maxout units of ranks $k_1,\ldots, k_{n_1}$, either in the non-generic or in the generic cases.
Then we apply the construction of parameters and the region counting argument from \citet[Proposition~3.4]{sharp2021} to this layer, to obtain a function with $\sum_{j=0}^{n_0}\sum_{S\subseteq \binom{[n_1]}{j}} \prod_{i\in S}(k_i-1)$ linear regions.
For each of the units $i=1,\ldots, n_1$, one may choose a generic vector $v_i\in\mathbb{R}^n$ and define the weights and biases of the pre-activation features as $w_{ij} = \frac{j}{k_i} v_i$ and $b_{ij} = -g(\frac{j}{k_i}+\epsilon_i)$, $j=1,\ldots, k_i$, where $g\colon\mathbb{R}\to\mathbb{R}$ is any strictly convex function and $\epsilon_i$ is chosen generically.
Then the non-linear locus of each unit consists of $k_i-1$ parallel hyperplanes with a generic shift $\epsilon_i$, and the normal vectors $v_i$ of different units are in general position.
The number of regions defined by such an arrangement of hyperplanes in $\mathbb{R}^n$ can be computed using Zaslavsky's theorem, giving the indicated result.
It remains to show that, for $n_0+1\leq k_1,\ldots, k_{n_1}$, there are positive measure perturbations of these parameters that do change the number of regions. 
By the lower semi-continuity discussed in Section~\ref{sec:numbers}, the number of regions does not decrease for sufficiently small generic perturbations of the parameters. 
To show that it does not increase, we note that, by Theorem~\ref{thm:upper_bound} this number of regions is the maximum that can be attained by a layer of $n_1$ maxout units of ranks $k_1,\ldots, k_n$. 

\emph{Deep network.} For the first statement, we use the first item to obtain maxout units of any desired ranks $1\leq k_{li}\leq K$, $l=1,\ldots, L$, $i=1,\ldots, {n_l}$, and then apply the construction of parameters from \citet[Proposition~3.11]{sharp2021} to this network, to obtain the indicated number of regions. 

For the second statement, we use Proposition~\ref{prop:genericposorth} to have the units behave as maxout units of any desired ranks over $[0,1]^{n_0}$.
For the $l$-th layer, we divide the $n_l$ units into $n_0$ blocks $x^{(l)}_{ij}$, $i=1,\ldots, n_0$, $j=1,\ldots, \frac{n_l}{n_0}$. 
For $i=1,\ldots, n_0$, the $i$-th block consists of $\frac{n_l}{n_0}$ maxout units of rank $k_{li}$.
We can choose the weights and biases so that over $[0,1]^{n_0}$, the nonlinear locus of the $i$-th block $(x^{(1)}_{i,1},\ldots, x^{(1)}_{i,\frac{n_1}{n_0}})$ consists of $\frac{n_l}{n_0}(k_{li}-1)$ parallel hyperplanes with normal $e_i$, and the alternating sum $\sum_{j=1}^{n_l/n_0}(-1)^j x^{(l)}_{ij}$ is a zig-zag function along the direction $e_i$ which maps $(0,1)^{n_0}$ to $(0,1)$, and maps any point in $\mathbb{R}^{n_0}\setminus [0,1]^{n_0}$ to a point in $\mathbb{R}\setminus [0,1]$.
In this way, the $l$-th layer, followed by a linear layer $\mathbb{R}^{n_l}\to\mathbb{R}^{n_0}$, maps $(0,1)^{n_0}$ onto $(0,1)^{n_0}$ in a $\prod_{i=1}^{n_0}(\frac{n_l}{n_0}(k_{li}-1)+1)$ to one manner. 
Sufficiently small perturbations of the parameters do not affect this general behavior.
The composition of $L$ such layers gives the desired number of regions over $(0,1)^{n_0}$. 
\end{proof}

\subsection{Minimum number of activation regions} 

One can easily construct parameters so that the represented function is identically zero. However, these are very special parameters. 
Moreover, it can be shown that the number of linear regions of a maxout network is a lower semi-continuous function of the parameters, in the sense that sufficiently small generic perturbations of the parameters do not decrease the number of linear regions \citep[Proposition~3.2]{sharp2021}. 
Hence, the question arises: What is the smallest number of linear regions that will occur with positive probability over the parameter space (i.e.\ for all parameters except for a null set).
For example, in the case of shallow ReLU networks, it is known that the number of regions for generic parameters is equal to the maximum.
For maxout networks we saw in Theorem~\ref{thm:positive_measure} that several numbers of linear regions can happen with positive probability. 
We prove the following lower bound on the number of regions for maxout networks with generic parameters. 

\theoremlowerbound*

First we observe that for generic parameters, the number of linear regions of the function represented by a network is bounded below by the number of linear regions of the network restricted to the first layer.
This is not trivial, since the deeper layers could in principle map the values from the first layer to a constant value, resulting in a function with a single linear region. However, for maxout networks this only happens for a null set of parameters.  
\begin{proposition}
\label{prop:deep_bound}
The number of activation regions of a maxout network is at least as large as the number of regions of the first layer. 
Moreover, for generic parameters the number of linear regions is equal to the number of activation regions. 
\end{proposition}

\begin{proof}
The number of regions never reduces as we pass through the network. 
The region is either kept as it is or split into parts by a neuron.
The fact that for generic parameters activation regions correspond to linear regions is Lemma~\ref{lem:act_vs_lin}. 
\end{proof}

In order to lower bound the number of regions of a single layer, we use the correspondence between linear regions and the upper vertices of the corresponding lifted Newton polytope, Proposition~\ref{prop:Newton}.
We first observe that the Newton polytope of a shallow maxout units is the Minkowski sum of the Newton polytopes of the individual units.
Recall that the Minkowski sum of two sets $A$ and $B$ is the set $A+B = \{a+b\colon a\in A, b\in B\}$. 

\begin{proposition}
\label{prop:minkowski-sum}
Consider a layer of maxout units, $f\colon \mathbb{R}^n\to \mathbb{R}^m; \; f_i(x) = \max\{w_{ir}\cdot x + b_{ir}\colon r=1,\ldots, k\}$. Let $f(x) = \sum_{i=1}^m f_i(x)$.
Then the lifted Newton polytope of $f$ is the Minkowski sum of the lifted Newton polytopes of $f_1,\ldots, f_m$, $P_f = \sum_{i=1}^m P_{f_i}$.
\end{proposition} 
\begin{proof}
This follows from direct calculation. 
Details can be found in the works of \cite{zhang2018tropical} and \cite{sharp2021}. 
\end{proof}

Next, a family of polytopes $P_i = \operatorname{conv}\{(w_{i,r},b_{i,r})\in\mathbb{R}^{\nin+1}\colon r=1,\ldots,K\}$ with generic $(w_{i,r},b_{i,r})$, $r=1,\ldots, K$, $i=1,\ldots, n_1$, is in general orientation.
For such a family, the Minkowski sum $P=P_1+\cdots+P_{n_1}$ has at least as many vertices as a Minkowski sum of $n_1$ line segments in general orientation:  

\begin{proposition}[{\citealt[Corollary 8.2]{adiprasito2017lefschetz}}]
\label{prop:shallow_bound}
    The number of vertices of a Minkowski sum of $m$ polytopes in general orientation is lower bounded by the number of vertices of a sum of $m$ line segments in general orientations. 
\end{proposition}

From this, we derive a lower bound on the number of upper vertices of a Minkowski sum of polytopes in general orientations. 

\begin{proposition}
\label{prop:lower-minkowski}
The number of upper vertices of a Minkowski sum of $n_1$ polytopes in $\mathbb{R}^{\nin+1}$ in general orientation is at least $\sum_{j=0}^{\nin}\binom{n_1}{j}$, and the number of strict upper vertices is at least $\binom{n_1-1}{ \nin}$. 
\end{proposition}

\begin{proof}
Consider the sum $P=P_1+\cdots+P_{n_1}$ of polytopes $P_i=\{(w_{i,r},b_{i,r})\colon r=1,\ldots, k\}$, $i=1,\ldots, n_1$.
The set of upper vertices consists of 1) strict upper vertices and 2) vertices which are both upper and lower.
The number of strict upper vertices of a Minkowski sum of $n_1$ positive dimensional polytopes in general orientations in $\mathbb{R}^{\nin+1}$ is at least $\binom{n_1-1}{ \nin}$  \cite[Corollary 3.8]{sharp2021}. 

Now note that the vertices which are upper and lower are precisely the vertices of the Minkowski sum $Q=Q_1+\cdots+Q_{n_1}$ of the poltyopes $Q_i = \operatorname{conv}\{w_{i,r}\in \mathbb{R}^{\nin}\colon r=1,\ldots, k\}$, $i=1,\ldots, n_1$.
By Proposition~\ref{prop:shallow_bound} the total number of vertices of a Minkowski sum is at least equal to the number of vertices of a Minkowski sum of line segments.
The latter is the same as the number of regions of a central hyperplane arrangement in $\nin$ dimensions, which is $\binom{n_1-1}{\nin-1} + \sum_{j=0}^{\nin-1}\binom{n_1}{j}$. 

Hence for any generic Minkowski sum of $n_1$ positive-dimensional polytopes in $\nin+1$ dimensions, the number of upper vertices is at least 
\begin{equation*} 
\binom{n_1-1}{ \nin} + \binom{n_1-1}{\nin-1} + \sum_{j=0}^{\nin-1}\binom{n_1}{j} 
= \binom{n_1}{ \nin}  + \sum_{j=0}^{\nin-1}\binom{n_1}{j} 
= \sum_{j=0}^{\nin}\binom{n_1}{j}. 
\end{equation*}
This concludes the proof.
\end{proof}

Now we have all tools we need to prove the theorem. 

\begin{proof}[Proof of Theorem \ref{thm:lower_bound}]
By Proposition~\ref{prop:deep_bound}, the number of regions is lower bounded by the number of regions of the first layer.
We now derive a lower bound for the number of regions of a single layer with $\nin$ inputs and $n_1$ maxout units. 
In view of Propositions \ref{prop:Newton} and \ref{prop:minkowski-sum}, we need to lower bound the number of upper vertices of a generic Minkowksi sum. 
The bounded regions correspond to the strict upper vertices. 
The result follows from Proposition~\ref{prop:lower-minkowski}.
\end{proof}

\begin{remark}
\label{remark:lower_bound_relu}
    The statement of Theorem \ref{thm:lower_bound} does not apply to ReLU networks unless they have a single layer of ReLUs. Indeed, for a network with 2 layers of ReLUs there exists a positive measure subset of parameters for which the represented functions have only 1 linear region. 
    To see this, consider a ReLU network with pre-activation features of the units in the second layer always being non-positive. 
    A subset of parameters required to achieve this is defined as a solution to a set of inequalities (for instance, when the input weights and biases of the second layer are non-positive) and has a positive measure. 
    For such pre-activation features, the ReLUs in the second layer always output $0$  and there is a single linear region for the network. 
\end{remark}

\section{Expected number of activation regions of a single maxout unit} 
\label{app:single_unit}

We discuss a single maxout unit with $n$ inputs. In this case, the $r$-partial activation patterns correspond to the $r$-dimensional upper faces of a polytope given as the convex hull of the points $(w_k,b_k)\in\mathbb{R}^{n+1}$, $k=1,\ldots, K$. 
The statistics of faces of random polytopes have been studied in the literature \citep{
Hug2004, 
10.2307/30037329, 
10.1214/009117906000000791}. 
We will use the following result. 
\begin{theorem}[{\citealt[Theorem~1.1]{Hug2004}}]
\label{thm:gaussianpolytopes}
If $v_1,\ldots, v_K$ are sampled iid according to the standard normal distribution in $\mathbb{R}^{d}$, then, the number of $s$-faces of the convex hull $P_K=\operatorname{conv}\{v_1,\ldots,v_K\}$, denoted $f_s(P_K)$, has expected value 
 \begin{align}
 \mathbb{E} f_s(P_K) &\sim \bar c(s,d) (\log K)^\frac{d-1}{2}, 
 \intertext{and the union $s$-faces of $P_K$, denoted $\operatorname{skel}_s(P_K)$, has expected volume}
 \mathbb{E} \operatorname{vol}_{s} (\operatorname{skel}_s(P_K)) &\sim c(s,d) (\log K)^{\frac{d-1}{2}}, 
\end{align}
 where $\bar c(s,d)$ and $c(s,d)$ are constants depending only on $s$ and $d$. 
 \end{theorem}

Based on this, we obtain the following upper bound for the expected number of linear regions of a maxout unit with iid Gaussian weights and biases. 

\begin{restatable}[Expected number of regions of a large-rank Gaussian maxout unit]{proposition}{expectedsingleunit}
\label{prop:single_unit}
Consider a rank-$K$ maxout unit with $\nin$ inputs. If the weights and biases are sampled iid from a standard normal distribution, then for large $K$ the expected number of non-empty $r$-partial activation regions satisfies  
 \begin{align*}
\mathbb{E} [\# \text{\normalfont{ }$r$-partial activation regions}] 
 &\leq \bar c(r,\nin) (\log K)^\frac{\nin}{2}.
\end{align*}
where $\bar c(r,\nin)$ is a constants depending solely on $r$ and $\nin$. 
\end{restatable}

\begin{proof}[Proof of Proposition~\ref{prop:single_unit}]
We use the correspondence between $r$-partial activation regions and the upper $r$-dimensional faces of the lifted Newton polytope (Proposition~\ref{prop:Newton}).
The total number of $s$-dimensional faces of a polytope is an upper bound on the number of upper $s$-dimensional faces.
Now we just apply Theorem~\ref{thm:gaussianpolytopes}. 
\end{proof}

We can use the above result to upper bound the expectation value of the number of regions of a maxout network with iid Gaussian weights and biases. 
In particular, for a shallow maxout network we have the following. 

\begin{proposition}[Expected number of linear regions of a large-rank Gaussian maxout layer]
Consider network $\mathcal{N}$ with $n_0$ inputs and a single layer of $n_1$ rank-$K$ maxout units. 
Suppose the weights and biases are sampled iid from a standard normal distribution. 
Then, for sufficiently large $K$, the expected number of linear regions is bounded as 
$$
\mathbb{E} [\# \text{{ }\normalfont linear regions}] 
\leq \sum_{j=0}^{n_0} \binom{n_1}{j} ( \bar c(\nin) (\log K)^\frac{\nin}{2}-1)^j, 
$$
where $\bar c(\nin)$ is a constant depending solely on $n_0$.
This upper bound behaves as $O( n_1^{n_0} (\log K)^{\frac{1}{2}\nin^2})$ in $n_1$ and $K$. 
\end{proposition}
\begin{proof}
By \citet[Theorem~3.6]{sharp2021}, the maximum number of regions of a layer with $n_0$ inputs and $n_1$ maxout units of ranks $k_1,\ldots, k_{n_1}$ is 
$$
\max [\# \text{{ }\normalfont linear regions}] = \sum_{j=0}^{n_0}\sum_{S\in\binom{[n_1]}{j}} \prod_{i\in S} (k_i-1). 
$$
Consider now our network with $n_1$ maxout units of rank $K$. For a given probability distribution over the parameter space, denote $\Pr(k_1,\ldots, k_{n_1})$ the probability of the event that the $i$-th unit has $k_i$ linear regions, $i=1,\ldots, n_1$.
If the parameters of the different units are independent, we have 
\begin{align*}
\mathbb{E} [\# \text{{ }\normalfont linear regions}]  
& \leq \sum_{1\leq k_1,\ldots, k_{n_1}\leq K} \Pr(k_1,\ldots, k_{n_1})
\sum_{j=0}^{n_0}\sum_{S\in\binom{[n_1]}{j}} \prod_{i\in S} (k_i-1)\\
& = \sum_{j=0}^{n_0}\sum_{S\in\binom{[n_1]}{j}} \prod_{i\in S} (\mathbb{E}[k_i] -1) . 
\intertext{
If the weights and biases of each unit are iid normal, Proposition~\ref{prop:single_unit} allows us to upper bound the latter expression by }
& \leq \sum_{j=0}^{n_0} \binom{n_1}{j} ( \bar c(\nin) (\log K)^\frac{\nin}{2}-1)^j . 
\end{align*}
This concludes the proof. 
\end{proof}

\section{Proofs related to the expected volume}
\label{app:boundary_volume}

The following is a maxout version of a result obtained by \citet[Theorem 6]{pmlr-v97-hanin19a} for the case of networks with single-argument piecewise linear activation functions. 

\begin{restatable}[Upper bound on the expected volume of $\mathcal{X}_{\net,r}$]{lemma}{calcvollemma}
    \label{lem:calc_vol}
    Consider a rank-$K$ maxout network $\net$ with input dimension $\nin$, output dimension $1$, and random weights and biases satisfying:
    \begin{enumerate}[leftmargin=*]
        \item The distribution of all weights has a density with respect to the Lebesgue measure. 
        \item Every collection of biases has a conditional density with respect to Lebesgue measure given the values of all weights and other biases. 
    \end{enumerate}
    
    Then, for any bounded measurable set $S \subset \mathbb{R}^{\nin}$ and any $r \in \{1, \ldots, \nin\}$, the expectation value of the $(\nin - r)$-dimensional volume of $\mathcal{X}_{\net, r}$ inside $S$ is upper bounded as 
    \begin{align*}
        &\mathbb{E}[\vol_{\nin - r}(\mathcal{X}_{\net, r} \cap S)]\\
        &\leq 
        \sum_{J\in \mathcal{S}_r} \int\limits_{S} \mathbb{E} \left[ \rho_{\mathbf{b}^r}((\mathbf{w}^m - \mathbf{w}^r) \cdot \mathbf{x}^m_{-1} + \mathbf{b}^m) \enspace \| \mathbf{J} ((\mathbf{w}^m - \mathbf{w}^r) \cdot \mathbf{x}^m_{-1} + \mathbf{b}^m)\| \right] \enspace d x, 
    \end{align*}
    where, for any given $r$-partial activation sub-pattern $J=(J_z)_{z\in\mathcal{Z}}\in\mathcal{S}_r$, for any given $J_z$ we denote its smallest element by $j_0$, we let $\rho_{\mathbf{b}^r}$ denote the joint conditional density of the biases of pre-activation features $j\in J_z\setminus\{j_0\}$ of the neurons $z\in\mathcal{Z}$, given all other network parameters, we let $g\colon \mathbb{R}^{\nin}\to\mathbb{R}^r$; $x\mapsto (\mathbf{w}^m - \mathbf{w}^r) \cdot \mathbf{x}^m_{-1} + \mathbf{b}^m := 
    ((w_{z, j_0} - w_{z, j}) \cdot x_{l(z)-1} + b_{z, j_0})_{z\in \mathcal{Z}, j\in J_z\setminus\{j_0\}}\in\mathbb{R}^r$, denote $\mathbf{J} g$ the $r \times \nin$ Jacobian of $g$, and $\|\mathbf{J}g(x)\| = \det \left( (\mathbf{J}g(x)) (\mathbf{J}g(x))^\top \right)^{\frac12}$, and the inner expectation is with respect to all parameters aside these biases.
\end{restatable}

\begin{proof}[Proof of Lemma~\ref{lem:calc_vol}]
The proof follows the arguments of \citet[Theorem~6]{pmlr-v97-hanin19a}. 
The main difference is that maxout units are generically active and the activation regions of maxout units may involve several pre-activation features and additional inequalities.
To obtain the upper bound, we will discard certain inequalities, and separate one distinguished pre-activation feature $j_0$ for each neuron participating in a sub-pattern, which allows us to relate inputs in the corresponding activation regions to bias values and apply the co-area formula. 

    Recall that an $r$-partial activation sub-pattern ${J} \in\mathcal{S}_r$ is a list of patterns $J_z\subseteq[K]$ of cardinality at least $2$ for a collection of participating neurons $z \in \mathcal{Z}$, with $\sum_{z\in \mathcal{Z}}(|J_z|-1) = r$. 
    Further, for any given $J_z$ we denote $j_0$ its smallest element. 
    When discussing a particular sub-pattern, we will write $m = |\mathcal{Z}|$ for the number of participating neurons.
    Finally, recall that $\mathcal{H}({J}, \theta) = \bigcap_{z \in \mathcal{Z}} \mathcal{H} (J_z,\theta)$. 

    By Corollary~\ref{cor:smaller_sum}, with probability $1$ with respect to $\theta$, 
    \begin{align*}
        \vol_{\nin - r}(\mathcal{X}_{\net, r}(\theta)) = \sum_{{J} \in \mathcal{S}_r} \vol_{\nin - r} (\mathcal{H} (J, \theta)).
    \end{align*}
Fix $J\in\mathcal{S}_r$. In the following we prove that 
    \begin{align*}
    \mathbb{E}[\vol_{\nin - r} (\mathcal{H}(J,\theta) \cap S)]\leq \int\limits_{S} \mathbb{E} \left[ \rho_{\mathbf{b}^r}((\mathbf{w}^m - \mathbf{w}^r) \cdot \mathbf{x}^m_{-1} + \mathbf{b}^m) \enspace \| \mathbf{J} ((\mathbf{w}^m - \mathbf{w}^r) \cdot \mathbf{x}^m_{-1} + \mathbf{b}^m)\| \right] \enspace d x.
    \end{align*}
    We first note that 
    \begin{align}
        \vol_{\nin - r}(\mathcal{H}(J,\theta) \cap S) = \int\limits_{\mathcal{H}(J,\theta) \cap S } d\vol_{\nin - r}(x).
        \label{eq:vol_general}
    \end{align}

    For each $z\in\mathcal{Z}$ and $J_z$ we can pick an element $j_0\in J_z$ and express $\mathcal{H}(J_z,\theta)$ in terms of $(|J_z| - 1)$ equations and $(K - |J_z|)$ inequalities (not necessarily linear), 
    \begin{align}
        \begin{split}
            \mathcal{H}(J_z,\theta) =  \{x \in \mathbb{R}^{\nin} \ | \ &w_{z, j_0} \cdot x_{l(z)-1} + b_{z, j_0} = w_{z, j} \cdot x_{l(z)-1} + b_{z, j}, \quad \forall j \in J_z \setminus \{j_0\};\\
            &w_{z, j_0} \cdot x_{l(z)-1} + b_{z, j_0} > w_{z, i} \cdot x_{l(z)-1} + b_{z, i}, \quad \forall i \in [K] \setminus J_z \}. 
        \end{split}
        \label{eq:eq_ineq_repr}
    \end{align}
    Here, $x_{l(z)-1}$ are the activation values of the units in the layer preceding unit $z$, depending on the input $x$. 
    Since $\sum_{z} (|J_z|-1) = r$, the set $\mathcal{H}(J,\theta)$ is defined by $r$ equations (in addition to inequalities). 
    We will denote with $\mathbf{b}^r \in \mathbb{R}^{r}$ the vector of biases $b_{z, j}$ that are involved in these $r$ equations, with subscripts $(z,j)$ with $j\in J_z\setminus \{j_0\}$ and $z\in\mathcal{Z}$. 
    
    We take the expectation of \eqref{eq:vol_general} with respect to the conditional distribution of $\mathbf{b}^r$ given the values of all the other network parameters. 
    We have assumed that this has a density. 
    Denoting the conditional density of $\mathbf{b}^r$ by $\rho_{\mathbf{b}^r}$, this is given by 
    \begin{align}
        \int\limits_{\mathbb{R}^r}
        \int\limits_{\mathcal{H}(J,\theta) \cap S } d\vol_{\nin - r}(x)
        \rho_{\mathbf{b}^r}(\mathbf{b}^r) 
        d\mathbf{b}^r .
        \label{eq:eq5a}
    \end{align}
    The equations in \eqref{eq:eq_ineq_repr} imply that $b_{z, j} = (w_{z, j_0} - w_{z, j}) \cdot x_{l(z)-1} + b_{z, j_0}$ for any $x\in \mathcal{H}(J,\theta)$. 
    We write all these equations concisely as $\mathbf{b}^r = (\mathbf{w}^m - \mathbf{w}^r) \cdot \mathbf{x}^m_{-1} + \mathbf{b}^m$. 
    Then \eqref{eq:eq5a} becomes 
    \begin{align}
        \int\limits_{\mathbb{R}^r} \int\limits_{\mathcal{H}(J,\theta) \cap S } \rho_{\mathbf{b}^r}((\mathbf{w}^m - \mathbf{w}^r) \cdot \mathbf{x}^m_{-1} + \mathbf{b}^m) \enspace d\vol_{\nin - r}(x) d\mathbf{b}^r.
        \label{eq:eq5}
    \end{align}
    
    We will upper bound the volume of $\mathcal{H}(J,\theta)$ by the volume of the corresponding set without the inequalities, 
    \begin{align*}
        \mathcal{H}'(J,\theta)
        \defeq \bigcap\limits_{z \in \mathcal{Z}} \left\{ x \in \mathbb{R}^{\nin} \ | \ w_{z, j_0} \cdot x_{l(z)-1} + b_{z, j_0} = w_{z, j} \cdot x_{l(z)-1} + b_{z, j}, \quad \forall j \in J_z \setminus \{j_0\} \right \}.
    \end{align*}
    
    Since $\mathcal{H}(J,\theta) \subseteq \mathcal{H}'(J,\theta)$, we can upper bound \eqref{eq:eq5} by 
    \begin{align}
        \int\limits_{\mathbb{R}^r} \int\limits_{\mathcal{H}'(J,\theta) \cap S } \rho_{\mathbf{b}^r}((\mathbf{w}^m - \mathbf{w}^r) \cdot \mathbf{x}^m_{-1} + \mathbf{b}^m) \enspace d\vol_{\nin - r}(x) d\mathbf{b}^r. 
        \label{eq:pre-co-area}
    \end{align}
    
    Now we will use the co-area formula to express \eqref{eq:pre-co-area} as an integral over $S$ alone. 
    Recall that the co-area formula says that if $\psi \in L^1(\mathbb{R}^n)$ and $g: \mathbb{R}^n \rightarrow \mathbb{R}^r$ with $r \leq n$ is Lipschitz, then
    \begin{align*}
        \int_{\mathbb{R}^r}\int_{g^{-1}(t)} \psi(x) d \vol_{n-r}(x)dt = \int_{\mathbb{R}^n} \psi(x) \| \mathbf{J} g(x)\| d\vol_n(x), 
    \end{align*}
    where $\mathbf{J}g$ is the $r\times n$ Jacobian of $g$ and $\|\mathbf{J}g(x)\|=\det((\mathbf{J}g(x))(\mathbf{J}g(x))^\top)^\frac12$. 
    
In our case $r = r$, $n = \nin$, which satisfy $r \leq \nin$.
    Further, $\mathbf{b}^r \in \mathbb{R}^{r}$ plays the role of $t \in\mathbb{R}^r$, and $\mathbb{R}^{\nin}\to\mathbb{R}^r$; $x\mapsto \rho_{\mathbf{b}^r}((\mathbf{w}^m - \mathbf{w}^r) \cdot \mathbf{x}^m_{-1} + \mathbf{b}^m)$ plays the role of $\psi$.
    Since $(\mathbf{w}^m - \mathbf{w}^r) \cdot \mathbf{x}^m_{-1} + \mathbf{b}^m$ is continuous and $S$ is bounded, assuming $\rho_{\mathbf{b}^r}$ is continuous, this is in $L^1(S)$.
    Finally, we set $g\colon S \rightarrow \mathbb{R}^r$; $x\mapsto ((\mathbf{w}^m - \mathbf{w}^r) \cdot \mathbf{x}^m_{-1} + \mathbf{b}^m)$, which is Lipschitz.
    
    Hence \eqref{eq:pre-co-area} can be expressed as 
    \begin{align*}
        \int\limits_{S} \rho_{\mathbf{b}^r}((\mathbf{w}^m - \mathbf{w}^r) \cdot \mathbf{x}^m_{-1} + \mathbf{b}^m) \enspace \|\mathbf{J} ((\mathbf{w}^m - \mathbf{w}^r) \cdot \mathbf{x}^m_{-1} + \mathbf{b}^m)\| \enspace d x.
    \end{align*}
    
    Taking expectation with respect to all other weights and biases, and interchanging the integral over $S$ with the expectation (according to Fubini's theorem, since the integral is non-negative), 
    \begin{align*}
        \int\limits_{S} \mathbb{E} \left[ \rho_{\mathbf{b}^r}((\mathbf{w}^m - \mathbf{w}^r) \cdot \mathbf{x}^m_{-1} + \mathbf{b}^m) \enspace \| \mathbf{J} ((\mathbf{w}^m - \mathbf{w}^r) \cdot \mathbf{x}^m_{-1} + \mathbf{b}^m)\| \right] \enspace d x.
    \end{align*}
    Summing over all $r$-partial activation sub-patterns $J\in\mathcal{S}_r$ gives the desired result. 
\end{proof}

Based on the preceding Lemma~\ref{lem:calc_vol}, now we derive a more explicit upper bound expressed in terms of properties of the probability distribution of the network parameters. 

\semimainupperboundtheorem*
\begin{proof}[Proof of Theorem~\ref{th:semi_main_upper_bound}]
    By Lemma~\ref{lem:calc_vol}, $\mathbb{E}\left[\vol_{\nin - r}(\mathcal{X}_{\net, r} \cap S) \right]$ is upper bounded by
    \begin{align*}
        \sum_{ J \in\mathcal{S}_r} \int\limits_{S} \mathbb{E} \left[ \rho_{\mathbf{b}^r}((\mathbf{w}^m - \mathbf{w}^r) \cdot \mathbf{x}^m_{-1} + \mathbf{b}^m) \enspace \|\mathbf{J} ((\mathbf{w}^m - \mathbf{w}^r) \cdot \mathbf{x}^m_{-1} + \mathbf{b}^m)\| \right] \enspace d x. 
    \end{align*}
    
    Since we have assumed that for any collection of $t$ biases the conditional density given all weights and the other biases can be upper-bounded with $C_{\text{\normalfont bias}}^t$, we have $\rho_{\mathbf{b}^r}((\mathbf{w}^m - \mathbf{w}^r) \cdot \mathbf{x}^m_{-1} + \mathbf{b}^m) \leq C_{\text{\normalfont bias}}^r$. 
    
    As for the term $\mathbb{E}[\|\mathbf{J} ((\mathbf{w}^m - \mathbf{w}^r) \cdot \mathbf{x}^m_{-1} + \mathbf{b}^m)\|]$, note that 
    \begin{align}
        &\|\mathbf{J} ((\mathbf{w}^m - \mathbf{w}^r) \cdot \mathbf{x}^m_{-1} + \mathbf{b}^m)\| \nonumber\\
        &= \det\left( \mathbf{J} ((\mathbf{w}^m - \mathbf{w}^r) \cdot \mathbf{x}^m_{-1} + \mathbf{b}^m)^T
        \mathbf{J} ((\mathbf{w}^m - \mathbf{w}^r) \cdot \mathbf{x}^m_{-1} + \mathbf{b}^m) \right)^{1/2} \nonumber\\
        &= \det \big( \text{Gram} \big( \nabla((w_{z_1, j_0} - w_{z_1, j_1}) \cdot x_{l(z_1)-1} + b_{z_1, j_0}), \dots, \nonumber\\
        & \hspace{60pt} \nabla((w_{z_m, j_0} - w_{z_m, j_{r_m}}) \cdot x_{l(z_m)-1} + b_{z_m, j_0}) \big) \big)^{1/2}, 
        \label{eq:gram}
    \end{align}
    where for any $v_1,\ldots, v_r \in \mathbb{R}^n$,
  $(\text{Gram}(v_1, \dots, v_r))_{i, j} = \langle v_i, v_j \rangle$
    is the associated Gram matrix.
    
    It is known that the Gram determinant can also be expressed in terms of the exterior product of vectors,  meaning that \eqref{eq:gram} can be written as 
    \begin{align*}
        &  \| \nabla((w_{z_1, j_0} - w_{z_1, j_1}) \cdot x_{l(z_1)-1} + b_{z_1, j_0}) \land \dots \land \nabla((w_{z_m, j_0} - w_{z_m, j_{r_m}}) \cdot x_{l(z_m)-1} + b_{z_m, j_0}) \|, 
    \end{align*}
    which is the the $r$-dimensional volume of the parallelepiped in $\mathbb{R}^{\nin}$ spanned by $r$ elements. 
    Therefore, for $ J\in\mathcal{S}_r$ with participating neurons $Z$, we can upper bound this expression by \citep[see][]{GOVER201028} 
    \begin{align*}
        &\prod\limits_{z \in {Z}} \prod\limits_{j \in J_z \setminus \{j_0\}} \|\nabla ((w_{z, j_0} - w_{z_i, j}) \cdot x_{l(z)-1} + b_{z, j_0}) \|\\
        &\leq \prod\limits_{z \in  {Z}} \prod\limits_{j \in J_z \setminus \{j_0\}} 2 \max \left\{ \|\nabla (w_{z, j_0} \cdot x_{l(z)-1}) \|,  \|\nabla (w_{z, j} \cdot x_{l(z)-1}) \| \right\}\\
        &\leq
        2^r \max_{\substack{ z \in  {Z}, j \in J_z}} \left\{ \|\nabla (w_{z, j} \cdot x_{l(z)-1}) \|\right\}^{r}.
    \end{align*}
    In the second line we use the triangle inequality. Considering the assumption that we have made on the gradients, for the expectation we obtain the upper bound $(2 \cgrad)^r$. 
    
    By Lemma \ref{lem:trivial_upper_bound}, we can upper-bound the number of entries of the sum $\sum_{J \in \mathcal{S}_r}$ with $\binom{rK}{2r} \binom{N}{r}$. Combining everything, we get the final upper bound
    \begin{align*}
        (2 \cgrad \cbias)^r \binom{rK}{2r} \binom{N}{r} \vol_{\nin}(S).
    \end{align*}
    This concludes the proof. 
\end{proof}

\section{Proofs related to the expected number of regions}
\label{app:expected_number}

\mainresulttheorem*
\begin{proof}[Proof of Theorem~\ref{th:main_result}] 
    The proof follows closely the arguments of \citet[Proof of Theorem~10]{NIPS2019_8328}, whereby we use appropriate supporting results for maxout networks and need to accommodate the combinatorics depending on $K$. 
    Fix a network $\net$ with rank-$K$ maxout units, input dimension $\nin$ and output dimension $1$. Let $0\leq r\leq \nin$. 
    For $N \leq \nin$, the statement follows direction from the simple upper bound on the number of distinct $r$-partial activation patterns given in Lemma~\ref{lem:trivial_upper_bound}. 
    
    Consider now the case $N \geq \nin$. 
    Fix a closed cube $C \subseteq \mathbb{R}^{\nin}$ of sidelength $\delta > 0$. For any $t \in \{0, \dots, \nin\}$ let 
    \begin{align*}
        C_t \defeq t\text{\normalfont-skeleton of } C 
    \end{align*}
    denote the union of $t$-dimensional faces of $C$. 
    For example, $C_0$ is the set of $2^{\nin}$ vertices of $C$,  $C_{\nin - 1}$ is the set of $2 \nin$ facets of $C$, and $C_{\nin}$ is $C$. 
    In general, $C_t$ consists of $\binom{\nin}{t} 2^{\nin-t}$ faces of dimension $t$, each with $t$-volume $\delta^t$. 
    Hence, 
    \begin{align}
        \vol_t(C_t) = \binom{\nin}{t} 2^{\nin-t} \delta^t.
        \label{eq:volt_ct}
    \end{align}
    
    For any choice of $\theta$ let 
    \begin{align*}
    \mathcal{V}_t(\theta) \defeq \mathcal{X}_{\net, t}(\theta) \cap C_t.
    \end{align*}
    
    By Lemma~\ref{lem:intersection_upper_bound} below, for any $t$ and almost every choice of $\theta$, the set $\mathcal{V}_t(\theta)$ is a finite set of points.
    For each $t \in \{0, \dots, \nin\}$, we also define 
    \begin{align*}
        C_{t, \varepsilon} \defeq \{ x \in \mathbb{R}^{\nin} \ | \ \dist(x, C_t ) \leq \varepsilon \}, 
    \end{align*}
    the $\varepsilon$-thickening of $C_t$. 
    For almost every $\theta$, Lemma~\ref{lem:intersection_volume} ensures the existence of an $\varepsilon > 0$ such that for all $v\in \mathcal{V}_{t}(\theta)$, the radius-$\varepsilon$ balls $B_{\varepsilon}(v)$ are contained in $C_{t,\varepsilon}$ and are disjoint. 
    Hence, writing $\omega_{\nin-t}$ for the $(\nin-t)$-volume of the $(\nin-t)$-dimensional ball with unit radius, 
    \begin{align*}
        \vol_{\nin - t} (X_{\net, t} \cap C_{t, \varepsilon}) \geq \sum\limits_{v \in \mathcal{V}_t} \varepsilon^{\nin - t} \omega_{\nin - t} = \# \mathcal{V}_t \cdot \varepsilon^{\nin - t} \omega_{\nin - t}.
    \end{align*}
    
    Therefore, for all but a measure $0$ set of $\theta \in \mathbb{R}^{\# \text{\normalfont params}}$, there exists $\varepsilon > 0$ so that
    \begin{align*}
        \frac{\vol_{\nin - t}(X_{\net, t} \cap C_{t, \varepsilon})}{\varepsilon^{\nin - t} \omega_{\nin - t}} \geq \# \mathcal{V}_t.
    \end{align*}
    
    Thus taking the limit $\varepsilon \rightarrow 0$ and taking expectation with respect to the parameter $\theta$, and applying Fatou's lemma to upper bound the result by the expression with exchanged limit and expectation, 
    \begin{align*}
         \mathbb{E} \left [ \# \mathcal{V}_t \right ] \leq \mathbb{E} \left[ \lim \limits_{\varepsilon \rightarrow 0} \frac{\vol_{\nin - t}(\mathcal{X}_{\net, t} \cap C_{t, \varepsilon})}{\varepsilon^{\nin - t} \omega_{\nin - t}} \right ] \leq \lim \limits_{\varepsilon \rightarrow 0} \mathbb{E} \left[  \frac{\vol_{\nin - t}(\mathcal{X}_{\net, t} \cap C_{t, \varepsilon})}{\varepsilon^{\nin - t} \omega_{\nin - t}} \right ].
    \end{align*}
    
    Then,
    \begin{align*}
        \mathbb{E} \left [ \# \mathcal{V}_t \right ] \leq &
        \lim \limits_{\varepsilon \rightarrow 0} \mathbb{E} \left[  \frac{\vol_{\nin - t}(\mathcal{X}_{\net, t} \cap C_{t, \varepsilon})}{ \vol_{\nin}(C_{t, \varepsilon})} \cdot \frac{\vol_{\nin}(C_{t, \varepsilon})}{\varepsilon^{\nin - t} \omega_{\nin - t}} \right ]\\
        = &\lim \limits_{\varepsilon \rightarrow 0} \mathbb{E} \left[  \frac{\vol_{\nin - t}(\mathcal{X}_{\net, t} \cap C_{t, \varepsilon})}{ \vol_{\nin}(C_{t, \varepsilon})} \right ] \cdot \lim \limits_{\varepsilon \rightarrow 0}  \frac{\vol_{\nin}(C_{t, \varepsilon})}{\varepsilon^{\nin - t} \omega_{\nin - t}} \\
        \leq& (
        2 \cgrad \cbias)^t \binom{tK}{2t} \binom{N}{t} \vol_t(C_t). 
    \end{align*}
    To obtain the last line, the first term is upper bounded using Theorem~\ref{th:semi_main_upper_bound}, and the second term is evaluated using 
    \begin{align*}
        \lim\limits_{\varepsilon \rightarrow 0} \frac{\vol_{\nin} ( C_{t, \varepsilon}) } {\varepsilon^{\nin - t} \omega_{\nin - t}} = \vol_t (C_t). 
    \end{align*}
    
    Combining this with Lemma~\ref{lem:intersection_upper_bound} and the formula \eqref{eq:volt_ct} for $\vol_t(C_t)$, we find
    \begin{align}
        &\mathbb{E} \left[  \# \{ r\text{\normalfont -partial activation regions with } \mathcal{R}(J^r;\theta) \cap C \neq \emptyset\} \right]\nonumber \\
        &\stackrel{\phantom {\delta \geq 1 / (2 \cgrad \cbias)}}{\leq} \sum\limits_{t = r}^{\nin} \binom{t}{r} K^{t-r} (2 \cgrad \cbias)^t \binom{tK}{2t} \binom{N}{t} \binom{\nin}{t} 2^{\nin - t} \delta^t \nonumber\\
        & \stackrel{\delta \geq 1 / (2 \cgrad \cbias)}{\leq} (2 \delta \cgrad \cbias)^{\nin} \binom{\nin K}{2 \nin} (2 K) ^{\nin - r} \sum\limits_{t = r}^{\nin} \binom{t}{r} \binom{N}{t} \binom{\nin}{t}.
        \label{eq:theorembound}
    \end{align}
    The last line uses the assumption that $\delta \geq 1 / (2 \cgrad \cbias)$ and Lemma~\ref{lem:technical1}, which states that $\binom{tK}{2t} \leq \binom{nK}{2n}$ for $t\leq n$. 

    In the following we simplify \eqref{eq:theorembound}. 
    Note that $\binom{t}{r} \leq \sum_{r = 0}^{t} \binom{t}{r} = 2^{t} \leq 2^{\nin}$.
    Hence \eqref{eq:theorembound} can be upper bounded by 
    \begin{align*}
        &(4 \delta \cgrad \cbias )^{\nin} \binom{\nin K}{2 \nin}  (2 K) ^{\nin - r} \sum\limits_{t = r}^{\nin} \binom{N}{t} \binom{\nin}{t}\\
        &= (4 \delta \cgrad \cbias)^{\nin} \binom{\nin K}{2 \nin}  (2 K) ^{\nin - r} \sum\limits_{t = r}^{\nin} \binom{\nin}{t}^2 \frac{ \binom{N}{t} }{\binom{\nin}{t}}.
    \end{align*}
    
    Using $\nin \leq N$, observe that 
    \begin{align*}
        \frac{ \binom{N}{t} }{\binom{\nin}{t}} = &\frac{N! \cdot (\nin - t)!}{(N - t)! \cdot \nin!} \leq N^{t} \cdot \frac{(\nin - t)!}{\nin!} = \frac{N^{\nin}}{N^{\nin - t}} \cdot \frac{(\nin - t)!}{\nin!} = \frac{N^{\nin}}{\nin!} \cdot \frac{(\nin - t)!}{N^{\nin - t}}\\
        \leq &\frac{N^{\nin}}{\nin!} \cdot \frac{(\nin - t)^{\nin - t}}{N^{\nin - t}} \leq \frac{N^{\nin}}{\nin!}.
    \end{align*}
    
    Also, using Vandermonde's identity, observe that 
    \begin{align*}
        \sum\limits_{t = 0}^{\nin} \binom{\nin}{t}^2 =\binom {2 \nin}{ \nin} \leq 4^{\nin}.
    \end{align*}
    
   Combing everything, \eqref{eq:theorembound} is upper bounded by
   \begin{align*}
        &(16 \delta \cgrad \cbias)^{\nin} \binom{\nin K}{2 \nin}  (2 K) ^{\nin - r} \frac{N^{\nin}}{\nin!} = (32 K \cgrad \cbias)^{\nin} \binom{\nin K}{2 \nin}  \frac{N^{\nin}}{ (2 K)^r \nin!} \vol(C).
   \end{align*}
   
   Setting $T
   = 2^5 \cgrad \cbias$, we get
    \begin{align*}
        \frac{(T K N )^{\nin} \binom{\nin K}{2 \nin} } {(2 K)^r \nin!} \vol(C).
    \end{align*}
This concludes the proof. 
\end{proof}

We state and prove lemmas used in the proof of Theorem \ref{th:main_result}. 
\begin{lemma}
    \label{lem:technical1}
    For any $t \leq n$, $\binom{tK}{2t} \leq \binom{nK}{2n}$. 
\end{lemma}
\begin{proof}
To see this, note that $\binom{tK}{2t} \leq \binom{nK}{2n}$ is equivalent to  the following: 
    \begin{align*}
        \frac{(Kr)!}{(2r)!(Kr - 2r)!} &\leq \frac{(Kn)!}{(2n)!(Kn - 2n)!}\\
        \frac{(2n)!}{(2r)!} \frac{(Kn - 2n)!}{(Kr - 2r)!}  &\leq  \frac{(Kn)!}{(Kr)!}\\
        \prod_{i=1}^{2n - 2r} (2r + i) \prod_{j=1}^{(K-2)n - (K-2) r} (Kr - 2r + j) &\leq \prod_{k=1}^{Kn - Kr} (Kr + k). 
    \end{align*}
    Since $\prod_{i=1}^{2n - 2r} (2r + i) \leq \prod_{k=1}^{2n - 2r} (Kr + k)$ and $\prod_{j=1}^{(K-2)n - (K-2) r} (Kr - 2r + j) \leq \prod_{k=2n - 2r + 1}^{Kn - Kr} (Kr + k)$ the inequality holds. 
\end{proof}

\begin{lemma}
    \label{lem:intersection_upper_bound}
    For almost every $\theta$, for each $t\in\{0,\ldots, \nin\}$, the set $\mathcal{V}_t(\theta) = \mathcal{X}_{\mathcal{N},t}(\theta)\cap C_t$ consists of finitely many points and  
    \begin{align}
            \# \{ r\text{\normalfont -partial activation regions } \mathcal{R}(J^r,\theta)
            \text{ \normalfont with } \mathcal{R}(J^r,\theta) \cap C \neq \emptyset \} \leq \sum\limits_{t=r}^{\nin} \binom{t}{r} K^{t-r} \# \mathcal{V}_t(\theta),
        \label{eq:rpart_with_v}
    \end{align}
    where $\# \mathcal{V}_t(\theta)$ is the number of points in $\mathcal{V}_t(\theta)$. 
\end{lemma}

\begin{proof} 
    The proof is similar to the proof of \cite[Lemma 12]{NIPS2019_8328}. The difference lies in the types of equations that appear in the partial activation regions of maxout networks.
    The dimension of $\mathcal{V}_t(\theta)$ is $0$ with probability $1$, because the set $C_t$ has dimension $t$ and, by Lemma \ref{lem:hyperplane_ball}, with probability $1$ the set $\mathcal{X}_{\net, t}$ coincides locally with a subspace of codimension $t$. 
    The intersection of two generic affine spaces of complementary dimension has dimension $0$. 
    
    Now we prove \eqref{eq:rpart_with_v}. 
    If $J^r$ is an $r$-partial activation pattern and $\mathcal{R}(J^r,\theta)\cap C\neq\emptyset$, then the closure $\operatorname{cl} \mathcal{R}(J^r,\theta) \cap C$ is a non-empty polytope.
    The intersection is bounded because $C$ is bounded, and, by Lemma~\ref{lem:conv_act_reg}, the closure of $\mathcal{R}(J^r,\theta)$ is a polyhedron. 
    As a non-empty polytope, this set has at least one vertex.
    Generically, if a vertex is in an $(\nin-t)$-face of $\operatorname{cl}\mathcal{R}(J^r,\theta)$, then it is in a $t$-face of $C$. 
    Hence, with probability $1$,  
    \begin{align*}
        \mathcal{R}(J^r,\theta) \cap C \neq \emptyset \quad \Rightarrow \quad \exists \ t \in \{r, \dots, \nin\} \quad \text{s.t.} \quad \operatorname{cl}\mathcal{R}(J^r,\theta) \cap \mathcal{V}_t \neq \emptyset. 
    \end{align*}
    Thus, with probability $1$,
    \begin{align*}
    \# \{ r\text{\normalfont -partial activation regions with } \mathcal{R}(J^r,\theta) \cap C \neq \emptyset\} \leq \sum\limits_{t = r}^{\nin} T_t \# \mathcal{V}_t,
    \end{align*}
    where $T_t$ is the maximum over all $v \in \mathcal{V}_t$ of the number of $r$-partial activation regions whose closure contains $v$.
    
    To complete the proof, it remains to check that, with probability 1, 
    \begin{align*}
        T_t\leq \binom{t}{r} K^{t-r}.
    \end{align*}
    
    By the definition of $\mathcal{X}_{\mathcal{N},t}$, each $v\in \mathcal{V}_t$ is an element of exactly one $t$-partial activation region defined by $t$ equations. 
    To upper bound the number of $r$-partial activation regions that contain $v$, we upper bound the number of ways in which one can get an $r$-partial region from this $t$-partial region. 
    We have $\binom{t}{r}$ options to pick $r$ equations that will remain satisfied. 
    In each case, there are at most $t-r$ neurons for which we need to specify a pre-activation feature attaining the maximum, for a total of at most $K^{t-r}$ options. 
    This concludes the proof. 
\end{proof}

\begin{lemma}
    \label{lem:intersection_volume}
    Fix $t \in \{0, \dots \nin \}$. For almost every choice of $\theta$, there exists $\varepsilon > 0$ (depending on $\theta$) so that the balls $\mathcal{B}_{\varepsilon}(v)$ of radius $\varepsilon$ centered at $v \in \mathcal{V}_t$ are disjoint and 
    \begin{align*}
        \vol_{\nin - t}(\mathcal{X}_{\net, t} \cap \mathcal{B}_{\varepsilon}(v)) = \varepsilon^{\nin - t} \omega_{\nin - t},
    \end{align*}
    where $\omega_t$ is the volume of a unit ball in $\mathbb{R}^t$.
\end{lemma}

\begin{proof}
    The proof is similar to the proof of \citet[Lemma~13]{NIPS2019_8328}, whereby we use Lemma~\ref{lem:intersection_upper_bound} and the results for maxout networks obtained in Section~\ref{app:intro_proofs}.
    By Lemma~\ref{lem:intersection_upper_bound}, with probability $1$ over $\theta$, each $\mathcal{V}_t$ is a finite set of points. 
    Hence, we may choose $\varepsilon > 0$ sufficiently small so that the balls $B_{\varepsilon}(v)$ are disjoint. Moreover, by Lemma \ref{lem:hyperplane_ball}, in a sufficiently small neighborhood of $v \in \mathcal{V}_t$, the set $\mathcal{X}_{\net, t}$ coincides with a $(\nin - t)$-dimensional subspace.
    The $(\nin - t)$-dimensional volume of this subspace in $B_{\varepsilon}(v)$ is the volume of $(\nin - t)$-dimensional ball of radius $\varepsilon$, which equals $\varepsilon^{\nin - t} \omega_{\nin - t}$, completing the proof.
\end{proof}

To conclude this section, we compare the results on the numbers of activation regions of maxout and ReLU networks in Table \ref{tb:comparison}. 
\begin{table}[ht]
    \caption{Comparison of the activation region results for maxout and ReLU networks.}
    \label{tb:comparison}
    \centering
    {\small 
    \begin{tabular}{p{0.29\textwidth} >{\centering}p{0.3\textwidth} >{\centering\arraybackslash}p{0.3\textwidth} }
        \toprule
        & ReLU network & Maxout network \\
        \midrule
        Generic lower bound on the number of linear regions for a deep network & $1$, \  Remark~\ref{remark:lower_bound_relu} & $\sum_{j=0}^{\nin}{\binom{n_1}{j}}$, \ Theorem~\ref{thm:lower_bound} \\
        \midrule
        Trivial upper-bound on the number of $r$-partial activation regions & $\binom{N}{r} 2^{N - r}$, \ \citep[Theorem~10]{NIPS2019_8328} & $\binom{r K}{2r} \binom{N}{r} K^{N - r}$, \ Lemma~\ref{lem:trivial_upper_bound}, see also Proposition~\ref{proposition:maxnrpartialactivation}\\
        \midrule
        Upper-bound on the expected number of $r$-partial activation regions, $N \geq \nin$
        &
        $\frac{(T N)^{\nin}}{2^r \nin!}$, $T = 2^5 \cgrad \cbias$, \ \citep[Theorem~10]{NIPS2019_8328} 
        & 
        $\frac{(T K N)^{\nin} \binom{\nin K}{2 \nin} }{(2 K)^r \nin!}$, $T = 2^5 \cgrad \cbias$, \ Theorem~\ref{th:main_result}
        \\
        \midrule
        Upper bound on the expected $(\nin - r)$-dimensional volume of the non-linear locus
        & $(2 \cgrad \cbias)^r \binom{N}{r}$, \ \citep[Corollary~7]{pmlr-v97-hanin19a} & $(2 \cgrad \cbias)^r \binom{rK}{2r} \binom{N}{r}$, \ Theorem~\ref{th:semi_main_upper_bound} \\
        \bottomrule
    \end{tabular}
    }
\end{table}

\section{Upper bounding the constants}
\label{app:constants}

We briefly discuss the constants $C_{\text{bias}}$ and $C_{\text{bias}}$ in the hypothesis of Theorem~\ref{th:main_result}. 
The constant $C_{\text{bias}}$ can be evaluated at initialization using the definition, since we know the distribution of biases. Recall that we defined $C_{\text{\normalfont bias}}$ as an upper bound on
\begin{align*}
    \left( \sup \limits_{b_{1}, \dots, b_{t}\in \mathbb{R}} \rho_{b_{1}, \dots, b_{t}} (b_{1}, \dots, b_{t}) \right)^{1/t},
\end{align*}
where $\rho_{b_{1}, \dots, b_{t}}$ is the conditional distribution of any collection of biases given all the other weights and biases in $\net$ and $t \in \mathbb{N}$.
If the biases are sampled independently, independently of the weights, this equals $\sup_{b \in \mathbb{R}} \rho_{b} (b)$. 
Then, for instance, 
for a normal distribution with standard deviation $\sqrt{C/n_l}$, the constant $C_{\text{bias}}$ can be chosen as 
\begin{align*}
    \max_{l \in \{0, \dots, L - 1\}} \sqrt{\frac{n_l}{2 \pi C}}. 
\end{align*}

The constant $C_{\text{grad}}$ was defined as an upper bound on
\begin{align*}
    \left(\sup\limits_{x \in \mathbb{R}^{\nin}} \mathbb{E}[\| \nabla \zeta_{z, k}(x) \|^t]\right)^{1/t}.
\end{align*}

Therefore we need to upper-bound $\mathbb{E} \left[ \| \nabla \zeta_{z, k}(x) \|^{t} \right]$.
This expression stands for the $t$-th moment of the L$2$ norm of the gradient of a pre-activation feature $\zeta_{z,k}$ in a network, with respect to the input to the network. 

One possible calculation is as follows. We consider $J_x = [\nabla_x \mathcal{N}_1(x;\theta),\ldots, \nabla_x \mathcal{N}_{n_L}(x;\theta)]^\top$ the Jacobian of the output vector  with respect to the input, for a given parameter $\theta$ and input $x$. 
Note that the gradient $\nabla \zeta_{z, k}(x)$ for a pre-activation feature of a unit in the $l$-th layer of a network is a row in the Jacobian matrix of an $l$-layer network. 
Therefore, $ \| \nabla \zeta_{z, k}(x) \|$ can be upper-bounded by the spectral norm $\|J_x\|$ of the Jacobian, and the moments of the Jacobian norm can be used as an upper-bound on the $t$-th moments of the gradient norm, $t \geq 1$. 

\begin{proposition}[Upper bound on the moments of the Jacobian matrix norm]
    \label{prop:jacobian_upper}
    Let $\net$ be a fully-connected feed-forward network with maxout units of rank $K$ and a linear last layer. 
    Let the network have $L$ layers of widths $n_1, \ldots, n_L$ and $\nin$ inputs. 
    Assume that the weights and biases of the units in the $l$-th layer are sampled iid from a Gaussian distribution with mean $0$ and variance $c / n_{l-1}$, $l = 1, \ldots, L$ and $c$ is some constant $c \in \mathbb{R}, c > 0$. Then 
    \begin{align*}
        &\mathbb{E}[\| J_x \|^t] \leq 
        c^{t/2} n_0^{-t/2} \mathbb{E}[\chi_{n_{L}}^t]  \prod_{l = 1}^{L-1} \mathbb{E} \left[ \left( \frac{c}{n_l} \sum_{i=1}^{n_l}  m^{(K)}_{n_{l-1}, i} \right)^{t/2} \right],
    \end{align*}
    where $J_x$ is the Jacobian as defined above, $x \in \mathbb{R}^{n_0}$; $t \geq 1, t \in \mathbb{N}$; 
    $m^{(K)}_{n_{l-1}, i}$ is the largest order statistic in a sample of size $K$ of $\chi_{n_{l-1}}^2$ variables. 
    Recall that the largest order statistic is a random variable defined as the maximum of a random sample, and that a sum of squares of $n$ independent Gaussian variables has a chi-squared distribution $\chi_n^2$. 
\end{proposition}
\begin{proof}

Our first goal will be to upper-bound $\|J_x\| = \sup_{\|u\| = 1} \| J_x u \|$. 
The Jacobian $J_x$ of $\mathcal{N}(x)\colon \mathbb{R}^{n_0} \rightarrow \mathbb{R}^{n_L}$ 
can be written as a product of matrices $\overline{W}^{(l)}$, $l=1,\ldots,L$ depending on the activation region of the input $x$. 
The matrix $\overline{W}^{(l)}$ consists of rows $\overline{W}_i^{(l)} = W_{i, k_i}^{(l)}\in\mathbb{R}^{n_{l-1}}$, where $k_i = \argmax_{k\in[K]}\{W^{(l)}_{i,k}x^{(l - 1)} + b^{(l)}_{i,k}\}$ for $i = 1, \ldots, n_l$,  
and $x^{(l-1)}$ is the $l$-th layer's input. 
For the last layer, which is linear, we have $\overline W^{(L)} = W^{(L)}$. 
Thus for any given $u\in\mathbb{R}^{n_0}$ we have 
\begin{align*}
    \| J_x u \| = \|W^{(L)} \overline{W}^{(L - 1)} \cdots \overline{W}^{(1)} u\|.
\end{align*}

Consider some $u^{(0)}$ with $\|u^{(0)}\|=1$ and assume $\|\overline{W}^{(1)} u^{(0)}\|\neq 0$. 
Note that for fixed $u^{(0)}$, the probability of $\overline{W}^{(1)}$ being such that $\|\overline{W}^{(1)} u^{(0)}\| = 0$ is $0$. 
Multiplying and dividing by $\| \overline{W}^{(1)} u^{(0)}\|$ we get
\begin{align*}
    &\|W^{(L)} \overline{W}^{(L - 1)} \cdots \overline{W}^{(1)} u^{(0)}\| \frac{\| \overline{W}^{(1)} u^{(0)}\|}{\| \overline{W}^{(1)} u^{(0)}\|}\\
    = &\left\|W^{(L)} \overline{W}^{(L - 1)} \cdots \overline{W}^{(2)} \frac{ \overline{W}^{(1)} u^{(0)}}{\| \overline{W}^{(1)} u^{(0)}\|}\right\| \| \overline{W}^{(1)} u^{(0)}\|\\
    = &\left\|W^{(L)} \overline{W}^{(L - 1)} \cdots \overline{W}^{(2)} u^{(1)}\right\| \| \overline{W}^{(1)} u^{(0)}\|,
\end{align*}
where $u^{(1)} = \frac{ \overline{W}^{(1)} u^{(0)}}{\| \overline{W}^{(1)} u^{(0)}\|}$. Notice, $\|u^{(1)}\| = 1$.
Repeating this procedure layer-by-layer, we get 
\begin{align*}
    \|W^{(L)} u^{(L-1)}\| \| \overline{W}^{(L-1)} u^{(L-2)}\| \cdots \| \overline{W}^{(2)} u^{(1)}\| \| \overline{W}^{(1)} u^{(0)}\| . 
\end{align*} 

Now consider one of the factors, $\| \overline{W}^{(l)} u^{(l-1)}\|$. 
We have 
\begin{align*}
    &\| \overline{W}^{(l)} u^{(l-1)}\|^2 = \sum_{i=1}^{n_l} \langle \overline{W}_i^{(l)}, u^{(l-1)} \rangle^2 
    \stackrel{\substack{\text{Cauchy–Schwarz}\\ \|u^{(l-1)}\| = 1}}{\leq}
    \sum_{i=1}^{n_l} \|\overline{W}_{i}^{(l)}\|^2
    \leq \sum_{i=1}^{n_l} \max_{k \in [K]} \left\{ \|W_{i,k}^{(l)}\|^2 \right\}.
\end{align*}
Notice that this upper bound only depends on $W^{(l)}$ and is independent of all other weight matrices and of the input vector. 

According to our assumptions, $W_{i,k}^{(l)} \stackrel{d}{=} \sqrt{\frac{c}{n_{l-1}}} v$, where $v$ is a standard Gaussian random vector in $\mathbb{R}^{n_{l-1}}$.  
Therefore, $\|W_{i,k}^{(l)}\|^2 \stackrel{d}{=} \frac{c}{n_{l-1}} \chi_{n_{l-1}}^2$ has the distribution of a  chi-squared random variable scaled by $c/n_{l-1}$. 
Moreover, since the vectors $W_{i,1}^{(l)},\ldots, W_{i,K}^{(l)}$ consist of the same number of separate iid entries, the variables $\|W_{i,1}^{(l)}\|^2,\ldots, \|W_{i,K}^{(l)}\|^2$ are iid. 
In turn, $\max_{k \in [K]} \left\{ \|W_{i,k}^{(l)}\|^2 \right\} \stackrel{d}{=} \frac{c}{n_{l-1}} m^{(K)}_{n_{l-1}, i}$, where $m^{(K)}_{n_{l-1}, i}$ is the largest order statistic in a sample of size $K$ of $\chi_{n_{l-1}}^2$ variables.

Notice that $\|W^{(L)} u^{(L-1)}\|^2 \stackrel{d}{=} \frac{c}{n_{L-1}} \chi^2_{n_{L}}$. 
To see this, recall that if $u$ is a fixed vector and $w$ is a Gaussian random vector with mean $\mu$ and covariance matrix $\Sigma$, then the product $u^\top w$ is Gaussian with mean $u^\top \mu$ and variance $u^\top \Sigma u$. 
Hence, since $W^{(L)}_i$ is a Gaussian vector with mean zero and covariance matrix $\Sigma = \frac{c}{n_{L-1}}I$, $W^{(L)}_i u^{(L-1)}$ is Gaussian with mean zero and variance $\frac{c}{n_{L-1}}\|u^{(L-1)}\|^2=\frac{c}{n_{L-1}}$.

Combining everything, we get
\begin{align*}
    &\|J_x\| = \sup_{\|u\| = 1} \| J_x u \| \leq \left( \frac{c}{n_{L-1}} \chi^2_{n_{L}}\right)^{1/2}
    \left( \frac{c}{n_{L - 2}} \sum_{i=1}^{n_{L - 1}} m^{(K)}_{n_{L - 2}} \right)^{1/2} \cdots
    \left( \frac{c}{n_{0}} \sum_{i=1}^{n_{1}} m^{(K)}_{n_0} \right)^{1/2}
    \\
    = &c^{L/2} \chi_{n_{L}} \left(\prod_{l=0}^{L-1} n_l^{-1/2}\right) \prod_{l = 1}^{L-1} \left(\sum_{i=1}^{n_l} m^{(K)}_{n_{l-1}, i} \right)^{1/2}.
\end{align*}

Now using the monotonicity of the expectation, the moments of the right hand side upper-bound those of the left hand side. 
Moreover, using the independence of the individual factors, the expectation factorizes. 
For the $t$-th moment we get 
    \begin{align*}
        \mathbb{E}[\| J_x \|^t]
        & \leq \mathbb{E} \left[ c^{t L/2} \chi_{n_{L}} \left(\prod_{l=0}^{L-1} n_l^{-1/2}\right) \prod_{l = 1}^{L-1} \left(\sum_{i=1}^{n_l}  m^{(K)}_{n_{l-1}, i} \right)^{t/2} \right]\\
        & = c^{t/2} n_0^{-t/2} \mathbb{E}[\chi_{n_{L}}^t]  \prod_{l = 1}^{L-1} \mathbb{E} \left[ \left( \frac{c}{n_l} \sum_{i=1}^{n_l}  m^{(K)}_{n_{l-1}, i} \right)^{t/2} \right].
    \end{align*}
\end{proof}

\begin{corollary}[Upper bound on $C_{\text{grad}}$]
    Under the same assumptions as in Proposition \ref{prop:jacobian_upper}, assuming that $c$ is set according to He initialization, meaning $c = 2$, or maxout-He initialization (see Table \ref{tb:std} for specific values of $c$ for various $K$), the following expression can be used as the value for $C_{\text{grad}}$: 
    \begin{align*}
        \left(\frac{c}{\nin}\right)^{1/2} \left(n_L (n_L + t)^{\frac{t}{2} - 1}\right)^{1/t} \prod_{l = 1}^{L-1} \left( \mathbb{E} \left[ \left( \frac{c}{n_l} \sum_{i=1}^{n_l}  m^{(K)}_{n_{l-1}, i} \right)^{t/2} \right]\right)^{1/t}, 
    \end{align*}
    where $m^{(K)}_{n_{l-1}, i}$ is the largest order statistic in a sample of size $K$ of $\chi_{n_{l-1}}^2$ variables. 
\end{corollary}
\begin{proof}
    The constant $C_{\text{grad}}$ was defined as an upper bound on 
    \begin{align*}
        \left(\sup\limits_{x \in \mathbb{R}^{\nin}} \mathbb{E}[\| \nabla \zeta_{z, k}(x) \|^t]\right)^{1/t}.
    \end{align*}

    Therefore, using the upper-bound on the moments of the Jacobian norm from Proposition \ref{prop:jacobian_upper}, an upper-bound on the following expression can be used as a value for $C_{\text{grad}}$: 
    \begin{align*}
        c^{1/2} n_0^{-1/2} \left(\mathbb{E}[\chi_{n_{L}}^t] \right)^{1/t} \prod_{l = 1}^{L-1} \left( \mathbb{E} \left[ \left( \frac{c}{n_l} \sum_{i=1}^{n_l}  m^{(K)}_{n_{l-1}, i} \right)^{t/2} \right]\right)^{1/t}.
    \end{align*}
    
    The moments of the chi distribution are 
    \begin{align*}
        \mathbb{E}\left[\chi_{n_{L}}^t\right] = 2^{t/2} \frac{\Gamma((n_L + t) / 2)}{\Gamma(n_L / 2)}.
    \end{align*}
    Using an upper-bound on a Gamma function ratio \citep[see][Equation 12]{jameson2013inequalities}, this can be upper-bounded with 
    \begin{align*}
        n_L (n_L + t)^{\frac{t}{2} - 1}. 
    \end{align*}
    
    The factor involving $m^{(K)}_{n_{l-1}}$ can be upper-bounded by considering the explicit expression for the moments of the largest order statistic of chi-squared variables.
    The closed form for these moments is available \citep[see][]{nadarajah2008explicit}, but they have complicated form and we will keep the factor involving $m^{(K)}_{n_{l-1}}$ as it is. 
    Then the total upper bound is
    \begin{align*}
        \left(\frac{c}{\nin}\right)^{1/2} \left(n_L (n_L + t)^{\frac{t}{2} - 1}\right)^{1/t} \prod_{l = 1}^{L-1} \left( \mathbb{E} \left[ \left( \frac{c}{n_l} \sum_{i=1}^{n_l}  m^{(K)}_{n_{l-1}, i} \right)^{t/2} \right]\right)^{1/t}.
    \end{align*}
\end{proof}

Estimating the moments of the gradient of maxout networks is a challenging topic, as can be seen from the above discussion, and is worthy of a separate investigation. 
It might be possible to obtain tighter upper-bounds on it and on $C_{\text{grad}}$, a question that we leave for future work. 

\section{Expected number of regions for networks without bias}
\label{app:zero_bias}

Zero biases of ReLU networks were discussed in \citet{NIPS2019_8328} and studied in detail in \citet{steinwart2019sober}.
There is no distribution on the biases in the zero bias case, meaning that conditions on the biases from Theorem \ref{th:main_result} are not satisfied.
We closely follow the proofs in \citet{NIPS2019_8328} and show that the arguments similar to those regarding the zero bias case in the ReLU networks also apply to the maxout networks.
According to Lemma \ref{lem:zero_regions}, activation regions of zero-bias maxout networks are convex cones, see Figure \ref{fig:zero_regions} for the illustration.
In Corollary \ref{cor:zero_bias_upper} we come to a conclusion that the number of activation regions in expectation in a network with zero biases grows as $O(\nin (K N)^{\nin - 1} \binom{K (\nin - 1)}{2 (\nin - 1)})$. 

\begin{figure}[!htb]  
    \centering
    \begin{tabular}{ccc}
         Zero bias & Small bias & Non-zero bias\\
         \includegraphics[width=0.15\textwidth]{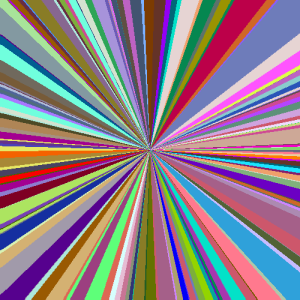} &
         \includegraphics[width=0.15\textwidth]{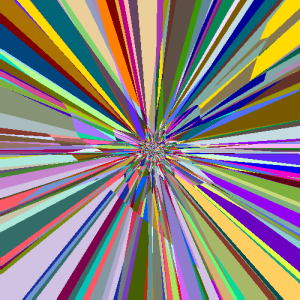} &
         \includegraphics[width=0.15\textwidth]{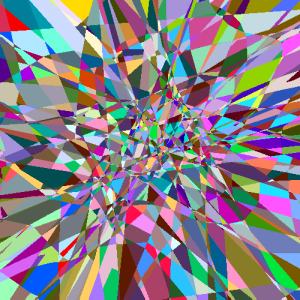}\\
    \end{tabular}
    \caption{
    Linear regions of a $3$ layer network with $100$ units and the maxout rank $K = 2$. The network was initialized with the maxout-He distribution.
    Activation regions of a maxout network with zero biases are convex cones.
    Small biases are initialized as the biases sampled from the maxout-He distribution multiplied by $0.1$.
    Majority of linear regions of a network with small biases are cones and the ones that are not are small and concentrated around zero.
    }
    \label{fig:zero_regions}
\end{figure}

\begin{lemma}
    \label{lem:zero_regions}
    Let $\net$ be a maxout network with biases set to zero. Then,
    \begin{enumerate}[(a)]
        \item $\net$ is
        nonnegative
        homogeneous
        : $\net(cx) = c \net (x)$ for each $c \geq 0$.
        \label{item:equivriant}
        \item For every activation region $\mathcal{R}$ of $\net$, and every point $x$ in $\mathcal{R}$, all points $cx$ are also in $\mathcal{R}$ for $c > 0$ and $\mathcal{R}$ is a convex polyhedral cone.
        \label{item:cone}
    \end{enumerate}
\end{lemma}
\begin{proof}[Proof of Lemma~\ref{lem:zero_regions}]
    Each neuron of the network computes a function of the form $z(x_1, \dots, x_n) = \max_{k \in [K]} \left\{ \sum_{i=1}^n w_{i, k} \cdot x_i \right\}$. Note that for any $c \geq 0$:
    \begin{align*}
        z(c x_1, \dots, c x_n) = \max_{k \in [K]} \left\{ c \sum_{i=1}^n w_{i, k} \cdot x_i \right\} = c \cdot \max_{k \in [K]} \left\{ \sum_{i=1}^n w_{i, k} \cdot x_i \right\} =  c \cdot z(x_1, \dots, x_n).
    \end{align*}
    Therefore, each neuron is equivariant under multiplication by a nonnegative constant $c$, and thus the overall network as well, proving \ref{item:equivriant}. If $c > 0$, the activation patterns for $x$ and $cx$ are also identical, since for any inequality in the activation region definition we have
    \begin{align*}
        \sum_{i=1}^n w_{i, j} \cdot c x_i > \sum_{i=1}^n w_{i, j'} \cdot c x_i \iff \sum_{i=1}^n w_{i, j} \cdot x_i > \sum_{i=1}^n w_{i, j'} \cdot x_i, \quad j,j' \in [K].
    \end{align*}
    
    This implies that $x$ and $cx$ lie in the same activation region, and that $\mathcal{R}$ is a convex polyhedral cone, see e.g. \citet{chandru2011optimization}. This proves \ref{item:cone}.
\end{proof}

\begin{proposition}[Networks without biases do not have more regions]
    \label{prop:zero_bias}
    Suppose that $\net$ is a maxout network with biases and conditions from Theorem \ref{th:main_result} are satisfied. Let $\net_{0}$ be the same network with all biases set to $0$. Then, the total number of activation regions (in all the input space) for $\net_{0}$ is no more than that for $\net$.
\end{proposition}
\begin{proof}[Proof of Proposition \ref{prop:zero_bias}]
    We define an injective mapping from activation regions of $\net_0$ to regions of $\net$. For each region $\mathcal{R}$ of $\net_0$, pick a point $x_{\mathcal{R}} \in \mathcal{R}$. By Lemma \ref{lem:zero_regions}, $c x_{\mathcal{R}} \in \mathcal{R}$ for each $c > 0$. Let $\net_{1 / c}$ be the network obtained from $\net$ by dividing all biases by $c$, and observe that $\net(c x_{\mathcal{R}}) = c \net_{1/c} ( x_{\mathcal{R}})$, with the same activation pattern between the two networks.
    
    By picking $c$ sufficiently large, $\net_{1/c}$ becomes arbitrarily close to $\net_0$. Therefore, for some sufficiently large $c$, $\net_0(c x_{\mathcal{R}})$ and $\net(c  x_{\mathcal{R}})$ have the same pattern of activations. Regions of $\net$ in which $c x_{\mathcal{R}}$ lies are distinct for all distinct $\mathcal{R}$. Thus, the number of regions of $\net$ is at least as large as the number of regions of $\net_0$.
\end{proof}

We obtain following corollary of Theorem \ref{th:main_result} for the zero-bias case.

\begin{corollary}[Expected number of activation regions of zero-bias networks]
    \label{cor:zero_bias_upper}
    Suppose that $\net_0$ is a fully-connected feed-forward maxout network with zero biases, $\nin$ inputs, a total of $N$ rank $K$ maxout units. Also, suppose that all conditions from Theorem \ref{th:main_result}, except for the conditions on the biases, are satisfied.
    Then there exists a constant $T'$ depending on $C_{\text{grad}}$
    so that
    \begin{align*}
        \mathbb{E}[\# \text{\normalfont activation regions of } \mathcal{N}_0 ] \leq
        \begin{cases}
            K^N, \quad N \leq \nin
            \\[5pt]
            2 \nin \frac{(T' K N)^{\nin - 1} \binom{K (\nin - 1)}{2 (\nin - 1)}}{(\nin - 1)!}, \quad N \geq \nin
        \end{cases}.
    \end{align*}
    The expectation is taken with respect to the distribution of weights in $\net_0$.
\end{corollary}
\begin{proof}[Proof of Corollary \ref{cor:zero_bias_upper}]
    Based on Proposition \ref{prop:zero_bias} we can use the same upper bound as for the networks with biases, thus for the case $N \leq \nin$, the expectation is upper bounded with $K^N$.
    
    Now consider the case $N \geq \nin$.
    We will add biases to $\net_0$ in such a way that the bias conditions of Theorem \ref{th:main_result} are satisfied with some $C_{\text{bias}}'$. Denote the resulting network with $\net$.
    Then, by Proposition~\ref{prop:zero_bias}, $\net$ has a region corresponding to each region of $\net_0$.
    All the corresponding regions in $\net$ are unbounded because according to Proposition~\ref{prop:zero_bias} for any $x_{\mathcal{R}}$ from a region of $\net_0$ there exists a constant $c > 0$ so that $c x_{\mathcal{R}}$ belongs to a region in $\net$. Since all regions in $\net_0$ are unbounded, all corresponding regions in $\net$ are unbounded under such a mapping.
    
    Therefore, to obtain the result, it is enough to upper-bound the number of unbounded activation regions of $\net$. 
    Similarly to the proof of Theorem~\ref{th:main_result}, consider a hypercube with a side length $\delta>\delta_0$, large enough to interest all the unbounded regions. Then the total number of unbounded activation regions of $\net$ is upper bounded by the sum of the numbers of activation regions intersecting each of the hypercube $2 \nin$ facets, each of dimension $(\nin - 1)$.
    By Theorem~\ref{th:main_result}, the expected number of activation regions of $\net$ in $\mathbb{R}^{\nin - 1}$ is upper bounded with ${(\delta 2^5 C_{\text{grad}} C_{\text{bias}}' K N)^{\nin - 1} \binom{K (\nin - 1)}{2 (\nin - 1)}} / {(\nin - 1)!}$. Denoting $\delta 2 ^5 C_{\text{grad}} C_{\text{bias}}'$ with $T'$ and combining everything we get the desired result.
\end{proof}

\section{Proofs related to the decision boundary}
\label{app:decision_boundary}

\subsection{Simple upper bound on the number of pieces of the decision boundary}
A network used for multi-class classification into $M \in \mathbb{N}, M \geq 2$ classes can be seen as a network with a rank $M$ maxout unit on top.
Therefore, to discuss the decision boundary, we consider $r$-partial activation regions, $r \geq 1$, with at least one equation in the last unit. 
With $J_{\text{DB}}^r$, we denote the $r$-partial activation patterns corresponding to such regions and with $\mathcal{X}_{\text{DB}, r} \defeq \bigcup_{J_{\text{DB}}^r \in \mathcal{P}_{\text{DB}, r}} \mathcal{R}(J_{\text{DB}}^r; \theta)$ their union.
All decision boundary is then written as $\mathcal{X}_{DB}$. 

\begin{lemma}[Simple upper bound on the number of $r$-partial activation patterns of the decision boundary]
    \label{lem:decision_patterns}
    Let $r\in\mathbb{N}_+$. The number of $r$-partial activation patterns in the decision boundary of a network with a total of $N$ rank-$K$ maxout units is upper bounded by $|\mathcal{P}_{\text{DB}, r}| \leq \sum_{i = 1} ^{\min\{M - 1, r\}}  \binom{M}{i + 1} \binom{K(r - i)}{2 (r - i)} \binom{N}{r - i} K^{N - r + i}$. 
    The number of $r$-partial activation sub-patterns is upper bounded by $|\mathcal{S}_{\text{DB}, r}|\leq \sum_{i = 1} ^{\min\{M - 1, r\}}  \binom{M}{i + 1} \binom{K(r - i)}{2 (r - i)} \binom{N}{r - i}$. 
\end{lemma}
\begin{proof}[Proof of Lemma~\ref{lem:decision_patterns}]
    Activation patterns for the decision boundary regions should have at least one equality in the upper unit.
    At the same time, the maximum possible number of equations in the last unit is $\min\{M - 1, r\}$.
    To get all suitable activation patterns we need to sum over all these options. 
    
    Now consider a fixed number of equations $i \in \{1, \dots, \min\{M - 1, r\}\}$. The number of ways to choose them is $\binom{M}{i + 1}$ and the number of options for the all other units in the network is given by Lemma~\ref{lem:trivial_upper_bound} for $r - i$. Combining everything, we get the claimed statement. 
\end{proof}

\subsection{Lower bound on the maximum number of pieces of the decision boundary}
The lower bound in the second item of Theorem~\ref{thm:upper_bound} is based on a construction of parameters for which the network maps an $n$-cube in the input space to an $n$-cube in the output space in many-to-one fashion. 
This means that any feature implemented over the last layer will replicate multiple times over the input layer.
We infer the following lower bound on the maximum number of pieces of the decision boundary of a maxout network. 

\begin{proposition}[Lower bound on the maximum number of pieces of the decision boundary]
\label{cor:max-decision}
Consider a network $\mathcal{N}$ with $n_0$ inputs and $L$ layers of $n_1,\ldots, n_L$ rank-$K$ maxout units followed by an $M$-class classifier.
Suppose $n\leq n_0$, $\frac{n_l}{n}$ even, and $e_l = \min\{n_0,\ldots, n_{l-1}\}$. 
Denote by $N(M,n)$ the maximum number of boundary pieces implemented by an $M$-class classifier over an $n$-cube. 
Then the maximum number of linear pieces of the decision boundary of $\mathcal{N}$ is lower bounded by $N(M,n) \prod_{l=1}^{L} (\frac{n_l}{n} (K-1)+1)^n$.
If $n\geq M$ or $n\geq 4$, $N(M,n)=\binom{M}{2}$. 
\end{proposition}

The asymptotic order of this bound is $\Omega(M^2\prod_{l=1}^L(n_lK)^{n_0})$.  
\begin{proof}
We use the construction of parameters from \citet[Proposition 3.11]{sharp2021} refining a previous construction for ReLU networks \citep{NIPS2014_5422} to have the network represent a many-to-one map. 
There are $\prod_{l=1}^{L} (\frac{n_l}{n} (K-1)+1)^n$ distinct linear regions whose image in the output space of the last layer contains an $n$-cube. 
The linear pieces of the decision boundary of an $M$-class classifier over an $n$-cube at the $L$-th layer will have a corresponding multiplicity over the input space.
An $M$-class classifier is implemented as $\mathbb{R}^M\to[M]$; $y=(y_1,\ldots, y_M) \mapsto \operatorname{argmax}_{r\in [M]}y_r$. 
This has $\binom{M}{2}$ boundaries, one between any two classes. 
If $n\geq M$, then the image of the preceding layers intersects all of these boundaries.
More generally, the number of boundary pieces of an $M$-class classifier over $n$-dimensional space can be seen to correspond to the number of edges of a polytope with $M$ vertices in $n$-dimensional space. 
The trivial upper bound $N(M,n) \leq \binom{M}{2}$ is attained if $1<\lfloor \frac{n}{2} \rfloor$. This follows form the celebrated Upper Bound Theorem for the maximum number of faces of convex polytopes \citep{mcmullen_1970}. 
\end{proof}

\subsection{Upper bound on the expected volume of the decision boundary}

\lemmaboundaryvolume*
\begin{proof}[Proof of Theorem~\ref{th:decision_boundary_volume}]
    Using Lemma~\ref{lem:calc_vol}, but considering only $r$-partial activation patterns that belong to the decision boundary, volume of the $(\nin-r)$-skeleton of the decision boundary can be upper-bounded with
    \begin{align*}
        \sum_{\hat{J}^r_{\text{DB}}} \int\limits_{S} \mathbb{E} \left[ \rho_{\mathbf{b}^r}((\mathbf{w}^m - \mathbf{w}^r) \cdot \mathbf{x}^m_{-1} + \mathbf{b}^m) \enspace \| \mathbf{J} ((\mathbf{w}^m - \mathbf{w}^r) \cdot \mathbf{x}^m_{-1} + \mathbf{b}^m)\| \right] \enspace d x.
    \end{align*}
    Upper-bounding the integral as in Theorem \ref{th:semi_main_upper_bound}, but using Lemma~\ref{lem:decision_patterns} to count the number of entries in the sum, we get the final upper-bound
    \begin{align*}
        (2 \cgrad \cbias)^r \sum_{i = 1} ^{\min\{M - 1, r\}} \binom{N}{r - i} \binom{K(r - i)}{2 (r - i)} \binom{M}{i + 1} \vol_{\nin}(S).
    \end{align*}
\end{proof}

\subsection{Upper bound on the expected number of pieces of the decision boundary}
\begin{restatable}[Upper bound on the expected number of $r$-partial activation regions of the decision boundary]{lemma}{lemma-decision-boundary}
    \label{lem:decision_boundary}
    
    Let $\net$ be a fully-connected feed-forward maxout network, with $\nin$ inputs, a total of $N$ rank $K$ maxout units, and $M$ linear output units used for multi-classification. 
    Fix $r \in \{1, \dots, \nin\}$.
    Then, under the assumptions of Theorem \ref{th:main_result}, there exists $\delta_0\leq 1/(2 \cgrad \cbias)$ such that for all cubes $\mathcal{C}\subseteq\mathbb{R}^{\nin}$ with side length $\delta>\delta_0$,
    \begin{align*}
        \frac{\mathbb{E}\big[\substack{\# \text{ \normalfont$r$-partial activation regions in} \\ \text{\normalfont the decision boundary of } \mathcal{N} \text{ \normalfont in } \mathcal{C}}\big]}{\vol(\mathcal{C})} \leq
        \begin{cases}
            \sum_{i = 1}^{\min\{M - 1, r\}}  \binom{M}{i + 1}  \binom{K(r - i)}{2 (r - i)} \binom{N}{r - i} K^{N - r + i} , \quad N \leq \nin
            \\[5pt]
            \frac{(2^4 \cgrad \cbias N)^{\nin} (2 K)^{\nin - 1} }{\nin!}\\
            \ \times \sum_{i = 1}^{\min\{M - 1, \nin\}}  \binom{M}{i + 1}  \binom{K(\nin - i)}{2 (\nin - i)} \frac{\prod_{j=1}^{i}(\nin - j + 1)}{\prod_{j=1}^{i}(N - 1 + j)}, \quad N \geq \nin
        \end{cases}.
    \end{align*}
    Here the expectation is taken with respect to the distribution of weights and biases in $\net$. 
\end{restatable}

\begin{proof}[Proof of Lemma~\ref{lem:decision_boundary}]
    Result for the case $N \leq \nin$ arises from Lemma~\ref{lem:decision_patterns}.
    Consider $N \geq \nin$.
    The proof closely follows the proof of Theorem~\ref{th:semi_main_upper_bound}, and we only highlight the differences.
    Based on Lemma~\ref{th:decision_boundary_volume},
    \begin{align*}
        \mathbb{E} \left [ \# \mathcal{V}_t \right ] \leq (
        2 \cgrad \cbias)^t \sum_{i = 1} ^{\min\{M - 1, t\}} \binom{N}{t - i} \binom{K(t - i)}{2 (t - i)} \binom{M}{i + 1} \vol_t(\mathcal{C}_t). 
    \end{align*}

    Therefore, the upper bound on the expected number of $r$-partial activation regions in the decision boundary is
    \begin{align*}
        &\sum\limits_{t = r}^{\nin} \binom{t}{r} K^{t-r} (2 \cgrad \cbias)^t \sum_{i = 1}^{\min\{M - 1, t\}} \binom{N}{t - i} \binom{K(t - i)}{2 (t - i)} \binom{M}{i + 1} \binom{\nin}{t} 2^{\nin - t} \delta^t\\
        &\leq (4 \delta \cgrad \cbias)^{\nin} (2 K) ^{\nin - r} \sum_{i = 1}^{\min\{M - 1, \nin\}}  \binom{M}{i + 1}  \binom{K(\nin - i)}{2 (\nin - i)} \sum\limits_{t = r}^{\nin} \binom{N}{t - i}\binom{\nin}{t}
    \end{align*}

    Re-writing $\binom{N}{t - i} \binom{\nin}{t}$ as $\binom{\nin}{t}^2 \frac{\binom{N}{t - i}}{\binom{\nin}{t}}$ we can upper-bound it with
    \begin{align*}
        4^{\nin} \frac{\prod_{j=1}^{i}(t - j + 1)}{\prod_{j=1}^{i}(N - t + j)} \frac{N^{\nin}}{\nin!} \leq 4^{\nin} \frac{\prod_{j=1}^{i}(\nin - j + 1)}{\prod_{j=1}^{i}(N - r + j)} \frac{N^{\nin}}{\nin!}.
    \end{align*}

    The final upper bound is then
    \begin{align*}
        \frac{(2^5 \cgrad \cbias K N)^{\nin}}{(2 K) ^{r} \nin!} \sum_{i = 1}^{\min\{M - 1, \nin\}}  \binom{M}{i + 1}  \binom{K(\nin - i)}{2 (\nin - i)} \frac{\prod_{j=1}^{i}(\nin - j + 1)}{\prod_{j=1}^{i}(N - r + j)} \vol(C).
    \end{align*}
    Dividing this expression by $\vol(C)$ we get the desired result.
\end{proof}

The next theorem follows immediately from Lemma~\ref{lem:decision_boundary} if $r$ is set to $1$. 

\theoremdecisionboundary*

\subsection{Lower bound on the expected distance to the decision boundary}

Now, using an approach similar to \citet[Corollary~5]{pmlr-v97-hanin19a}, who provided a lower bound on the expected distance to the boundary of linear regions, we discuss a lower bound on the distance to the decision boundary. 
We will use the following result from that work. 

\begin{lemma}[{\citealt[Lemma 12]{pmlr-v97-hanin19a}}]
\label{lem:hr_tube}
    Fix a positive integer $n \geq 1$, and let $Q \subseteq \mathbb{R}^n$ be a compact continuous piecewise linear submanifold
    with finitely many pieces. Define $Q_0 = \emptyset$ and let $Q_t$ be the union of the interiors of all $k$-dimensional pieces of $Q \setminus (Q_0 \cup \dots \cup Q_{t - 1})$. Denote by $T_{\varepsilon}(X)$ the $\varepsilon$-tubular neighborhood of any $X \subset \mathbb{R}^n$. We have $\vol_n (T_{\varepsilon}(Q)) \leq \sum_{t=0}^n \omega_{n - t} \varepsilon^{n - t} \vol_k (Q_t)$,where $\omega_d \defeq $ volume of ball of radius $1$ in $\mathbb{R}^d$.
\end{lemma}

We will prove the following. 

\cordbdistance*
\begin{proof}[Proof of Corollary \ref{cor:dist_to_db}]
    Let $x \in K$ be uniformly chosen. Then, for any $\varepsilon > 0$, using Markov’s inequality and Lemma \ref{lem:hr_tube}, we have
    \begin{align*}
        \mathbb{E} [\text{distance}(x, \mathcal{X}_{DB} )]
        \geq &\varepsilon P(\text{distance} (x, \mathcal{X}_{DB} ) > \varepsilon)
        = \varepsilon (1 - P (\text{distance}(x, \mathcal{X}_{DB} ) \leq \varepsilon))\\
        = & \varepsilon \left(1 - \mathbb{E} \left[\text{vol}_{\nin} (T_{\varepsilon} (\mathcal{X}_{DB} ) \right] \right)
        \geq \varepsilon \left( 1 - \sum_{t=1}^{\nin} \omega_{\nin - t} \varepsilon^{\nin - t} \mathbb{E}\left[ \text{vol}_{\nin - t} (\mathcal{X}_{DB}) \right] \right)
    \end{align*}
    
    The upper bound from Theorem \ref{th:decision_boundary_volume} can be upper bounded further with
    \begin{align*}
        \mathbb{E}[\vol_{\nin - t}(\mathcal{X}_{\text{DB}, t} \cap S)] \leq & (2 \cgrad \cbias)^t\sum_{i = 1} ^{\min\{M - 1, t\}}  \binom{M}{i + 1} \binom{K(t - i)}{2 (t - i)} \binom{N}{t - i} \vol_{\nin}(S)\\
        \leq & (2 \cgrad \cbias)^t (4 K^2 N)^{t-1} M^{m^* + 1} m^* \vol_{\nin}(S),
    \end{align*}
    where $m^* \defeq \min\{M - 1, t\}$.
    Then expectation of the distance can be lower bounded with
    \begin{align*}
        \varepsilon \left( 1 - \sum_{t=1}^{\nin} (2 \cgrad \cbias \varepsilon)^t (4 \varepsilon K^2 N)^{t-1} M^{m^* + 1} m^* \right)
        \geq
        \varepsilon \left( 1 - 2 \cgrad \cbias M^{m + 1} m \varepsilon \right),
    \end{align*}
    where $m \defeq \min\{M - 1, \nin \}$.
    Taking $\varepsilon$ to be 
    a small constant $c$ times
    $1 / ( 2 \cgrad \cbias M^{m + 1} m)$ completes the proof.
\end{proof}

\begin{remark}[Decision boundary of ReLU networks]
    All proofs consider the indecision locus of the last unit on top of the network and reuse results on the volume of the boundary and the number of activation regions. If one sets $K$ to $2$, these results differ only in $2^{-r}$ from those for the ReLU networks. Therefore, the decision boundary analysis should also apply to the ReLU networks if one sets $K$ to $2$ with a difference only in the constant. 
\end{remark}

\section{Counting algorithms}
\label{app:algorithm}

\subsection{Approximate counting of the activation regions}

First, we describe an approximate method for counting linear regions that is useful for quickly estimating the number of linear regions or plotting them.

We generate a grid of inputs in an $\nin$-dimensional cube, compute the gradients with respect to the input, which is simply a product of weights on the path that corresponds to a given input, and then sum the gradient values for each input dimension of one input. Then, we compute the number of unique sums and use it as the number of linear regions.
 
The method is not exact because it works by computing network gradients on a grid, so it is possible to miss a small region. Also, it does not distinguish between regions with the same gradient value, which is one more reason it might miss some linear regions and why it counts linear regions, not activation regions. However, from what we have seen, if the grid has many points, the difference between the exact and approximate method is not that big.

\subsection{Exact counting of the activation regions}
\label{subseq:exact_regions}

The algorithm starts with a cube in which we want to count the activation regions defined with a set of linear inequalities in $\mathbb{R}^{\nin}$.
We go through the network layer by layer, unit by unit, and for each unit, we determine if its pre-activation features attain a maximum on the regions obtained so far by checking the feasibility of the corresponding linear inequalities systems.
For this, we use linear programming.
More specifically, an interior-point method implementation from \texttt{scipy.optimize.linprog}.
The use of linear programming is justified since, according to Lemma \ref{lem:conv_act_reg}, the activation regions are convex.

The input to the simplex method becomes the combined system of inequalities for the region and the pre-activation feature. We set the objective to zero, meaning that any $x$ can satisfy it. 
One has to use non-strict inequalities in linear programming methods, implying the boundary of activation regions is also included.
We also add a small $\varepsilon = 1\mathrm{e}{-6}$ to avoid zero solutions in a zero bias case.
The inequalities for a pre-activation feature of some neuron $z$ have the form
\begin{align*}
    \begin{split}
        \{x \in \mathbb{R}^{\nin} \ | \  a_{z, j_0}(x; \theta) + b_{z, j_0} \geq a_{z, i}(x; \theta) + b_{z, i} + \varepsilon, \quad \forall i \in [K] \backslash [j_0] \}.
    \end{split}
\end{align*}
As a result, we get a new list of activation regions and pass it to the next unit.

To correctly estimate inequalities corresponding to a pre-activation feature on a specific region, one has to keep track of the function computed on this region, which has the form: $w_{{J}}^{(l)} \dots (w_{{J}}^{(0)} \cdot x + b_{{J}}^{(0)}) + \dots + b_{{J}}^{(l)}$, where $J$ is an activation pattern of the region.

The pseudocode for the algorithm is in Algorithm~\ref{alg:exact_maxout}, and the pseudocode for a check for one pre-activation feature is in Algorithm~\ref{alg:feature_check}.

\subsection{Exact counting of linear pieces in the decision boundary}

We define an algorithm for exactly counting linear pieces in the decision boundary based on the algorithm from Section \ref{subseq:exact_regions}.
Consider a classification problem with $M$ classes, and to describe the decision boundary, add a maxout unit of rank $M$ on top of the network.
To count the number of linear pieces in the decision boundary, for each pair of classes, go through all the activation regions of the network.
Construct a linear program for which the set of inequalities is given by a union of the region inequalities and inequalities which determine if the given classes attain maximum. Also, add the equality between these two classes. 
If the problem is feasible, there is a piece in the decision boundary.
At the end of this process, one gets the total number of linear pieces in the decision boundary. 

\subsection{Algorithm discussion}

There are two useful modifications to the method.
First, to count the number of regions in a ReLU network instead of systems of $(K - 1)$ linear inequalities, one can use inequalities of the form $w \cdot x + b \geq 0$ and $w \cdot x + b \leq 0$.

Second, to compute the number of activation regions in a slice, one can define a parametrization of the input space.
We consider as the slice of a cube $\mathcal{C}$ the 2-space through three points $x_1,x_2,x_3\in\mathbb{R}^{\nin}$, meaning the slice has the form $V=\{x = v_0 + v_1 y_1 + v_2 y_2 \in \mathbb{R}^{\nin} \colon (y_1,y_2)\in \mathbb{R}^2\cap \mathcal{C} \}$, where $v_0 = (x_1+x_2+x_3)/3\in\mathbb{R}^{\nin}$, and $v_1,v_2\in\mathbb{R}^{\nin}$ are an orthogonal basis of $\operatorname{span}\{x_2-x_1, x_3-x_1\}$, 
and $v_1,v_2$ are orthonormal.
We can evaluate the network function over such a slice by augmenting the network by a linear layer $\phi\colon \mathbb{R}^2 \to \mathbb{R}^{\nin}$ with weights $v_1,v_2$ and biases $v_0$. 
We used images from 3 different classes as the points that define the slice. 

\begin{algorithm}[t]
    \caption{Exactly Count the Number of Activation Regions in a Maxout Network}
    \label{alg:exact_maxout}
    \begin{algorithmic}[1]
        \Function{CountActivationRegions}{}
            \State \texttt{activation\_regions} = [\texttt{starting\_cube}]
            \For{\texttt{layer} in $\{1, \dots, L\}$}
                \For{\texttt{unit} in \texttt{layer}}
                \State \texttt{new\_activation\_regions} = []
                   \For{\texttt{region} in \texttt{activation\_regions}}
                        \For{ \texttt{feature} in \texttt{unit}}
                            \LineComment See Algorithm \ref{alg:feature_check}
                            \If{NewRegionCheck(\texttt{unit.features}, \texttt{feature}, \texttt{region})}
                                \State                                \texttt{new\_activation\_regions}.append(\texttt{new\_region})
                            \EndIf
                        \EndFor
                    \EndFor
                \State \texttt{activation\_regions} = \texttt{new\_activation\_regions}
                \EndFor
                \For{\texttt{region} in \texttt{activation\_regions}}
                    \State \texttt{region.function} = \texttt{region.next\_layer\_function}
                    \State \texttt{region.next\_layer\_function} = []
                \EndFor
            \EndFor
            \State \Return length(\texttt{activation\_regions)}
        \EndFunction
    \end{algorithmic}
    \label{algorith:exact_count}
\end{algorithm}

\begin{algorithm}[t]
    \caption{Auxiliary Function That Checks if a Pre-Activation Feature Creates a New Region}
    \label{alg:feature_check}
    \begin{algorithmic}[1]
        \Function{NewRegionCheck}{\texttt{unit\_features}, \texttt{feature}, \texttt{region}}
            \State \texttt{objective} = \texttt{zeros}
            \State \texttt{inequalites} = \texttt{region.inequalities}
            \State \texttt{unit\_features.weights} = \texttt{unit\_features.weights} $\times$ \texttt{region.weights}
            \State \texttt{unit\_features.biases} = \texttt{unit\_features.weights} $\times$ \texttt{region.biases}
            \State \hspace{22em} + \texttt{unit\_features.biases}
            \For{ \texttt{another\_feature} in \texttt{unit\_features} $\setminus$ \texttt{feature}}
                \State \texttt{inequalities}.append(\texttt{another\_feature.weights} - \texttt{feature.weights} $\times$ $x$
                \State \hspace{15em} $\leq$ \texttt{feature.bias} - \texttt{another\_feature.bias})
            \EndFor
            \If{LinearProgramming.Solve(\texttt{objective}, \texttt{inequalities})}
                \State \texttt{next\_layer\_function} = \texttt{region.next\_layer\_function}
                \State \hspace{15em} + [\texttt{feature.weights}, \texttt{feature.bias}]
                \State \Return Region(\texttt{inequalites}, \texttt{region.function}, \texttt{next\_layer\_function})
            \EndIf
            \State \Return \texttt{None}
        \EndFunction
    \end{algorithmic}
\end{algorithm}

We usually performed the computation in a 2D slice, which is reasonably fast because the number of regions is not large if the input dimension is not high, as suggested by Theorem~\ref{th:main_result}. 
Additionally, note that the check for a given unit is embarrassingly parallel, meaning the computation can be accelerated.
To demonstrate that the computation can be carried out in a reasonable time, we also analyse the algorithm's space-time complexity.

\paragraph{Space-time complexity of the algorithm}\mbox{}

To start, we estimate complexities for some number of activation regions $R$. 
Firstly, consider the space complexity.
Since we store all activation regions, the space requirement grows as $R$ multiplied by an activation region size.
We store a region as a constant size function computed on it and as a system of linear inequalities.
The maximum number of inequalities is attained when each of $N$ neurons adds a new system of inequalities to the region, while $K - 1$ inequalities determine that one pre-activation feature attains a maximum.
Therefore, the space complexity of the algorithm is $\mathcal{O} (R K N)$.

Now consider the time complexity.
Since we traverse the network unit by unit, and for each pre-activation feature of a unit and each available activation region, we solve a linear programming problem, the time complexity is $\mathcal{O}$ of $R K N$ times the time complexity of a linear programming method.
We have used an interior point method that has a polynomial-time complexity of $\mathcal{O}(\frac{n^3}{\log n} L)$, see \citet{anstreicher1999linear}, where $n$ is the dimension of the variables, which is the dimension of the network input $\nin$, and $L$ is the number of bits used to represent the method input.
The input is the set of inequalities, and as we have just discussed, its size is $\mathcal{O}(K N)$.
Combining everything, and using $\mathcal{O}(n^3 L)$ instead of $\mathcal{O}(\frac{n^3}{\log n} L)$ for simplicity, we get that the time complexity of the whole algorithm is $\mathcal{O}(R K^2 N^2 \nin^3 )$.

To get complexities for the average case, assume $N \geq \nin$.
Then, based on Theorem~\ref{th:main_result}, $R$ grows as $\mathcal{O}((K^3 N)^{\nin})$.
Therefore, the space complexity is $\mathcal{O} ( K N (K^3 N)^{\nin})$ and the time complexity is $\mathcal{O}( K^2 N^2 \nin^3 (K^3 N)^{\nin})$.
Both space and time complexities grow exponentially with the input dimension but polynomially with the number of neurons and a maxout unit's rank.

\section{Parameter initialization}
\label{app:init}

\subsection{He initialization}
We briefly recall the parameter initialization procedure for ReLU networks which is commonly referred to as ``He initialization'' \citep{he2015delving}. 
This follows the motivation of the work by Xavier and co-authors \citep{glorot2010understanding}. 
To train deep networks, one would like to avoid vanishing or exploding gradients. 
The approach formulates a sufficient condition for the norms of the activations across layers to not blow up or vanish. 
For ReLU networks this leads to sampling the weights from a distribution with standard deviation $\sqrt{2/n_l}$.

\subsection{He-like initialization for maxout (Maxout-He)}
\label{app:he-like-maxout}
    We follow the derivation from \citet{glorot2010understanding} and \citet{he2015delving} but for the case of maxout units. 
    We note that a He-like initialization for maxouts was considered by \citet{sun2018improving} but only for $K = 2$. 
    We focus on the forward pass and consider fully-connected layers. 
    The idea is to investigate the variance of the responses in each layer. 
    We use the following notations. 
    For a given layer $l$ with $d$ units and $n_l$ inputs, a (pre-activation) response is $\mathbf{y}_l = W_l \mathbf{x}_l + \mathbf{b}_l$, where $\mathbf{x}_l \in \mathbb{R}^{n_l}$ is an input vector to the layer, $W_l \in \mathbb{R}^{d \times n_l}$ is a matrix, $\mathbf{b}_l \in \mathbb{R}^d$ is a vector of biases.
    We have $\mathbf{x}_l = \phi(\mathbf{y}_{l-1})$, where $\phi$ is the activation function. 
    
    We assume the elements in $W_l$ are independent and identically distributed (iid). 
    We assume that the elements in $\mathbf{x}_l$ are also iid. 
    We assume that $\mathbf{x}_l$ and $W_l$ are independent of each other. 
    Denote $y_l$, $w_l$, and $x_l$ the random variables of each element in $\mathbf{y}_l$, $W_l$, and $\mathbf{x}_l$ respectively. 
    In the following we assume that biases are zero. 
    Then we have: 
    \begin{align*}
        \var[y_l] = n_l \var[w_l \cdot x_l].
    \end{align*}
    If we assume further that $w_l$ has zero mean, then the variance of the product of independent variables gives us:  
    \begin{align}
        \var[y_l] = n_l \var[w_l] \mathbb{E}[x_l^2].
        \label{eq:he_init_var}
    \end{align}

    We need to estimate $\mathbb{E}[x^2_l]$. 
    For ReLU, $\mathbb{E}[x_l^2] = \frac{1}{2} \var[y_{l-1}]$. 
    For maxout we get a different result. 
    Let $K$ be the rank of a maxout unit. 
    Then $x_l = \phi(y_{l-1}) = \max_{k \in [K]}\{y_{l-1,k}\}$.
    The $y_{l-1,1},\ldots, y_{l-1,K}$ are independent and have the same distribution.
    We denote $f(t)$ and $F(t)$ the pdf and cdf of this distribution.
    The cdf for $x_l = \max_{k\in[K]}\{y_{l-1,k}\}$ is, dropping the subscript $l-1$ of $y_{l-1,k}$ for simplicity of notation, 
    \begin{align*} 
        \Pr \left(\max_{k \in [K]}\{y_{k}\} < t \right) = \Pr \left( y_1,\ldots, y_K < t \right) = \prod_{k = 1}^K \Pr\left( y_k < t \right)= (F(t))^K. 
    \end{align*}
    In turn, the expectation of the square is 
    \begin{align*}
        \mathbb{E} \left[ \max_{k \in [K]}\{y_k\}^2 \right] = \int_{\mathbb{R}} t^2 \frac{d}{dt} \left[ \left( F(t) \right)^K \right] dt = K \int_{\mathbb{R}} t^2 \left( F(t) \right)^{K - 1} f(t) dt.
    \end{align*}
    
    Now we can apply this formula to discuss the cases of a uniform distribution on an interval and a normal distribution. If we assume that $w_{l - 1}$ has a symmetric distribution around zero, then $y_{l - 1}$ has zero mean and has a symmetric distribution around zero. 
    
    \paragraph{Uniform Distribution}
    Assuming $y_{l-1}$ has a uniform distribution on the interval $[-a, a]$, we get $\var[y_{l-1}] = a^2/3$, and 
    \begin{align*}
        &K = 2: \mathbb{E}[x_l^2] = \frac{a^2}{3} = \var[y_{l-1}],\\
        &K = 3: \mathbb{E}[x_l^2] = \frac{2 a^2}{5} =  \frac{6}{5}\var[y_{l-1}],\\
        &K = 4: \mathbb{E}[x_l^2] = \frac{7 a^2}{15} =  \frac{7}{5}\var[y_{l-1}],\\
        &K = 5: \mathbb{E}[x_l^2] = \frac{11 a^2}{21} =  \frac{11}{7}\var[y_{l-1}].
    \end{align*}
    More generally, $\mathbb{E}[x_l^2] = 4 a^2 (\frac{K}{K+2} - \frac{K}{K+1}+\frac{K}{4K})$. 
    
    \paragraph{Normal Distribution}
    Assuming $y_{l-1}$ has a normal distribution $\mathcal{N} (0, \sigma^2)$, the closed form solution is available for up to $K = 4$. We have:
    \begin{align*}
        &K = 2: \mathbb{E}[x_l^2] = \var[y_{l-1}],\\
        &K = 3: \mathbb{E}[x_l^2] = \frac{\sqrt{3} + 2 \pi}{2 \pi}\var[y_{l-1}],\\
        &K = 4: \mathbb{E}[x_l^2] = \frac{\sqrt{3} + \pi}{\pi}\var[y_{l-1}],\\
        &K = 5: \mathbb{E}[x_l^2] \approx 1.80002 \var[y_{l-1}].
    \end{align*}
    
    Inserting the expressions for $\mathbb{E}[x_l^2]$ into \eqref{eq:he_init_var}, 
    \begin{align*}
        \var[y_l] = n_l \var[w_l] c \var[y_{l-1}],
    \end{align*}
    where $c$ depends on the distribution and on $K$.
    Putting the results together for all layers,
    \begin{align*}
        \var[y_L] = \var[y_1] \prod_{l = 2}^L c n_l \var [w_l].
    \end{align*}
    A sufficient condition for this product not to increase or decrease exponentially in $L$ is that, for each layer, $c n_l \var [w_l] = 1$. This is achieved by setting the standard deviation (std) of $w_l$ as $\sqrt{1/c n_l}$. 
    For $K=2$ this is $\sqrt{1/n_l}$ for both uniform and normal distribution. 
    For a uniform distribution, we obtain the condition $\var[w_l] = \frac{1}{n_l (\frac{1}{4} - \frac{K}{(K+2)(K+1)})}$. 

We notice that for ReLU, the particular shape of the distribution of the (pre-activation) response $y_{l-1}$ does not impact the expected square of the activation $x_l$. Indeed, as soon as $w_l$ is assumed to be symmetric around zero, one obtains $\mathbb{E}[x_l^2] = \frac12 \var[y_{l-1}]$. 
In contrast, for maxout units of rank $K>2$, the particular shape of the distribution of $y_{l-1}$ does affect the value of $\mathbb{E}[x_l^2]$. This is why we obtain different conditions on the standard deviation of the weight distributions depending on the assumed response distribution. 
The situation is illustrated in Figure~\ref{fig:init}. 
    Among the possible distributions that one might assume for $y_{l-1}$, a normal distribution appears most natural. Therefore, we take the standard deviations obtained under this assumption as the ones defining the maxout-He initialization procedure. 
    The values of the std of $w_l$ for $K$ up to $5$ for normal distributions are shown in Table \ref{tb:std}.

\begin{figure}
    \centering
    \includegraphics[width=10cm]{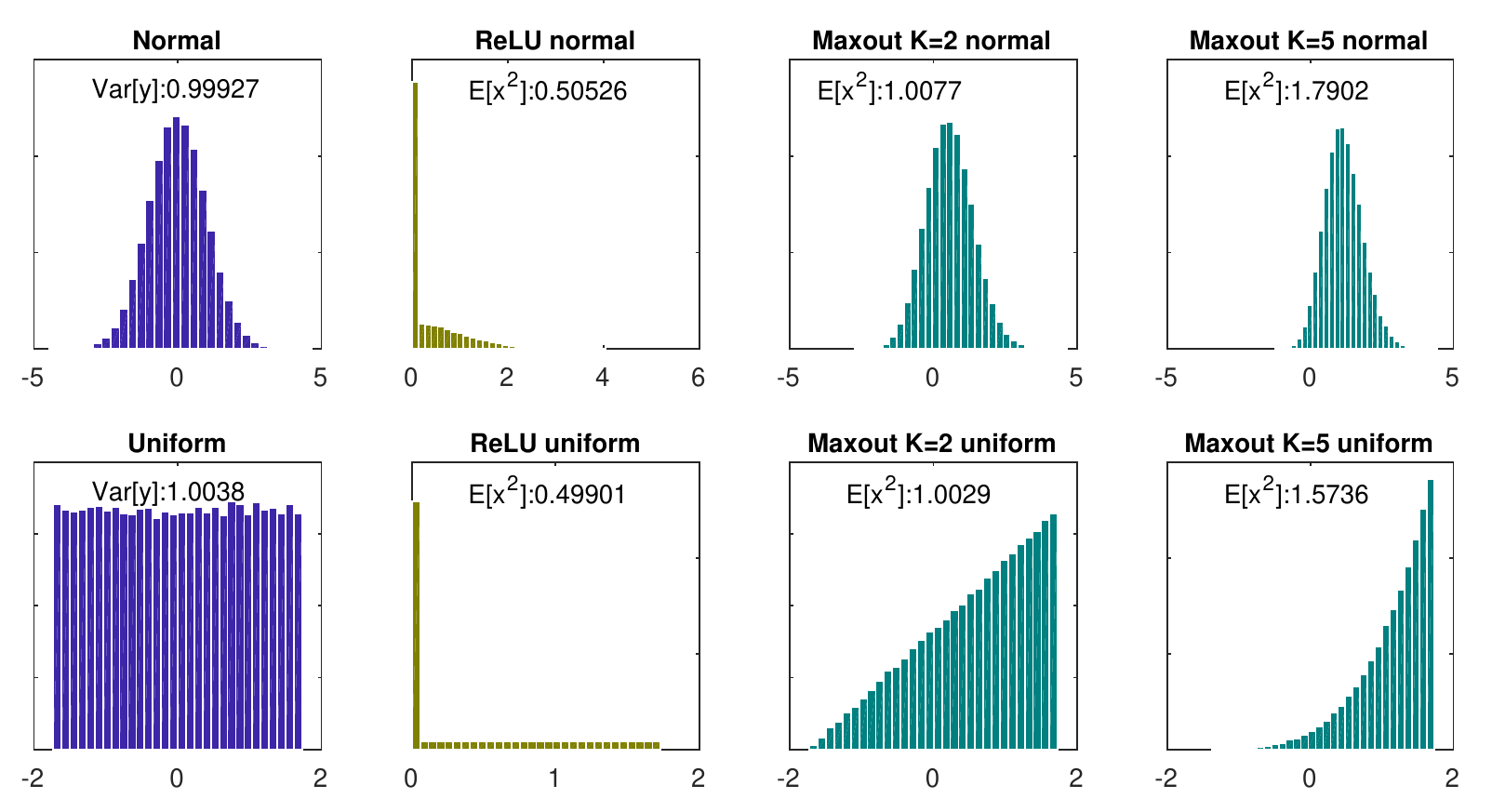}
    \caption{Shown are normal (top) and uniform (bottom) input distributions, as well as the corresponding response distributions for ReLU, maxout of rank $K=2$, and maxout of rank $K=5$. The expectation of the square response for maxouts of rank $K>2$ depends not only on the variance but also on the particular shape of the input distribution. }
    \label{fig:init}
\end{figure}
    
  \subsection{Sphere initialization}
  If we initialize the pre-activation features of a maxout unit independently, then we expect the number of regions of the unit will be significantly smaller than $K$, as discussed in Appendix~\ref{app:single_unit}.
  In view of Proposition~\ref{prop:Newton}, the number of regions of a maxout unit with weights $w_1,\ldots, w_K\in\mathbb{R}^n$ and biases $b_1,\ldots, b_K\in\mathbb{R}$ is equal to the number of upper vertices of the polytope $\operatorname{conv}\{(w_r,b_r)\colon r\in[K]\}$. 
  Hence one way to have each rank-$K$ maxout unit have $K$ linear regions over its input at initialization is to initialize the pre-activation feature parameters as points in the upper half-sphere $\{(w,b)\in\mathbb{R}^{n+1} \colon \|(w,b)\|=1, b>0\}$. 
  This can be done as follows.  
  For each pre-activation feature $i=1,\ldots, K$:  
  \begin{enumerate}[leftmargin=*]
    \item Sample $(w_i,b_i)$ from a Gaussian on $\mathbb{R}^{n+1}$. 
    \item Normalize $(w_i,b_i)/\|(w_i,b_i)\|$. 
    \item Replace $b_i$ with $|b_i|$.
  \end{enumerate}
  If desired, subtract a constant $c$ from each of the biases $b_1,\ldots, b_K$. 
  For instance one may choose $c$ so that the mean output of the maxout unit is approximately $0$ for inputs from a Gaussian distribution.
  We have used $c = 1 / \sqrt{K n_l}$ in our implementation, and Gaussian had zero mean and unit covariance. 
  
  \subsection{Many regions initialization}
  We can initialize the parameters of a maxout layer so that the layer has the largest possible number of linear regions over its input space. 
  A description of parameter choices maximizing the number of regions for a layer of maxout units has been given by \citet[Proposition~3.4]{sharp2021}. The number of regions of a layer of maxout units corresponds to the number of upper vertices of a Minkowski sum of polytopes. A construction maximizing the number of vertices of Minkowski sums was presented earlier by \cite{Weibel12}. 
 The procedure is as follows. Let the layer have input dimension $n$. 
 For each unit $j=1,\ldots, m$:  
  \begin{enumerate}[leftmargin=*]
      \item Sample a vector $v_j\in\mathbb{R}^{n}$ from a distribution which has a density. 
      \item For each pre-activation feature $i=1,\ldots, K$ set the weights and bias as $w_{j,i} = v_j  \cos(\pi i/K)$ and $b_{j,i} = \sin(\pi i/K)$. 
  \end{enumerate} 
  This construction ensures that each unit has $K$ linear regions separated by $K-1$ parallel hyperplanes, and the hyperplanes of different units are in general position. Then the number of regions of the layer is the one indicated in the first item of Theorem~\ref{thm:upper_bound}. 
 
 If desired, one can add some noise to each of the above parameters (e.g.\ standard normal times a small constant) in order to have a parameter distribution which has a density. 
 If desired, one can also normalize the initialization by subtracting an appropriate constant (e.g.\ to achieve a zero mean activation) and dividing by an appropriate standard deviation (e.g.\ to achieve that the activations have a unit mean norm). 
We were sampling $v_j$ from a Gaussian distribution with mean zero and std chosen according to maxout-He.
  
\subsection{Steinwart-like initialization for maxout}
\citet{steinwart2019sober} investigated initialization in ReLU networks. 
He suggested that having the nonlinear locus of different units evenly spaced over the input space at initialization could lead to faster convergence of training, which he also supported with experiments on the datasets from the UCI repository. 
We can formulate a version of this general idea for the case of maxout networks as follows. 
\begin{enumerate}[leftmargin=*]
\item 
Assume we have some generic initialization procedure for individual units, which gives us weights $w_1,\ldots,  w_K\in \mathbb{R}^n$ and biases $b_1,\ldots, b_K\in\mathbb{R}$. 
The initialization procedure could be for instance ``Sphere''. 
Upon initialization, our unit is computing a function $x\mapsto \max\{\langle w_1,x\rangle +b_1,\ldots, \langle w_K , x\rangle +b_K\}$ with non-linear locus that we denote $L$. 

\item 
For each unit, sample a vector $c$ uniformly from the cube $[-1,1]^{n}$.
Alternatively, sample $c$ as a random convex combination of a random subset of the training data, so that $c=\sum_{i=1}^m p_i x_i$, where $(p_1,\ldots, p_m)$ is a random probability vector and $x_1,\ldots, x_m$ are $m$ randomly selected training input examples. 
\item Now set the weights as $w_1,\ldots, w_K$ and the biases as $b_1 + \langle w_1,c\rangle, \ldots, b_K + \langle w_K,c\rangle$. 
Now our unit is computing a function $x\mapsto \max\{\langle w_k, x\rangle + b_k + \langle w_k, c\rangle \} = \max\{\langle w_k, x + c\rangle + b_k\}$. Hence the nonlinear becomes $L-c$. 
\end{enumerate}

\section{Experiments}
\label{app:experiments}

In this section, we provide details on the implementation and additional experimental results. 
All the experiments were implemented in Python using PyTorch \citep{NEURIPS2019_9015}, numpy \citep{harris2020array}, scipy \citep{scipy} and mpi4py \citep{dalcin2011parallel}, with plots created using matplotlib \citep{Hunter:2007}. 
In the experiments concerning the network training, we used the MNIST dataset \citep{lecun2010mnist}.
PyTorch, numpy, scipy and mpi4py are made available under the BSD license, matplotlib under the PSF license and MNIST dataset under the Creative Commons Attribution-Share Alike 3.0 license.
We conducted all experiments on a CPU cluster that uses Intel Xeon IceLake-SP processors (Platinum 8360Y) with 72 cores per node and 256 GB RAM. The most extensive experiments were usually running for 2-3 days on 32 nodes.
The computer implementation of the key functions is available on GitHub at \url{https://github.com/hanna-tseran/maxout_complexity}. 

For the MNIST experiments we use the Adam optimizer with mini-batches of size $128$ with learning rate $0.001$ and the standard Adam hyperparameters from PyTorch (betas are $0.9$ and $0.999$).
Counting at initialization was performed in the window $[-50, 50]^2$, in the training experiments in the window $[-400, 400]^2$ defined on the slice, and images of the regions and the decision boundary were obtained in the window $[-50, 50]^2$ also defined on the slice.
All results are averaged over 30 instances where applicable. 
The network architectures are specified in the individual experiments.
The parameter initialization procedures are implemented
following the descriptions in Appendix~\ref{app:init}. 
For the experiments counting the number of activation regions and pieces in the decision boundary we use home made implementations of the algorithms described in Appendix \ref{app:algorithm}. 
Further below we present the details and additional results of the individual experiments. 

\paragraph{Details on Figure~\ref{fig:1}} 
We consider a network with 2 input units, three layers of rank-3 maxout units of width 3, and a single linear output unit. 
We fix three parameter vectors $\theta_0,\theta_1,\theta_2$ drawn from a normal distribution over the parameter space and define a grid of parameter values $\theta(\xi_1,\xi_2) = \theta_0 + \xi_1\theta_1 + \xi_2\theta_2$ with $(\xi_1,\xi_2)$ taking 102400 uniformly spaced values in $[-1,1]^2$. 
For each of these parameter values, we estimate the number of linear regions that the represented function has over the square $[-1,1]^2$ in the input space. 
To this end, we evaluate the gradient of the function over 102400 uniformly spaced input points and take the number of distinct values an estimate for the number of linear regions.
Then we plot the estimated number of linear regions as a function of $(\xi_1,\xi_2)$. 
A subset of 25 out of the evaluated functions is shown in Figure~\ref{fig:1a}. 
    
\paragraph{Comparison to the upper bound} 
Figures \ref{fig:full_formula} and  \ref{fig:formula_different_K} complement Figure \ref{fig:formula}. Figure \ref{fig:full_formula} compares the number of activation regions and linear pieces in the decision boundary to the formulas both with and without the constants, while Figure \ref{fig:formula_different_K} demonstrates the results for different values of $K$. 

\paragraph{Effects of the depth and the number of units on the number of linear regions}
Results adding more information to Figure \ref{fig:neurons_to_depth} are in Figure \ref{fig:neurons_to_depth2}. It shows that ReLU networks and maxout networks with $K = 2$ have a similar number of activation regions that does not depend on the network depth but rather on the total number of units. This figure also shows that maxout networks with ranks $K > 2$ tend to have fewer regions as the depth increases, but the number of units remains constant and that the difference in the number of regions becomes more apparent for larger ranks. 

\paragraph{Effects of different initializations on training}
Figure \ref{fig:training_dif_params} is a more detailed version of Figure \ref{fig:loss}. It shows how convergence speed changes for different network depths and different maxout ranks given different initializations. 
The improvement from maxout-He, sphere, and many regions initializations compared to ReLU-He initialization becomes more noticeable with larger network depth and larger maxout rank. 
We have also checked how the Steinwart initialization affects the convergence speed, but found no significant difference in this particular experiment. We used the approach where $c$ is taken as a convex combination of all training data points (weights $p$ uniformly at random from the probability simplex). The results are shown in Figure~\ref{fig:steinwart}. 
    
\paragraph{Effects of different initializations on the number of activation regions and pieces in the decision boundary during training}
Figure \ref{fig:init_in_training} adds more information to Figure \ref{fig:training} and demonstrates how the number of activation regions and linear pieces in the decision boundary changes for different initializations during training on the MNIST dataset. 
We observe that the number of activation regions and pieces of the decision boundary increase for all tested initialization procedures as training progresses. Nonetheless, the number remains much lower than the theoretical maximum. 
Figure \ref{fig:db_regions_evolution} illustrates how linear regions and the decision boundary evolve during training. 

\begin{figure}
    \centering
    \includegraphics[width=5cm]{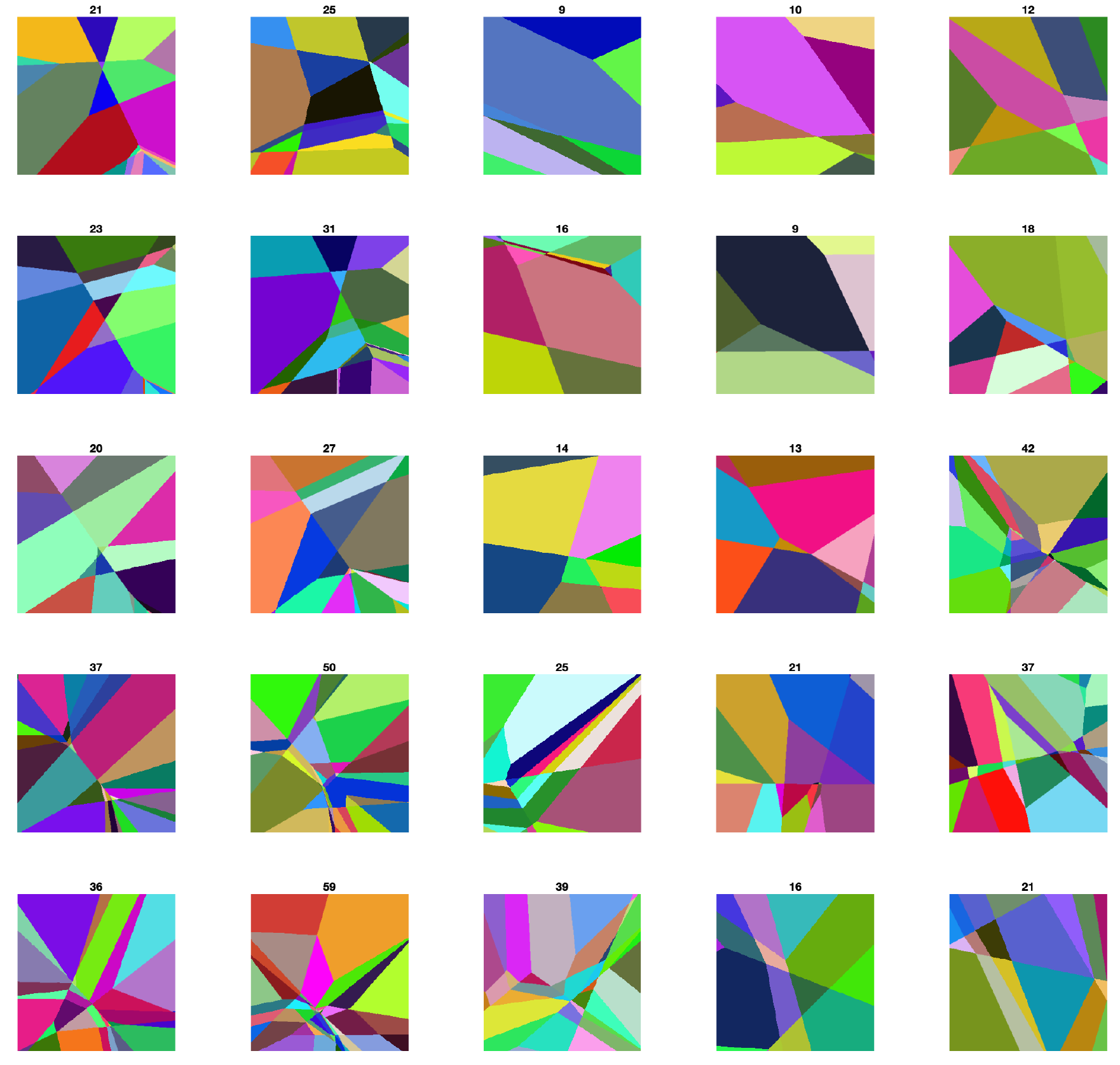}
    \caption{A few functions represented by a maxout network for different parameter values in a 2D slice of parameter space. 
    For each function we plot regions of the input space with different gradient values using different colors. 
    }
    \label{fig:1a}
\end{figure}

\begin{figure}
    \begin{subfigure}{\textwidth}
        \setlength\tabcolsep{2pt}
        \begin{tabular}{ccc}
            \centering
            \small{Formula without $K$ and constants} &
            \small{Formula with $K$ without constants} &
            \small{Full formula} \\
            
            \includegraphics[width=0.32\textwidth]{images/formula/78186f13e936442abc08d02c3e9879c12.png} &
            \includegraphics[width=0.32\textwidth]{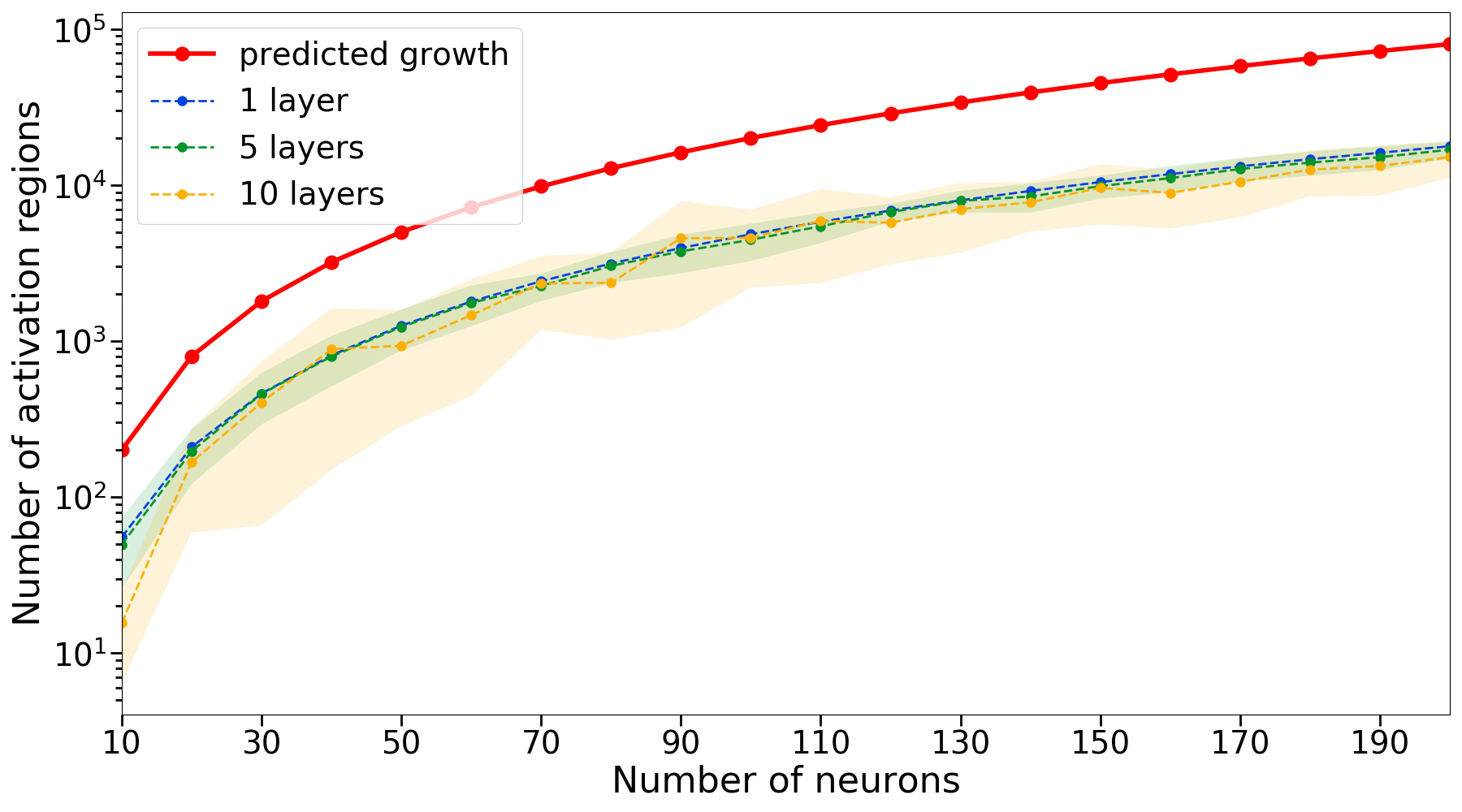} &
            \includegraphics[width=0.32\textwidth]{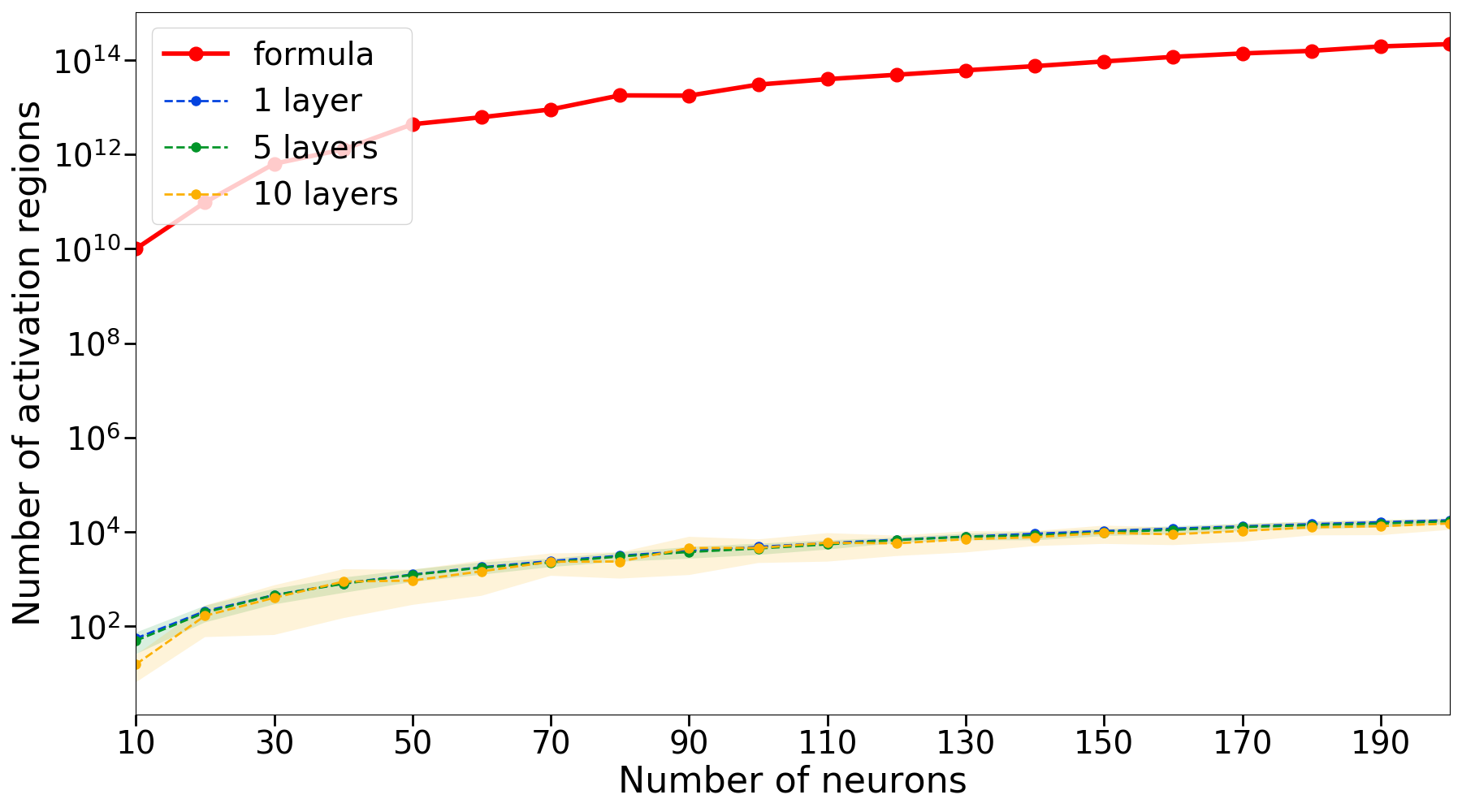}
        \end{tabular}
        \caption{\small Number of activation regions for a network with ReLU-He normal initialization.}
    \end{subfigure}
    \vspace{.2cm}
    
    \begin{subfigure}{\textwidth}
        \setlength\tabcolsep{2pt}
        \begin{tabular}{ccc}
            \centering
            \small{Formula without $K$ and constants} &
            \small{Formula with $K$ without constants} &
            \small{Full formula} \\
            \includegraphics[width=0.32\textwidth]{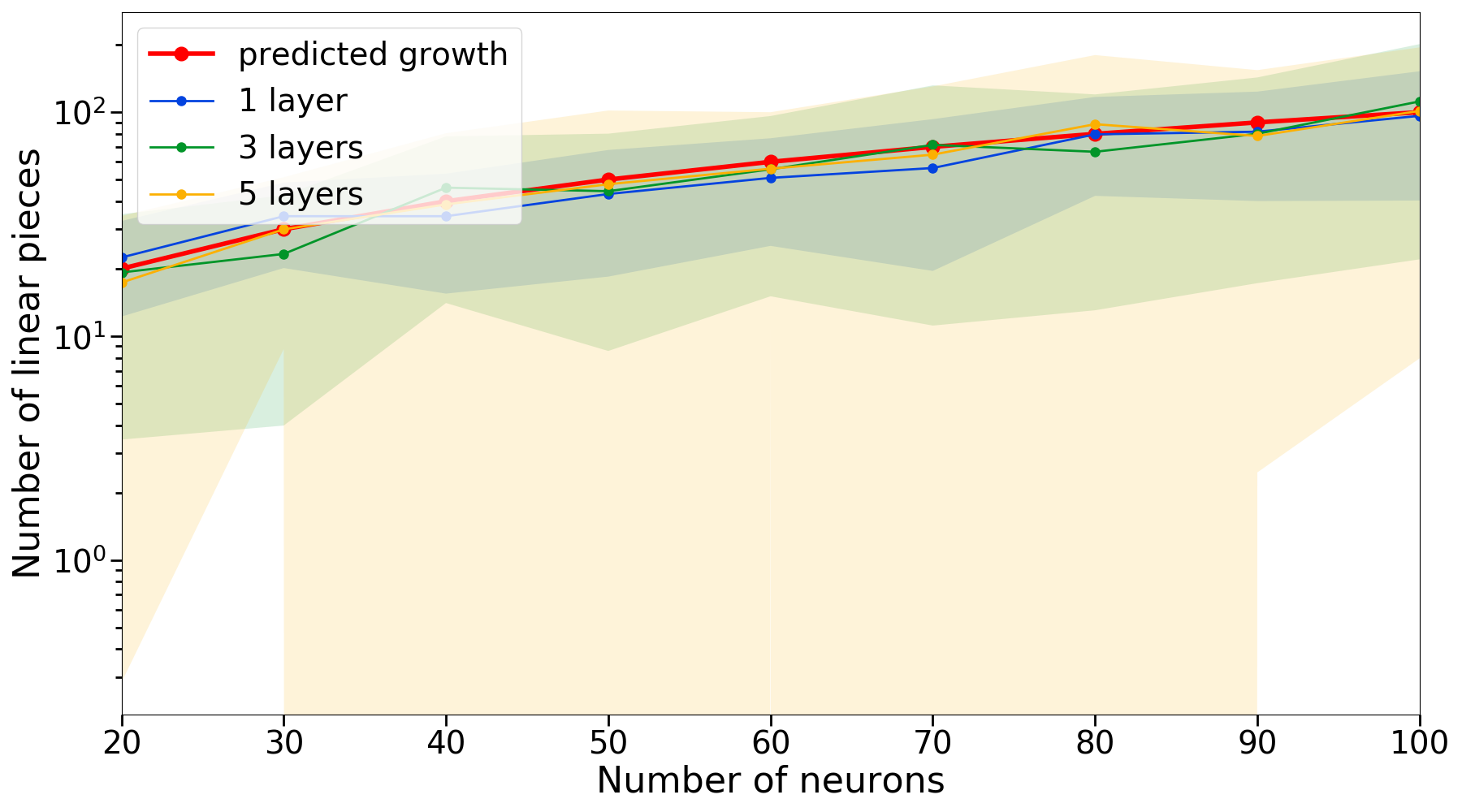}&
            \includegraphics[width=0.32\textwidth]{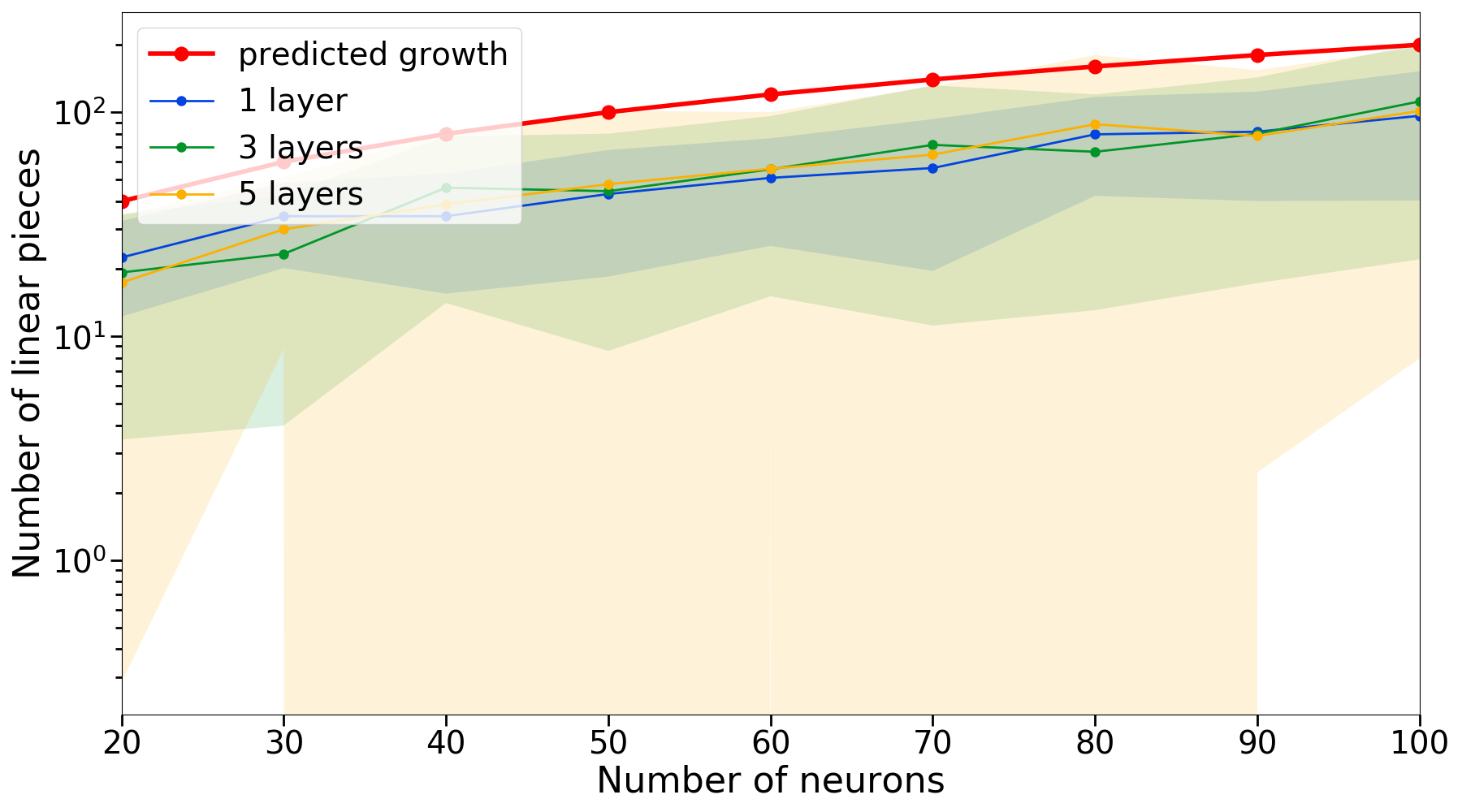} &
            \includegraphics[width=0.32\textwidth]{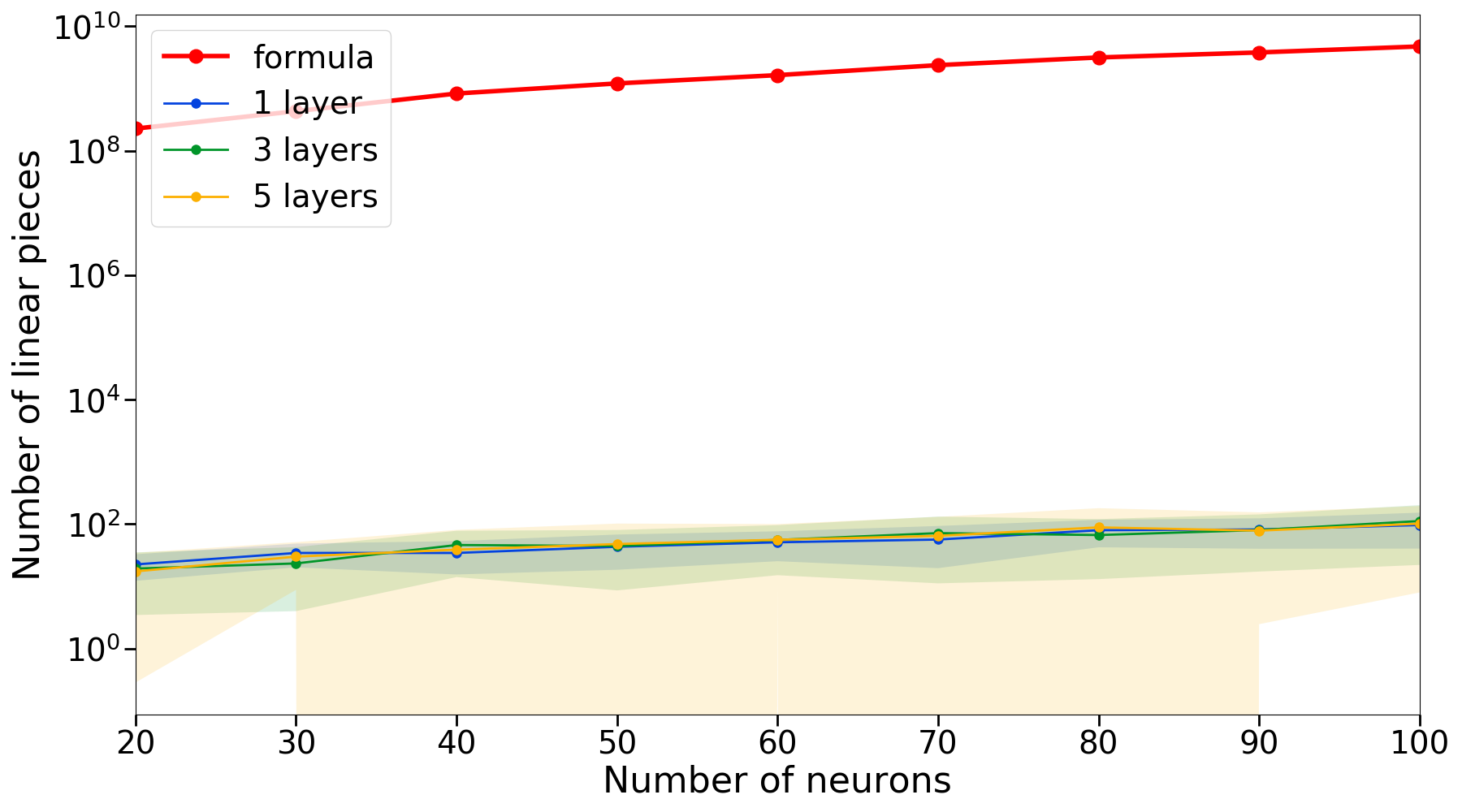}
        \end{tabular}
        \caption{\small Number of linear pieces in the decision boundary for a network with maxout-He normal initialization.}
    \end{subfigure}
    \caption{Comparison to the formulas with and without the constants for the number of activation regions and linear pieces in the decision boundary from Theorem \ref{th:main_result} and Theorem \ref{th:decision_boundary} respectively. Networks had $100$ units and maxout rank $K = 2$. The settings are similar to those in Figure~\ref{fig:formula}.}
    \label{fig:full_formula} 
\end{figure}

\begin{figure}
    \begin{subfigure}{\textwidth}
        \setlength\tabcolsep{2pt}
        \begin{tabular}{ccc}
            \centering
            \small{Formula without $K$ and constants} &
            \small{Formula with $K$ without constants} &
            \small{Full formula} \\
            
            \includegraphics[width=0.32\textwidth]{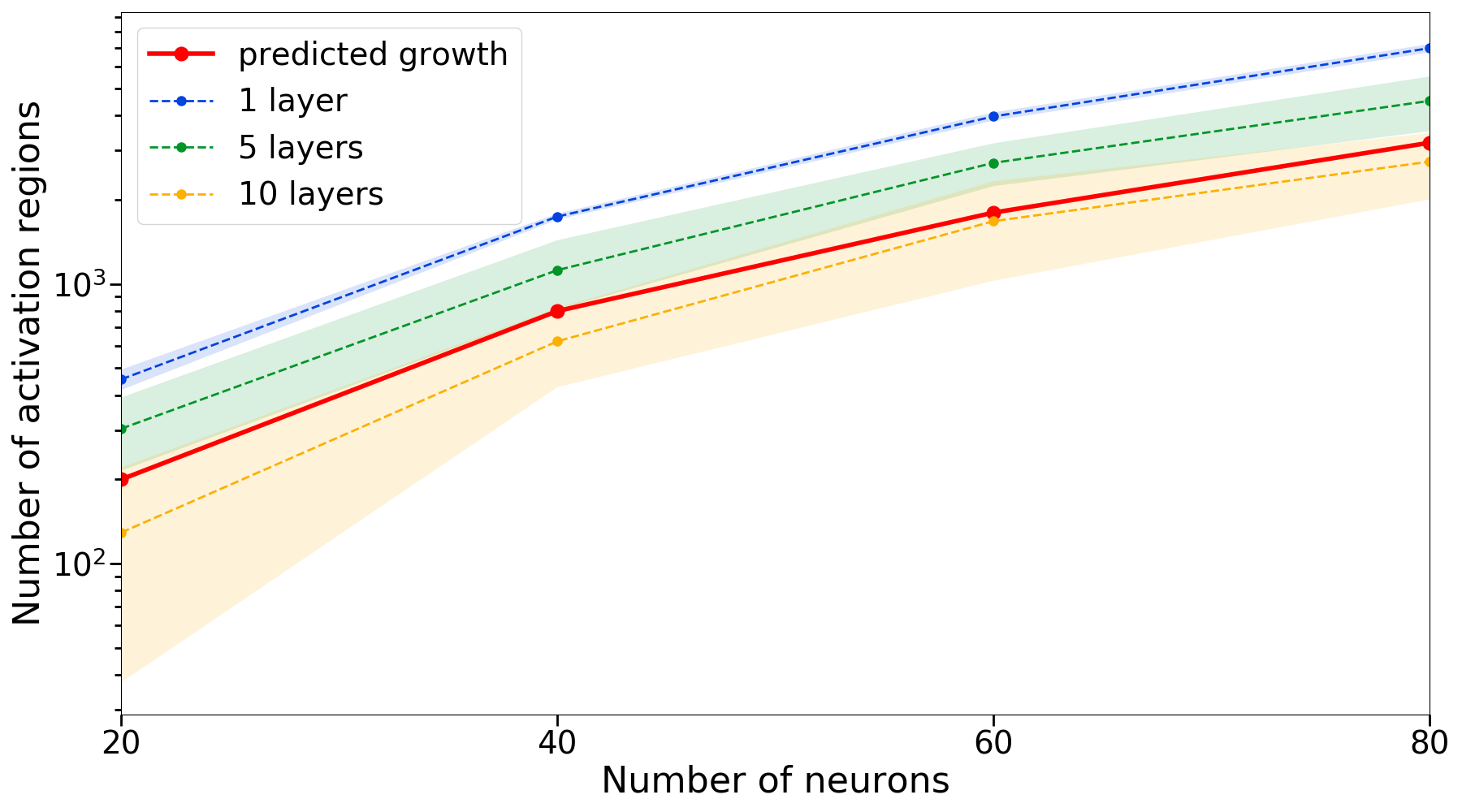} &
            \includegraphics[width=0.32\textwidth]{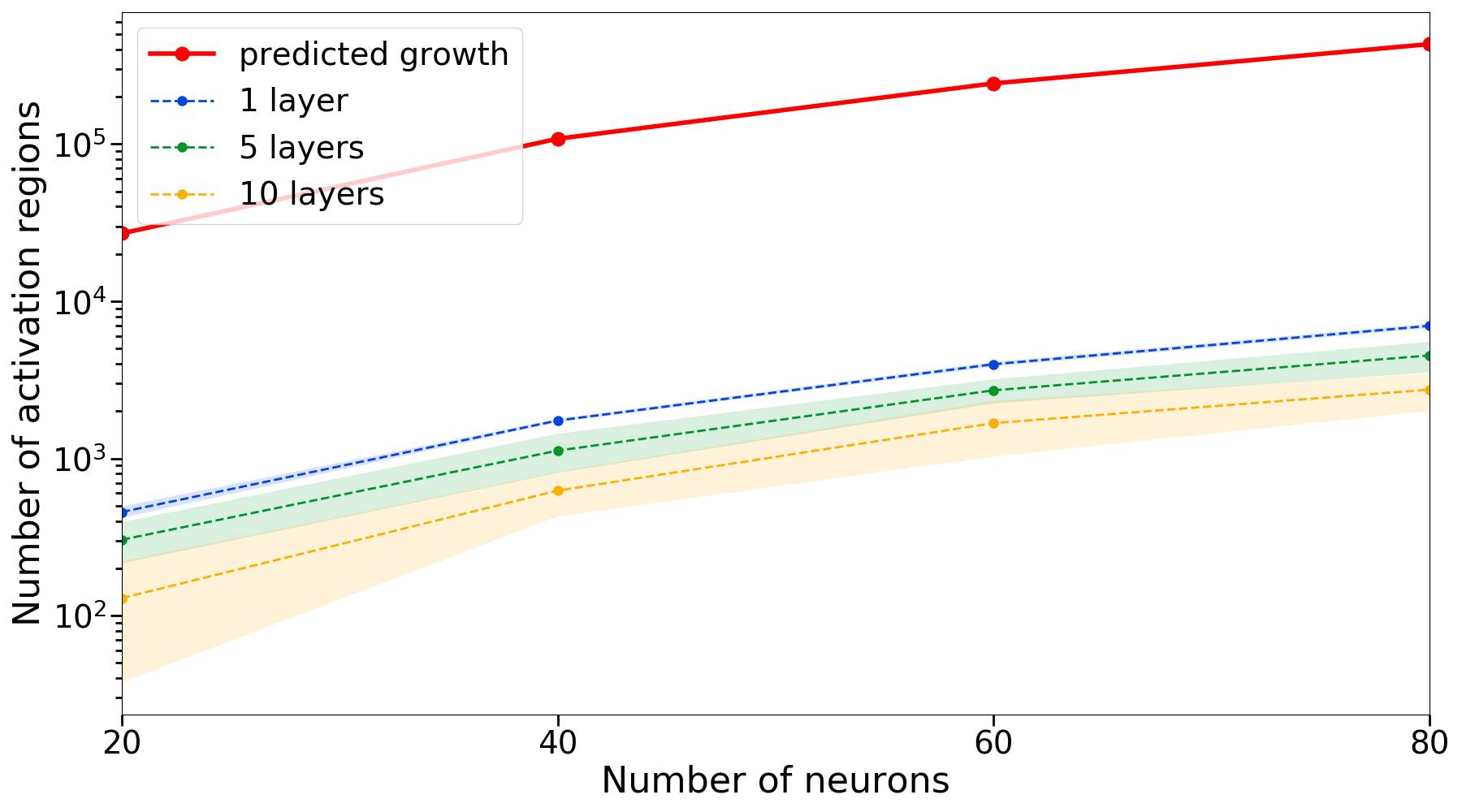} &\includegraphics[width=0.32\textwidth]{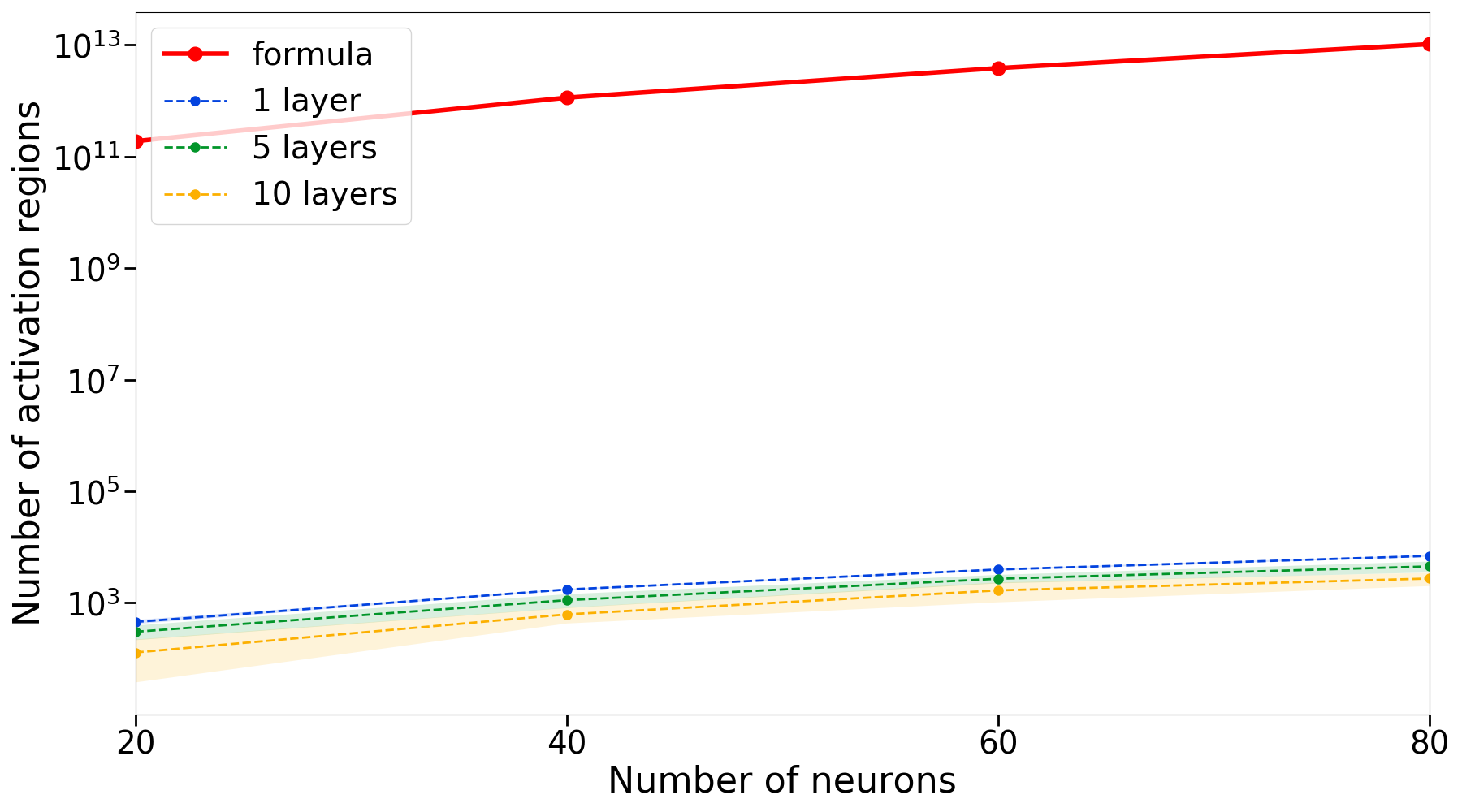}
        \end{tabular}
        \caption{\small $K = 3$.}
    \end{subfigure}
    \vspace{.2cm}
    
    \begin{subfigure}{\textwidth}
        \setlength\tabcolsep{2pt}
        \begin{tabular}{ccc}
            \centering
            \small{Formula without $K$ and constants} &
            \small{Formula with $K$ without constants} &
            \small{Full formula} \\
            \includegraphics[width=0.32\textwidth]{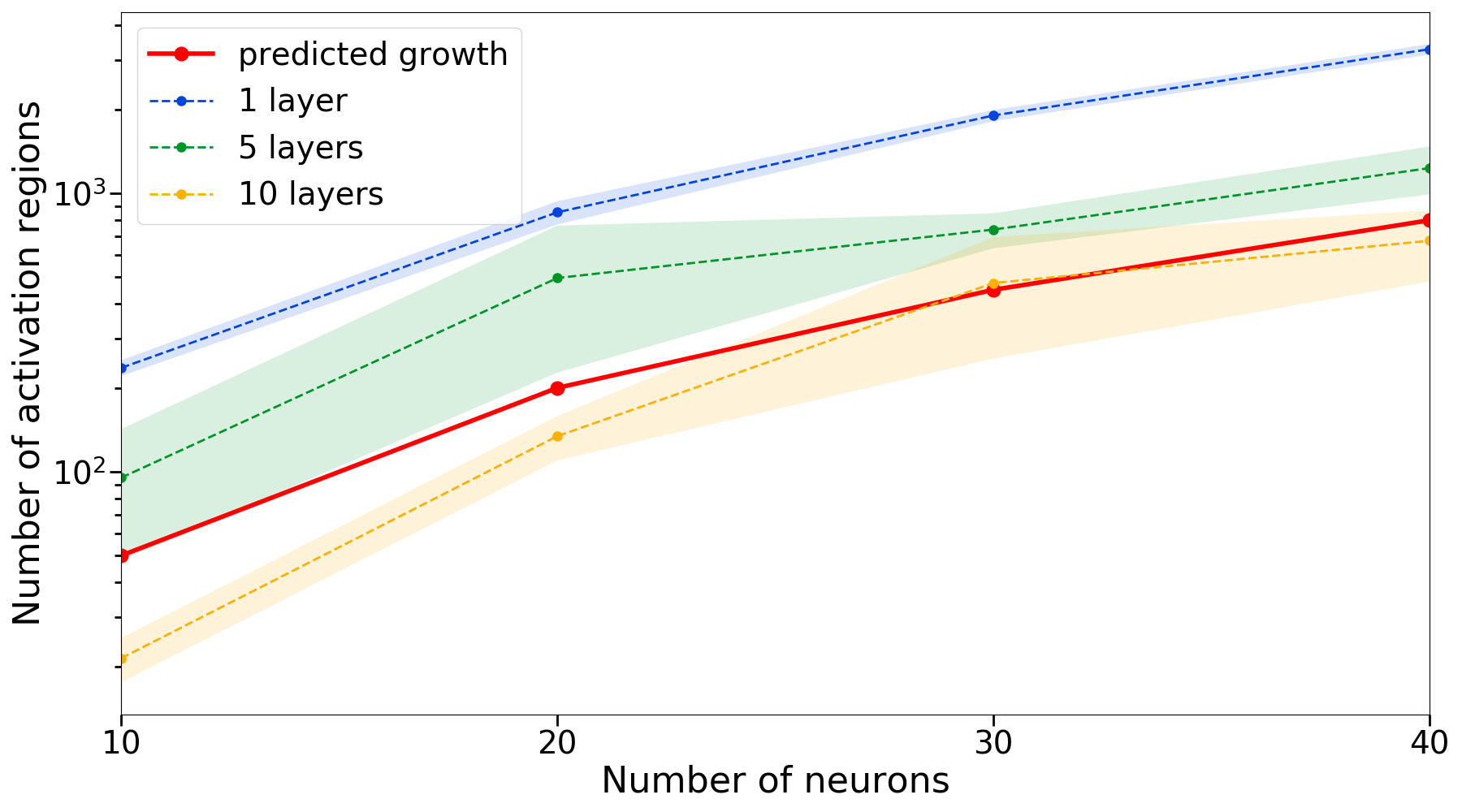} &
            \includegraphics[width=0.32\textwidth]{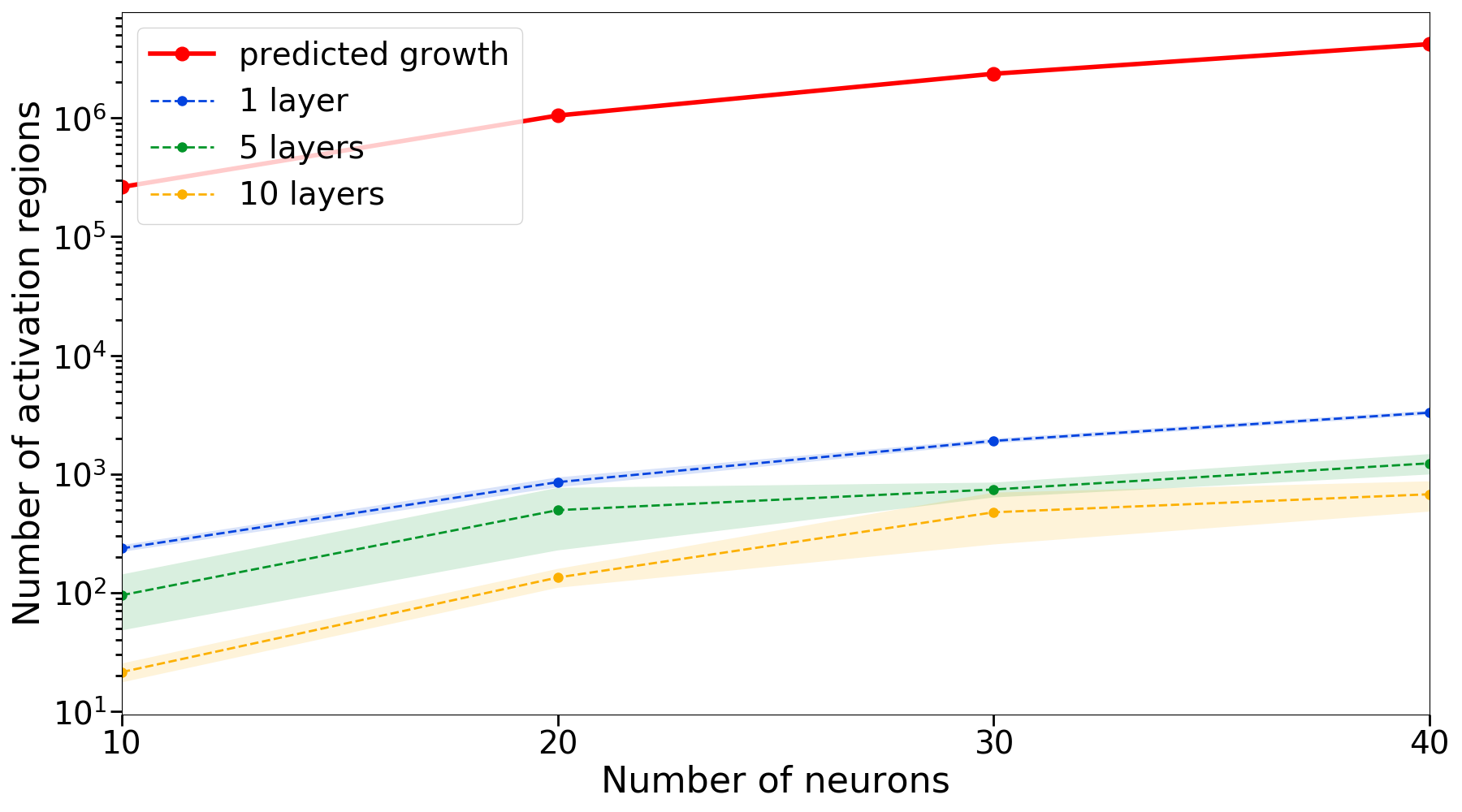} &
            \includegraphics[width=0.32\textwidth]{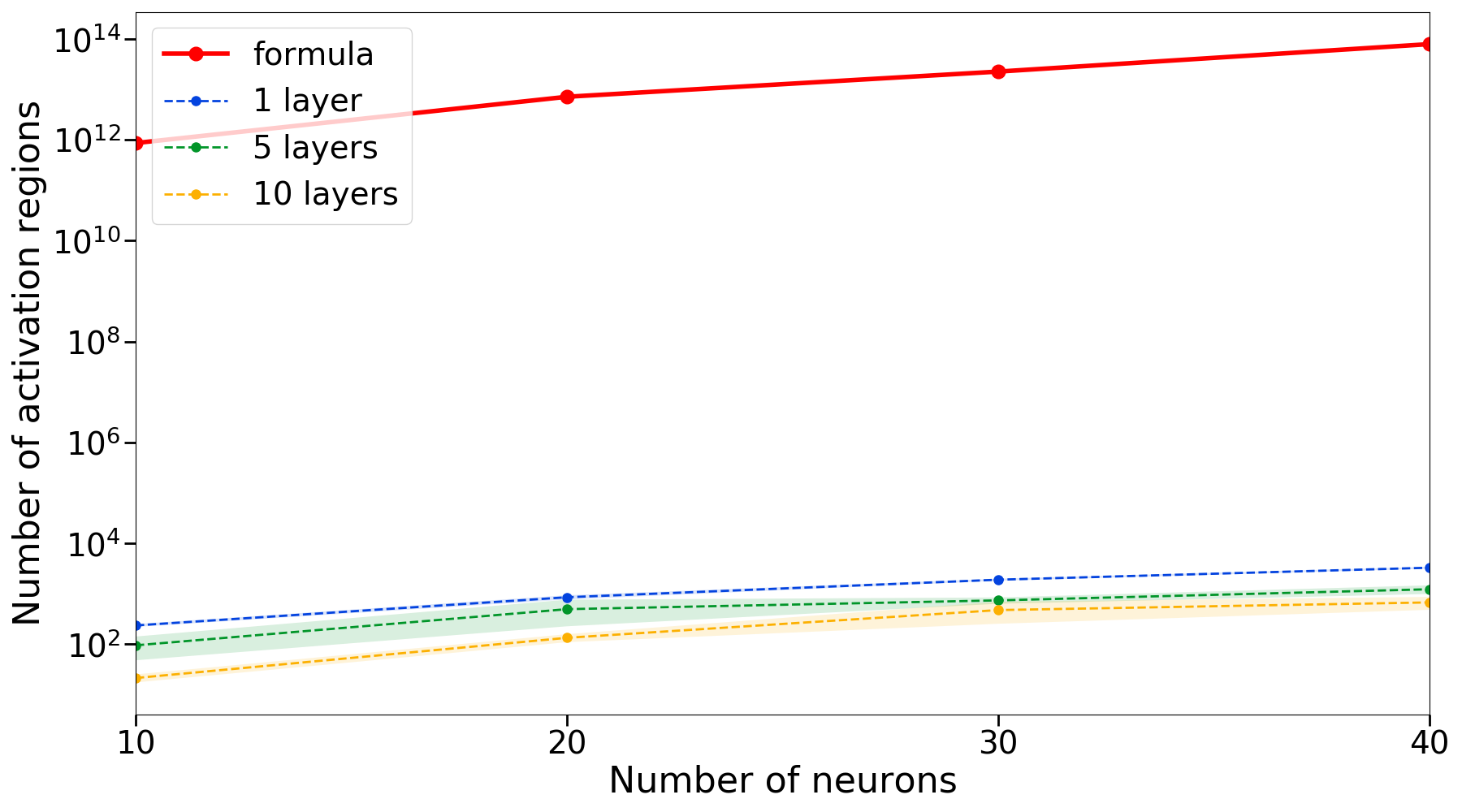}
        \end{tabular}
        \caption{\small $K = 5$.}
    \end{subfigure}
    
    \caption{Comparison to the formula from Theorem~\ref{th:main_result} for maxout ranks $K = 3$ and $K = 5$. The networks were initialized with maxout-He normal initialization.
    We observe the increase in the number of activation regions as the maxout rank increases, and for networks with higher maxout rank deeper networks tend to have less regions than less deep networks with the same rank.}
    \label{fig:formula_different_K} 
\end{figure}

\begin{figure}
    \centering
    \begin{subfigure}{0.46\textwidth}
        \centering
        \includegraphics[trim=10 10 10 10, clip, width=.46\textwidth] {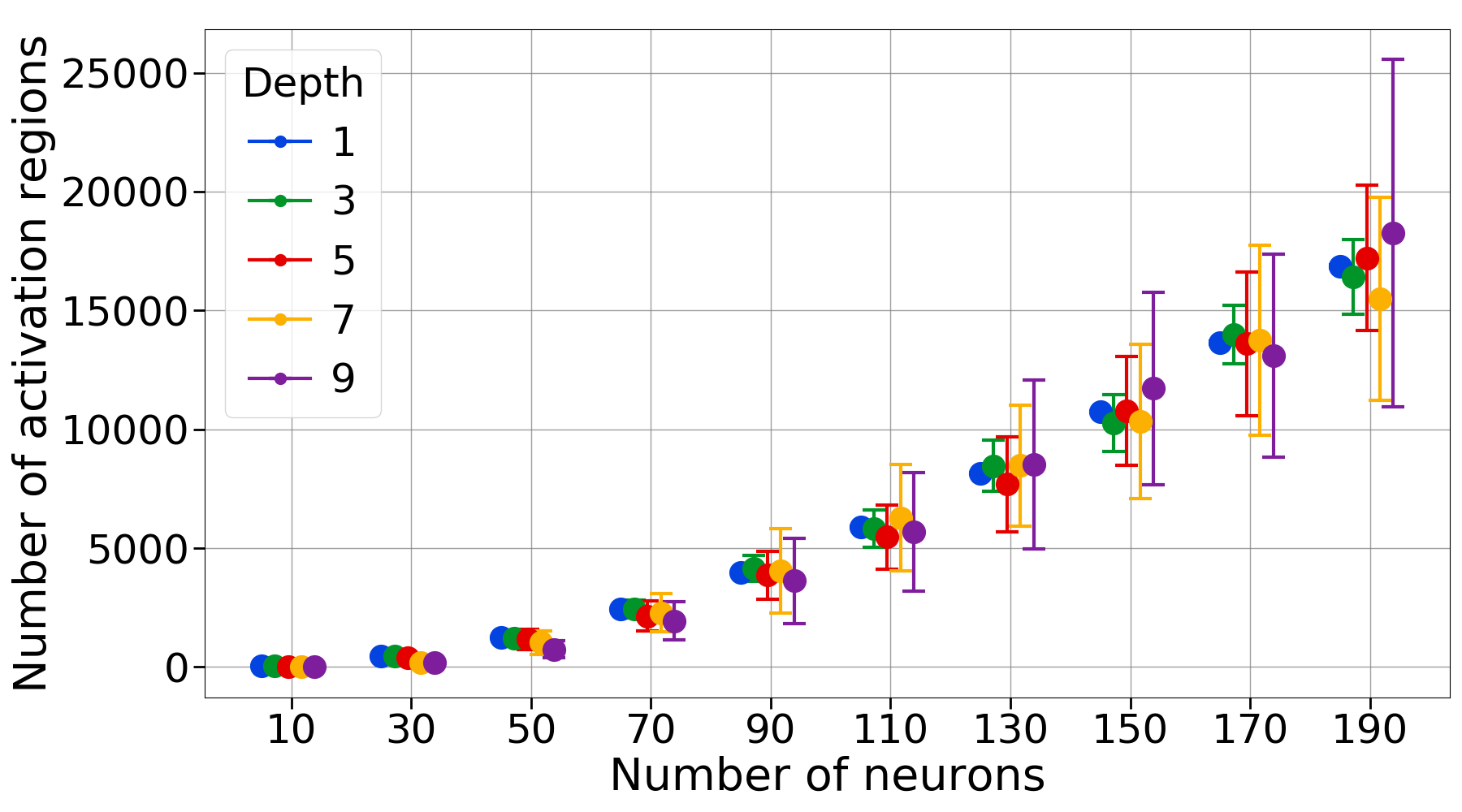}
        \includegraphics[trim=10 10 10 10, clip, width=.52\textwidth] {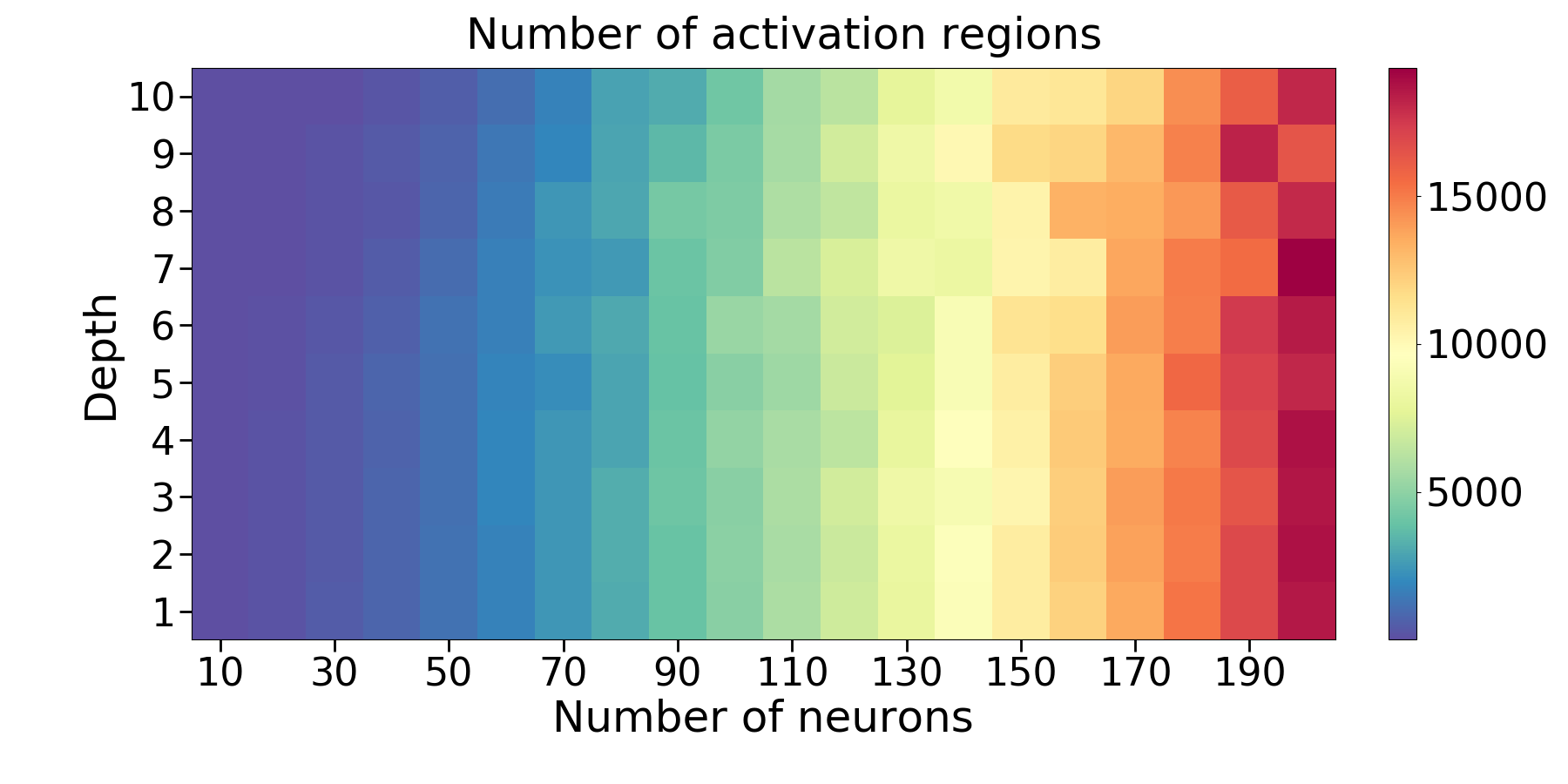}
        \caption{ReLU network with ReLU-He normal initialization.}
    \end{subfigure}
    \hspace{10pt}
    \begin{subfigure}{0.46\textwidth}
        \centering
        \includegraphics[trim=10 10 10 10, clip, width=.46\textwidth] {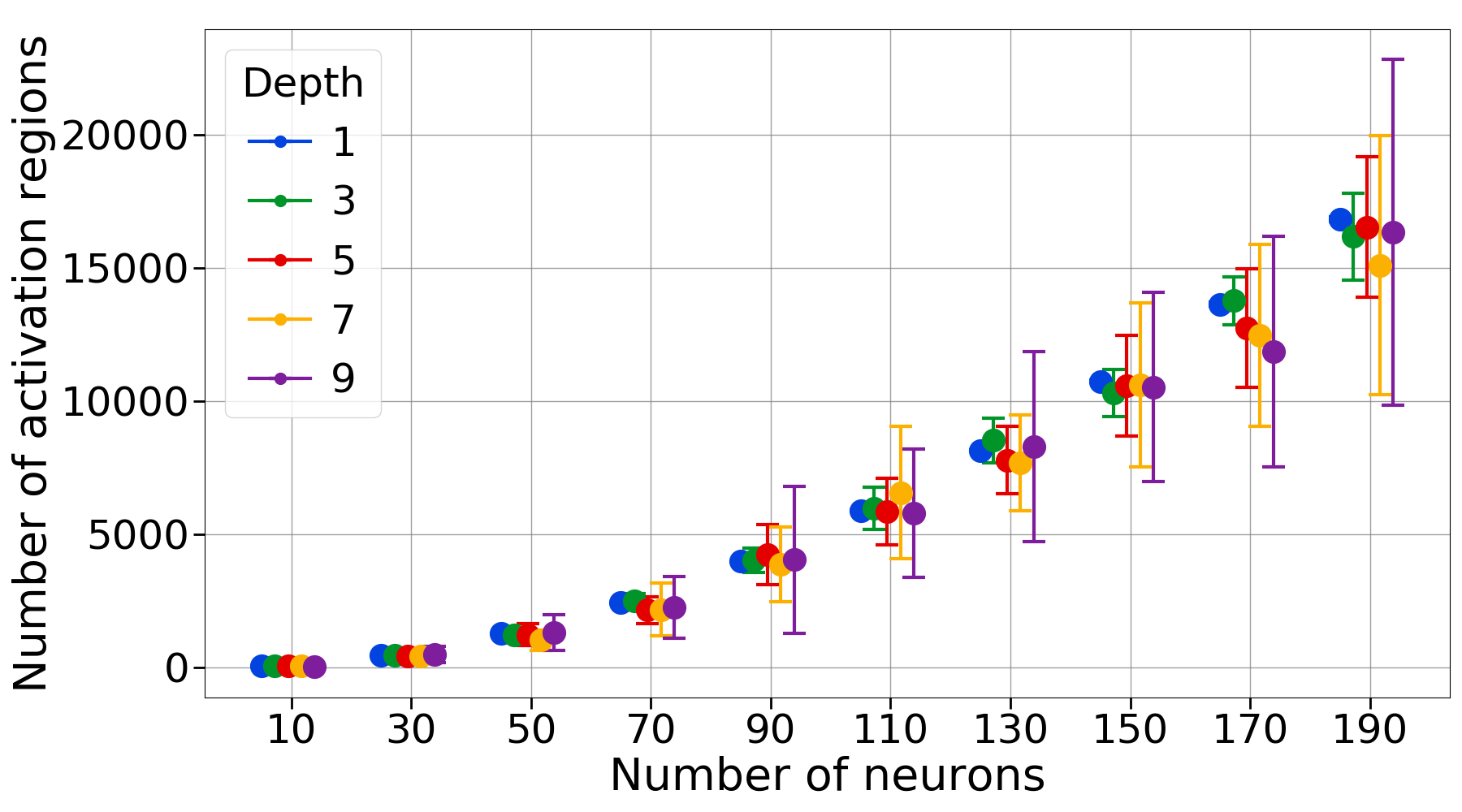}
        \includegraphics[trim=10 10 10 10, clip, width=.52\textwidth] {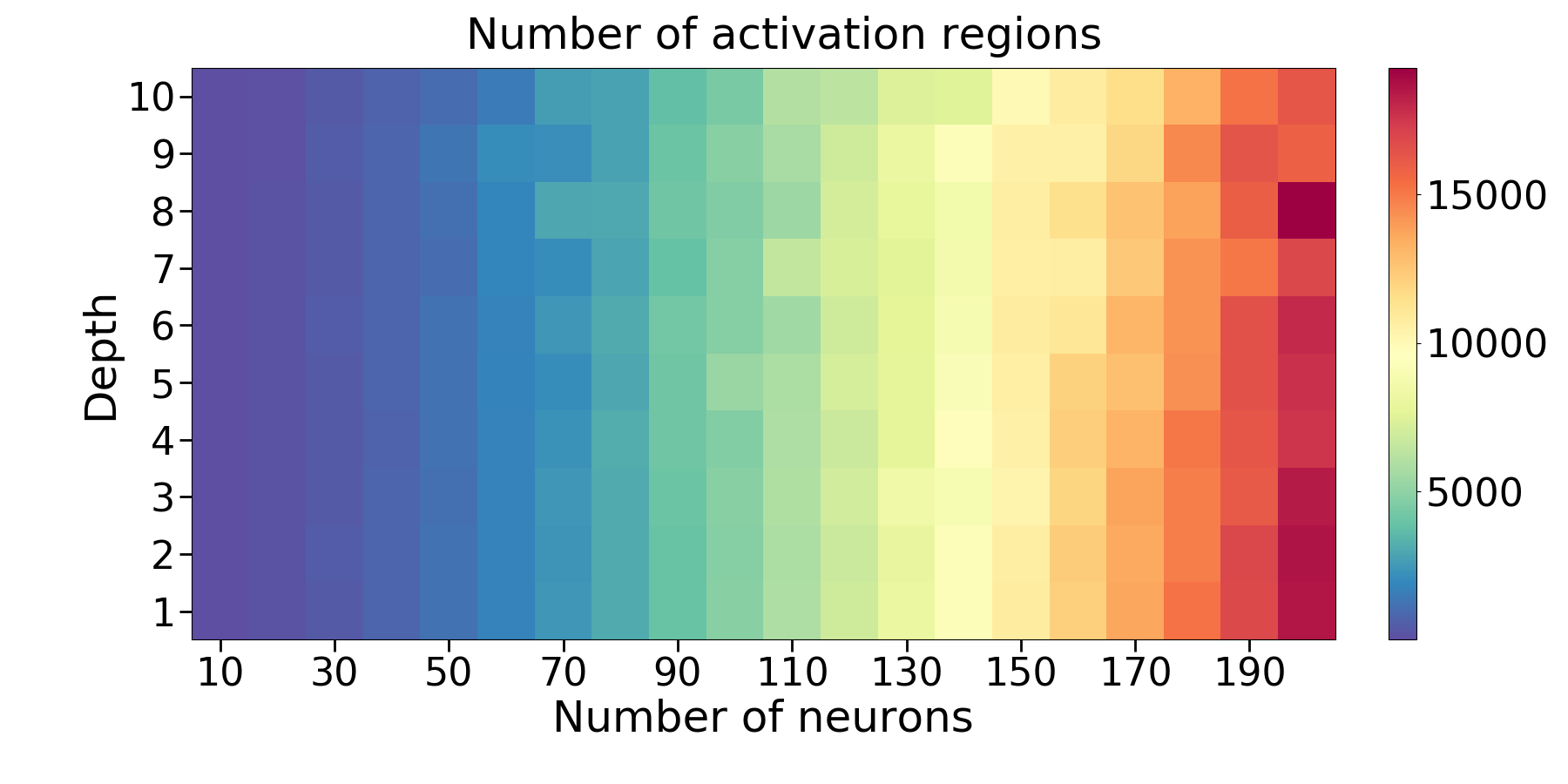}
        \caption{Maxout network with maxout rank $K = 2$ and ReLU-He normal initialization.}
    \end{subfigure}
    
    \begin{subfigure}{0.46\textwidth}
        \centering
        \includegraphics[trim=10 10 10 20, clip, width=.44\textwidth] {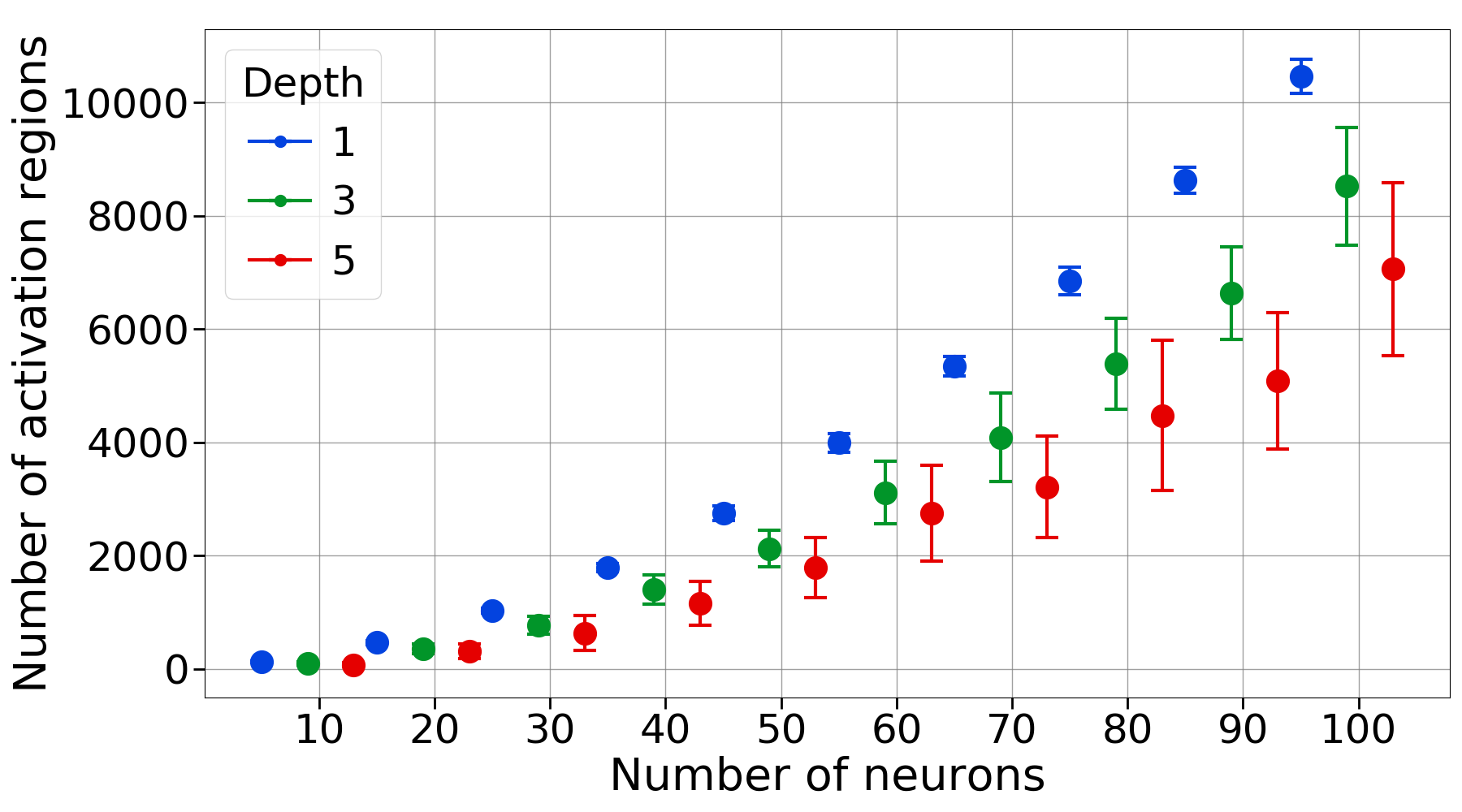}
        \includegraphics[trim=30 40 15 10, clip, width=.48\textwidth] {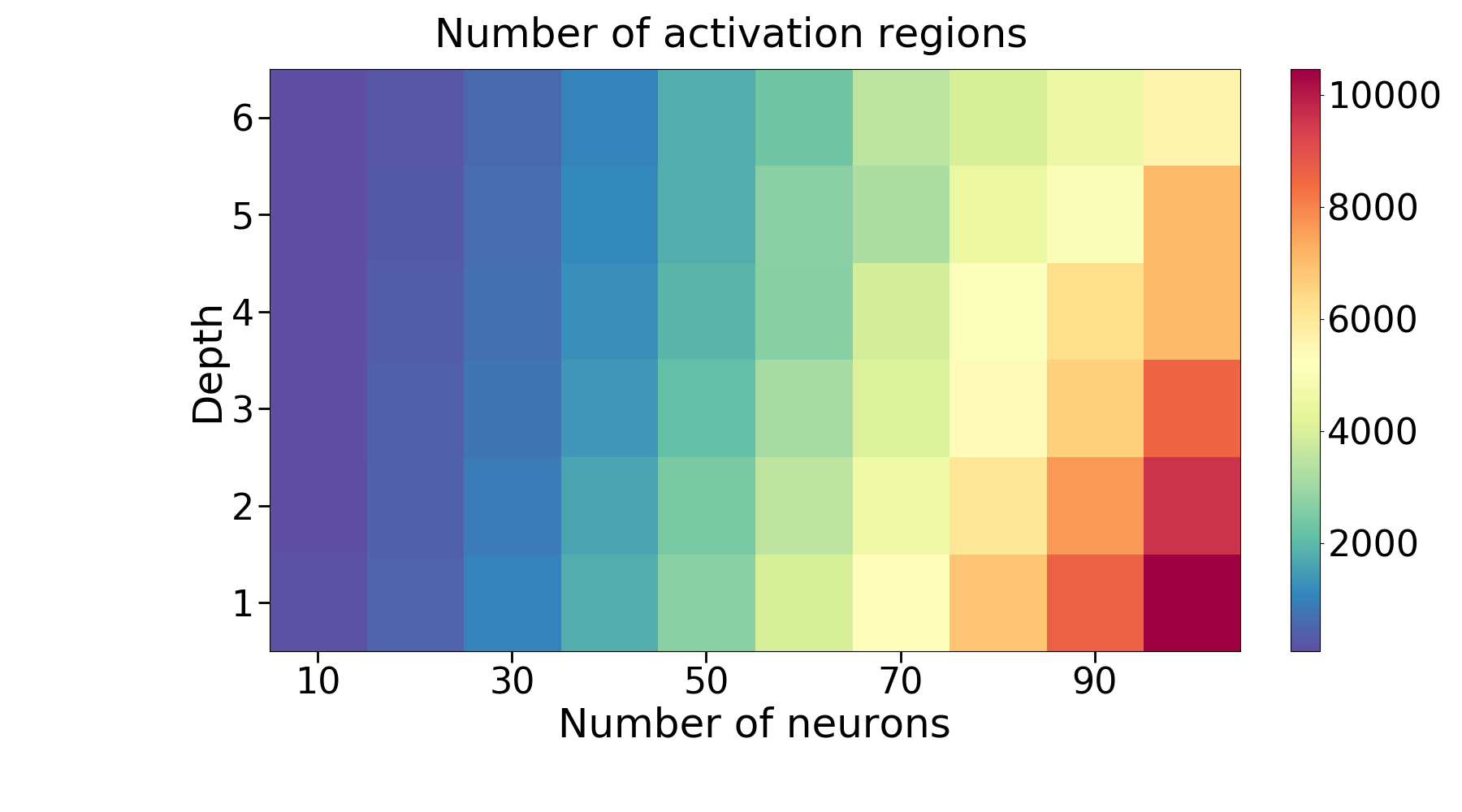}
        \caption{Maxout network with $K = 3$. Maxout-He normal initialization.}
    \end{subfigure}
    \hspace{10pt}
    \begin{subfigure}{0.46\textwidth}
        \centering
        \includegraphics[trim=10 10 10 20, clip, width=.44\textwidth] {images/neurons_to_depth/f42a25d6537d431fa276af1075803f320.png}
        \includegraphics[trim=30 40 15 10, clip, width=.36\textwidth] {images/neurons_to_depth/f42a25d6537d431fa276af1075803f321.png}
        \caption{Maxout network with $K = 5$. Maxout-He normal initialization.}
    \end{subfigure}
    
    \caption{Difference between the effects of depth and number of neurons on the number of activation regions. These plots are additional to Figure \ref{fig:neurons_to_depth} and have similar settings.
    ReLU and maxout networks with $K = 2$ have a similar number of linear regions.
    For maxout rank $K > 2$ deeper networks tend to have less regions than less deep networks with the same rank. For $K = 3$ the gaps between different depths are smaller than for $K=5$.
    }
    \label{fig:neurons_to_depth2} 
\end{figure}

\begin{figure}
    \centering
    \begin{tabular}{cc}
        \centering
        \small{Loss} &
        \small{Accuracy} \\
        
        \includegraphics[width=0.4\textwidth]{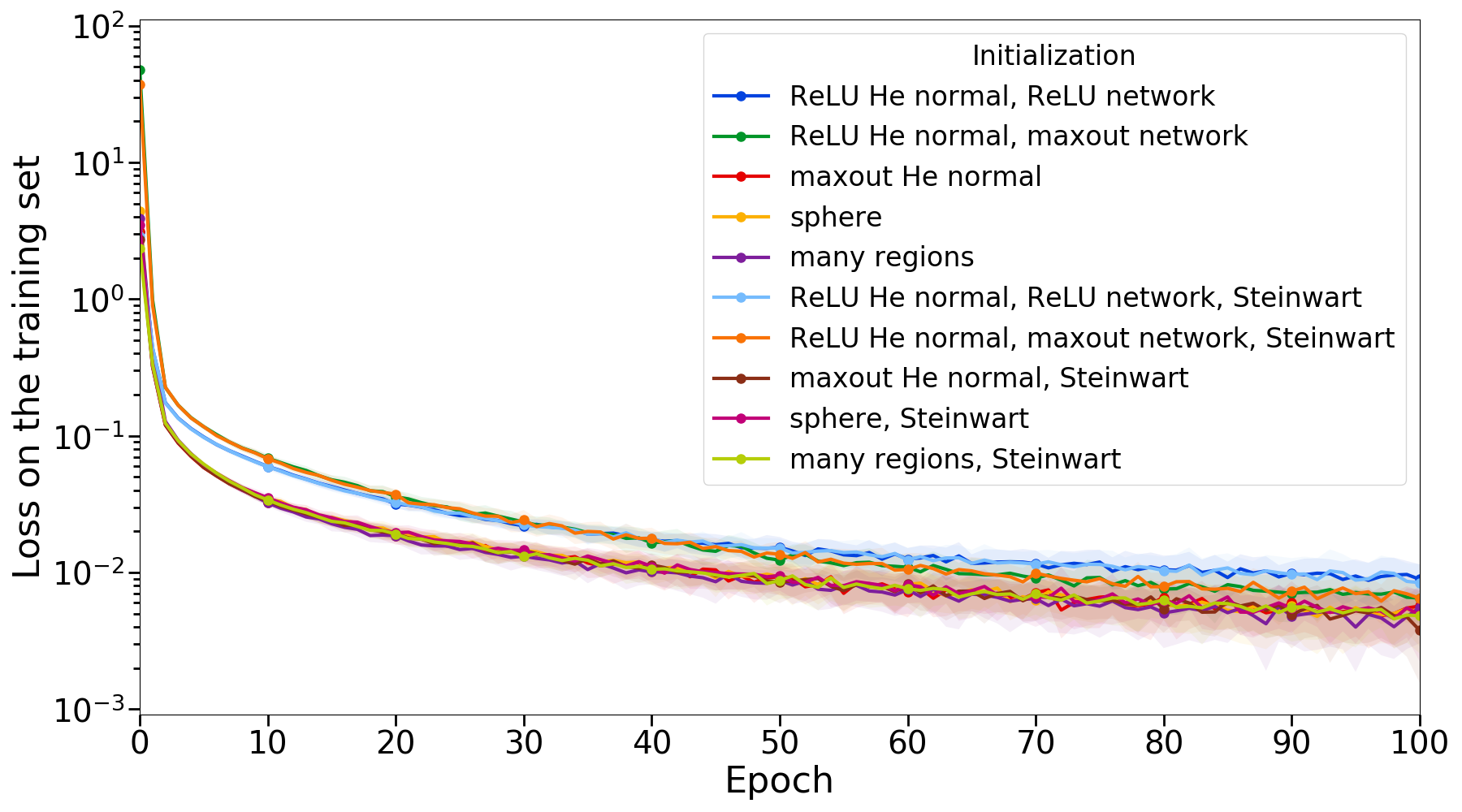} &
        \includegraphics[width=0.4\textwidth]{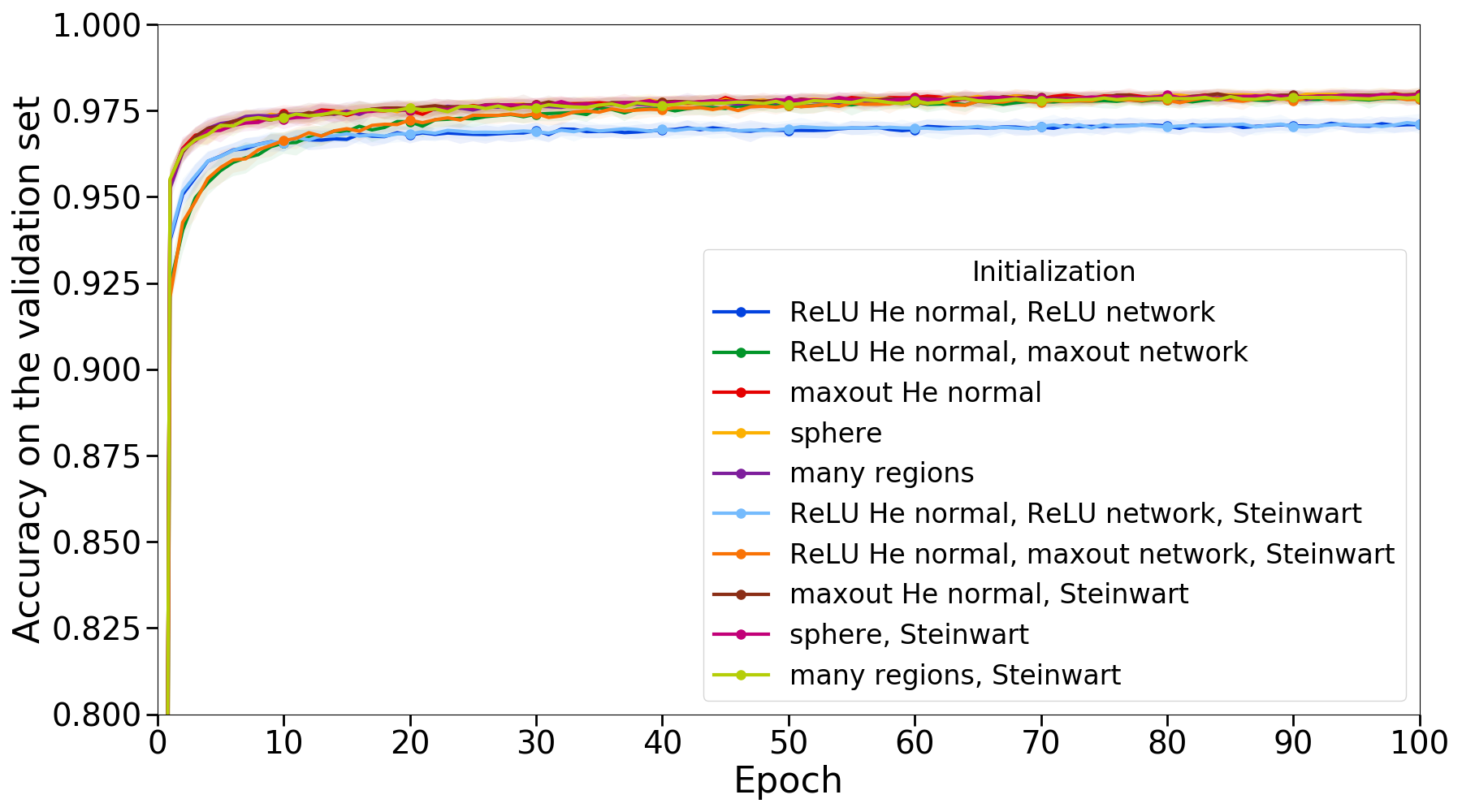}
    \end{tabular}
    \caption{Effect of the Steinwart initialization approach on the convergence speed during training on the MNIST dataset for a network with $200$ units and $5$ layers. Maxout rank was $K = 5$. 
    In this experiment, for various initialization procedures, the addition or omission of a random shift of the non-linear regions of different units led to similar training curves. }
    \label{fig:steinwart} 
\end{figure}

\begin{figure}

    \begin{subfigure}{\textwidth}
        \centering
        \setlength\tabcolsep{0pt}
        \begin{tabular}{ccc}
            \centering
            &
            \small{$K = 2$} &
            \small{$K = 5$} \\
            
            \begin{minipage}[c]{0.03\textwidth}
                \begin{flushright}
                    {\rotatebox{90}{\small $1$ layer}}
                \end{flushright}
            \end{minipage} &
            
            \begin{tabular}{cc}
                \centering
                \small{Loss} &
                \small{Accuracy} \\ \includegraphics[width=0.24\textwidth]{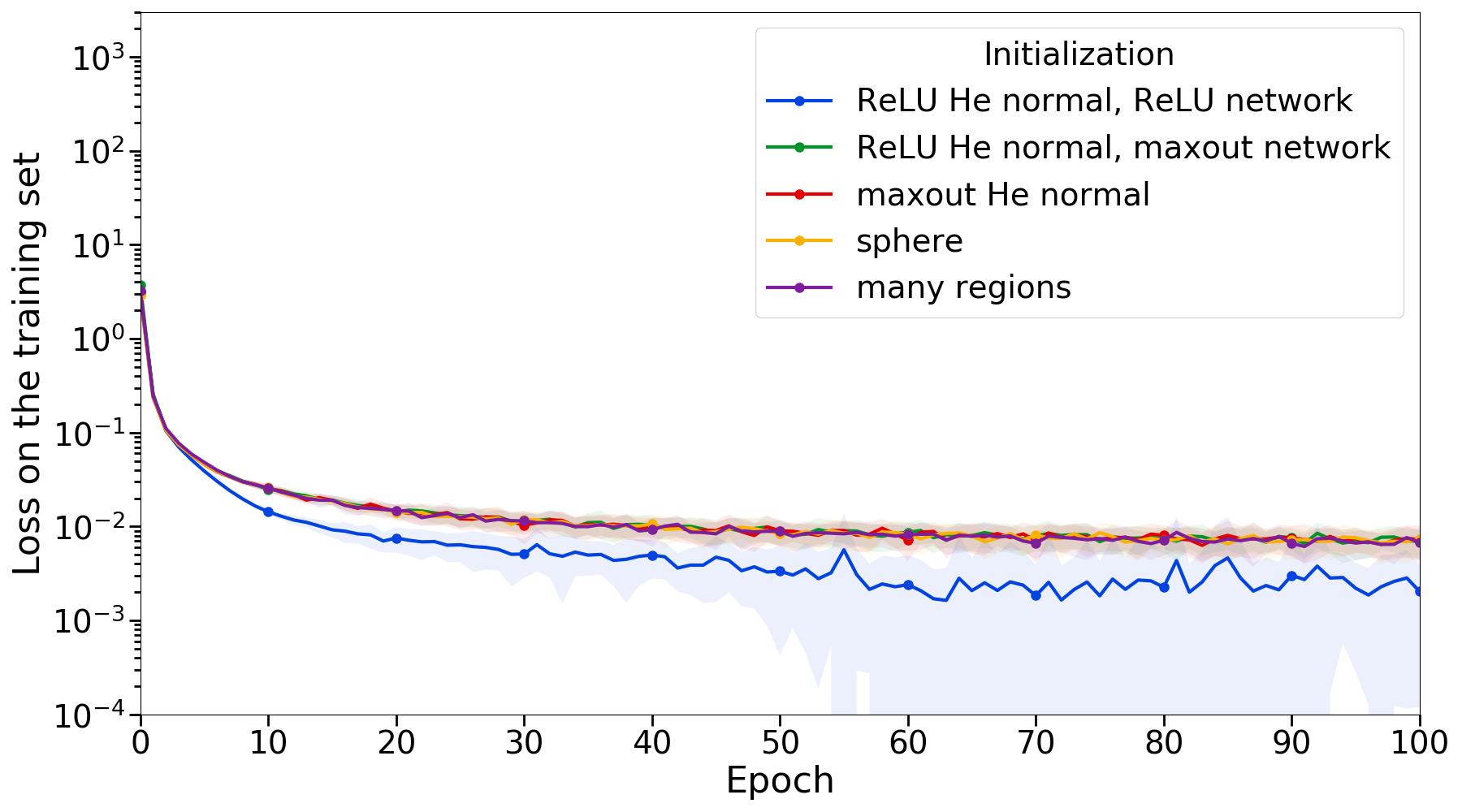} &
                \includegraphics[width=0.24\textwidth]{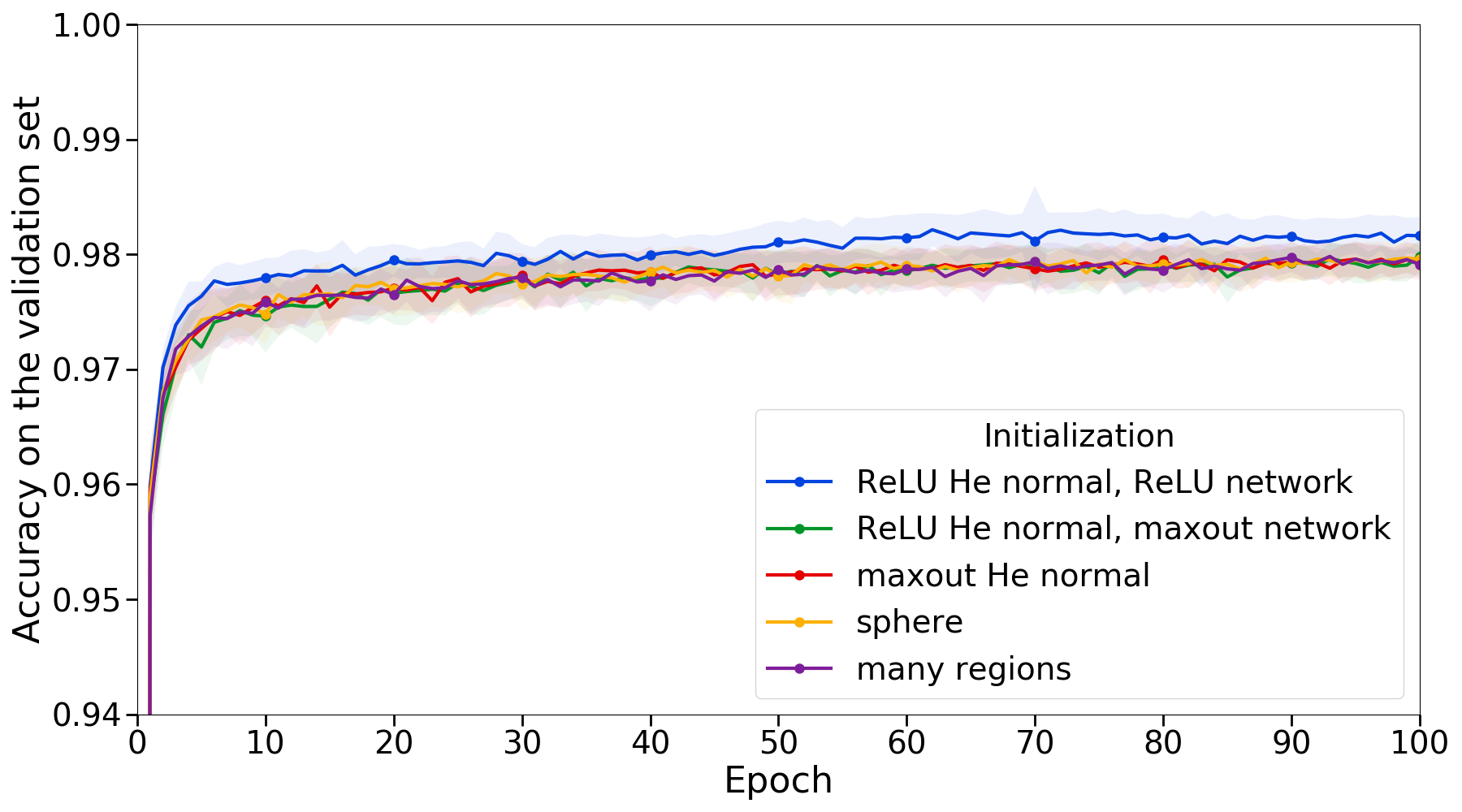}
            \end{tabular} &
            
            \begin{tabular}{cc}
                \centering
                \small{Loss} &
                \small{Accuracy} \\
                
                \includegraphics[width=0.24\textwidth]{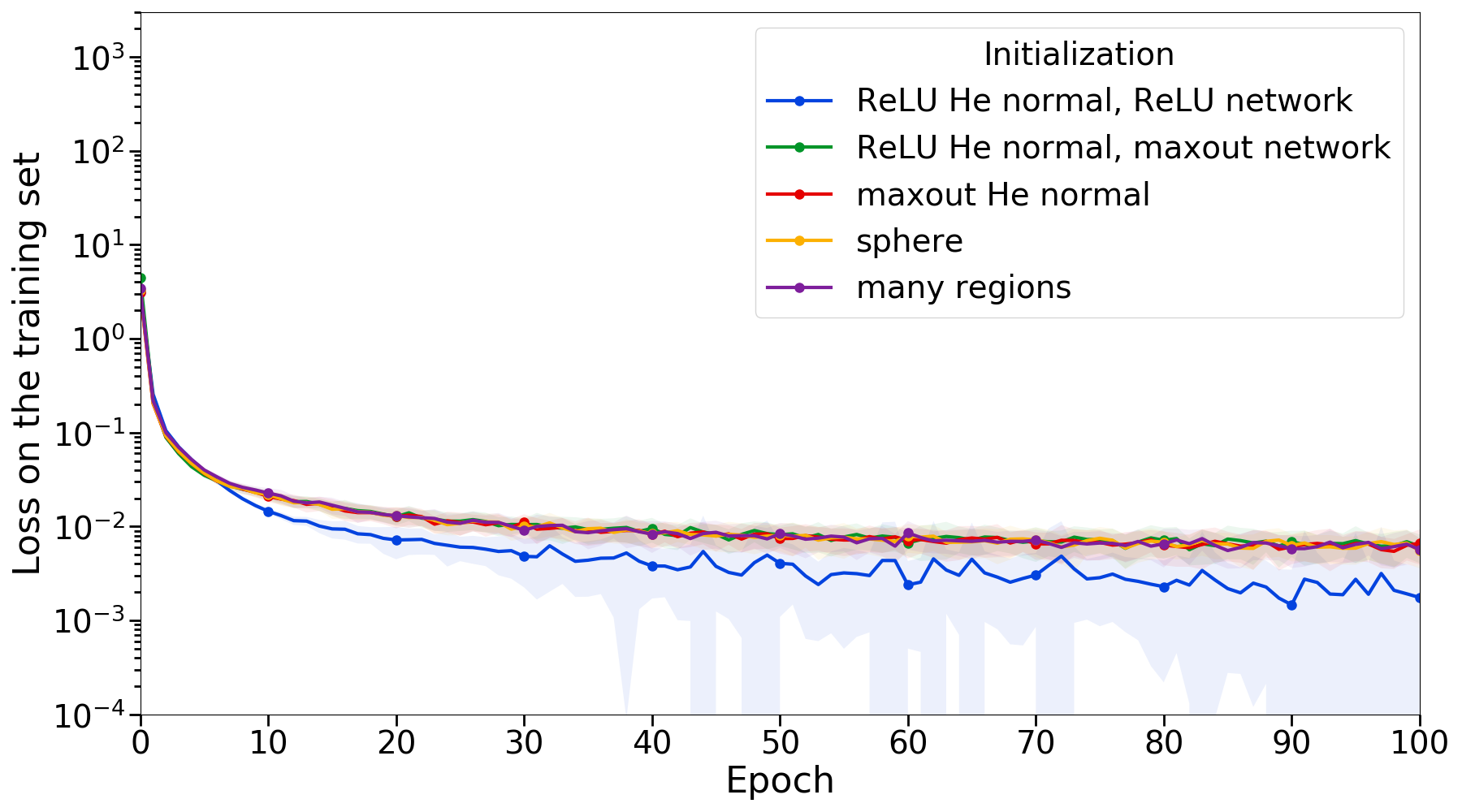} &
                \includegraphics[width=0.24\textwidth]{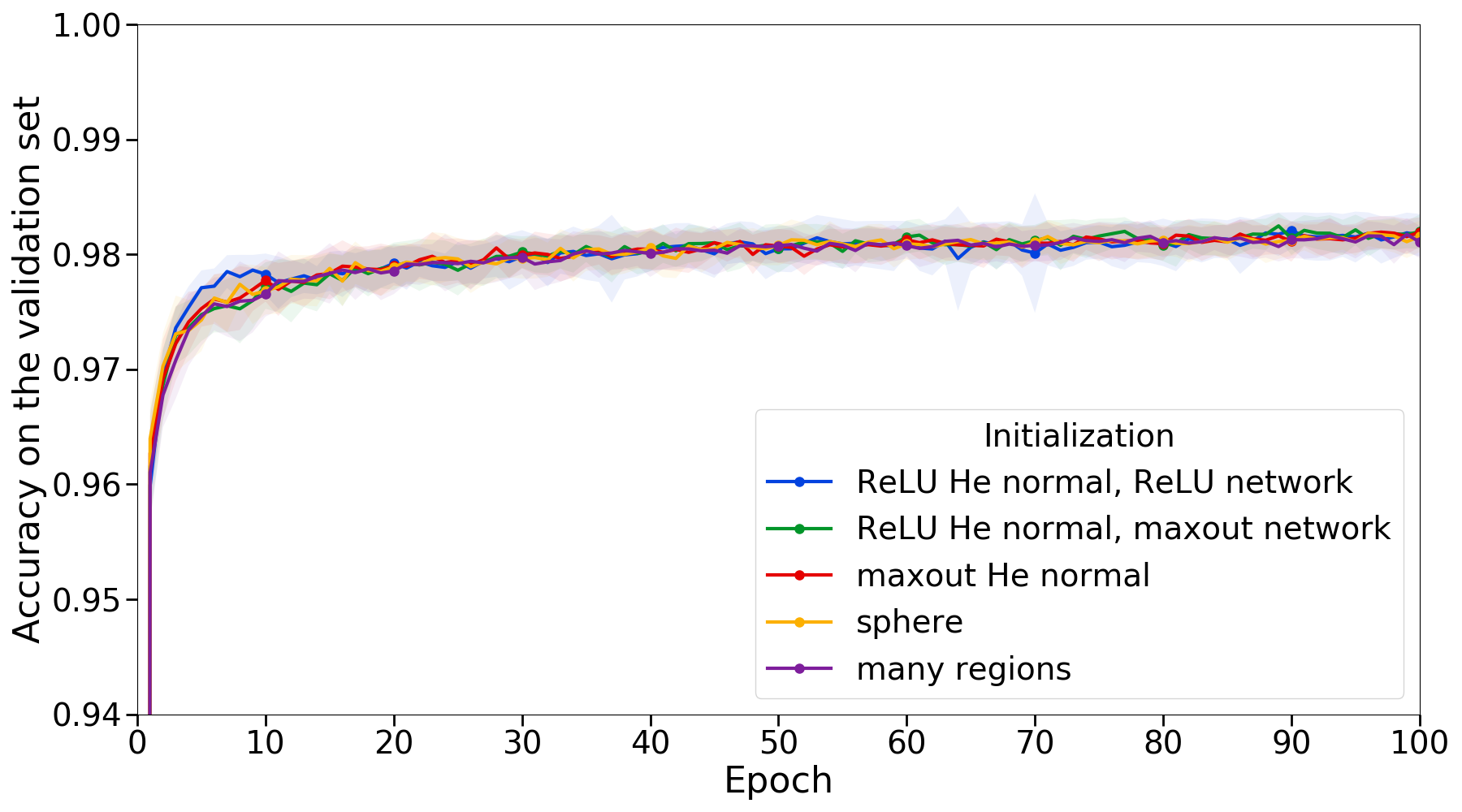}
            \end{tabular}
        \end{tabular}
    \end{subfigure}
    \vspace{.2cm}
    
    \begin{subfigure}{\textwidth}
        \centering
        \setlength\tabcolsep{0pt}
        \begin{tabular}{ccc}
            \centering
            
            \begin{minipage}[c]{0.03\textwidth}
                \begin{flushright}
                    {\rotatebox{90}{\small $3$ layers}}
                \end{flushright}
            \end{minipage} &
            
            \begin{tabular}{cc}
                \centering
                \includegraphics[width=0.24\textwidth]{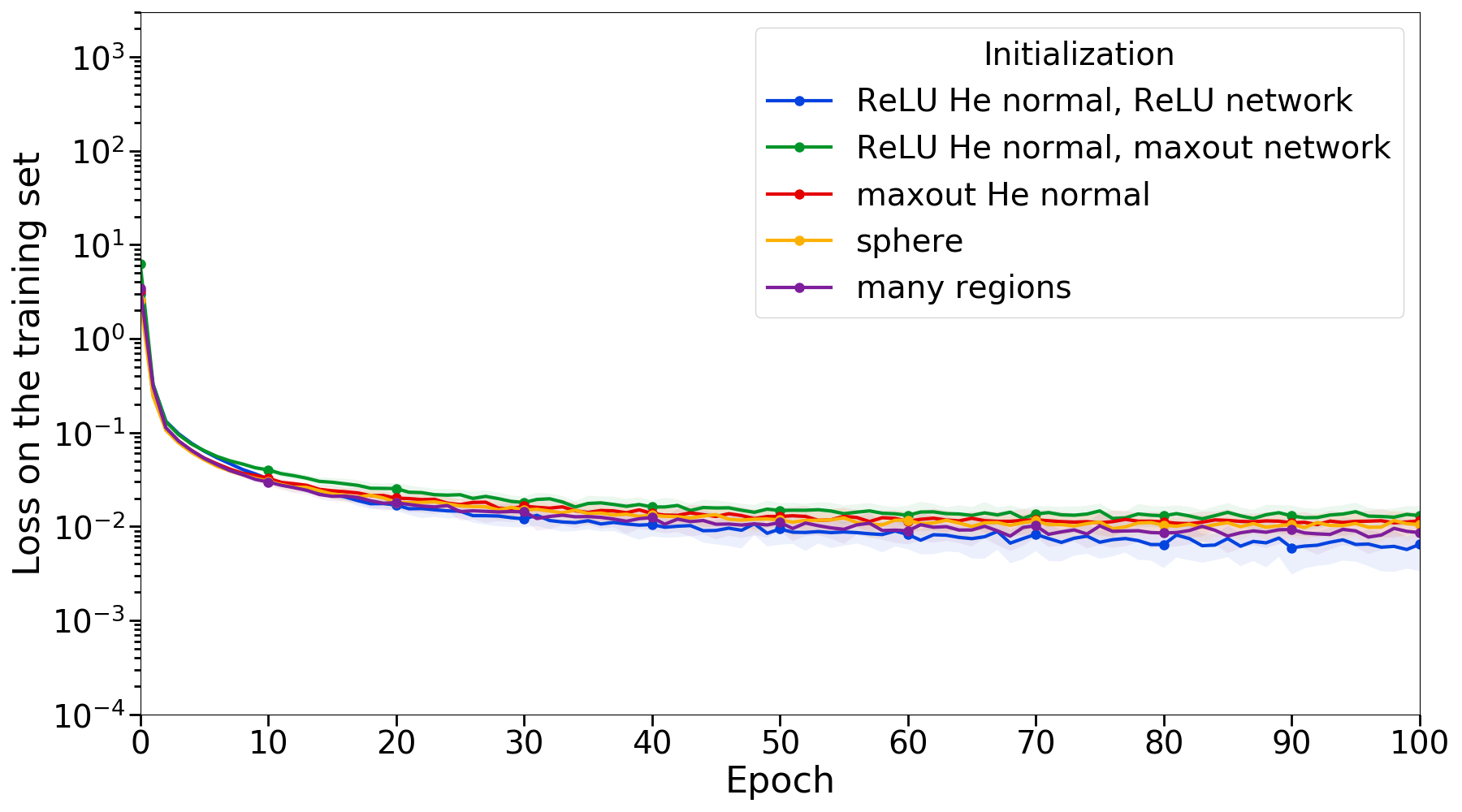} &
                \includegraphics[width=0.24\textwidth]{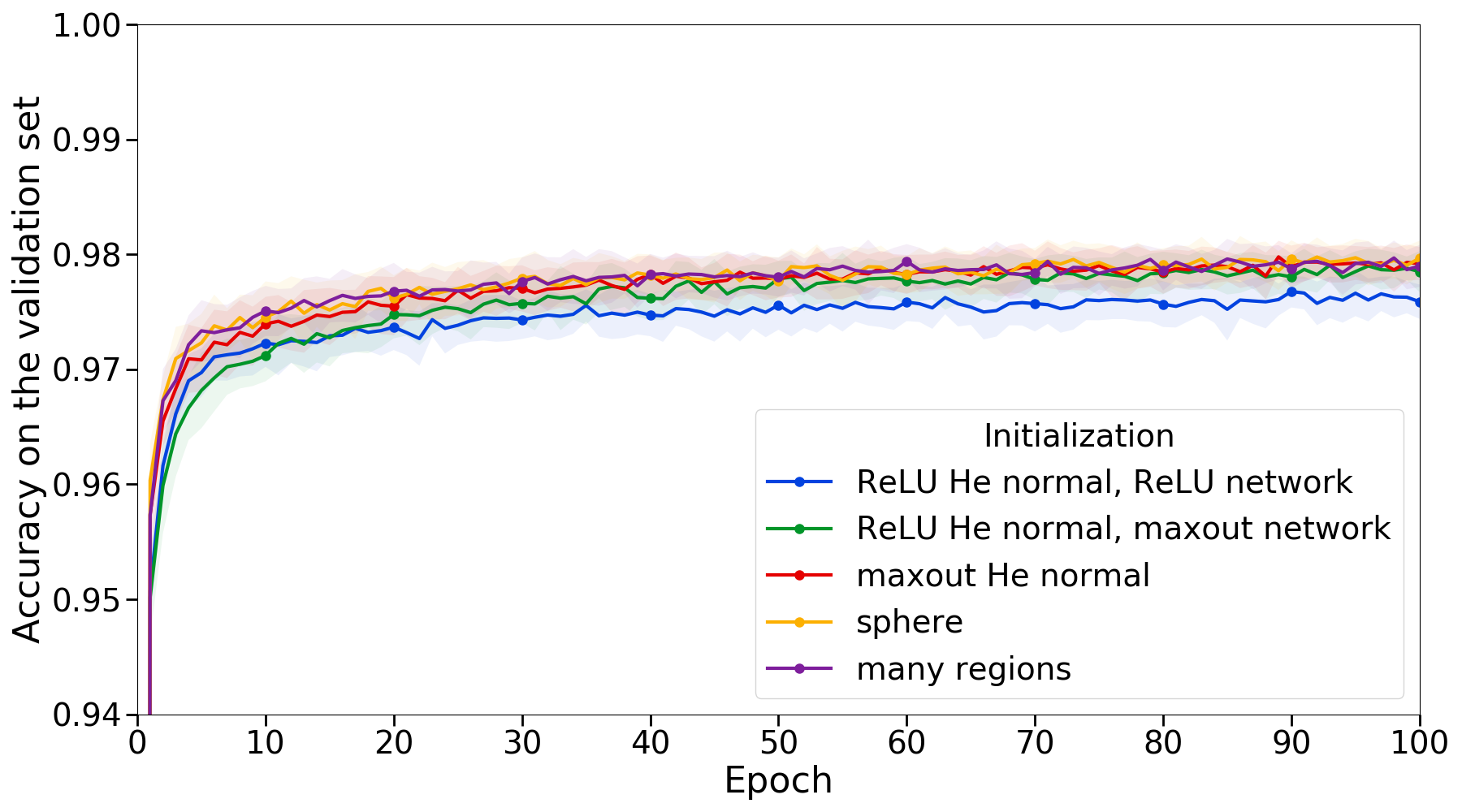}
            \end{tabular} &
            
            \begin{tabular}{cc}
                \centering
                \includegraphics[width=0.24\textwidth]{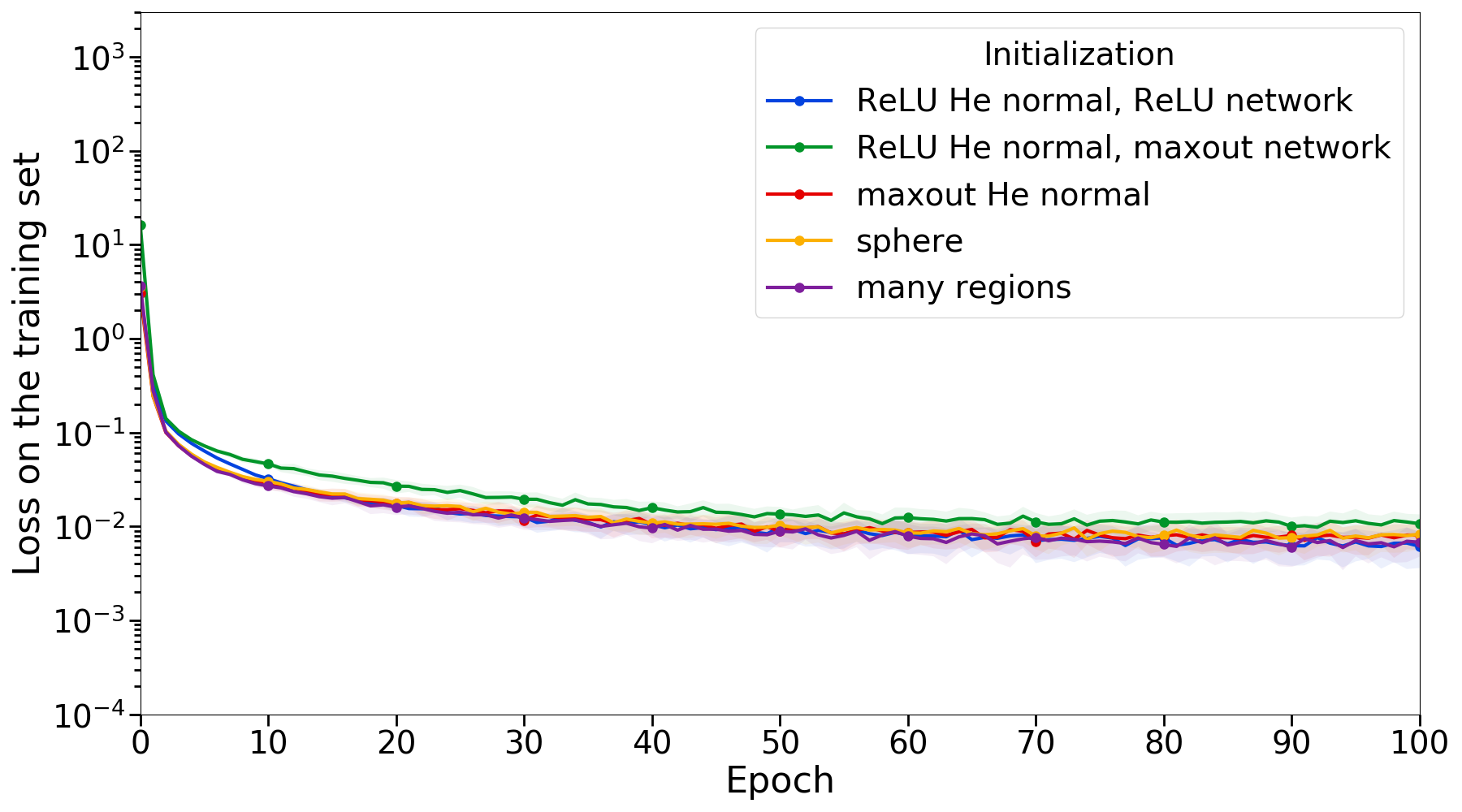} &
                \includegraphics[width=0.24\textwidth]{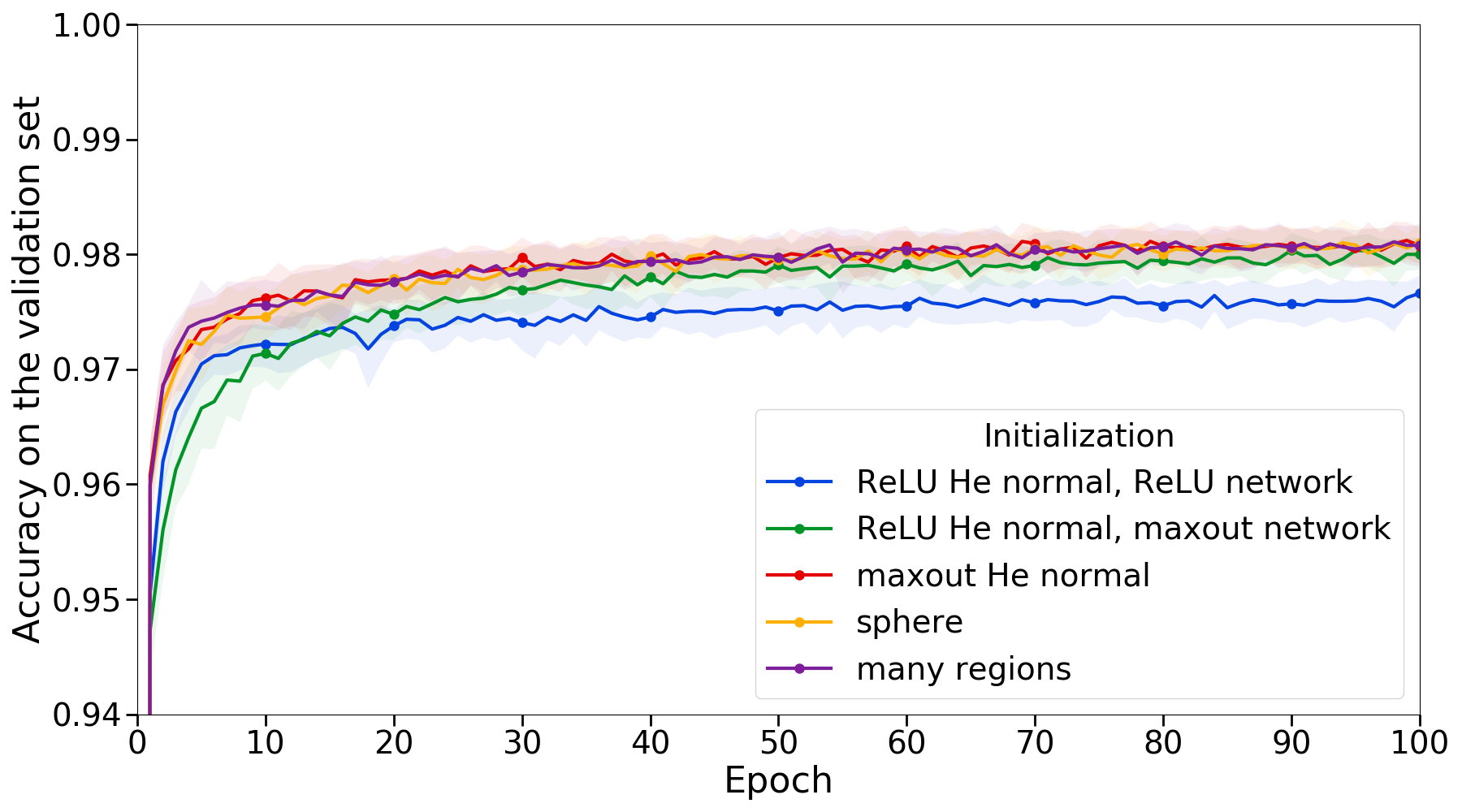}
            \end{tabular}
        \end{tabular}
    \end{subfigure}
    \vspace{.2cm}
    
    \begin{subfigure}{\textwidth}
        \centering
        \setlength\tabcolsep{0pt}
        \begin{tabular}{ccc}
            \centering
            
            \begin{minipage}[c]{0.03\textwidth}
                \begin{flushright}
                    {\rotatebox{90}{\small $5$ layers}}
                \end{flushright}
            \end{minipage} &
            
            \begin{tabular}{cc}
                \centering
                \includegraphics[width=0.24\textwidth]{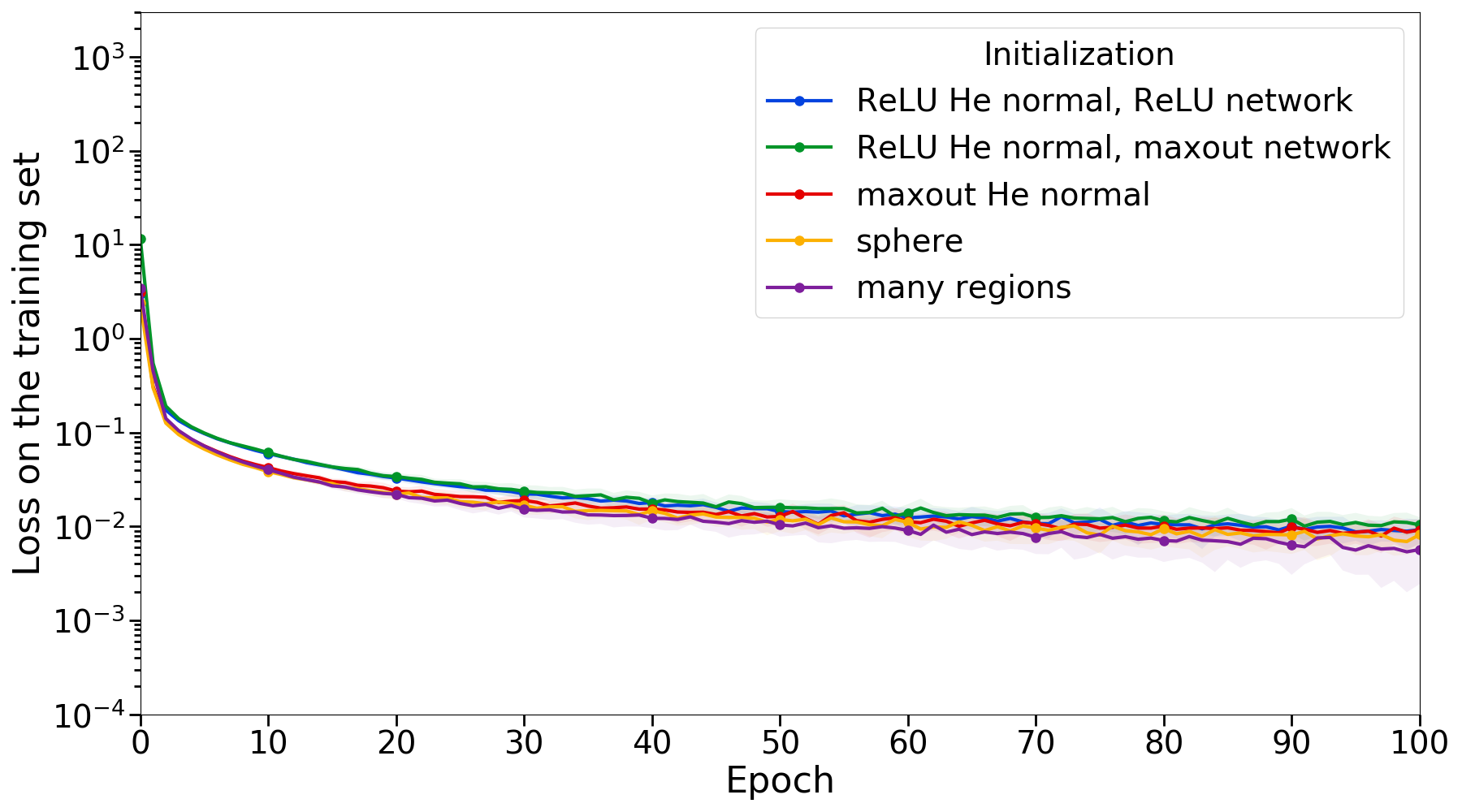} &
                \includegraphics[width=0.24\textwidth]{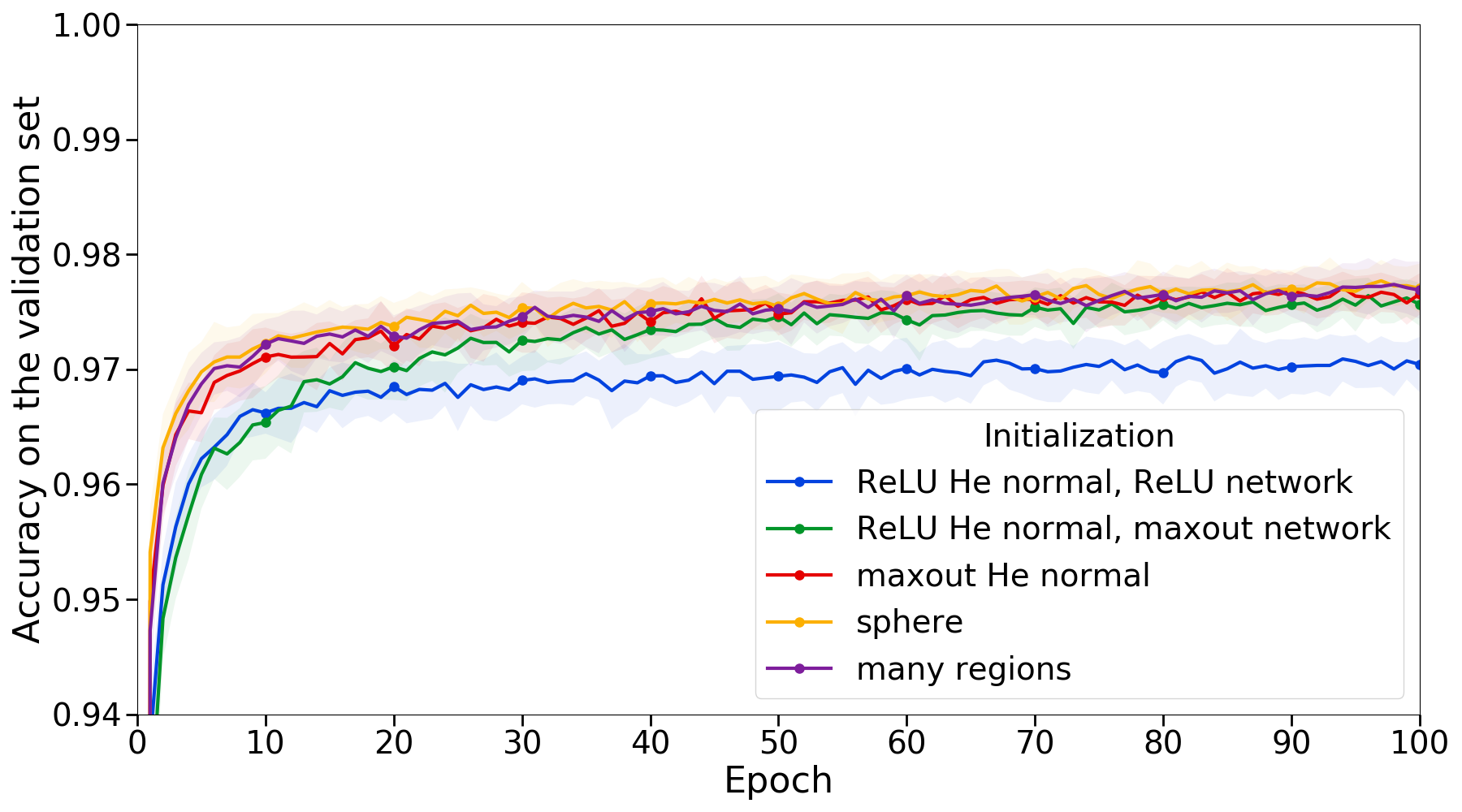}
            \end{tabular} &
            
            \begin{tabular}{cc}
                \centering
                \includegraphics[width=0.24\textwidth]{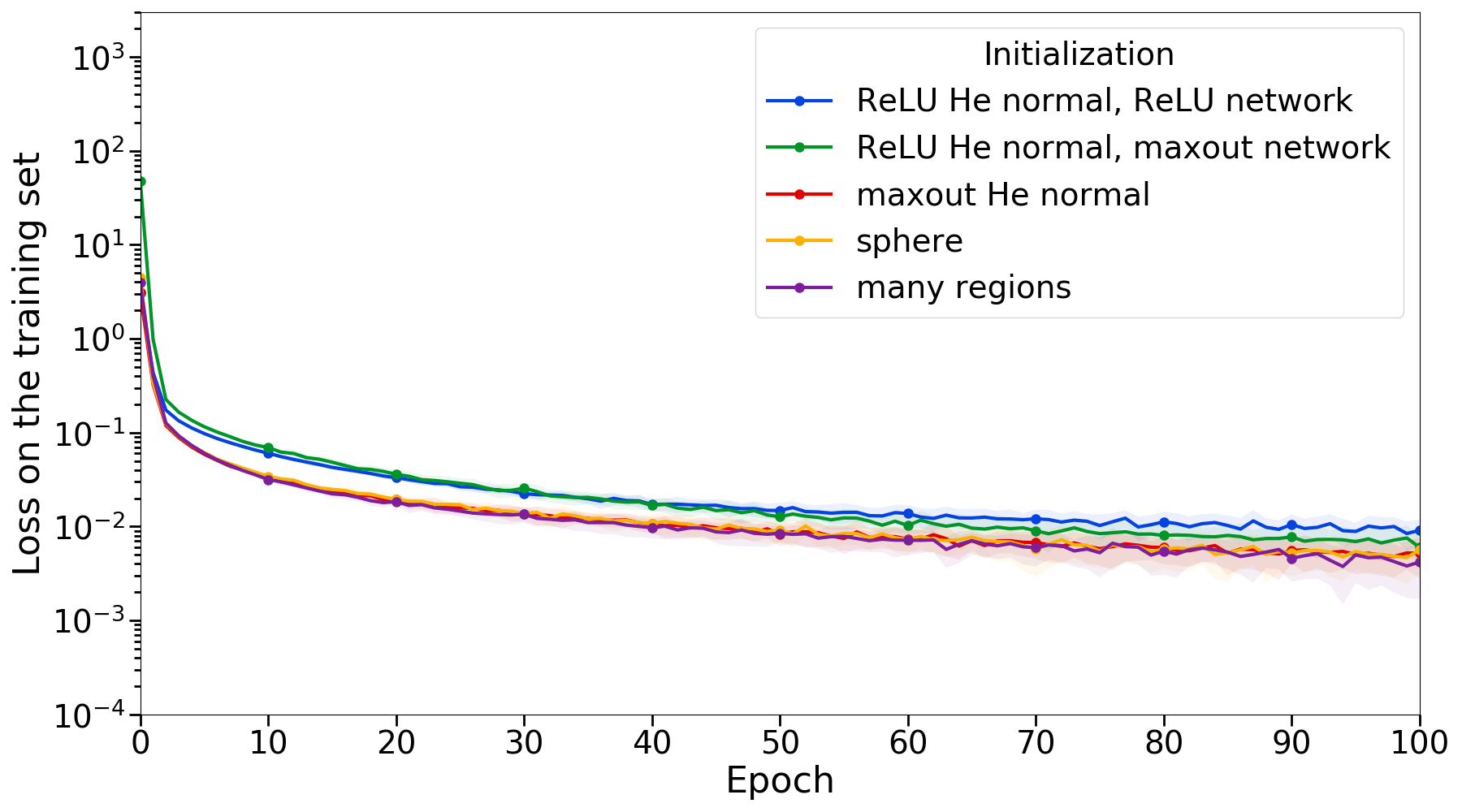} &
                \includegraphics[width=0.24\textwidth]{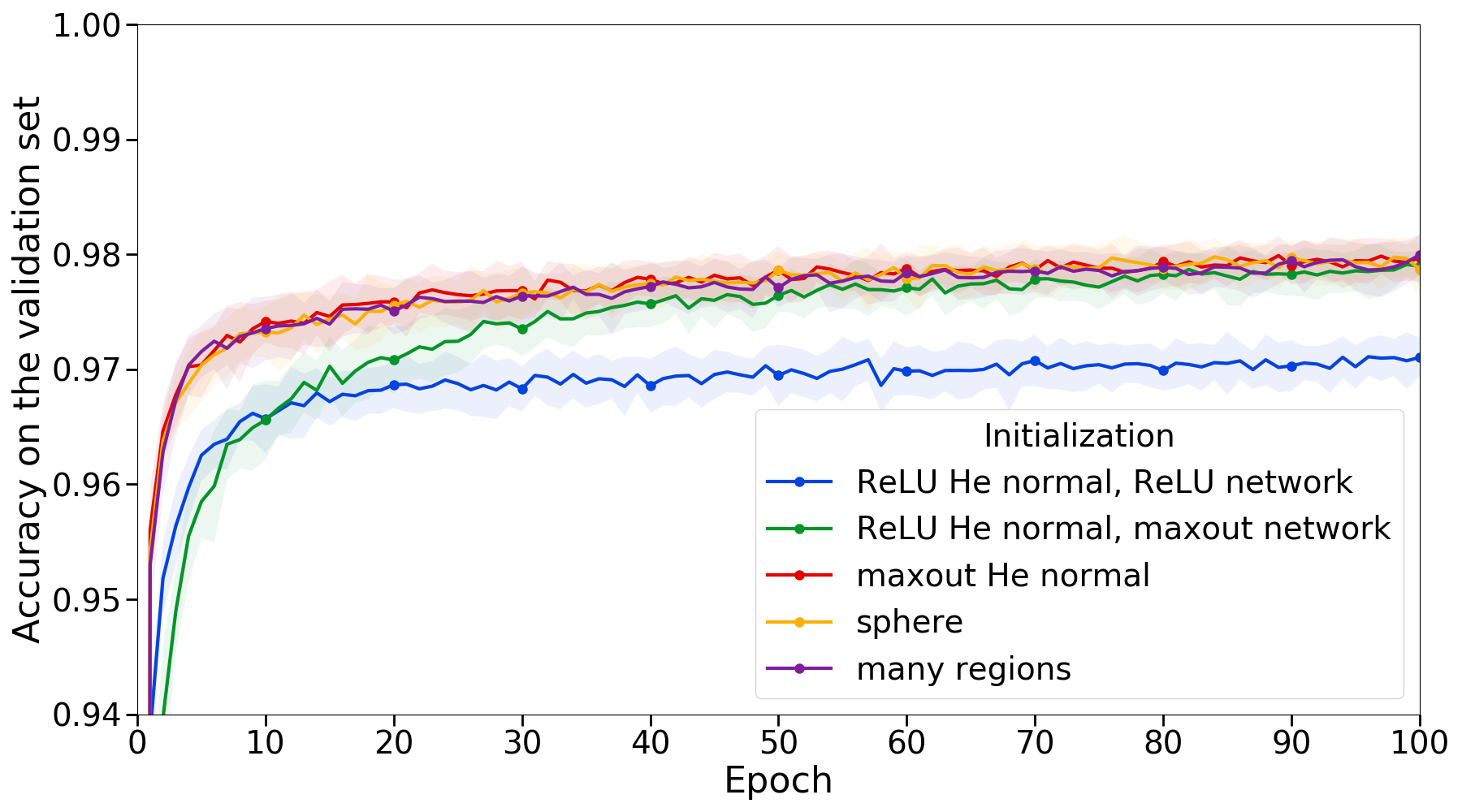}
            \end{tabular}
        \end{tabular}
    \end{subfigure}
    \vspace{.2cm}
    
    \begin{subfigure}{\textwidth}
        \centering
        \setlength\tabcolsep{0pt}
        \begin{tabular}{ccc}
            \centering
            
            \begin{minipage}[c]{0.03\textwidth}
                \begin{flushright}
                    {\rotatebox{90}{\small $10$ layers}}
                \end{flushright}
            \end{minipage} &
            
            \begin{tabular}{cc}
                \centering
                \includegraphics[width=0.24\textwidth]{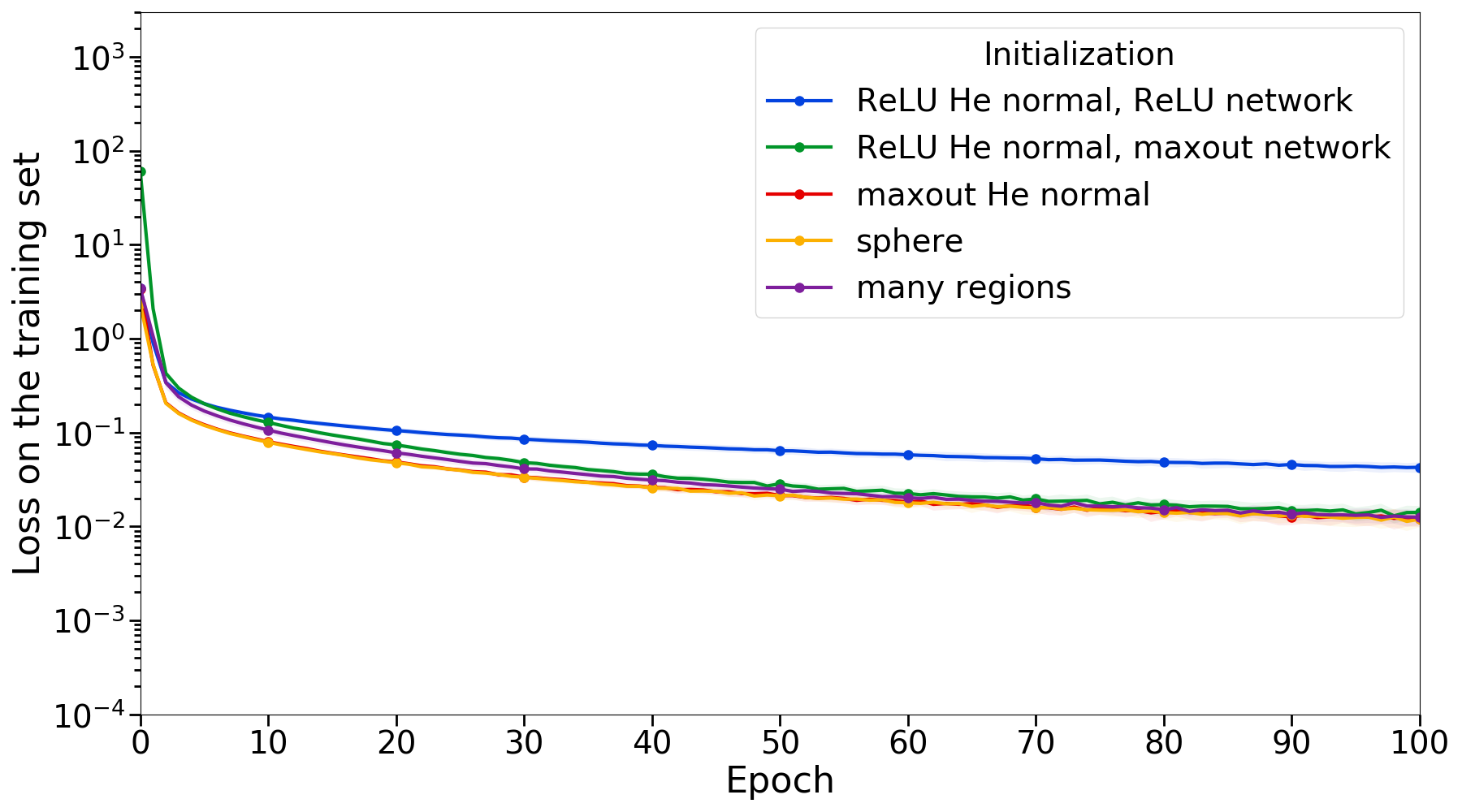} &
                \includegraphics[width=0.24\textwidth]{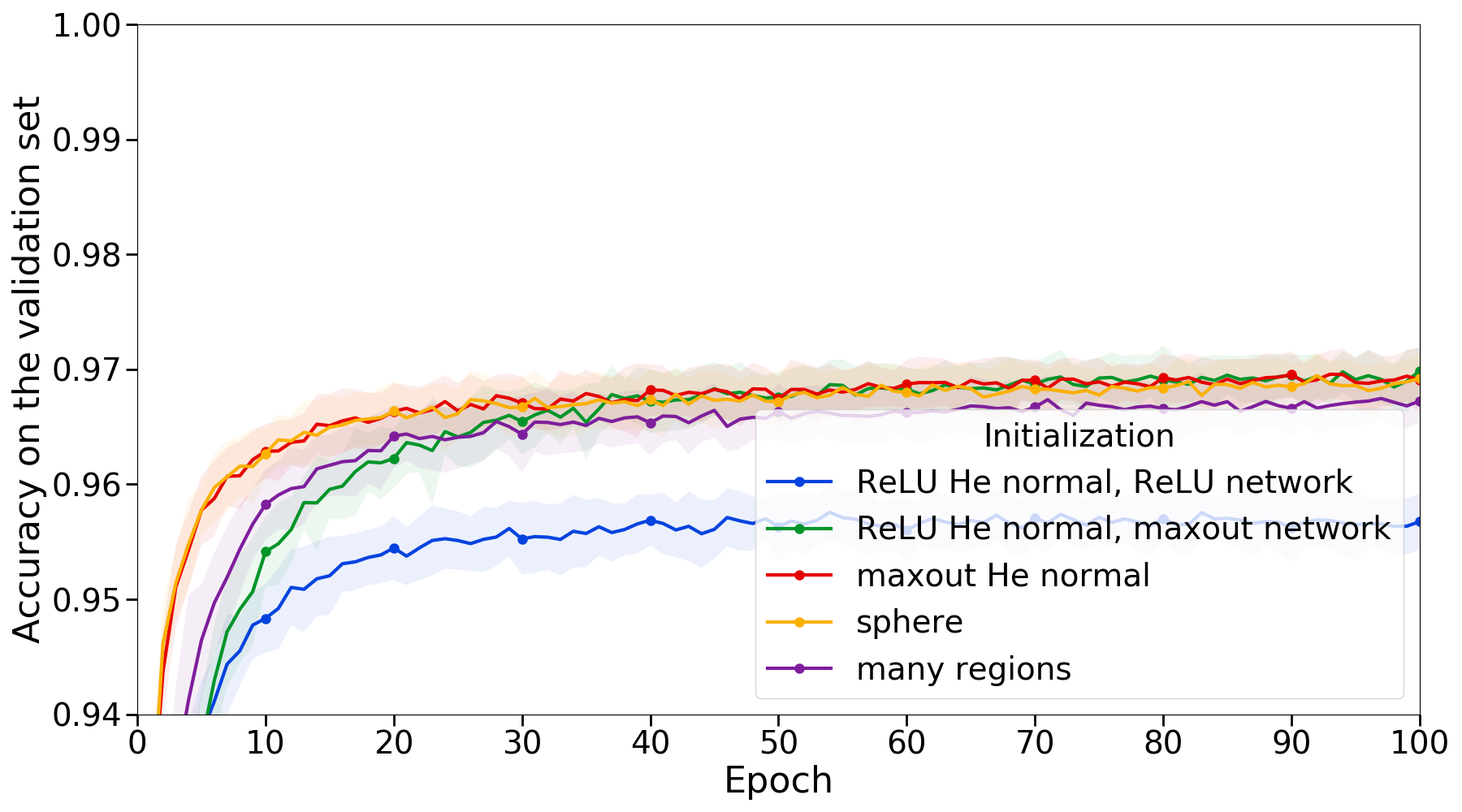}
            \end{tabular} &
            
            \begin{tabular}{cc}
                \centering
                \includegraphics[width=0.24\textwidth]{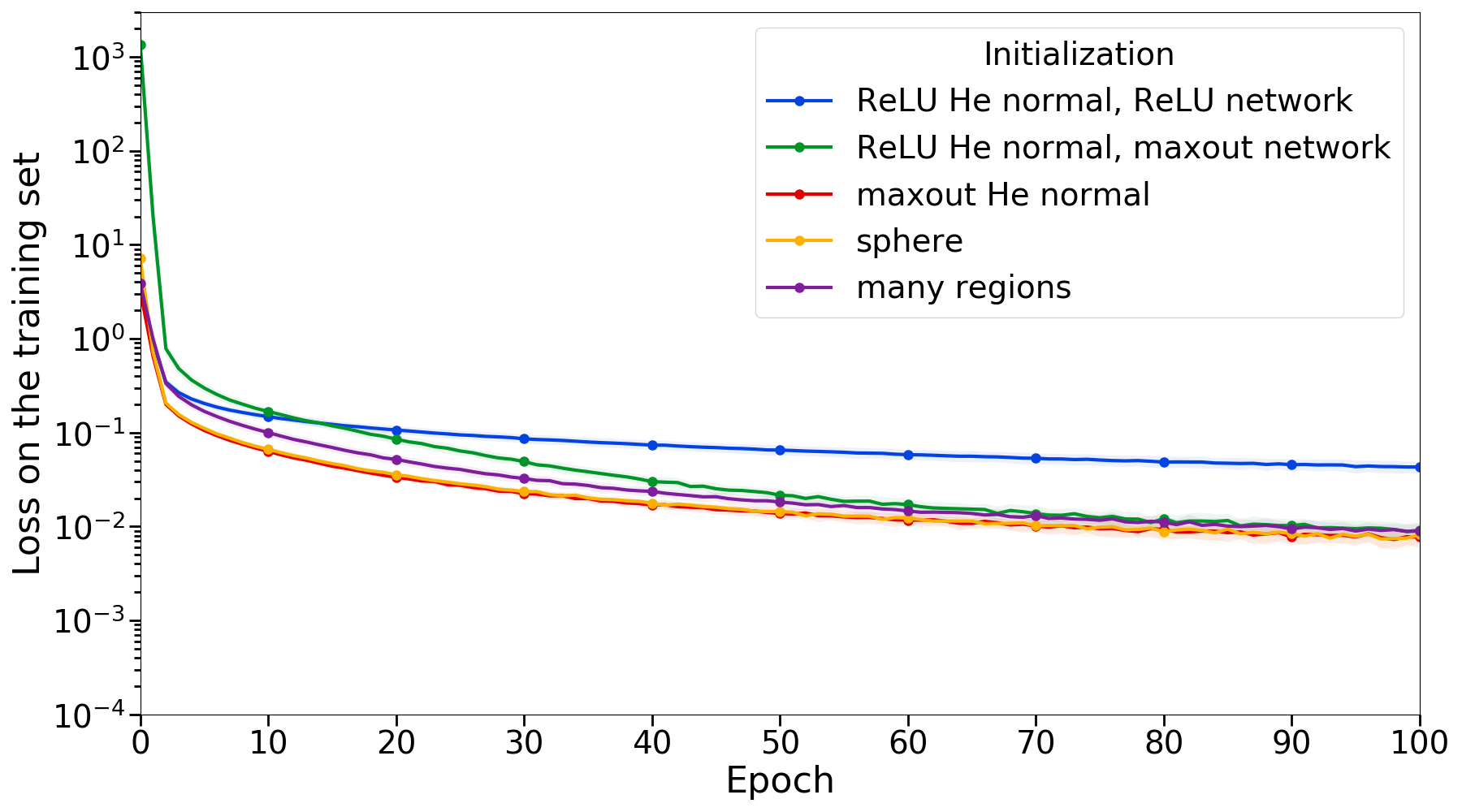} &
                \includegraphics[width=0.24\textwidth]{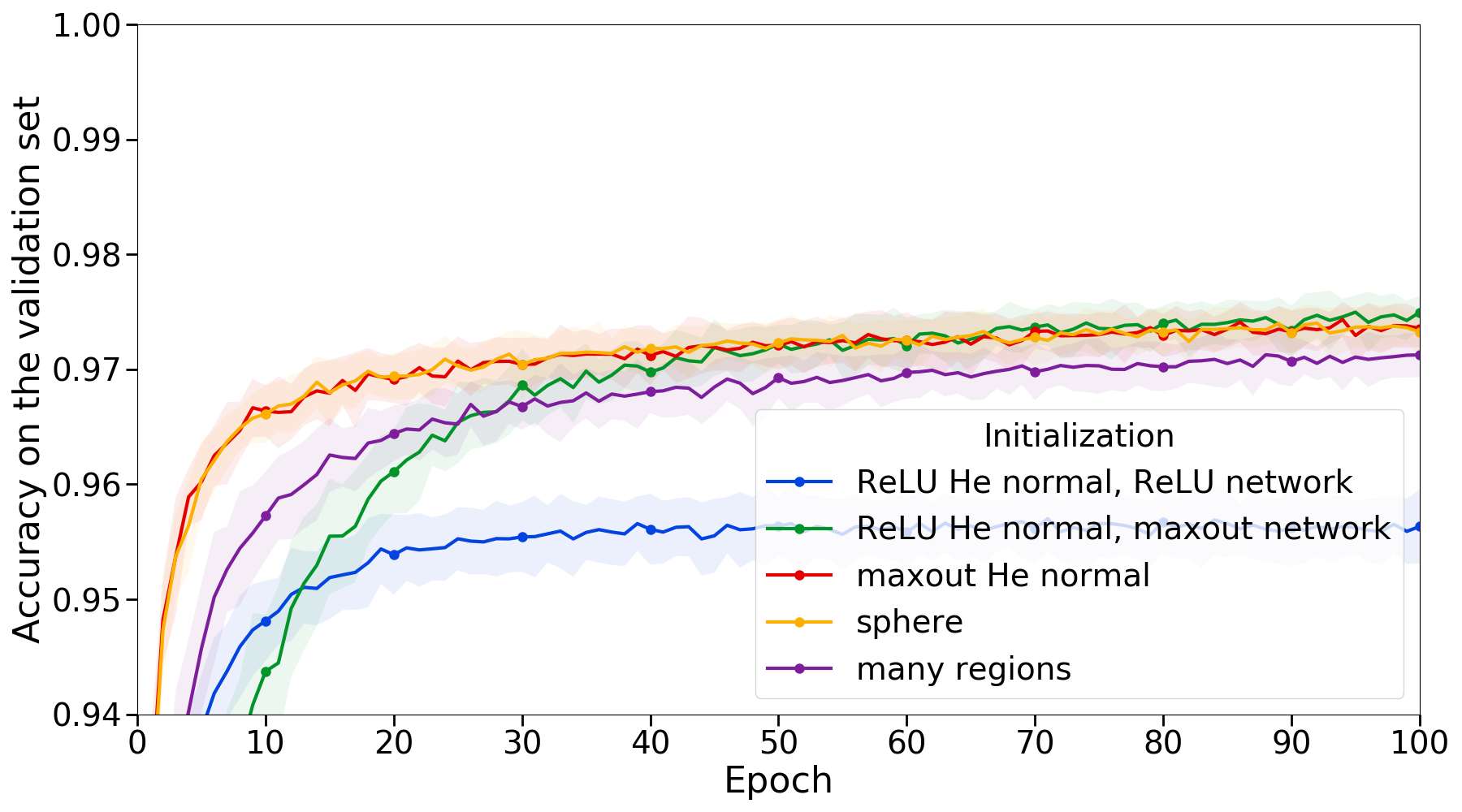}
            \end{tabular}
        \end{tabular}
    \end{subfigure}
    
    \caption{Effect of the initialization on the convergence speed during training on the MNIST dataset of networks with $200$ units depending on the network depth and the maxout rank. Maxout-He, sphere, and many regions initializations behave similarly, and the improvement in the convergence speed becomes more noticeable for larger network depth and maxout rank.}
    \label{fig:training_dif_params} 
\end{figure}

\begin{figure}
    
    \begin{subfigure}{\textwidth}
        \centering
        \setlength\tabcolsep{1pt}
        \begin{tabular}{cccccc}
            \centering
            \small{Before training} &
            \small{$20$ epochs} &
            \small{$40$ epochs} &
            \small{$60$ epochs} &
            \small{$80$ epochs} &
            \small{$100$ epochs} \\
            
            \includegraphics[width=0.16\textwidth]{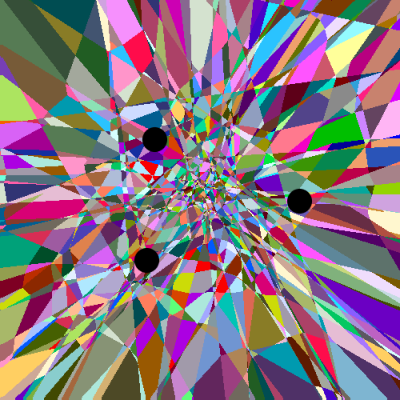} &
            \includegraphics[width=0.16\textwidth]{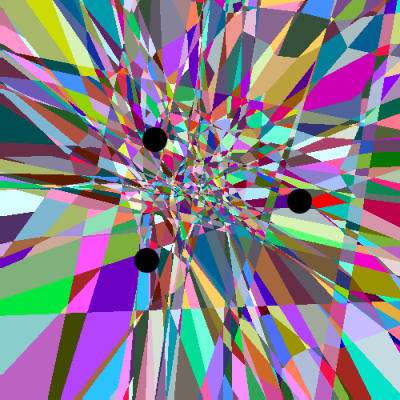} &
            \includegraphics[width=0.16\textwidth]{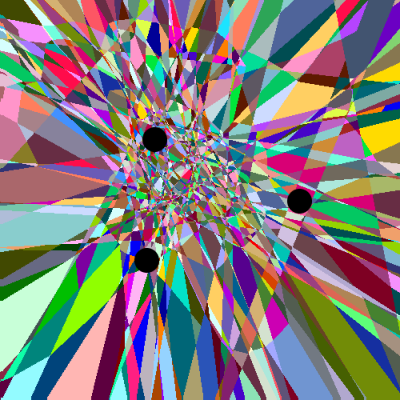} &
            \includegraphics[width=0.16\textwidth]{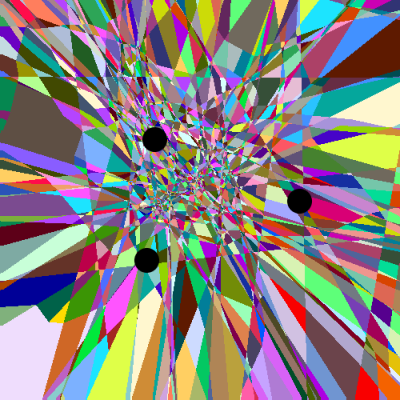} &
            \includegraphics[width=0.16\textwidth]{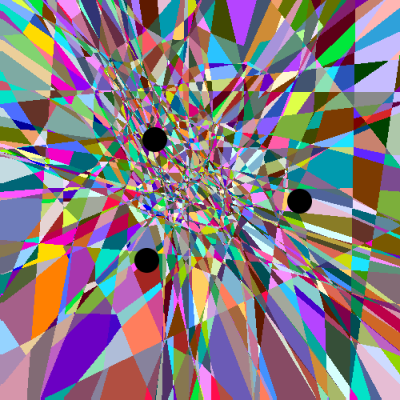} &
            \includegraphics[width=0.16\textwidth]{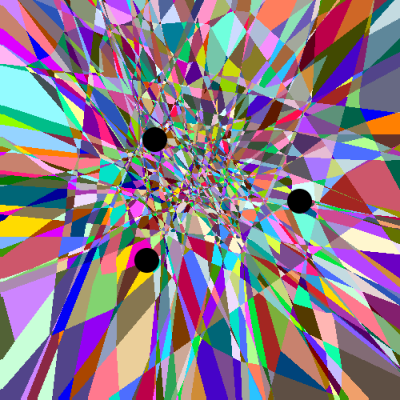}
        \end{tabular}
        \caption{\small Linear regions.}
    \end{subfigure}
    \vspace{.2cm}
    
    \begin{subfigure}{\textwidth}
        \centering
        \setlength\tabcolsep{1pt}
        \begin{tabular}{cccccc}
            \centering
            \small{Before training} &
            \small{$20$ epochs} &
            \small{$40$ epochs} &
            \small{$60$ epochs} &
            \small{$80$ epochs} &
            \small{$100$ epochs} \\
            
            \includegraphics[width=0.16\textwidth]{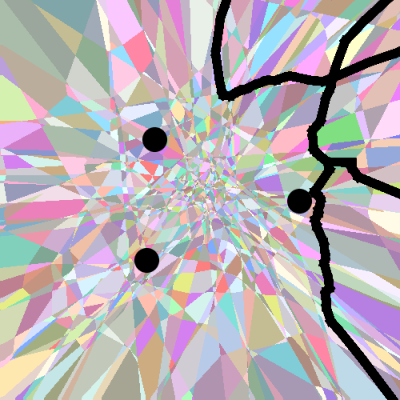} &
            \includegraphics[width=0.16\textwidth]{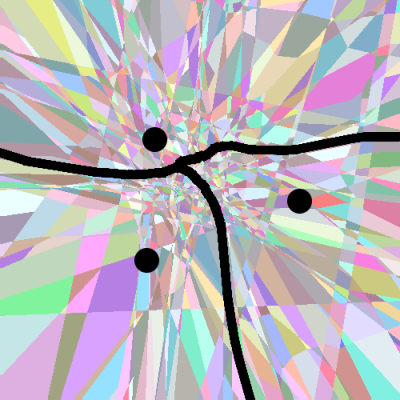} &
            \includegraphics[width=0.16\textwidth]{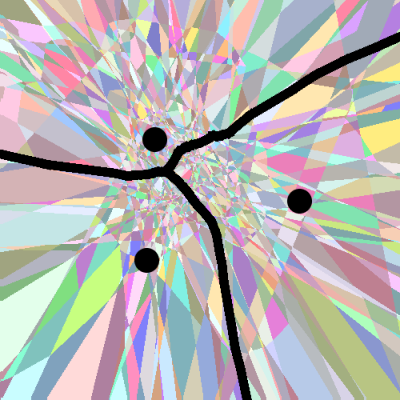} &
            \includegraphics[width=0.16\textwidth]{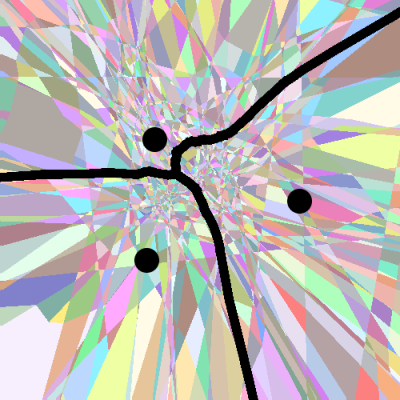} &
            \includegraphics[width=0.16\textwidth]{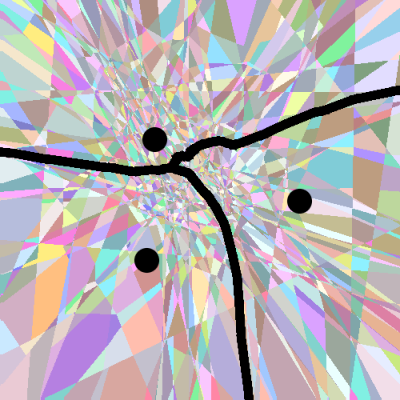} &
            \includegraphics[width=0.16\textwidth]{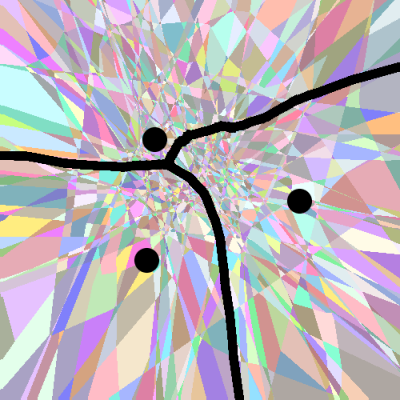}
        \end{tabular}
        \caption{\small Decision boundary.}
    \end{subfigure}
    
    \caption{Evolution of the linear regions and the decision boundary during training on MNIST in a 2D slice determined by three random input points from the dataset. The network had $3$ layers, a total of 100 maxout units of rank $K = 2$, and was initialized with the maxout-He initialization.}
    \label{fig:db_regions_evolution} 
\end{figure}

\begin{figure}

    \begin{subfigure}{\textwidth}
        \setlength\tabcolsep{2pt}
        \begin{tabular}{ccc}
            \centering
            \small{Linear regions} &
            \small{Decision boundary} &
            \small{Loss}  \\
            
            \includegraphics[width=0.32\textwidth]{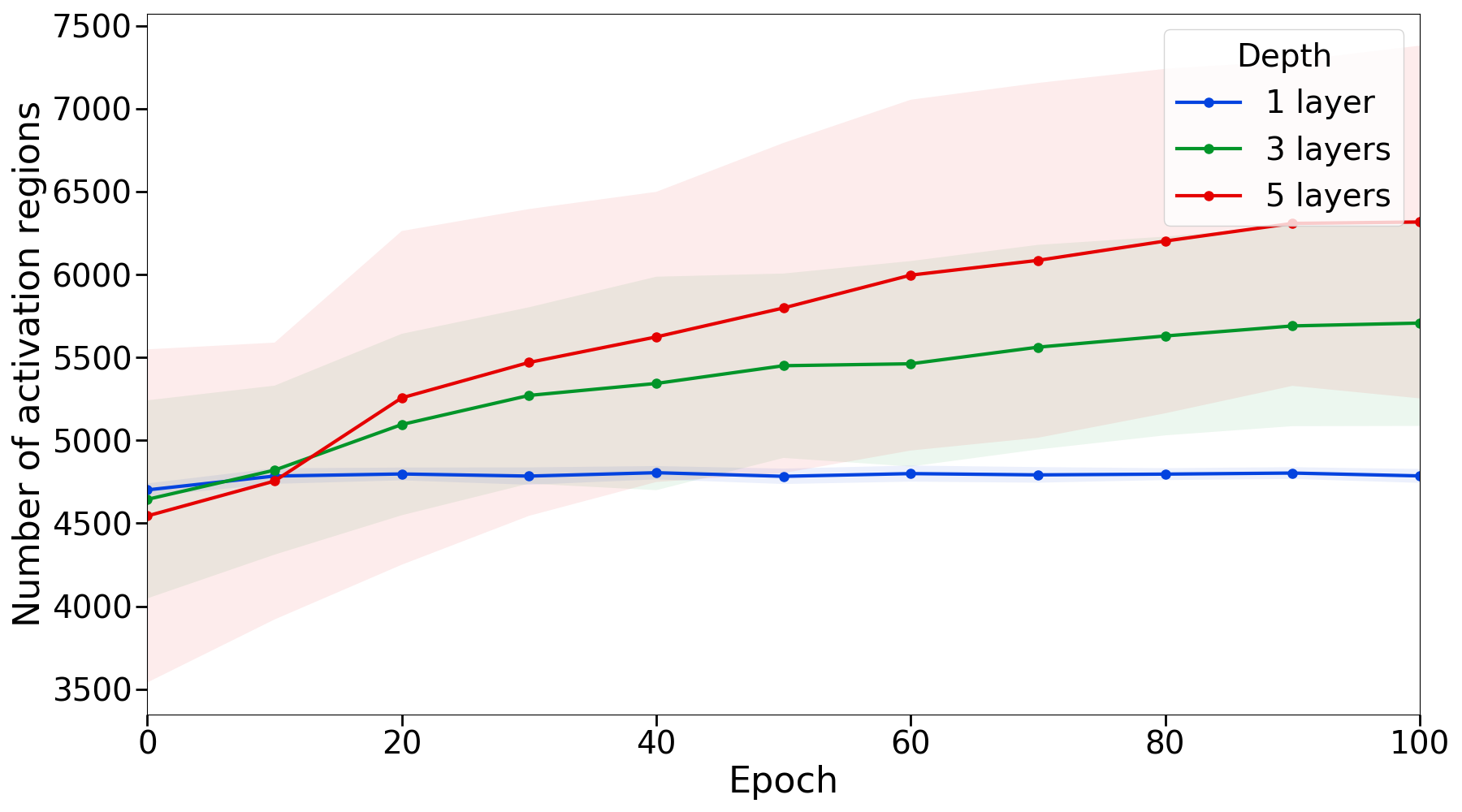} &
            \includegraphics[width=0.32\textwidth]{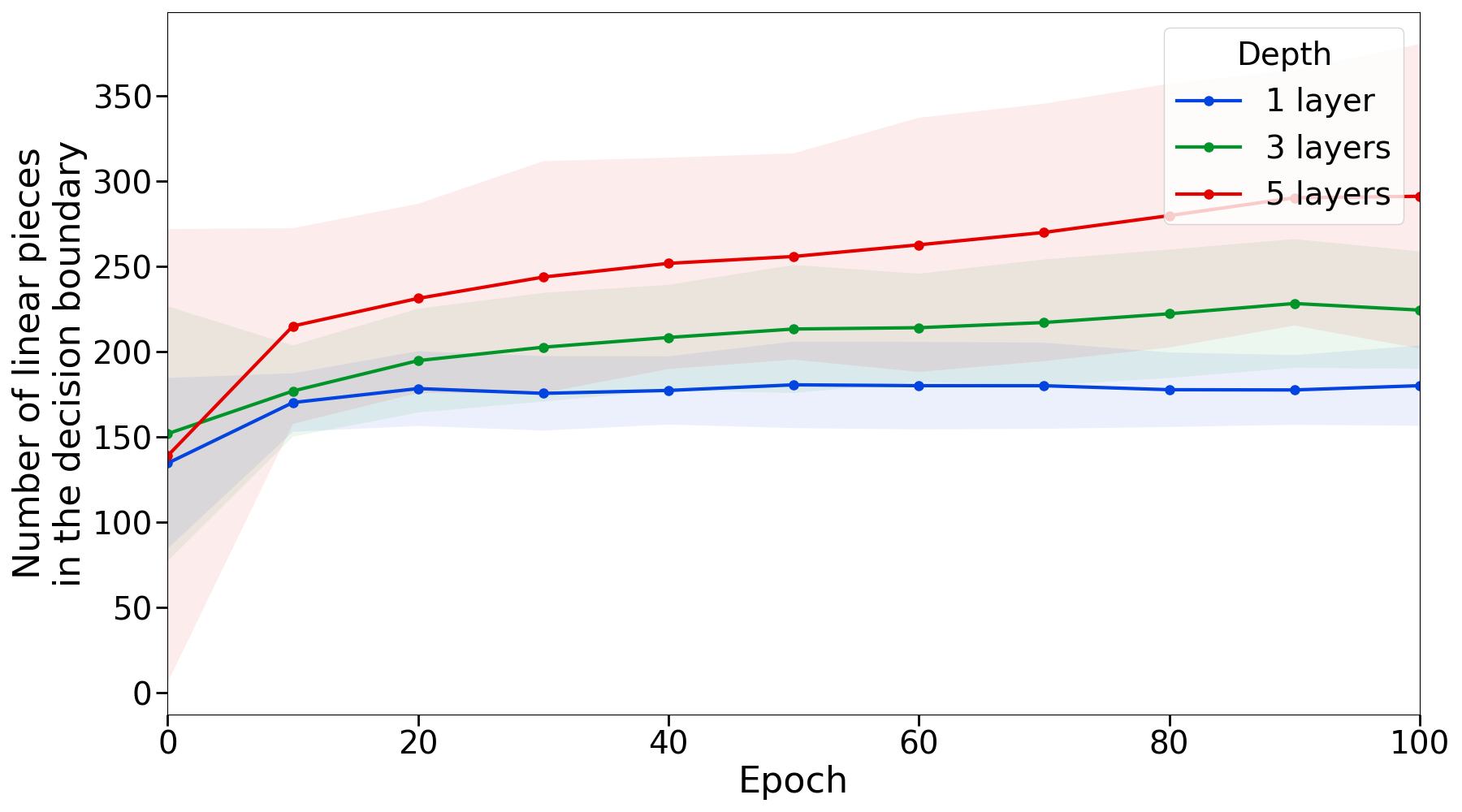} &
            \includegraphics[width=0.32\textwidth]{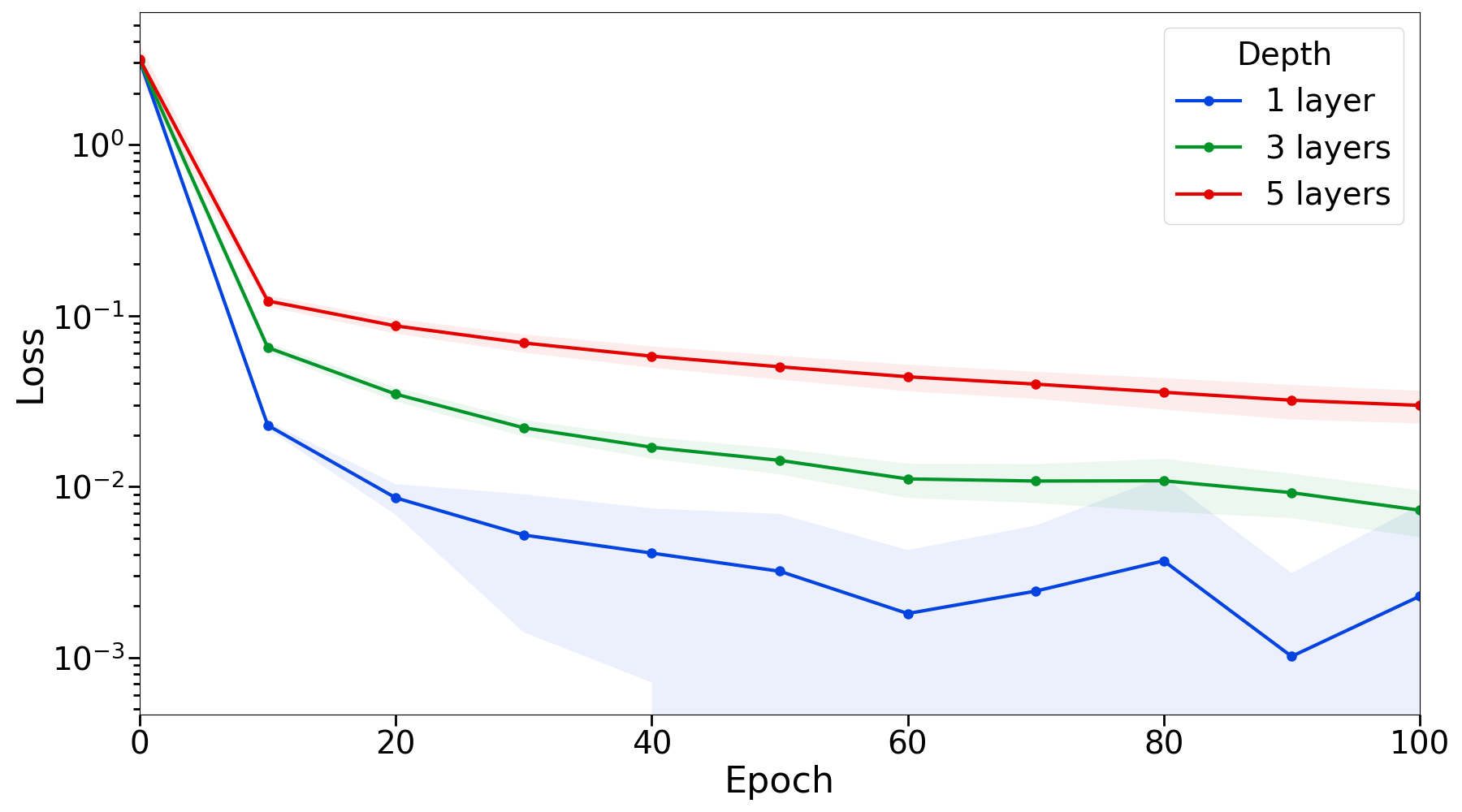}
        \end{tabular}
        \caption{\small ReLU network with the ReLU-He normal initialization.}
    \end{subfigure}
    \vspace{.2cm}

    \begin{subfigure}{\textwidth}
        \setlength\tabcolsep{2pt}
        \begin{tabular}{ccc}
            \centering
            \includegraphics[width=0.32\textwidth]{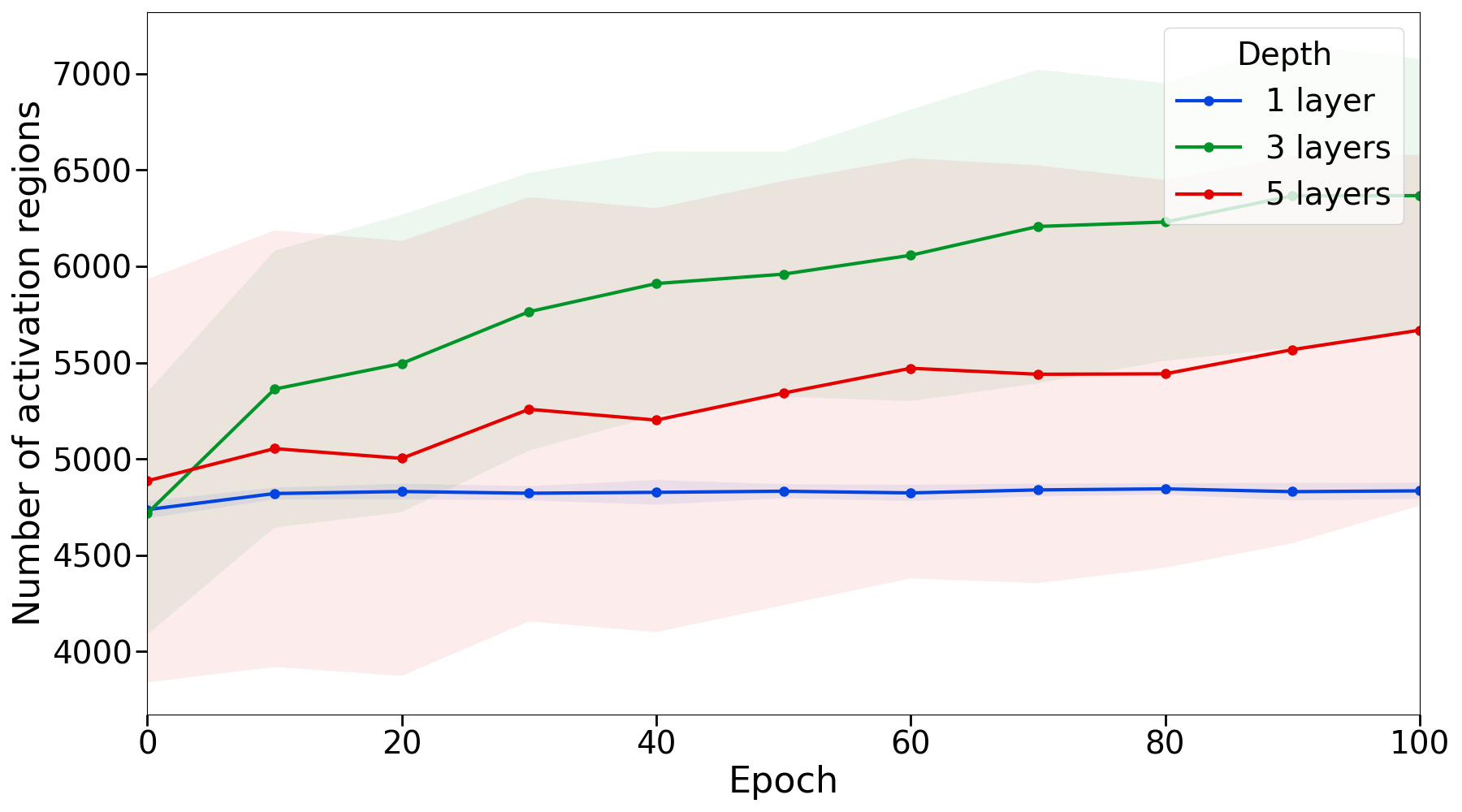} &
            \includegraphics[width=0.32\textwidth]{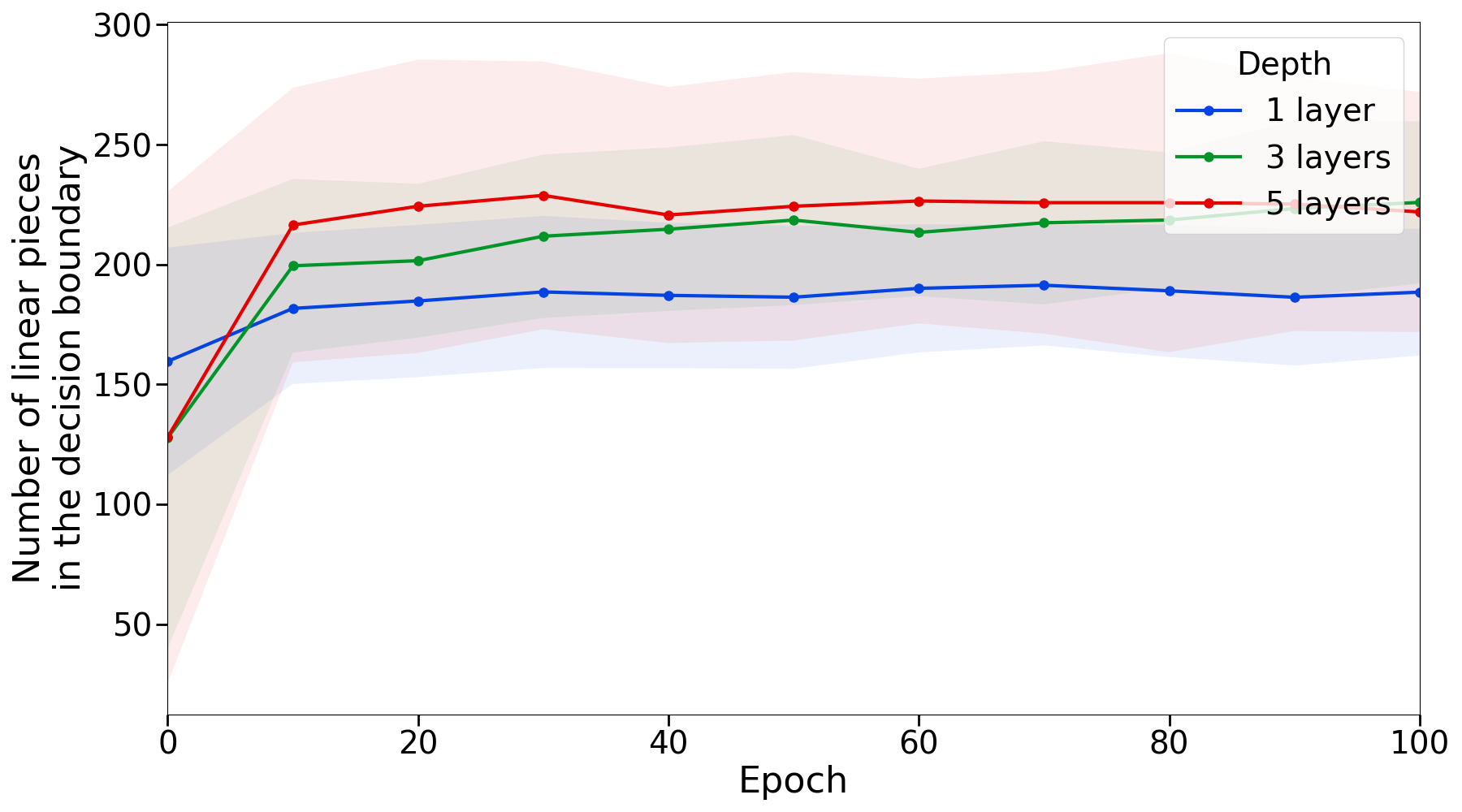} &
            \includegraphics[width=0.32\textwidth]{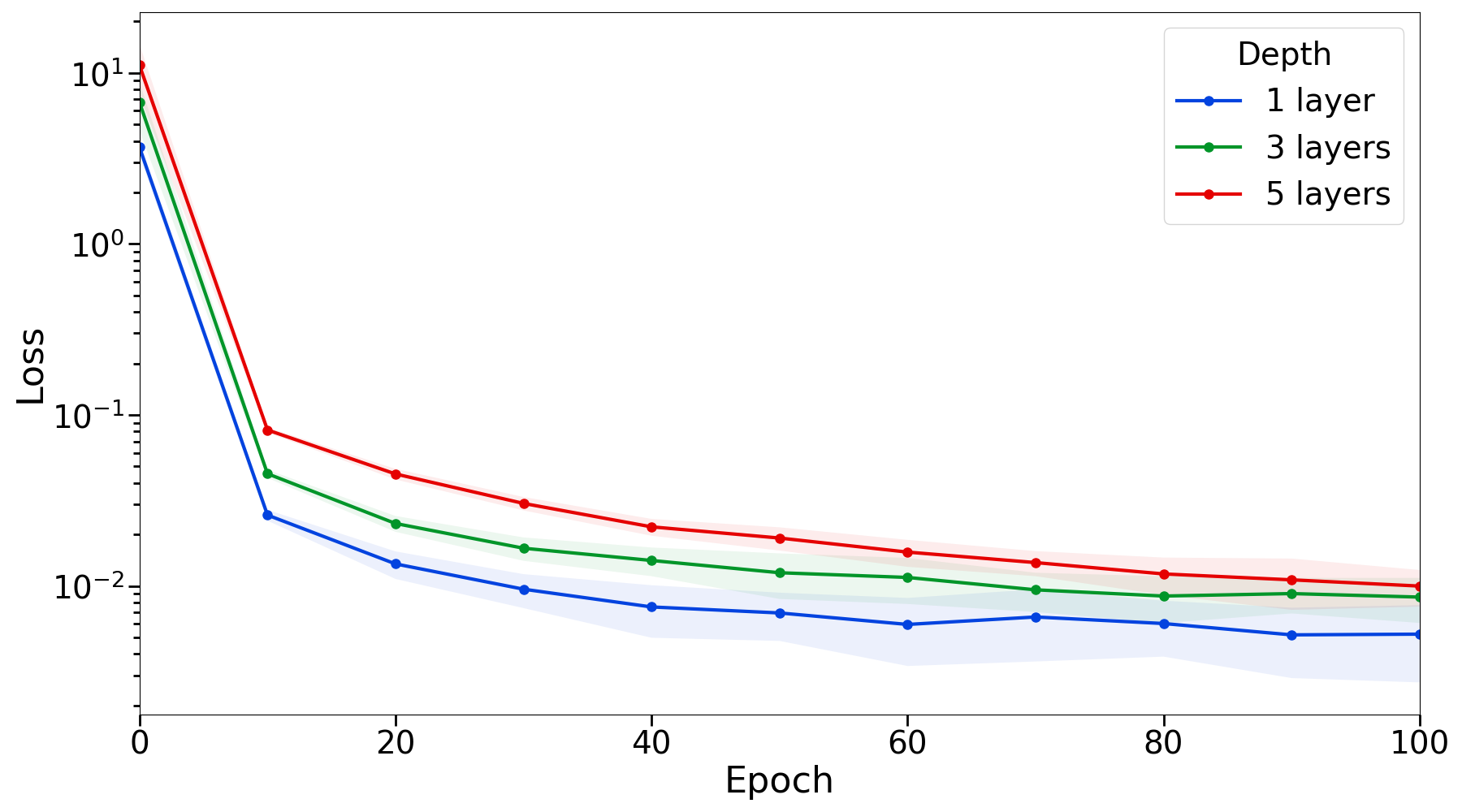}
        \end{tabular}
        \caption{\small Maxout network with the ReLU-He normal initialization.}
    \end{subfigure}
    \vspace{.2cm}

    \begin{subfigure}{\textwidth}
        \setlength\tabcolsep{2pt}
        \begin{tabular}{ccc}
            \centering
            \includegraphics[width=0.32\textwidth]{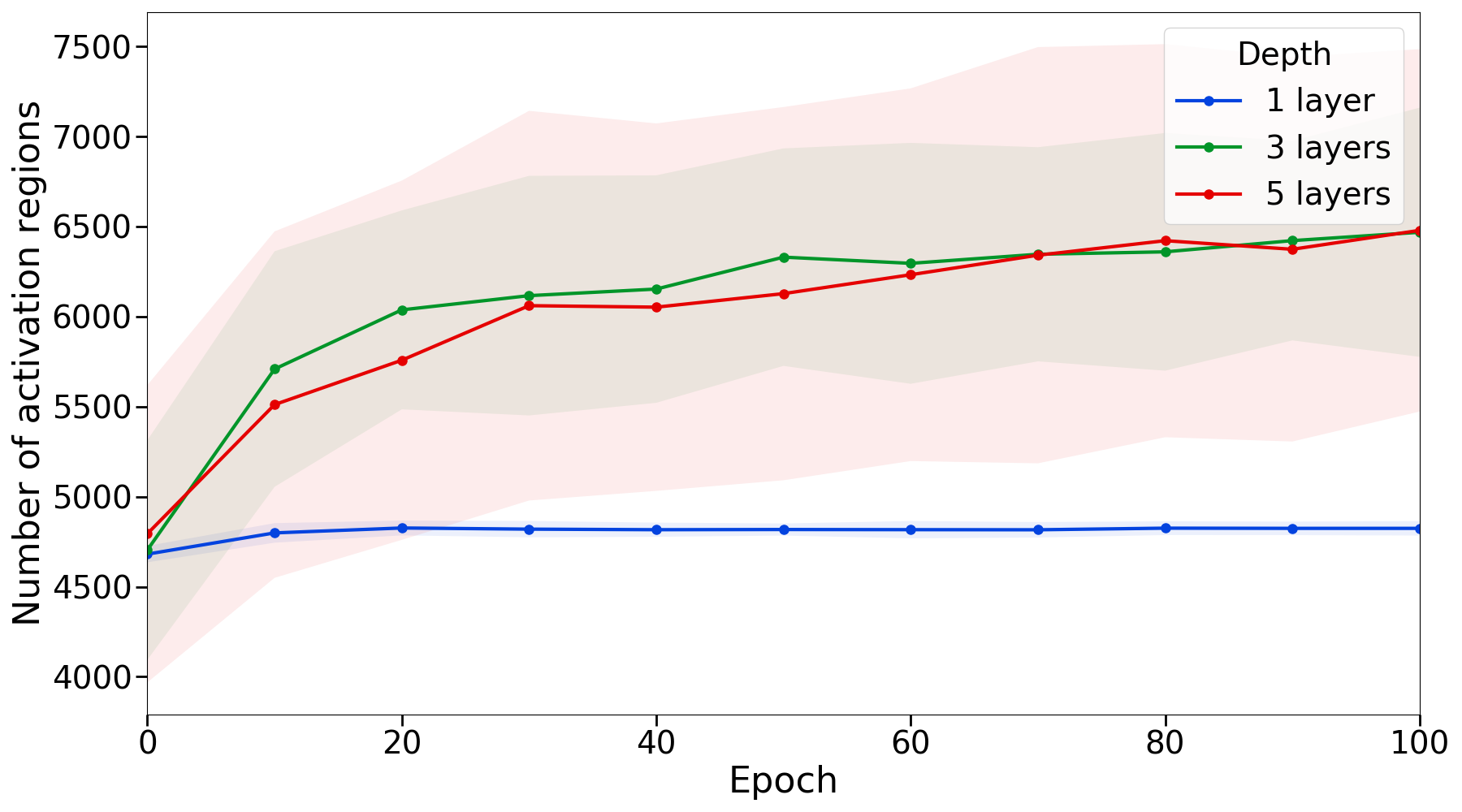} &
            \includegraphics[width=0.32\textwidth]{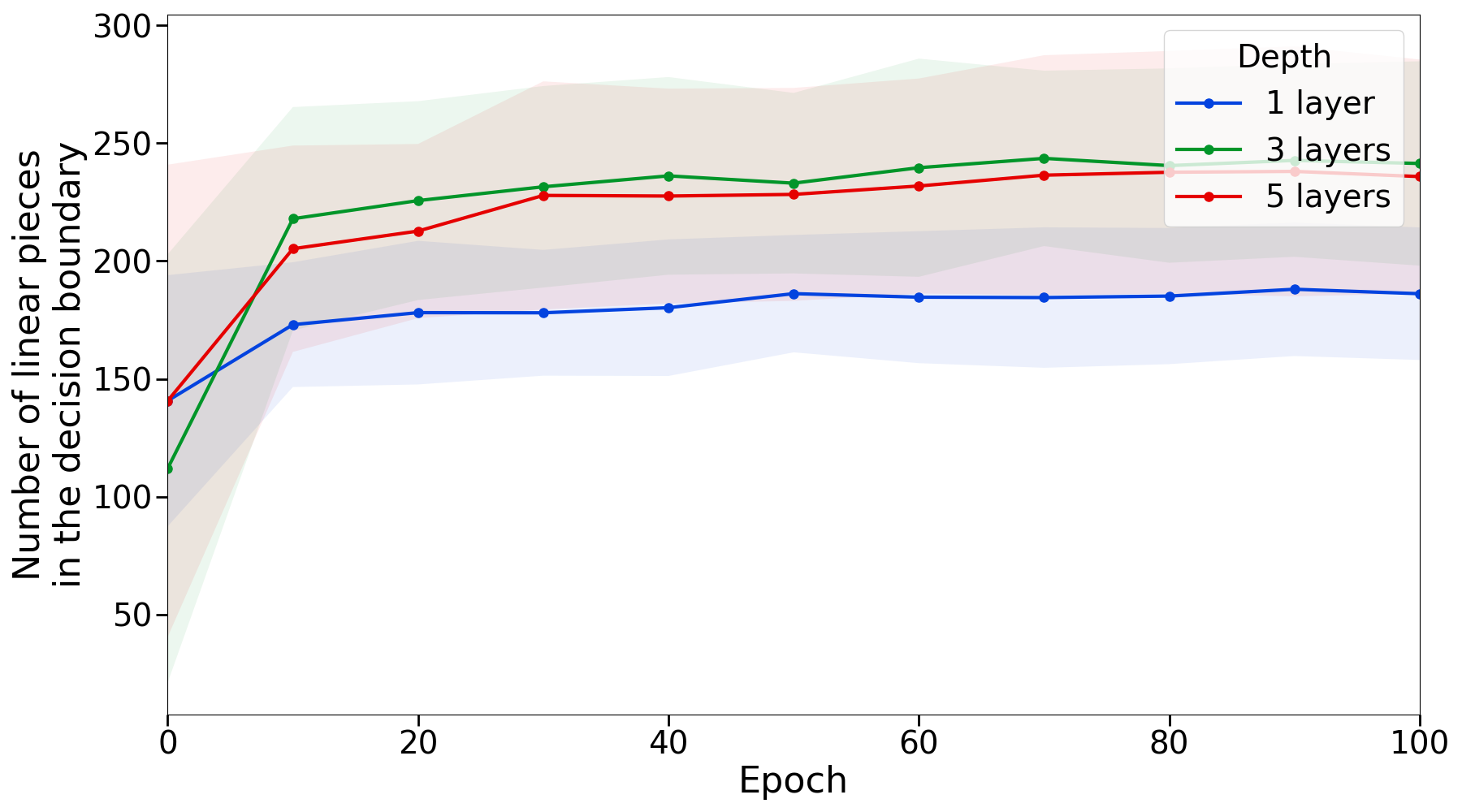} &
            \includegraphics[width=0.32\textwidth]{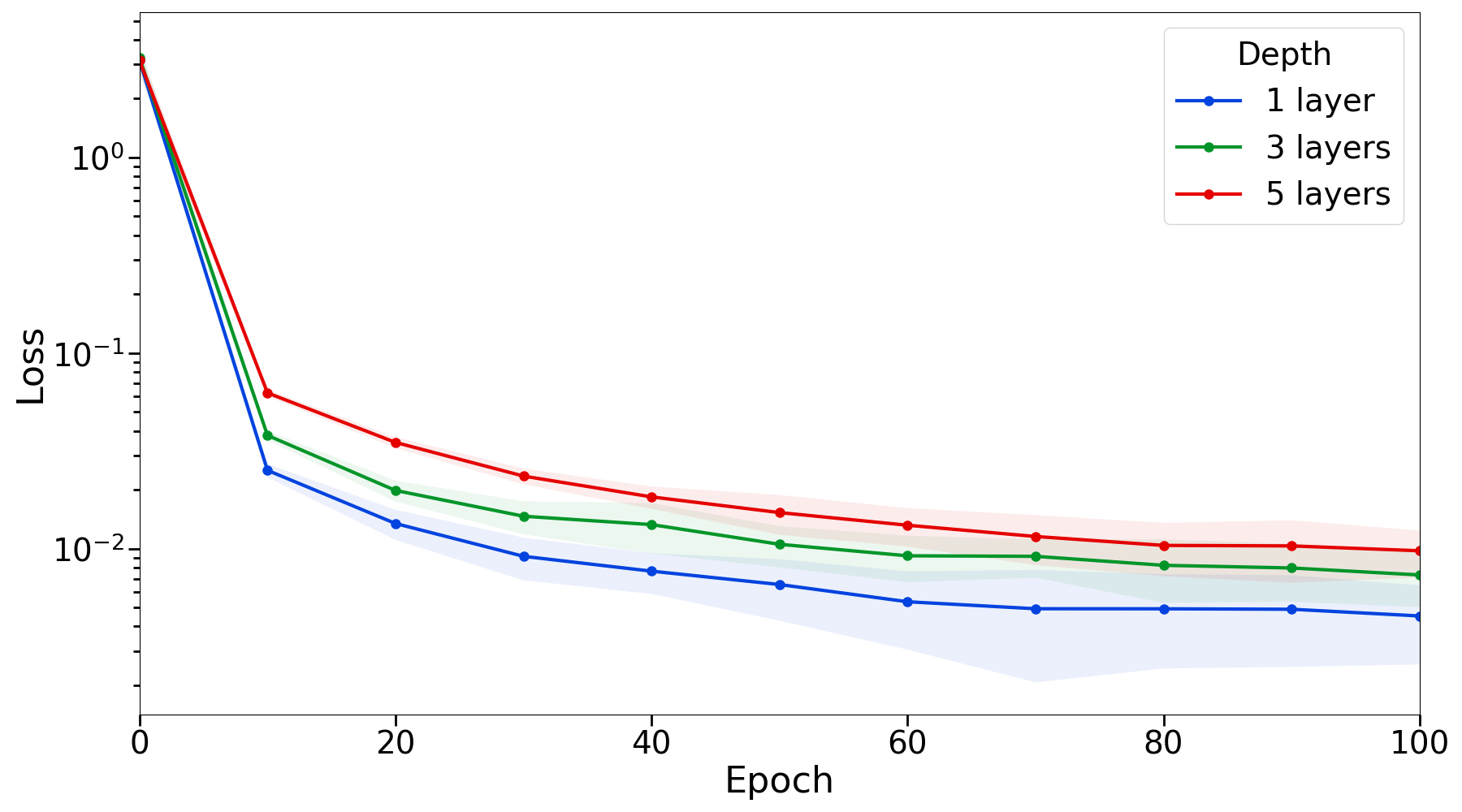}
        \end{tabular}
        \caption{\small Maxout network with the maxout-He normal initialization.}
    \end{subfigure}
    \vspace{.2cm}

    \begin{subfigure}{\textwidth}
        \setlength\tabcolsep{2pt}
        \begin{tabular}{ccc}
            \centering
            \includegraphics[width=0.32\textwidth]{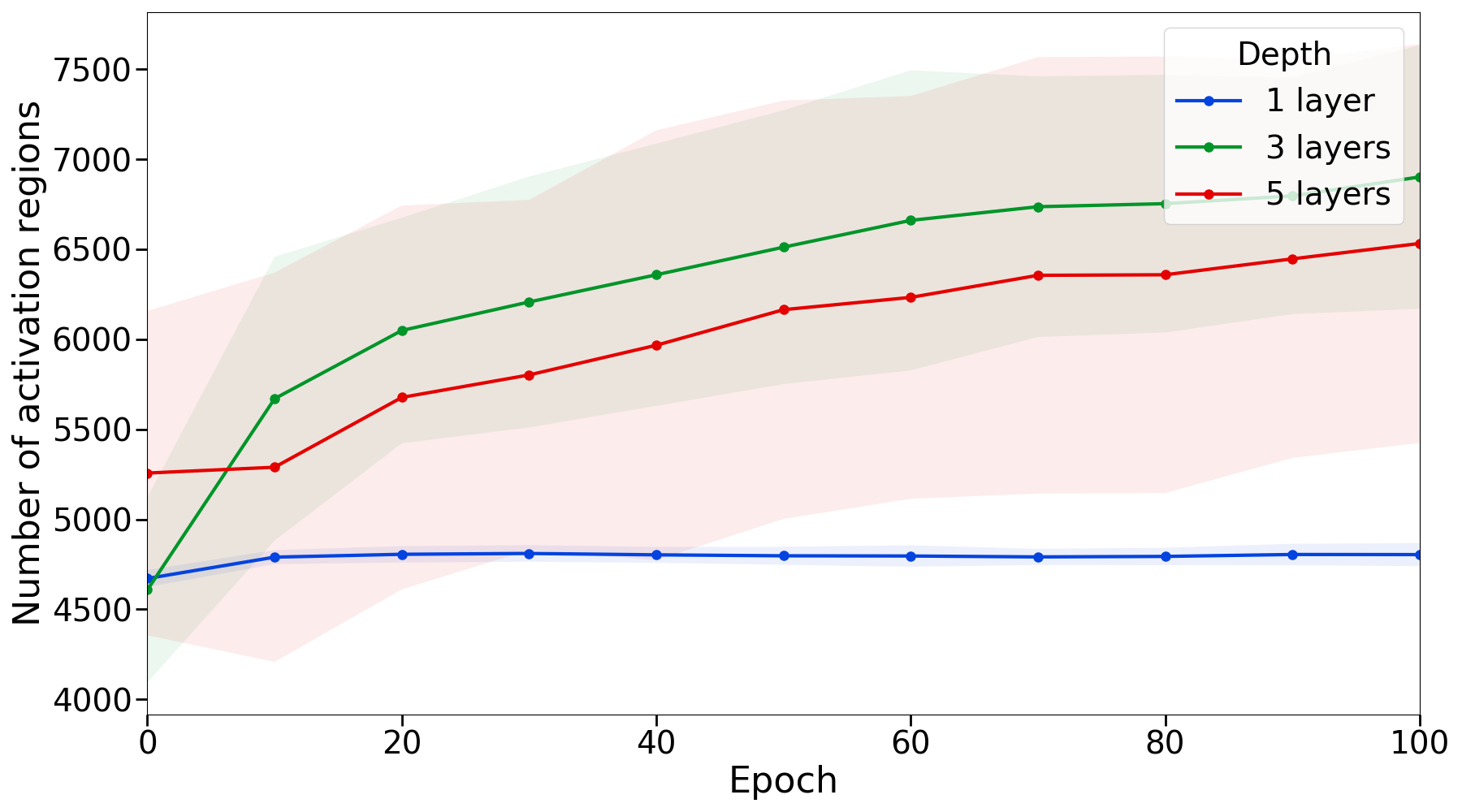} &
            \includegraphics[width=0.32\textwidth]{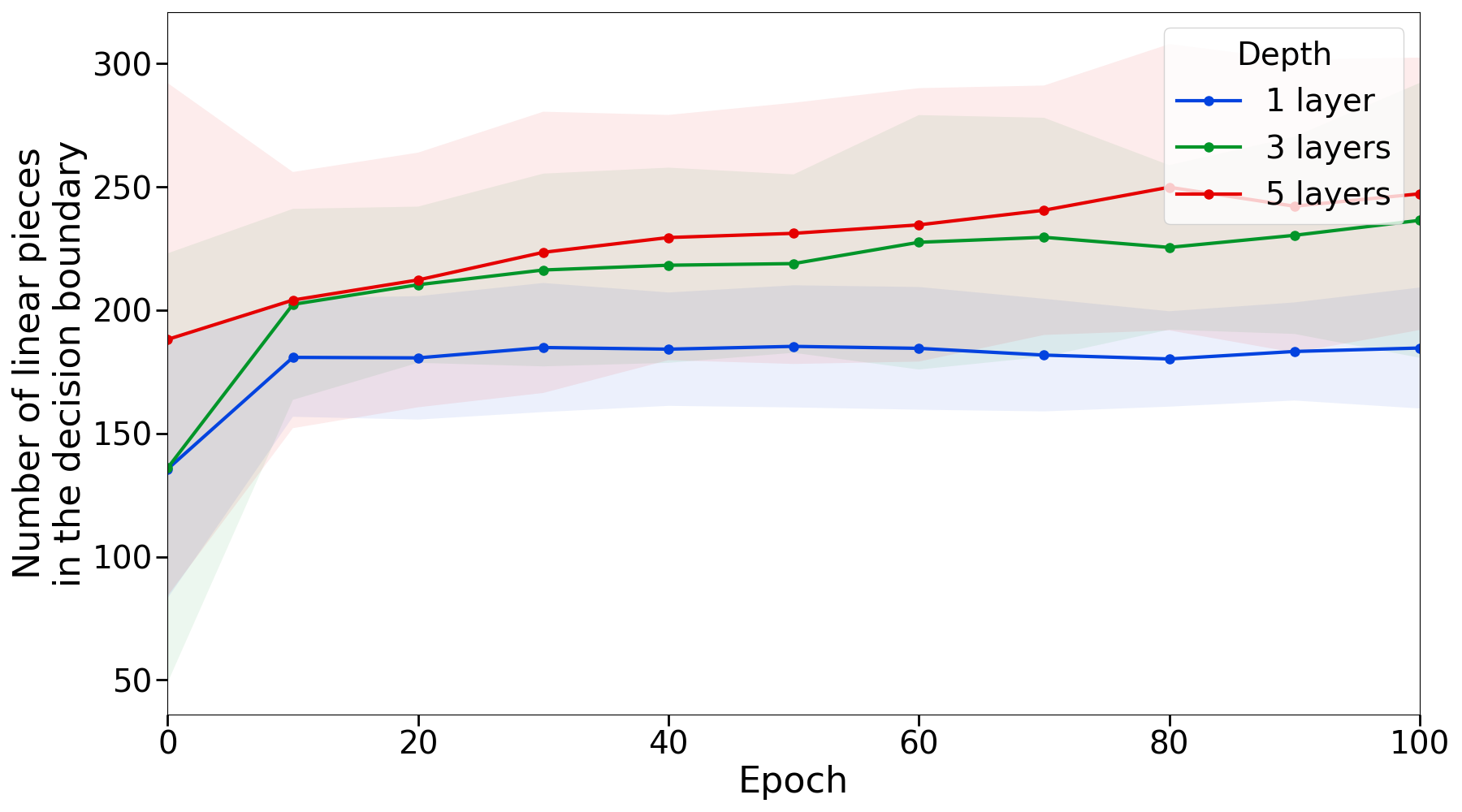} &
            \includegraphics[width=0.32\textwidth]{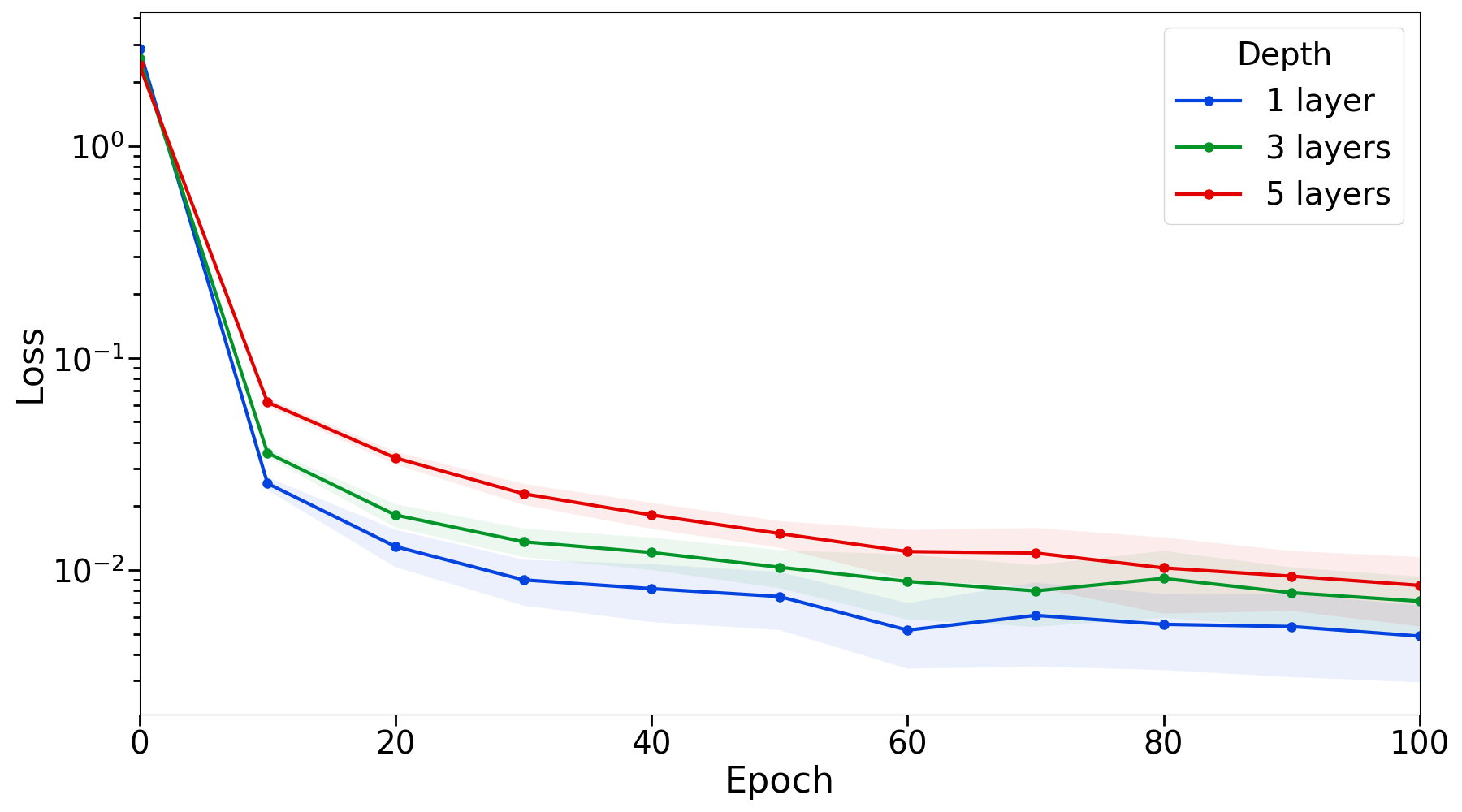}
        \end{tabular}
        \caption{\small Maxout network with the sphere initialization.}
    \end{subfigure}
    \vspace{.2cm}

    \begin{subfigure}{\textwidth}
        \setlength\tabcolsep{2pt}
        \begin{tabular}{ccc}
            \centering
            \includegraphics[width=0.32\textwidth]{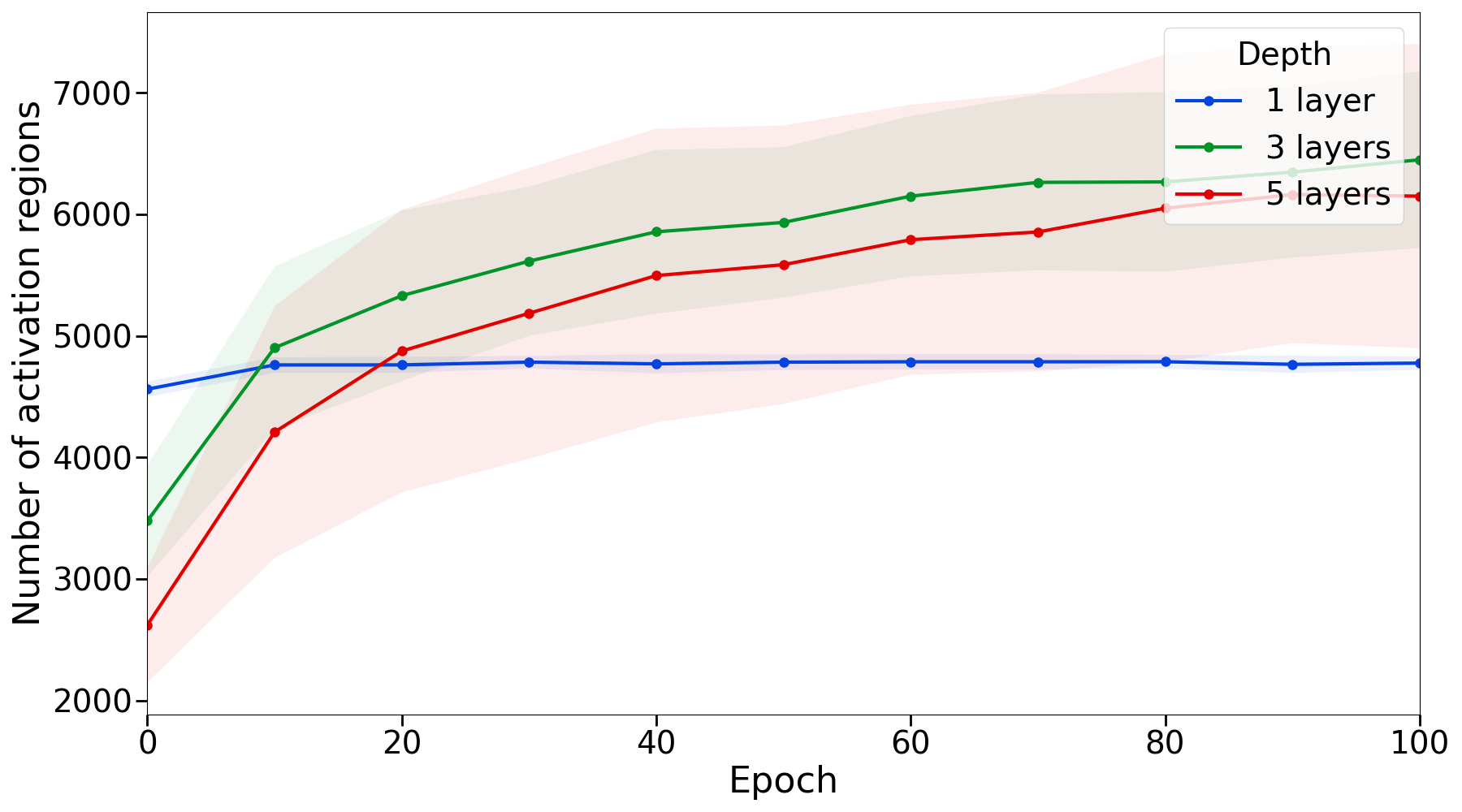} & 
            \includegraphics[width=0.32\textwidth]{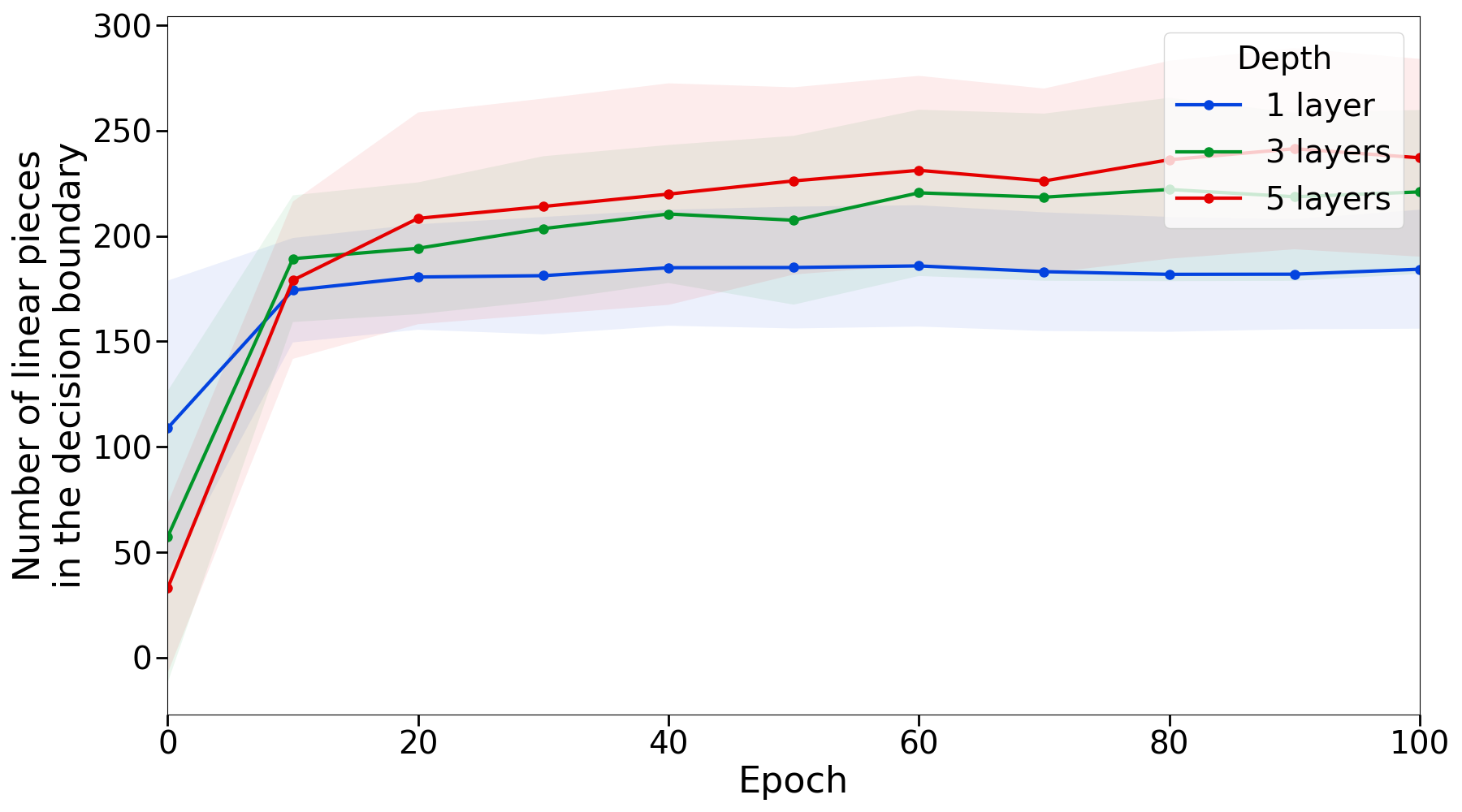} &
            \includegraphics[width=0.32\textwidth]{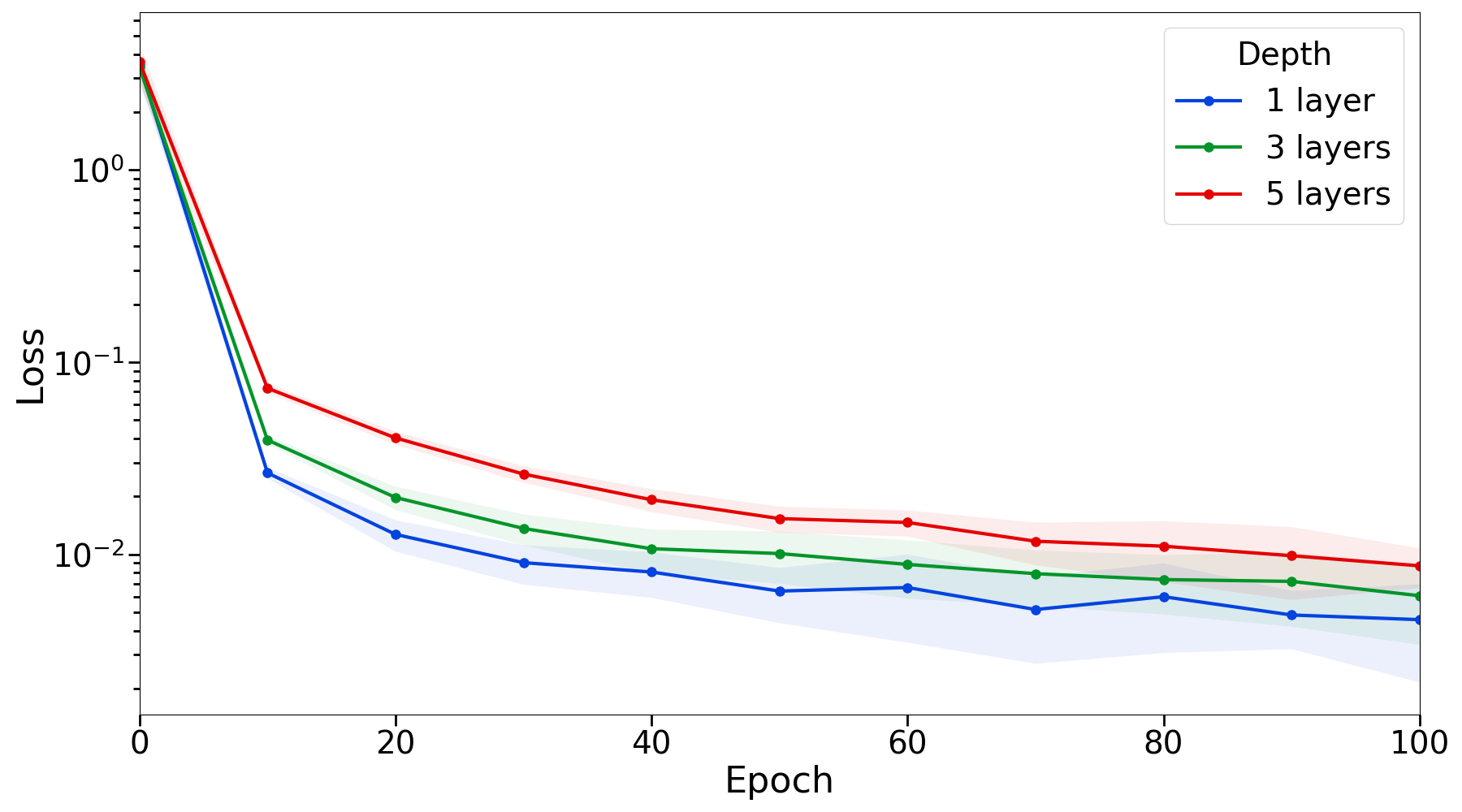}
        \end{tabular}
        \caption{\small Maxout network with the many regions initialization.}
    \end{subfigure}
    
    \caption{Change in the number of linear regions and the decision boundary pieces during $100$ training epochs given different initializations. Networks had $100$ neurons and for maxout networks $K=2$. Both the number of linear regions and linear pieces of the decision boundary increases during training for all initializations but remain much smaller than the theoretical maximum. The settings were the same as in Figure \ref{fig:training}.}
    \label{fig:init_in_training} 
\end{figure}

\end{document}